\documentclass[lefttitle]{behroozarxiv}

\usepackage[utf8]{inputenc}
\usepackage[T1]{fontenc}

\usepackage{amsmath}
\usepackage{amssymb}
\usepackage{amsfonts}
\usepackage{amsthm}
\usepackage{amsbsy}
\usepackage{amstext}
\usepackage{mathtools}
\usepackage{dsfont}
\usepackage{relsize}
\usepackage{bm}
\usepackage{xfrac}
\usepackage{nicefrac}

\makeatletter
\@ifpackageloaded{xcolor}{}{\usepackage[table]{xcolor}}
\makeatother

\usepackage{microtype}
\usepackage{textcomp}
\usepackage{setspace}
\usepackage{comment}

\usepackage{booktabs}
\usepackage{array}
\usepackage{tabularx}
\usepackage{longtable}
\usepackage{makecell}
\usepackage{stackengine}

\usepackage{enumitem}

\usepackage{graphicx}
\usepackage{epsfig}
\usepackage{float}
\usepackage{caption}

\usepackage{tikz}
\usetikzlibrary{arrows.meta,positioning,fit,calc,backgrounds}

\usepackage{algorithm}
\usepackage{algpseudocode}

\providecommand{\STATE}{\State}
\providecommand{\FOR}{\For}
\providecommand{\ENDFOR}{\EndFor}
\providecommand{\IF}{\If}
\providecommand{\ENDIF}{\EndIf}
\providecommand{\ELSE}{\Else}
\providecommand{\ELSIF}{\ElsIf}

\newcommand{\AlgComment}[1]{\hfill{\textcolor{blue}{$\triangleright$ #1}}}

\usepackage[most]{tcolorbox}

\newtcolorbox{examplebox}[1]{
  colback=gray!6,
  colframe=gray!45,
  title=\textbf{#1},
  fonttitle=\small,
  fontupper=\small,
  rounded corners,
  boxrule=0.5pt,
  left=6pt,
  right=6pt,
  top=6pt,
  bottom=6pt
}

\newtcolorbox{protocolbox}[1]{
  colback=blue!4,
  colframe=blue!45!black,
  title=\textbf{#1},
  fonttitle=\small,
  fontupper=\small,
  rounded corners,
  boxrule=0.5pt,
  left=6pt,
  right=6pt,
  top=6pt,
  bottom=6pt
}

\newtcolorbox{promptbox}[1]{
  colback=gray!4,
  colframe=gray!45,
  title=#1,
  fonttitle=\bfseries,
  fontupper=\small,
  breakable,
  rounded corners,
  boxrule=0.5pt,
  left=1mm,
  right=1mm,
  top=1mm,
  bottom=1mm
}

\usepackage[numbers,sort&compress]{natbib}

\usepackage{orcidlink}

\makeatletter

\theoremstyle{definition}
\@ifundefined{definition}{
  \newtheorem{definition}{Definition}[section]
}{}

\theoremstyle{plain}
\@ifundefined{assumption}{
  \newtheorem{assumption}[definition]{Assumption}
}{}
\@ifundefined{theorem}{
  \newtheorem{theorem}[definition]{Theorem}
}{}
\@ifundefined{lemma}{
  \newtheorem{lemma}[definition]{Lemma}
}{}
\@ifundefined{proposition}{
  \newtheorem{proposition}[definition]{Proposition}
}{}
\@ifundefined{corollary}{
  \newtheorem{corollary}[definition]{Corollary}
}{}

\theoremstyle{remark}
\@ifundefined{remark}{
  \newtheorem{remark}[definition]{Remark}
}{}

\makeatother

\DeclareMathAlphabet{\mathpzc}{OT1}{pzc}{m}{it}

\providecommand{\Top}{\mathsf{Top}}
\providecommand{\Lap}{\mathsf{Lap}}

\providecommand{\norm}[1]{\left\lVert #1\right\rVert}

\makeatletter
\newcommand*{\centernot}{%
  \mathpalette\@centernot
}
\def\@centernot#1#2{%
  \mathrel{%
    \rlap{%
      \settowidth\dimen@{$\m@th#1{#2}$}%
      \kern.5\dimen@
      \settowidth\dimen@{$\m@th#1=$}%
      \kern-.5\dimen@
      $\m@th#1\not$%
    }%
    {#2}%
  }%
}
\makeatother

\renewcommand\thesection{\arabic{section}}
\renewcommand\thesubsection{\arabic{section}.\arabic{subsection}}

\makeatletter

\def\@seccntformatinl#1{\csname the#1dis\endcsname\hskip 1em\relax}
\makeatother

\makeatletter
\newcommand{\AppendixOnlyTOC}{%
  \section*{Appendix Contents}%
  \@starttoc{atoc}%
}
\makeatother

\newcommand{\appsection}[1]{%
  \section{#1}%
  \addcontentsline{atoc}{section}{\protect\numberline{\thesection}#1}%
}

\makeatletter
\newcommand*{\rom}[1]{\expandafter\@slowromancap\romannumeral #1@}
\makeatother

\def\BibTeX{{\rm B\kern-.05em{\sc i\kern-.025em b}\kern-.08em
    T\kern-.1667em\lower.7ex\hbox{E}\kern-.125emX}}

\title{Differentially Private Semantic Plans for Aggregate Insight Generation}

\author{Behrooz~Razeghi~\textsuperscript{\orcidlink{0000-0001-9568-4166}}}

\affiliation{School of Engineering and Applied Sciences, Harvard University}

\abstract{
\texttt{URANIA} provides end-to-end differential privacy (DP) for summaries
of data-dependent clusters. However, its cluster--keyword release does not
directly provide collection-wide aggregates for semantic concepts defined
independently of the protected corpus. Records may express several concepts,
records expressing the same concept may be assigned to different clusters,
and cluster identities need not correspond across analyses. Consequently,
cluster-level statistics do not directly provide comparable measurements of
predefined concepts across collections or repeated analyses.
We introduce \texttt{DP-SPIN}, a trusted-curator framework for aggregate
measurement and summarization over semantic concepts fixed independently of
the protected target records. Each record is mapped to a bounded sparse
nonnegative vector over these concepts, whose sum forms a semantic sketch. A
differentially private mechanism releases a semantic plan containing admitted
concepts and noisy masses; normalized semantic-support values and support bins
are obtained by post-processing. For user-level privacy, each user's aggregate
contribution is clipped to a fixed bound. The language model receives only the
plan and fixed decoding instructions, while a public verifier checks concept
mentions, reported values, comparisons, and rank claims against the released
plan. The final summary is differentially private by post-processing.
We establish record- and user-level DP guarantees under add/drop and
replacement adjacency. We evaluate \texttt{DP-SPIN} under record-level
privacy on CFPB complaint narratives, Amazon All Beauty reviews, and Yelp
restaurant reviews, and under user-level privacy on Amazon and Yelp. We
compare \texttt{DP-SPIN} with non-private plan and summary references, DP
keyword and category histogram baselines, and a \texttt{URANIA}-style
baseline with a fixed public keyword vocabulary.
}

\preprint{arXiv preprint, DP-SPIN}
\correspondence{\email{behroozrazeghi@seas.harvard.edu}, \email{behrooz.razeghi@gmail.com}%
}
\codeurl{https://github.com/BehroozRazeghi/dp-spin}

\begin{document}
\maketitle

\vspace{5pt}

\section{Introduction}

Organizations use language models to analyze sensitive text collections, including consumer complaints, reviews, support interactions, and conversational data, and to produce collection-level summaries of recurring themes, failure modes, and intents. Although the released report is collection-level, its computation can still depend on individual records. In aggregate-insight pipelines, a rare or distinctive record may affect a cluster, label, exemplar, keyword, or generated statement. Releasing only aggregate text therefore does not by itself bound the influence of an individual record on the output. \texttt{CLIO} illustrates this setting through LLM-based summarization, clustering, filtering, and privacy auditing of Claude.ai conversations \citep{tamkin2024clio}. \texttt{CLIOPATRA} demonstrates a limitation of such heuristic safeguards by showing that adversarially inserted conversations can induce disclosure of sensitive information from target conversations \citep{annamalai2026cliopatra}.

Prompting a language model not to reveal identifying information does not by itself establish differential privacy (DP). More generally, withholding raw records from the final generation prompt is insufficient if clusters, labels, keywords, exemplars, or other generation inputs remain non-private functions of the protected corpus. An end-to-end DP analysis must specify the adjacency relation and account for every protected-corpus-dependent object exposed outside the trusted curator. Such an object must either be fixed independently of the protected corpus or be covered by a DP mechanism, with the privacy costs of multiple releases composed accordingly. The released private object also determines which corpus-dependent quantities are available to support claims in the final report.

\begin{figure}
    \centering
    \includegraphics[width=\linewidth]{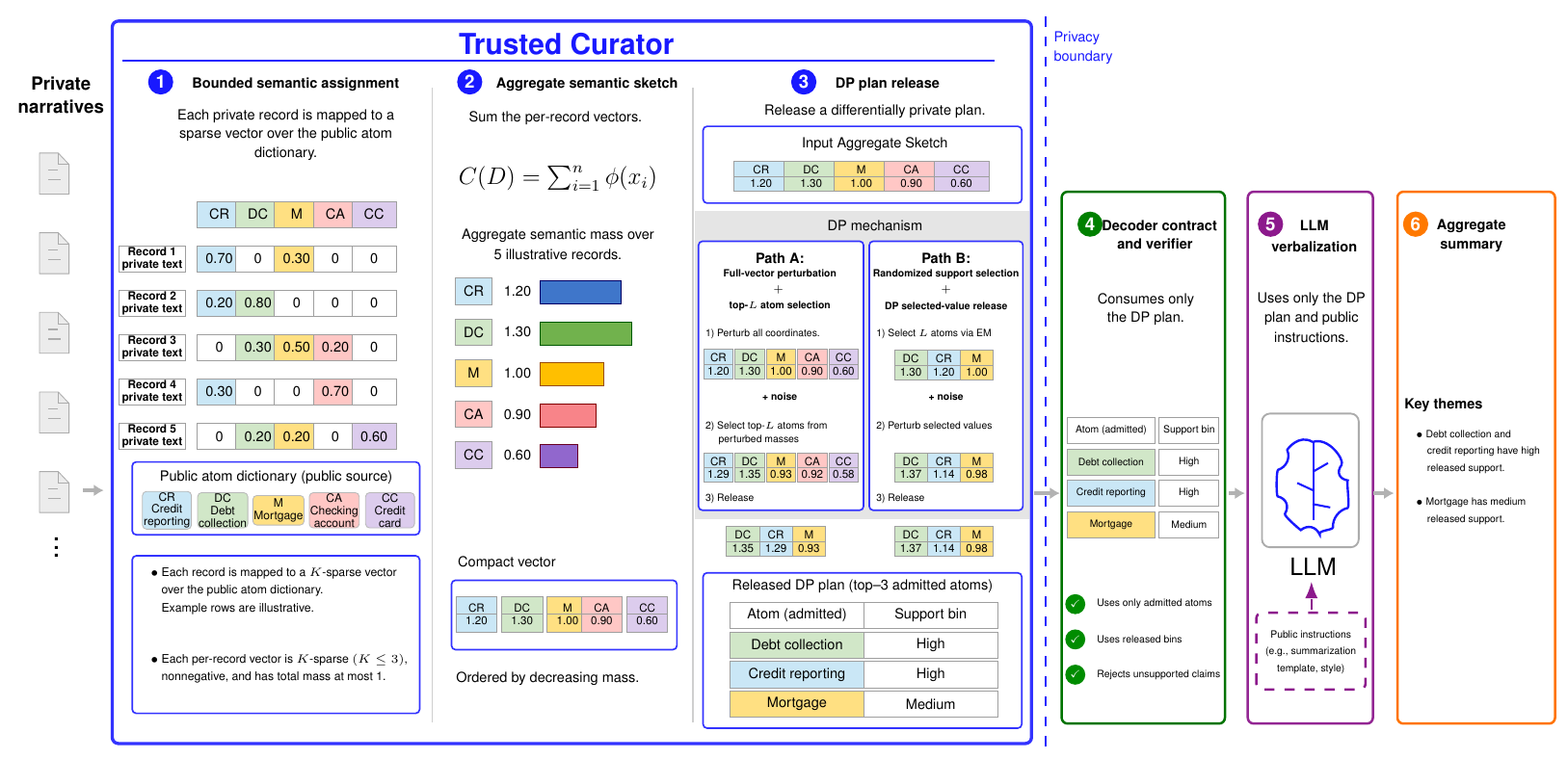}
    \caption{Illustrative record-level \texttt{DP-SPIN} example.}
    \label{fig:illustration-DPSPIN}
\end{figure}

A closely related formal-DP framework is \texttt{URANIA} \citep{liu2025urania}. \texttt{URANIA} privately clusters chatbot conversations, filters clusters using noisy sizes, releases DP per-cluster keyword histograms, and generates descriptions from released keywords. Its cluster--keyword release therefore supports private discovery and description of data-dependent groups. This is distinct from an analysis in which the target quantities are collection-wide aggregates over semantic concepts defined independently of the protected corpus. A record may express several such concepts, while records expressing the same concept may lie in different clusters. Consequently, cluster sizes and per-cluster keyword histograms do not directly provide the aggregate contribution to each predefined concept. Because the clusters are data-dependent, their identities also need not correspond across separately analyzed corpora. Cluster-level statistics therefore do not directly provide comparable measurements of the same predefined concepts across analyses. Private description of data-dependent groups and private measurement of predefined concepts are both meaningful objectives, but they require different released statistics.

\texttt{URANIA}'s formal analysis uses add/drop adjacency for one conversation and identifies protection of all conversations from one user as future work. When a user contributes multiple conversations, those conversations can affect the private clustering stage and multiple downstream cluster statistics. Its stated record-level guarantee therefore does not provide a fixed user-level bound independent of the number of conversations contributed by that user. A user-level construction requires corresponding per-user contribution control and privacy accounting over the affected private stages.

Natural-language generation introduces a separate requirement. Differential privacy is preserved under post-processing, but post-processing does not ensure that generated claims are supported by the released DP object. A summary generated only from a DP release may mention an unreleased concept, report an unsupported support level, reverse a comparison, or assert a rank not justified by the released values. Such a summary remains differentially private by post-processing but may be inconsistent with the release. Privacy and consistency with released evidence are therefore separate properties.

We study aggregate reporting over a finite set of semantic concepts when the protected unit may be one record or all records associated with one user. This setting presents three technical challenges. First, contributions must
be represented so that the private aggregate-release mechanisms have explicit record- and user-level sensitivity bounds. Second, decoder-visible concept labels must not expose non-private text derived from the protected corpus, while data-dependent concept admission and aggregate values must be released privately. Third, the generated report must remain within the semantic and quantitative information supported by the released object without exposing additional protected-corpus-dependent information to the decoder.

We introduce \textbf{\texttt{DP-SPIN}} (\textit{Differentially Private Semantic Plans for Insight Narration}), a
trusted-curator framework that addresses these requirements. Inside the curator, each protected record is mapped to a sparse nonnegative contribution over a finite semantic atom universe with bounded total mass. An atom represents a semantic concept that may appear in the final report. Record contributions are aggregated coordinate-wise into a semantic sketch whose coordinates measure aggregate semantic mass. For user-level privacy, all record contributions associated with one user are first aggregated and the resulting user vector is deterministically clipped to a public contribution bound. This gives an explicit sensitivity bound for the user-level semantic sketch independently of the number of records contributed by that user.

The curator releases a differentially private \emph{semantic plan} containing admitted atoms and noisy aggregate masses; normalized semantic-support values and support bins are obtained by post-processing.
The semantic plan is the only protected-corpus-dependent input to language generation and determines the concepts and released quantities available to support quantitative, comparison, and rank claims.

The first challenge concerns decoder-visible semantic labels. Adding noise to a numerical aggregate does not privatize a label, topic name, or description derived non-privately from protected records. In the core \texttt{DP-SPIN} formulation, the atom universe, labels, and descriptions are fixed using public or disjoint auxiliary data before the protected target corpus is processed. The protected target records affect the decoder-visible plan only through atom admission and the released noisy values. Any decoder-visible label derived from the protected target corpus would instead have to be covered by a differentially private release whose privacy cost is included in the overall accounting. This separates \emph{atom naming} from \emph{atom admission}.

The second challenge is private release of the high-mass atoms and their values. We analyze two mechanism families. Full-sketch mechanisms perturb every coordinate of the semantic sketch using Laplace or Gaussian noise and select the leading atoms by post-processing, so atom selection incurs no additional privacy cost. Sparse-release mechanisms privately admit atoms through sequential applications of the exponential mechanism and then release noisy values for the admitted coordinates. The latter avoids perturbing every coordinate of a large atom universe but assigns part of the privacy budget to atom admission. We establish record-level add/drop and replacement guarantees and user-level guarantees after deterministic clipping of each user's aggregate semantic contribution.

The third challenge is plan-consistent language generation. The language model receives only the released semantic plan and public instructions, not raw target records or non-private target-derived intermediates. We define a decoder contract under which concept mentions, reported semantic-support information, comparisons, and rank claims must be supported by the released plan. A public verifier rejects candidates that violate this contract. Candidate generation, verification, regeneration, and fallback generation depend only on the same DP plan and public information, so the final report remains DP by post-processing. The verifier establishes consistency with the released plan rather than correctness or completeness relative to the unreleased corpus.

\vspace{-3pt}

\paragraph{Contributions.}
\begin{enumerate}[leftmargin=16pt]
    \item
    We formulate differentially private aggregate insight generation through a structured semantic plan that is the only protected-corpus-dependent input to language generation. The final summary depends only on this differentially private plan and public information and is therefore differentially private by post-processing.\vspace{-5pt}
    \item 
    We introduce a semantic-sketch model with bounded sparse record contributions over a fixed semantic atom universe. Our formulation separates \textit{atom naming} from \textit{atom admission}. Decoder-visible atom labels and descriptions are fixed independently of the protected target records, whereas atom admission depends on the protected target records only through the differentially private plan-release mechanism.\vspace{-5pt}
    \item
    We analyze Gaussian and Laplace full-sketch mechanisms with top-$L$ selection by post-processing, and sequential exponential-mechanism atom admission followed by noisy release of the admitted masses. We establish record- and user-level DP guarantees under add/drop and replacement adjacency.
    \vspace{-5pt}
    \item
    We define plan-level utility measures and a decoder contract for plan-consistent verbalization. The utility measures quantify noisy-value error, top-$L$ support recovery, preserved non-private semantic-mass fraction, and support-mass loss. A public verifier checks atom mentions, numeric claims, support-bin assignments, pairwise comparisons, and rank claims for consistency with the released plan.\vspace{-5pt}
    \item
    We evaluate \texttt{DP-SPIN} on CFPB consumer complaint narratives, Amazon All Beauty reviews, and Yelp restaurant reviews using disjoint auxiliary and protected target splits. We measure plan utility, use the public verifier as a plan-consistency check, and compare with non-private plan and summary references, differentially private keyword and category histogram baselines, and a \texttt{URANIA}-style baseline using a fixed public keyword vocabulary.
\end{enumerate}

\vspace{-3pt}

\section{Related Work}

\vspace{-3pt}

\paragraph{Differential privacy.}
Differential privacy bounds how much the output distribution of a randomized mechanism can change under a specified adjacency relation, such as adding, removing, or replacing one protected contribution \citep{dwork2006calibrating,dwork2014algorithmic}. Our analysis uses standard primitives from the differential-privacy literature, including Laplace and Gaussian mechanisms for real-valued vector release, the exponential mechanism for private selection, and the post-processing and composition properties of differentially private mechanisms \citep{mcsherry2007mechanism,dwork2014algorithmic}.

\vspace{-3pt}

\paragraph{Aggregate insight approaches.}
\texttt{CLIO} analyzes Claude.ai conversations using LLM-generated summaries, clustering, filtering, and LLM-based privacy auditing to produce aggregate usage insights \citep{tamkin2024clio}. These safeguards are heuristic and do not provide a formal DP guarantee. \texttt{CLIOPATRA} shows that adversarially inserted conversations can bypass these layers and induce leakage from a target user's conversation \citep{annamalai2026cliopatra}.
\texttt{URANIA} provides end-to-end DP for a related cluster-based pipeline \citep{liu2025urania}. It privately clusters conversations, filters clusters using noisy sizes, releases DP per-cluster keyword histograms, and generates summaries from released keywords. These released statistics support private discovery and description of data-dependent groups. \texttt{DP-SPIN} addresses a different measurement target, namely collection-wide aggregates over semantic
concepts defined independently of the protected target records. A record may express several such concepts, while records expressing the same concept may lie in different clusters. Consequently, cluster sizes and per-cluster keyword histograms do not directly provide the collection-wide aggregate for each predefined concept. Because each
partition is learned from the analyzed corpus, cluster identities also need not correspond across analyses, so cluster-level statistics do not directly provide comparable measurements of fixed concepts. \texttt{DP-SPIN} instead releases a DP semantic plan over predefined concepts, with utility measures and verifier checks defined on the concepts and noisy values in that plan.
\texttt{URANIA}'s formal analysis uses add/drop adjacency for one conversation and identifies user-level DP as future work. When one user contributes multiple conversations, that user can affect the private clustering stage and multiple downstream cluster statistics. A fixed user-level guarantee therefore requires per-user contribution control and privacy accounting across these stages. \texttt{DP-SPIN} instead aggregates and clips each user's semantic contribution before private release, yielding a fixed sensitivity bound for the user-level semantic sketch.

\vspace{-3pt}

\paragraph{Prompt and context privacy.}
A separate line of work studies privacy for LLM inference by modifying or constraining the context supplied to the model. CAPE perturbs prompts using a context-aware differentially private mechanism \citep{wu2025cape}. DP-GTR uses group text rewriting under local differential privacy to reduce prompt-level disclosure \citep{li2025dp}. DP-Fusion bounds the influence of designated sensitive tokens in the input context on the output distribution \citep{thareja2025dp}. These methods operate on the context given to an LLM at inference time. In our setting, a trusted curator holds a collection of private records, releases a differentially private aggregate representation, and uses the LLM only to verbalize that released representation.

\vspace{-3pt}

\paragraph{Private top-$L$ selection.}
When the semantic atom universe is large, perturbing every coordinate of the full sketch may be unnecessary when the released plan contains only $L$ atoms. Differentially private selection methods, including the exponential mechanism and private top-$L$ selection procedures, select high-scoring coordinates without releasing the full vector \citep{mcsherry2007mechanism,qiao2021oneshot}. In \texttt{DP-SPIN}, this class of methods corresponds to atom admission, where the mechanism privately selects a subset of the fixed atom universe for the released plan. The admitted atom masses are then released by a separate differentially private value-release mechanism, and the privacy costs of atom admission and value release are composed explicitly.

\begin{figure}[!t]
\centering
\includegraphics[width=0.99\textwidth]{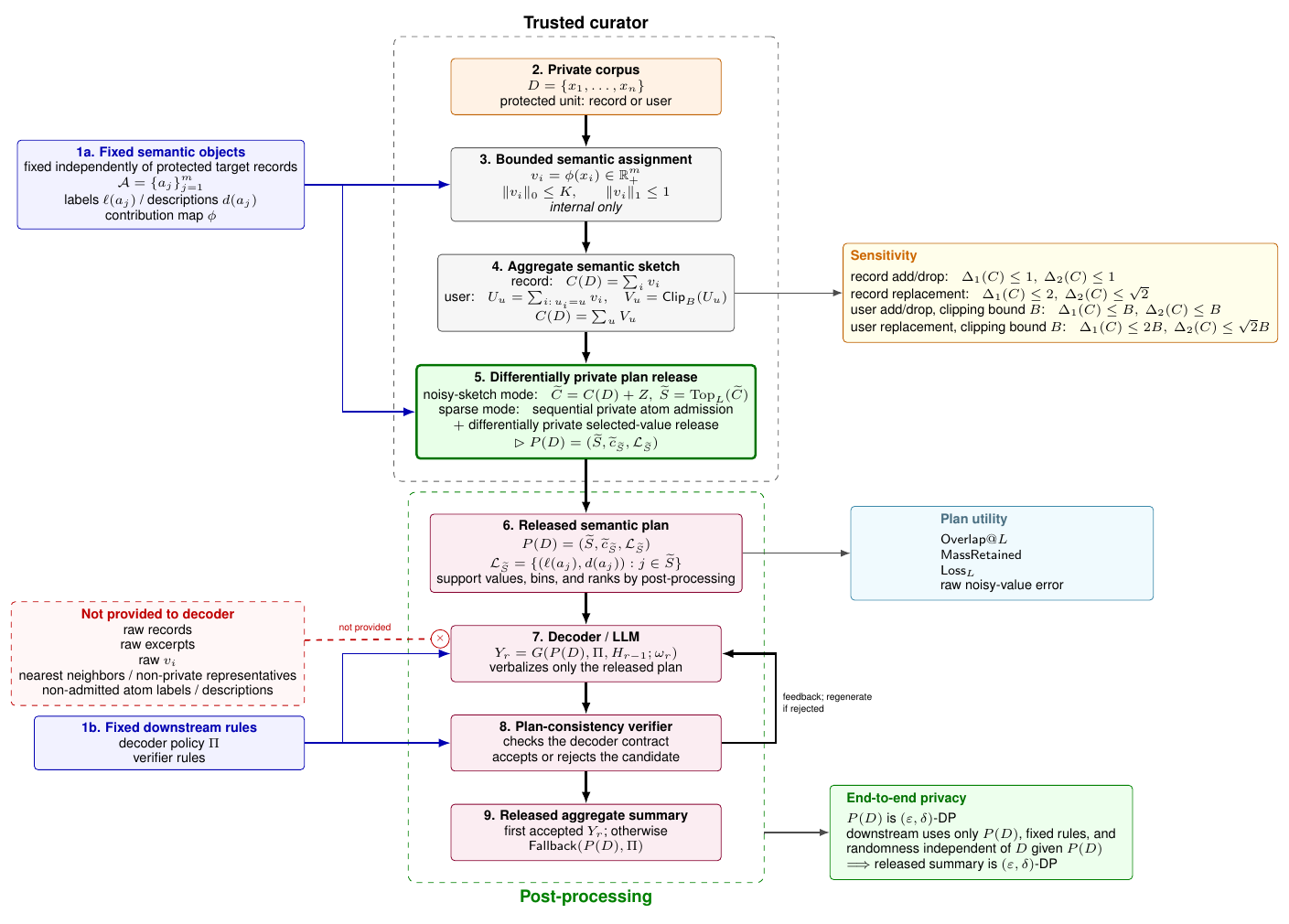}
\vspace{-4pt}
\caption{
\textbf{DP-SPIN pipeline.}
Private records are processed inside the trusted curator. Each record is mapped to a sparse norm-bounded semantic vector. For user-level privacy, record vectors are aggregated by user and the resulting user contribution is clipped before aggregation across users. The curator releases only a differentially private semantic plan $P(D)$ containing admitted atoms and released noisy values. Downstream computation receives only $P(D)$ and information fixed independently of the protected target records. The decoder and verifier never receive raw records, raw excerpts, per-record contribution vectors, non-private representatives, or non-admitted atoms. The released aggregate summary is therefore differentially private by post-processing.
}
\label{fig:private-plans-public-words}
\end{figure}

\vspace{-3pt}

\section{Problem Formulation}
\label{sec:problem-formulation}

\vspace{-3pt}

\paragraph{Input.}
The curator holds a finite dataset $D=\{x_1,\ldots,x_n\}$, where each $x_i\in\mathcal X$ is a private text record, such as a complaint, support ticket, conversation, survey response, or incident report.

\vspace{-4pt}

\paragraph{Output.}
The target released text output is a natural-language aggregate summary $Y\in\mathcal Y$ that describes recurring collection-level themes or semantic categories in $D$. In our framework, $Y$ is generated from a released differentially private semantic plan and public decoding instructions.

\vspace{-4pt}

\paragraph{Privacy requirement.}
The released summary $Y$ must satisfy differential privacy with respect to a protected unit specified by the deployment. In the record-level setting, one protected unit is one text record. We consider add/drop adjacency and replacement adjacency. In the user-level setting, one protected unit consists of all records associated with one user; the aggregate contribution of that user is deterministically clipped to a fixed contribution bound before release.

\vspace{-4pt}

\paragraph{Utility target.}
The formal utility object is the released semantic plan rather than the natural-language summary itself. We evaluate whether this plan recovers high-mass coordinates of the non-private semantic sketch and preserves non-private semantic mass relative to the non-private top-$L$ support. Generated summaries are evaluated separately for consistency with the released plan under the public verifier.

\subsection{Adjacency}

\begin{definition}[Add/drop adjacency]
Datasets $D$ and $D'$, viewed as multisets of protected units, are add/drop adjacent, denoted $D\sim_{\rm ad}D'$, if one dataset can be obtained from the other by adding or removing exactly one protected unit.
\end{definition}

\begin{definition}[Replacement adjacency]
Datasets $D$ and $D'$, viewed as multisets of protected units, are replacement adjacent, denoted $D\sim_{\rm rep}D'$, if they contain the same number of protected units and differ in exactly one protected unit.
\end{definition}

\begin{definition}[User-level add/drop adjacency]
Let each record have a user identifier. For a dataset $D$, let $D_u\subseteq D$ denote the submultiset of records associated with user $u$. Datasets $D$ and $D'$ are user-level add/drop adjacent, denoted $D\sim^{\rm user}_{\rm ad}D'$, if one dataset can be obtained from the other by adding or removing all records associated with exactly one protected user.
\end{definition}

\begin{definition}[User-level replacement adjacency]
View each dataset as a multiset of user-level protected units, where one unit contains all records associated with one user. Datasets $D$ and $D'$ are user-level replacement adjacent, denoted $D\sim^{\rm user}_{\rm rep}D'$, if they contain the same number of user-level protected units and differ in exactly one such unit.
\end{definition}

\vspace{-2pt}

\subsection{Fixed semantic atom universe}

Let $\mathcal{A}=\{a_1,\ldots,a_m\}$ be a finite semantic atom universe. Each atom $a_j$ has a decoder-visible label $\ell(a_j)$ and, optionally, a description $d(a_j)$. Atoms may represent topics, intents, complaint categories, risks, claim categories, or sentiment--topic pairs.

\begin{assumption}[Semantic atom universe fixed independently of protected records]
\label{ass:public-atoms}
The semantic atom universe $\mathcal A$ and all decoder-visible atom labels and descriptions are fixed independently of the protected target records, for example using public or disjoint auxiliary data.
%
\end{assumption}

An atom universe, label, or description derived from the protected target records would require a separate differentially private mechanism whose privacy cost is included in the overall accounting.

\subsection{Contribution map and semantic sketch}

\vspace{-3pt}

A semantic contribution map is a function $\phi:\mathcal{X}\to\mathbb{R}_+^m$ such that, for every record $x$,\vspace{-3pt}
\begin{equation}
\label{eq:phi-bound}
\norm{\phi(x)}_0\le K,\qquad \norm{\phi(x)}_1\le 1.
\end{equation}
The mapping rule $\phi$, including any fitted parameters used to evaluate it, is fixed independently of the protected target dataset. The map may use a fixed classifier, retrieval against atom descriptions, embeddings fitted on public or disjoint auxiliary data, or an LLM running inside the trusted curator. The privacy analysis relies on Assumption~\ref{ass:public-atoms}, the trusted-curator boundary, and the norm bound in \eqref{eq:phi-bound}. If an external service receives raw protected text to compute $\phi(x)$, that disclosure is outside the stated guarantee unless the service is part of the trusted curator.

\begin{remark}[Sparse normalized assignment]
One admissible construction scores each atom using a fixed scoring rule $r_j(x)$, preserves up to $K$ highest-scoring atoms under a fixed tie-breaking rule, and normalizes the resulting nonnegative weights to have total mass at most one. The scoring rule may be fixed in advance or learned from public or disjoint auxiliary data.%
\end{remark}

The non-private semantic sketch is defined according to the protected-unit setting. For record-level privacy,
\begin{equation}
\label{eq:sketch}
C(D)=\sum_{x\in D}\phi(x)\in\mathbb R_+^m.
\end{equation}
Coordinate $C_j(D)$ is the aggregate semantic mass assigned to atom $a_j$.
It need not be an integer count, since one record may distribute its
contribution across several atoms.

For user-level privacy, let $U_u(D)=\sum_{x\in D_u}\phi(x)$ denote the aggregate semantic contribution of user $u$. For a fixed user-level contribution bound $B>0$, define the $\ell_1$-clipping operator $\mathsf{Clip}_B(z) \coloneqq (Bz)/ \max\{B,\|z\|_1\}$, and set $V_u(D)=\mathsf{Clip}_B\!\left(U_u(D)\right)$. Then $\|V_u(D)\|_1\le B$, and the user-level semantic sketch is
\begin{equation}
\label{eq:user-sketch}
C(D)=\sum_u V_u(D).
\end{equation}
Throughout the paper, $C(D)$ denotes the semantic sketch for the protected-unit setting under consideration.

\vspace{-3pt}

\subsection{Semantic plan and decoder contract}

\vspace{-3pt}

A released semantic plan is
\begin{equation}
P(D)=\bigl(\widetilde S,\widetilde c_{\widetilde S}, \mathcal L_{\widetilde S}\bigr),
\end{equation}
where $\widetilde S\subseteq[m]$ is the admitted atom set, $\widetilde c_{\widetilde S}$ contains the raw noisy masses produced by the differentially private release mechanism, and $\mathcal L_{\widetilde S} = \{(\ell(a_j),d(a_j)):j\in\widetilde S\}$. For each admitted atom, define the decoder-facing nonnegative mass $\widetilde c_j^+=\max\{0,\widetilde c_j\}$, and, when normalized semantic support is reported, $\widetilde q_j = \frac{\widetilde c_j^+}{d_{\mathrm{public}}}$. Support bins and released ranks are deterministic functions of the released noisy values and fixed rules. We regard these derived quantities as fields
of the released semantic plan but suppress them from the compact notation for $P(D)$.

\begin{definition}[Decoder contract]
\label{def:decoder-contract}
A decoder satisfies the plan contract if every substantive claim in $Y$ is supported by one of the following objects:
\textbf{(i)} an admitted atom label or description included in the released plan;
\textbf{(ii)} a released noisy mass or support bin for an admitted atom;
\textbf{(iii)} a pairwise semantic-support comparison between two admitted atoms whose ordering is supported by released noisy masses under the public comparison margin;
\textbf{(iv)} a rank claim supported by released ranks; or
\textbf{(v)} public, data-independent text.
The decoder must not introduce raw examples, record-specific details, atom labels outside the released plan, or unsupported quantitative claims.
\end{definition}

A public verifier may reject summaries that violate the contract. Regeneration after rejection remains post-processing when every candidate summary, verifier decision, feedback message, and final output is computed only from the same released plan $P(D)$, public rules, and randomness independent of $D$ given $P(D)$.

\begin{remark}[Fixed-denominator semantic support]
When normalized semantic-support values are reported, $d_{\mathrm{public}}>0$ is fixed independently of the protected target records. The quantity $\widetilde q_j$ need not equal the fraction of protected units expressing atom $a_j$ and, because of the noisy numerator, need not be bounded above by one. Nonnegative clipping, normalization, and support-bin assignment are post-processing of the differentially private noisy masses. If the denominator is not public, it must be released privately and included in the privacy accounting.
\end{remark}

\vspace{-3pt}

\section{DP-SPIN Mechanism}

\vspace{-3pt}

Figure~\ref{fig:private-plans-public-words} summarizes the trusted-curator pipeline and the boundary between private computation and downstream post-processing. Algorithm~\ref{alg:dpspin} gives the corresponding procedure. The trusted curator computes bounded semantic contributions, aggregates them into a semantic sketch, and releases a differentially private semantic plan. All downstream generation and verification use only the released plan, public decoding instructions, and public verifier rules.

\begin{algorithm}[t]
\caption{\textsc{DP-SPIN}: Differentially Private Semantic Plans for Insight Narration}
\label{alg:dpspin}
\begin{algorithmic}[1]
\STATE \textbf{Input:} protected records $D=\{x_i\}_{i=1}^n$;
protected-unit mode $\mathsf{unit}\in\{\mathsf{record},\mathsf{user}\}$;
user identifiers and contribution bound $B$ when
$\mathsf{unit}=\mathsf{user}$; fixed semantic atom universe
$\mathcal A=\{a_j\}_{j=1}^m$; contribution map $\phi$; plan size $L$;
public denominator $d_{\mathrm{public}}$; decoder policy $\Pi$;
verifier $\mathsf V$; maximum attempts $R_{\max}$; privacy parameters.
\STATE $v_i \gets \phi(x_i)\in\mathbb R_+^m,\ i\in[n]$
\AlgComment{Stage I: $\|v_i\|_0\le K,\ \|v_i\|_1\le1$}
\IF{$\mathsf{unit}=\mathsf{record}$}
    \STATE $C\gets\sum_{i=1}^n v_i$
    \AlgComment{record-level semantic sketch}
\ELSE
    \FOR{each protected user $u$}
        \STATE $V_u\gets
        \mathsf{Clip}_{B} \left(\sum_{i:\,u_i=u}v_i\right)$
    \ENDFOR
    \STATE $C\gets\sum_u V_u$
    \AlgComment{user-level semantic sketch}
\ENDIF
\STATE $P \gets
\mathsf{DPPlanRelease}(C,\mathcal A,L,d_{\mathrm{public}};
\text{privacy parameters and sensitivity bounds})$
\AlgComment{Stage II: differentially private semantic plan}
\STATE $H_0\gets\emptyset$
\FOR{$r=1,\ldots,R_{\max}$}
    \STATE $Y_r\gets G(P,\Pi,H_{r-1};\omega_r)$
    \STATE $I_r\gets\mathsf V(Y_r,P,\Pi,H_{r-1})$
    \IF{$I_r=1$}
        \STATE \textbf{return} $Y_r$
    \ENDIF
    \STATE $F_r\gets
    \mathsf{Feedback}(Y_r,P,\Pi,H_{r-1},I_r)$
    \STATE $H_r\gets(H_{r-1},Y_r,I_r,F_r)$
\ENDFOR
\STATE \textbf{return} $\mathsf{Fallback}(P,\Pi)$
\AlgComment{deterministic plan-only template}
\end{algorithmic}
\end{algorithm}

\vspace{-3pt}

\subsection{Stage I: Bounded semantic assignment}

\vspace{-3pt}

For each private record $x_i$, the curator computes $v_i=\phi(x_i)\in\mathbb R_+^m$. The vectors $v_i$ are internal curator variables and are not provided to the decoder. By construction, each $v_i$ is $K$-sparse and has total mass at most one, i.e., $ \norm{v_i}_0\le K,\, \norm{v_i}_1\le 1$. For record-level privacy, $C(D)=\sum_i v_i$. For user-level privacy, record contributions are first aggregated by user, each user aggregate is clipped to the contribution bound $B$, and $C(D)=\sum_u V_u(D)$.

\vspace{-3pt}

\subsection{Stage II: Differentially private plan release}

\vspace{-3pt}

The curator converts the semantic sketch $C(D)$ into a differentially private semantic plan $P(D)$. The definition of $C(D)$ depends on the protected-unit setting as specified above. Algorithm~\ref{alg:plan-release} summarizes the plan-release mechanisms. We consider two release modes.

\vspace{-3pt}

\paragraph{Noisy-sketch selection.}
When perturbing the full $m$-dimensional sketch is feasible, the curator samples $ \widetilde C = C(D) +Z$, where $Z$ is Laplace or Gaussian noise calibrated to the global sensitivity of $C$  under the chosen adjacency relation. The admitted set is then computed as
\begin{equation}
\widetilde S=\Top_L(\widetilde C),   
\end{equation}
with deterministic public tie-breaking. The released plan contains the fixed labels and descriptions of atoms in $\widetilde S$, together with the corresponding noisy masses $\widetilde C_{\widetilde S}$. When semantic-support values and support bins are reported, they are computed from these noisy masses using the public denominator. The top-$L$ selection, support normalization, and binning are post-processing of the noisy sketch.

\paragraph{Sparse private atom admission.}
For sparse plan release, the curator does not perturb and release every atom coordinate. Instead, it selects an atom subset by a differentially private atom-admission mechanism and releases noisy masses only for the admitted atoms. Let $\widetilde S\subseteq[m]$, with $|\widetilde S|=L$, denote the admitted atom-index set.
The privacy budget is split as $\varepsilon_{\mathrm{sel}}=(1-\rho)\varepsilon$, $\varepsilon_{\mathrm{val}}=\rho\varepsilon$, where $\rho\in(0,1)$ is a public value-release budget fraction. Atom admission is performed sequentially with the exponential mechanism. At round $t\in\{1,\ldots,L\}$, let $S_{t-1}$ be the atom indices already admitted. The curator samples one new atom index $J_t\in [m]\setminus S_{t-1}$ using the coordinate utility $u(D,j)=C_j(D)$, with 
\begin{equation}
\mathsf{Pr} \left[ J_t = j \mid S_{t-1} \right]  \propto 
\exp\left(\frac{\varepsilon_{\mathrm{sel},t} \, C_j(D)}{2\Delta_{\mathrm{coord}}} \right),
\qquad j\in[m]\setminus S_{t-1},
\end{equation}
where $\Delta_{\mathrm{coord}}$ is the global sensitivity of the coordinate score $C_j(D)$ under the chosen adjacency relation. The per-round budgets satisfy $\sum_{t=1}^L \varepsilon_{\mathrm{sel},t} =\varepsilon_{\mathrm{sel}}$, and in the experiments we use the uniform allocation $\varepsilon_{\mathrm{sel},t}=\varepsilon_{\mathrm{sel}}/L$.

After $L$ rounds, the curator obtains $\widetilde S =\{J_1,\ldots,J_L\}$. It then forms the restricted vector
\begin{equation}
C_{\widetilde S}(D)=\bigl(C_j(D):j\in\widetilde S\bigr).
\end{equation}
The admission step privately selects the atom indices; it does not release the exact masses on those indices. Therefore $C_{\widetilde S}(D)$ is not released directly. Conditional on the realized admitted set $\widetilde S$, the curator applies a differentially private value-release mechanism to $C_{\widetilde S}(D)$. In our experiments, this mechanism uses Gaussian noise calibrated \citep{balle2018improving} to the sensitivity of the selected subvector, with privacy parameters $(\varepsilon_{\mathrm{val}},\delta_{\mathrm{val}})$. The released values are denoted by $\widetilde c_{\widetilde S}$.

The final sparse plan is $\big( \widetilde S, \widetilde c_{\widetilde S}, \{(\ell(a_j), d(a_j)):j\in\widetilde S\} \big)$. No value is released for atoms outside $\widetilde S$. Thus, the sparse mechanism has two differentially private stages, consisting of atom admission by the exponential mechanism and selected-value release by a differentially private value-release mechanism. The privacy costs of these stages are combined by sequential composition.

\begin{algorithm}[t]
\caption{\textsc{DPPlanRelease}}
\label{alg:plan-release}
\begin{algorithmic}[1]
\STATE \textbf{Input:} sketch $C\in\mathbb R_+^m$; fixed semantic atom universe
$\mathcal A$; plan size $L$; release mode $\mathsf{mode}$; public denominator
$d_{\mathrm{public}}$; privacy parameters and sensitivity bounds.
\IF{$\mathsf{mode}=\mathsf{sketch}$}
    \STATE Sample noise $Z$ calibrated to the sensitivity of $D\mapsto C(D)$.
    \STATE $\widetilde C \gets C+Z$ \AlgComment{differentially private noisy sketch}
    \STATE $\widetilde S \gets \Top_L(\widetilde C)$
    \STATE $\widetilde c_{\widetilde S}\gets \widetilde C_{\widetilde S}$
\ELSIF{$\mathsf{mode}=\mathsf{sparse}$}
    \STATE $\widetilde S\gets\emptyset$ \AlgComment{private atom admission}
    \FOR{$t=1,\ldots,L$}
        \STATE Sample $J_t\in[m]\setminus\widetilde S$ with
        $\Pr[J_t=j\mid \widetilde S]\propto \exp \left( \frac{\varepsilon_{\mathrm{sel},t}C_j}{2\Delta_{\mathrm{coord}}}\right), \; j\in[m]\setminus\widetilde S$.
        
        \STATE $\widetilde S\gets \widetilde S\cup\{J_t\}$
    \ENDFOR
    \STATE Sample noise $Z_{\widetilde S}$ calibrated to the sensitivity of
    $D\mapsto C_{\widetilde S}(D)$.
    \STATE $\widetilde c_{\widetilde S}\gets C_{\widetilde S}+Z_{\widetilde S}$
    \AlgComment{released noisy masses on admitted atoms}
\ENDIF
\STATE $\widetilde c_{\widetilde S}^{+} \gets \max\{0,\widetilde c_{\widetilde S}\}$
\AlgComment{coordinatewise post-processing}
\vspace{2pt}
\STATE Compute normalized semantic-support values, support bins, and released ranks by post-processing.
\STATE $\mathcal L_{\widetilde S}\gets \{(\ell(a_j),d(a_j)):j\in\widetilde S\}$
\STATE \textbf{return}
$P=(\widetilde S,\widetilde c_{\widetilde S},\mathcal L_{\widetilde S})$
\end{algorithmic}
\end{algorithm}

\subsection{Stage III: Contract-constrained verbalization}

The decoder receives only the released semantic plan $P(D)$ and public instructions. It does not receive raw records, raw excerpts, nearest neighbors, per-record contribution vectors, atoms outside the released plan, or non-private representatives. Under Assumption~\ref{ass:public-atoms}, all decoder-visible atom labels and descriptions are fixed independently of the protected target records. Let $H_{r-1}$ denote the candidate and verifier-feedback history available before attempt $r$. A candidate summary is generated as
\begin{equation}
Y_r=G(P(D),\Pi,H_{r-1};\omega_r),
\end{equation}
where $\Pi$ is the fixed decoder policy and $\omega_r$ denotes downstream decoder randomness at attempt $r$.

A public verifier (plan-consistency checker) validates whether each candidate summary satisfies the finite decoder contract in Definition~\ref{def:decoder-contract}. The decoder may make at most $R_{\max}$ attempts, where $R_{\max}$ is fixed independently of the protected data. Each attempt uses the same released plan $P(D)$, public decoding rules, public verifier feedback from previous attempts, and decoder randomness independent of $D$ given $P(D)$. The mechanism returns the first accepted candidate. If no candidate is accepted after $R_{\max}$ attempts, it returns a deterministic fallback summary computed from the same released plan. The final output is post-processing of the differentially private plan because decoding, verification, rejection, regeneration, verifier feedback, and the fallback rule depend only on $P(D)$, public rules, and randomness independent of $D$ given $P(D)$.

\section{Sensitivity Analysis}

For $p\in\{1,2\}$, let $\Delta_p(C)=\sup_{D\sim D'}\|C(D)-C(D')\|_p$ denote the global $\ell_p$-sensitivity under the adjacency relation  being considered.

\begin{lemma}[Record-level add/drop sensitivity]
\label{lem:adddrop}
Under \eqref{eq:phi-bound} and record-level add/drop adjacency, $\Delta_1(C)\le 1, \, \Delta_2(C)\le 1$.
\end{lemma}

\begin{proof}
Adding or removing one record changes the sketch by $\phi(x)$ or $-\phi(x)$. By Eq.~\eqref{eq:phi-bound}, $\|\phi(x)\|_1\le 1$. Since $\phi(x)\in\mathbb R_+^m$, we have $\|\phi(x)\|_2\le \|\phi(x)\|_1\le 1$.
\end{proof}

\begin{lemma}[Record-level replacement sensitivity]
\label{lem:replacement}
Under Eq.~\eqref{eq:phi-bound} and record-level replacement adjacency, $\Delta_1(C)\le 2$ and $\Delta_2(C)\le \sqrt{2}$. For $m\ge 2$, both bounds are tight.
\end{lemma}

\begin{proof}
Replacing $x$ by $x'$ changes the sketch by $\phi(x')-\phi(x)$. The $\ell_1$ bound follows from the triangle inequality $\|\phi(x')-\phi(x)\|_1 \le \|\phi(x')\|_1+\|\phi(x)\|_1 \le 2$. Let $u=\phi(x)$ and $v=\phi(x')$. Since $u,v\in\mathbb R_+^m$, we have $\norm{u-v}_2^2 =\norm{u}_2^2+\norm{v}_2^2-2\langle u,v\rangle \le \norm{u}_1^2+\norm{v}_1^2\le2$. Thus $\Delta_2(C)\le\sqrt{2}$. For $m\ge 2$, the bounds are attained by two distinct standard basis vectors.
\end{proof}

\begin{lemma}[User-level sensitivity after clipping]
\label{lem:user}
Suppose records are partitioned by user, and let $U_u(D)=\sum_{x\in D_u}\phi(x)$, $V_u(D)=\mathsf{Clip}_B\!\left(U_u(D)\right)$, $C(D)=\sum_u V_u(D)$. Under user-level add/drop adjacency, $\Delta_1(C)\le B$, $\Delta_2(C)\le B$. Under user-level replacement adjacency, $\Delta_1(C)\le2B$, $\Delta_2(C)\le\sqrt{2}B$.
\end{lemma}

\begin{proof}
Under user-level add/drop adjacency, the sketch changes by one clipped nonnegative vector $V_u$. Since $\|V_u\|_1\le B$, we also have $\|V_u\|_2\le B$. Under user-level replacement adjacency, one clipped contribution $V_u$ is replaced by another clipped contribution $V_{u'}$. Hence $\|V_{u'}-V_u\|_1 \le \|V_{u'}\|_1+\|V_u\|_1 \le 2B$.
Moreover, because $V_u,V_{u'}\in\mathbb R_+^m$, $\|V_{u'}-V_u\|_2^2 = \|V_{u'}\|_2^2+\|V_u\|_2^2-2\langle V_{u'},V_u\rangle \le \|V_{u'}\|_1^2+\|V_u\|_1^2 \le 2B^2$. Thus $\Delta_2(C)\le \sqrt{2}B$.
\end{proof}

\subsection{Release mechanisms}

\begin{theorem}[Laplace noisy-sketch mechanism]
\label{thm:laplace}
Let $C$ have global $\ell_1$-sensitivity $\Delta_1(C)$, and let $b=\Delta_1(C)/\varepsilon$. The mechanism $\widetilde C_j=C_j(D)+\eta_j$, $\eta_j \stackrel{\mathrm{i.i.d.}}{\sim}\Lap(b)$, where $\Lap(b)$ denotes the Laplace distribution with scale $b$, is $(\varepsilon,0)$-DP. In particular, $b=1/\varepsilon$ suffices for record-level add/drop adjacency, and $b=2/\varepsilon$ suffices for record-level replacement adjacency. Any atom set, noisy masses, normalized semantic-support values or bins computed with a public denominator, or semantic plan computed from $\widetilde C$ and public objects is post-processing.
\end{theorem}

\vspace{-4pt}

\begin{proof}
This is the Laplace mechanism applied to $C(D)$ with global $\ell_1$-sensitivity $\Delta_1(C)$. The remaining operations are post-processing.
\end{proof}

\begin{theorem}[Gaussian noisy-sketch mechanism]
\label{thm:gaussian}
Let $C$ have global $\ell_2$-sensitivity $\Delta_2(C)$. The mechanism $\widetilde C=C(D)+Z$, $Z\sim\mathcal N(0,\sigma^2 I_m)$, is $(\varepsilon,\delta)$-DP whenever $\sigma$ satisfies a valid Gaussian-mechanism calibration for sensitivity $\Delta_2(C)$. In particular, for $0<\varepsilon<1$ and $0 < \delta < 1$, the classical sufficient calibration is $\sigma\ge \frac{\Delta_2(C)\sqrt{2\log(1.25/\delta)}}{\varepsilon}$. Any atom set, released noisy masses, normalized semantic-support values or bins computed with a public denominator, or semantic plan computed from $\widetilde C$ and public objects is post-processing.
\end{theorem}

\vspace{-4pt}

\begin{proof}
This is the Gaussian mechanism applied to $C(D)$ with global $\ell_2$-sensitivity $\Delta_2(C)$. The remaining operations are post-processing.
\end{proof}

\begin{proposition}[Sequential private atom admission]
\label{prop:private-admission}
Let $1\le L\le m$. Let $C(D)\in\mathbb R_+^m$ be an aggregate semantic sketch whose protected-unit contribution is nonnegative and has $\ell_1$-norm at most $B>0$. Under add/drop adjacency, adding or removing one protected unit changes $C(D)$ by a nonnegative vector with $\ell_1$-norm at most $B$. In the record-level setting, $B=1$. In the user-level setting, $B$ is the user contribution bound after clipping. For each coordinate $j\in[m]$, define the utility $u(D,j)=C_j(D)$. Under add/drop adjacency,
\begin{equation}
\Delta_{\mathrm{coord}} \coloneqq \sup_{D\sim D'}\sup_{j\in[m]}
|u(D,j)-u(D',j)| \le B.
\end{equation}
Under replacement adjacency between protected units whose contributions are nonnegative and have $\ell_1$-norm at most $B$, the same coordinate sensitivity bound $\Delta_{\mathrm{coord}} \le B$ holds.

Consider the following sequential admission procedure. At round $t\in\{1,\ldots,L\}$, after atoms $S_{t-1}$ have already been admitted, sample one new atom index $J_t\in[m]\setminus S_{t-1}$ using the exponential mechanism with utility $u(D,j)$, sensitivity $\Delta_{\mathrm{coord}}$, and privacy parameter $\varepsilon_{\mathrm{sel},t}$:
\begin{equation}
\Pr[J_t=j\mid S_{t-1}]
= \frac{\exp\left(\frac{\varepsilon_{\mathrm{sel},t}u(D,j)}{2\Delta_{\mathrm{coord}}}\right)}{
\sum_{\ell\in[m]\setminus S_{t-1}}\exp\left(\frac{\varepsilon_{\mathrm{sel},t}u(D,\ell)}{2\Delta_{\mathrm{coord}}}\right)},
\qquad j\in[m]\setminus S_{t-1}.
\end{equation}
Let $\widetilde S=\{J_1,\ldots,J_L\}$. Then $\widetilde S$ is
\begin{equation}
\left(\sum_{t=1}^L\varepsilon_{\mathrm{sel},t},0\right)\text{-DP}. \nonumber
\end{equation}

If, conditional on every realization $\widetilde S=S$, the restricted vector $C_S(D)$ is released by an $(\varepsilon_{\mathrm{val}},\delta_{\mathrm{val}})$-DP value-release mechanism, then the complete sparse plan $P(D)= \bigl( \widetilde S, \widetilde c_{\widetilde S}, \{(\ell(a_j),d(a_j)):j\in\widetilde S\}\bigr)$ is
\begin{equation}
\left( \sum_{t=1}^L\varepsilon_{\mathrm{sel},t} + \varepsilon_{\mathrm{val}}, \delta_{\mathrm{val}} \right)\text{-DP}. \nonumber
\end{equation}
For Gaussian selected-value release, for every fixed set $S\subseteq[m]$ with $|S|=L$, the restricted map $D\mapsto C_{S}(D)$ has $\ell_2$-sensitivity at most $B$ under add/drop adjacency and at most $\sqrt{2}B$ under replacement adjacency.
\end{proposition}

\begin{proof}
First consider coordinate sensitivity. Under add/drop adjacency, two adjacent datasets differ by the addition or removal of one protected unit. The contribution of that protected unit has $\ell_1$-norm at most $B$, so each coordinate changes by at most $B$. Hence $\Delta_{\mathrm{coord}}\le B$.
Under replacement adjacency, one protected-unit contribution is replaced by another. Both contributions are nonnegative and have $\ell_1$-norm at most $B$. Each coordinate of each contribution therefore lies in $[0,B]$, so the absolute change of any coordinate is at most $B$. Thus $\Delta_{\mathrm{coord}}\le B$ under replacement adjacency as well.
Fix a round $t$ and condition on any previously selected set $S_{t-1}$. The candidate set $[m]\setminus S_{t-1}$ is then fixed. The round-$t$ selection rule is the exponential mechanism with utility sensitivity $\Delta_{\mathrm{coord}}$ and privacy parameter $\varepsilon_{\mathrm{sel},t}$. Thus, for any fixed history, round $t$ is
$(\varepsilon_{\mathrm{sel},t},0)$-DP. Adaptive sequential composition gives $\left(\sum_{t=1}^L\varepsilon_{\mathrm{sel},t},0\right)\text{-DP}$ for the selected set $\widetilde S$.
For every fixed selected set $\widetilde S$, the value-release mechanism is assumed to be $(\varepsilon_{\mathrm{val}},\delta_{\mathrm{val}})$-DP for the map $D\mapsto C_{\widetilde S}(D)$. Adaptive composition between the private selection stage and the selected-value release gives the stated privacy bound for $(\widetilde S,\widetilde c_{\widetilde S})$. Under Assumption~\ref{ass:public-atoms}, the atom universe, labels, and descriptions are fixed independently of $D$. Attaching the labels and descriptions indexed by $\widetilde S$ is therefore post-processing.
For the Gaussian selected-value release stated above, the sensitivity of $D\mapsto C_{\widetilde S}(D)$ is at most $B$ in $\ell_2$ under add/drop adjacency. Under replacement adjacency, the change in the selected subvector is $v'_S-v_S$, where $v_S,v'_S\in\mathbb R_+^{|S|}$ are the restrictions of the two protected-unit contributions to $S$ and satisfy $\|v_S\|_1,\|v'_S\|_1\le B$. Hence $\|v'_S-v_S\|_2\le\sqrt{2}B$.
\end{proof}

\begin{remark}[Atom labels and atom admission]
\label{rem:labels-vs-admission}
Proposition~\ref{prop:private-admission} protects atom admission from an atom universe fixed independently of the protected target records. It does not cover an atom universe, label, or description derived non-privately from the protected target records. Such objects would require a separate differentially private release and corresponding privacy accounting.
\end{remark}

\begin{theorem}[Private plans imply private summaries]
\label{thm:endtoend}
Let $\mathcal M:\mathcal D\to\mathcal P$ be an $(\varepsilon,\delta)$-differentially private mechanism that releases the semantic plan $P(D)=\mathcal M(D)$. Let $\Pi$ be a decoding policy, $\mathsf V$ a verifier,
$\mathsf{Feedback}$ a feedback rule, $\mathsf{Fallback}$ a fallback rule, and $R_{\max}\in\mathbb N$ a maximum number of decoding attempts. Assume that these objects are fixed independently of the protected dataset $D$ and that the downstream procedure has access to $D$ only through $P(D)$. Set $H_0=\varnothing$. At each reached attempt $r\in\{1,\ldots,R_{\max}\}$, generate $Y_r=G\!\left(P(D),\Pi,H_{r-1};\omega_r\right)$, and let $I_r = \mathsf V\!\left(Y_r,P(D),\Pi,H_{r-1};\xi_r\right) \in\{0,1\}$, where $I_r=1$ denotes acceptance. If $I_r=1$, the procedure terminates and returns $Y_r$. Otherwise, it computes $F_r = \mathsf{Feedback}\left( Y_r,P(D),\Pi,H_{r-1},I_r \right)$ and updates $H_r=(H_{r-1},Y_r,I_r,F_r)$. Conditional on $P(D)$, the distribution of all downstream randomness is independent of $D$. If no candidate is accepted after $R_{\max}$ attempts, the procedure returns $\mathsf{Fallback}(P(D),\Pi)$. Then the final released summary is $(\varepsilon,\delta)$-differentially private.
\end{theorem}

\vspace{-3pt}

\begin{proof}
For any released plan $P(D)$, candidate generation, verification, feedback, regeneration, the stopping rule, and fallback generation depend on the protected dataset only through $P(D)$ and on information fixed independently
of $D$. Hence the complete downstream procedure defines a randomized post-processing of $P(D)$. Since $P(D)=\mathcal M(D)$ is $(\varepsilon,\delta)$-differentially private and differential privacy is preserved under randomized post-processing, the final released summary is $(\varepsilon,\delta)$-differentially private.
\end{proof}

\vspace{-3pt}
 
\begin{remark}[Privacy boundary]
\label{rem:privacy-boundary}
Theorem~\ref{thm:endtoend} applies only when every dataset-dependent object observed by the decoder or verifier is contained in the released DP plan. The conclusion need not hold if the decoder or verifier receives any additional object that depends on the protected dataset and is exposed outside $P(D)$ without a corresponding differential-privacy guarantee.
\end{remark}

Appendix~\ref{app:utility-analysis} gives the plan-level utility statements and proofs.

\section{Experiments}
\label{sec:experiments}

We evaluate \texttt{DP-SPIN} as a framework for differentially private aggregate insight generation. Our experiments address four questions. First, we measure whether the released plan recovers high-mass atoms of the non-private semantic sketch. Second, we evaluate how plan utility varies with the privacy budget, atom dictionary size, plan size, and assignment sparsity. Third, we test whether generated summaries remain consistent with the released plan under the public verifier. Fourth, we compare \texttt{DP-SPIN} with non-private plan and summary references, differentially private histogram baselines, and a URANIA-style public-keyword baseline.

\vspace{-3pt}

\paragraph{Datasets and protected units.}
We use three text corpora: consumer complaint narratives from the CFPB Consumer Complaint Database \citep{cfpb_consumer_complaint_database}, Amazon Reviews 2023 \texttt{All\_Beauty} reviews \citep{hou2026bridging}, and restaurant reviews from the Yelp Open
Dataset \citep{yelp_open_dataset}. In record-level experiments, one text record is one protected unit, and the adjacency relation is add/drop adjacency. In user-level experiments, all records associated with one user are treated as one protected unit. The aggregate semantic contribution of that user is deterministically clipped to a fixed $\ell_1$ contribution bound before release. CFPB is evaluated only in the record-level setting because the prepared complaint records do not contain a stable user identifier for grouping multiple records from the same individual. Amazon and Yelp are evaluated in both record-level and user-level settings because their review records contain stable reviewer identifiers. Dataset-specific construction details are given in Appendices~\ref{app:cfpb-atom-construction}, \ref{app:amazon-atom-construction}, and~\ref{app:yelp-atom-construction}.

\vspace{-3pt}

\paragraph{Privacy boundary.}
Each dataset is split into an auxiliary split and a protected target split. The auxiliary split is used to construct decoder-visible atom labels and to fit the fixed text representation used by the assignment map. The protected target split is treated as the protected corpus. It affects the released output only through bounded record-to-atom assignments, aggregation into the semantic sketch, and the differentially private plan-release mechanism. The decoder never receives raw target records, target excerpts, nearest neighbors, per-record contribution vectors, non-private sketch values, non-private representatives or exemplars, atoms outside the released plan, or labels derived non-privately from the protected target split. For user-level experiments, the auxiliary and protected target splits are user-disjoint. Additional configuration details are given in Appendix~\ref{app:experimental-configuration}.

\vspace{-3pt}

\paragraph{Assignment map.}
Each target record is mapped to a sparse nonnegative vector over the fixed semantic atom universe. The atom dictionary and text representation are fixed before the protected target split is processed. For each target record, the assignment map computes similarity scores to fixed atom descriptions, preserves the $K$ highest-scoring atoms under a fixed public tie-breaking rule, and normalizes the nonnegative weights with total $\ell_1$ mass at most one. Thus each record contribution satisfies $\|v_i\|_0\le K$ and $\|v_i\|_1\le 1$, matching the bounded-contribution condition used in the privacy analysis. The main configuration uses assignment sparsity $K=3$ and released plan size $L=10$.

\vspace{-3pt}

\paragraph{Release mechanisms.}
We evaluate three \texttt{DP-SPIN} plan-release mechanisms. The Gaussian sketch mechanism perturbs the full semantic sketch with Gaussian noise calibrated to the global $\ell_2$-sensitivity under the chosen adjacency relation, then selects the admitted atom set by post-processing. The Laplace sketch mechanism uses Laplace noise calibrated to the global $\ell_1$-sensitivity and applies the same post-processing step. The sparse private atom-admission mechanism first admits $L$ atoms by sequential exponential-mechanism selection and then releases noisy masses only for the admitted atoms. For the full-sketch mechanisms, atom selection is post-processing of the noisy sketch. For the sparse mechanism, atom admission and selected-value release both consume privacy budget.

\vspace{-4pt}

\begin{figure}[!t]
    \centering
    \includegraphics[width=0.77\linewidth]{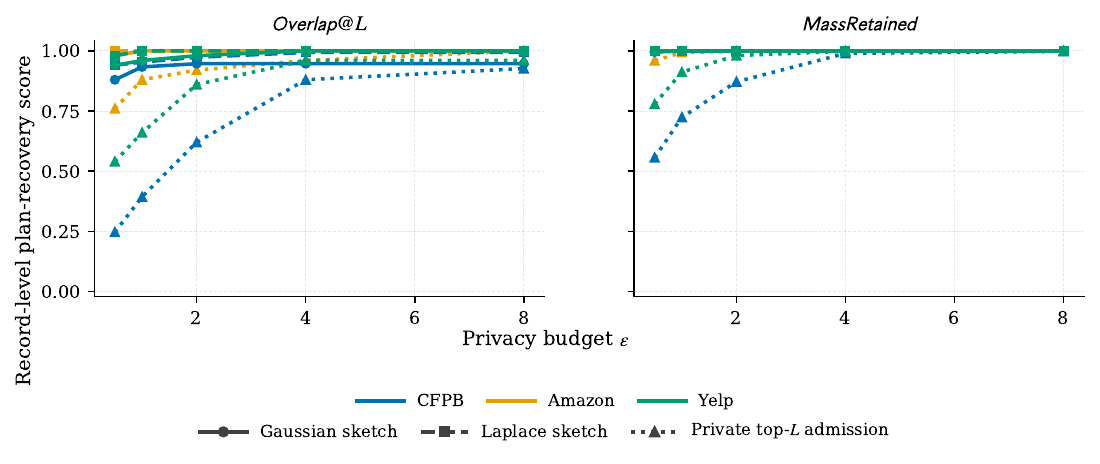}
    \vspace{-9pt}
    \caption{Record-level plan utility. The metrics are $\mathsf{Overlap}@L$ and $\mathsf{MassRetained}$, evaluated as the privacy budget $\varepsilon$ varies with assignment sparsity $K=3$ and released plan size $L=10$. Colors indicate datasets, and line styles indicate release mechanisms.}
    \vspace{-7pt}
    \label{fig:main-record-plan-utility}
\end{figure}
 
\begin{figure}[!t]
    \centering
    \includegraphics[width=0.77\linewidth]{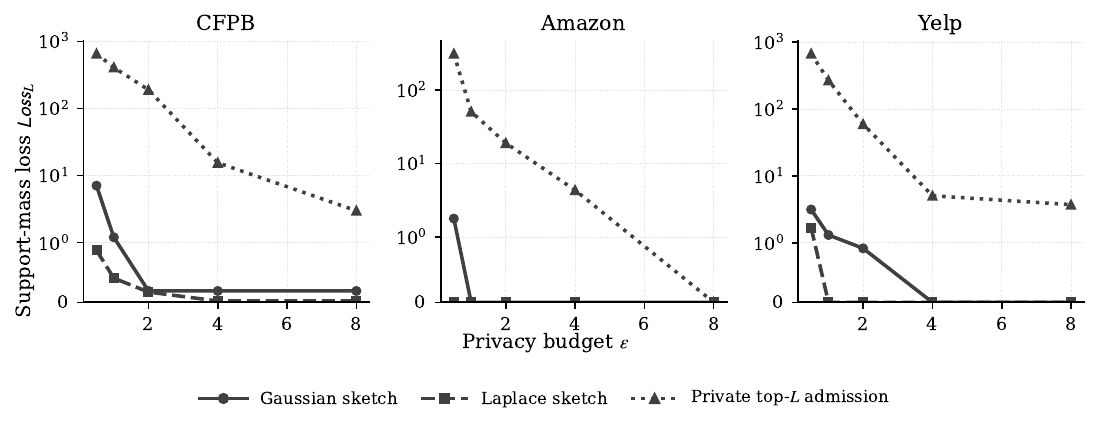}
    \vspace{-9pt}
    \caption{Record-level support-mass loss. The metric is $\mathsf{Loss}_L$, evaluated as the privacy budget $\varepsilon$ varies with assignment sparsity $K=3$ and released plan size $L=10$. Lower values indicate less non-private semantic mass lost by the released atom set. The vertical axis uses a symmetric logarithmic scale to show both near-zero full-sketch losses and larger private-admission losses.}
    \vspace{-7pt}
    \label{fig:main-record-support-loss}
\end{figure}

\paragraph{Baselines.}
We compare against non-private utility references and differentially private baselines. The non-private references include direct target-record summarization, heuristically redacted target-record summarization, non-private clustering, and non-private plan verbalization. These references measure utility when the decoder or preprocessing pipeline has access to target data; they do not satisfy the \texttt{DP-SPIN} decoder boundary. The differentially private baselines include fixed-vocabulary keyword and category histograms. Because no public reference implementation of \texttt{URANIA} was available at the time of our experiments, we independently implemented a \texttt{URANIA}-style baseline from the published method description. For comparison with \texttt{DP-SPIN}, we adapted this implementation to the matched public-vocabulary protocol used in our experiments. Our baseline preserves \texttt{URANIA}'s principal release structure: differentially private clustering, noisy cluster-size filtering, differentially private per-cluster keyword histograms, and summary generation from released keywords and noisy cluster-size information. The comparison evaluates this release structure rather than reproducing the original \texttt{URANIA} results.

\vspace{-4pt}

\paragraph{Metrics.}
We report $\mathsf{Overlap}@L$, kept non-private mass fraction, support-mass loss, and raw noisy-value errors relative to the non-private semantic sketch. Noisy-value errors are computed before decoder-facing nonnegative clipping. Because the template decoder constructs each summary from admitted plan entries, its verifier outcomes serve as consistency checks rather than measures of plan utility. For free-form decoded summaries, the public verifier checks consistency with the released plan. As secondary verbalization measures, we report lexical similarity to a non-private-plan reference summary, including token, $n$-gram, keyphrase, and TF--IDF measures. These text-level measures are not part of the privacy mechanism. Metric definitions are given in Appendix~\ref{app:evaluation-metrics}, and plan-level utility guarantees are given in Appendix~\ref{app:utility-analysis}.

\vspace{-4pt}

\subsection{Plan utility across datasets}
\label{subsec:main-plan-utility}

Figure~\ref{fig:main-record-plan-utility} shows record-level plan utility with $K=3$ and $L=10$. The Gaussian and Laplace full-sketch mechanisms preserve a large fraction of the semantic mass carried by the non-private top-$L$ support over the tested privacy budgets. Exact support recovery, measured by $\mathsf{Overlap}@L$, can be more sensitive than $\mathsf{MassRetained}$ when atoms near the non-private top-$L$ boundary have similar masses. A mechanism can therefore miss some non-private top-$L$ atoms while keeping most of the semantic mass of the non-private plan.

Figure~\ref{fig:main-record-support-loss} shows the corresponding support-mass loss $\mathsf{Loss}_L$. Under the main record-level configuration, the full-sketch mechanisms have lower support-mass loss than sparse private atom admission at the smaller tested $\varepsilon$ values. The mechanisms allocate the privacy budget differently. The full-sketch mechanisms spend the privacy budget on the noisy sketch and select atoms by post-processing, whereas sparse private atom admission allocates part of the privacy budget to atom admission before selected-value release. As $\varepsilon$ increases, the support-mass loss generally decreases, especially for private top-$L$ admission. Since $\mathsf{Loss}_L$ is unnormalized, its numerical scale depends on the dataset size and aggregate sketch mass; values should therefore be compared within each dataset panel.

\vspace{-4pt}

\subsection{Record-level and user-level settings}
\label{subsec:record-user-results}

Amazon and Yelp support both record-level and user-level evaluation. In the record-level setting, one review is the protected unit. In the user-level setting, one reviewer is the protected unit. All target records from the same reviewer are aggregated into a user contribution, which is clipped before summation across users and differentially private plan release.

Figure~\ref{fig:main-user-plan-utility} shows user-level plan utility for Amazon and Yelp at $K=3$, $L=10$, and user contribution bound $B=1$. The Gaussian and Laplace full-sketch mechanisms have $\mathsf{MassRetained}$ close to one across the tested privacy budgets. Private top-$L$ admission shows a larger dependence on the privacy budget, especially at smaller $\varepsilon$, but its plan utility improves as $\varepsilon$ increases. The decoder interface is unchanged in the user-level setting: it receives only the released differentially private plan.

\begin{figure}[!t]
    \centering
    \includegraphics[width=0.78\linewidth]{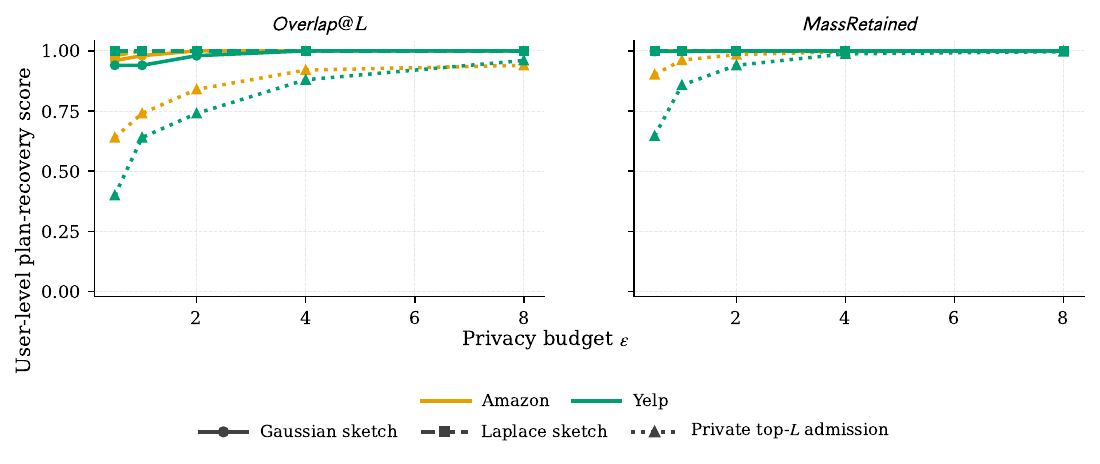}
    \vspace{-5pt}
    \caption{User-level plan utility for Amazon and Yelp. The metrics are $\mathsf{Overlap}@L$ and $\mathsf{MassRetained}$, evaluated as the privacy budget $\varepsilon$ varies with assignment sparsity $K=3$, released plan size $L=10$, and user contribution bound $B=1$. Colors indicate datasets, and line styles indicate release mechanisms.}
    \label{fig:main-user-plan-utility}
\end{figure}

\vspace{-2pt}

\begin{figure}[!t]
    \centering
    \includegraphics[width=0.78\linewidth]{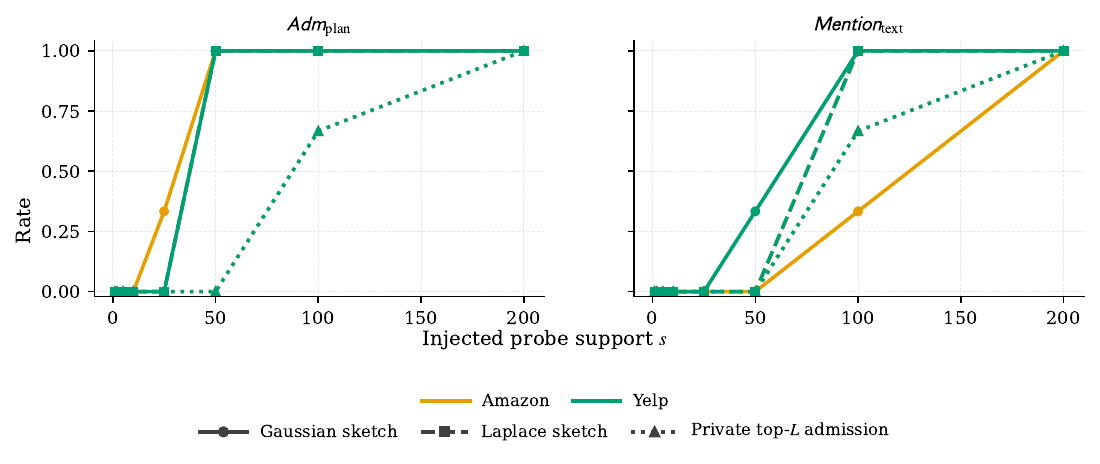}
    \vspace{-5pt}
    \caption{
    Controlled probe-atom behavior in the record-level setting. We report plan-admission rate $\mathsf{Adm}_{\mathrm{plan}}$ and summary mention rate $\mathsf{Mention}_{\mathrm{text}}$ as the injected probe support $s$ varies, with $K=3$ and $L=10$. The probe atom is a public coordinate added before the differentially private release.
    }
    \label{fig:main-controlled-probe-record}
\end{figure}

\subsection{Controlled probe-atom experiment}
\label{subsec:controlled-probe-main}

Figure~\ref{fig:main-controlled-probe-record} reports a controlled probe-atom experiment in the record-level setting. The probe atom is included in the fixed semantic atom universe before release, and synthetic probe records are inserted only into the protected target split. The experiment therefore evaluates private admission of a known public coordinate, not discovery of a new label from protected text.
Probe-atom admission is infrequent at small injected support values and generally increases as the injected support $s$ increases. The summary mention rate $\mathsf{Mention}_{\mathrm{text}}$ follows the same trend but can be lower than the plan-admission rate because an admitted atom need not be verbalized in every generated summary. The probe experiment characterizes how aggregate support affects admission; the differential-privacy guarantee follows from the mechanism analysis rather than from this empirical behavior.

\begin{table*}[t]
\centering
\vspace{5pt}
\caption{Release-conditioned OpenAI evaluation of selected record-level summaries across datasets. Rows compare the \texttt{DP-SPIN} release condition, differentially private keyword and category histograms, and the URANIA-style public-keyword baseline. Each summary is evaluated only with respect to the object released to its decoder. Scores are assigned on a 1--5 scale and reported as mean $\pm$ sample standard deviation across the judged summaries in each row, where each summary corresponds to one experimental seed. The text-safety score evaluates only the generated text and is not a differential-privacy guarantee.
}
\label{tab:openai-main-comparison}
\scriptsize
\resizebox{0.99\textwidth}{!}{%
\begin{tabular}{lllccccccc}
\toprule
Dataset & Method & Release configuration & $\varepsilon$ & Protected units & Judged summaries & Coverage & Insightfulness & Faithfulness  & Text safety \\
\midrule
CFPB & DP-SPIN & record; Gaussian sketch + top-$L$; m=200; K=3; L=10 & 1 & 2000 & 3 & 4.67 $\pm$ 0.58 & 3.67 $\pm$ 0.58 & 4.67 $\pm$ 0.58 & 5.00 $\pm$ 0.00 \\
CFPB & DP keyword histogram & fixed public keywords; Laplace histogram; top-10 & 1 & 2000 & 3 & 5.00 $\pm$ 0.00 & 3.00 $\pm$ 0.00 & 5.00 $\pm$ 0.00 & 5.00 $\pm$ 0.00 \\
CFPB & DP category histogram & fixed Issue universe; Laplace histogram; top-10 & 1 & 2000 & 3 & 5.00 $\pm$ 0.00 & 4.00 $\pm$ 0.00 & 5.00 $\pm$ 0.00 & 5.00 $\pm$ 0.00 \\
CFPB & URANIA-style public keywords & DP clustering; DP public-keyword histograms; clusters=10; keywords=10 & 1 & 2000 & 3 & 4.33 $\pm$ 0.58 & 4.00 $\pm$ 0.00 & 3.67 $\pm$ 0.58 & 5.00 $\pm$ 0.00 \\
Amazon & DP-SPIN & record; Gaussian sketch + top-$L$; m=46; K=3; L=10 & 1 & 2000 & 3 & 5.00 $\pm$ 0.00 & 4.00 $\pm$ 0.00 & 5.00 $\pm$ 0.00 & 5.00 $\pm$ 0.00 \\
Amazon & DP keyword histogram & fixed public keywords; Laplace histogram; top-10 & 1 & 2000 & 3 & 5.00 $\pm$ 0.00 & 3.67 $\pm$ 0.58 & 5.00 $\pm$ 0.00 & 5.00 $\pm$ 0.00 \\
Amazon & DP category histogram & fixed rating\_bin universe; Laplace histogram; top-3 & 1 & 2000 & 3 & 5.00 $\pm$ 0.00 & 3.33 $\pm$ 0.58 & 5.00 $\pm$ 0.00 & 5.00 $\pm$ 0.00 \\
Amazon & URANIA-style public keywords & DP clustering; DP public-keyword histograms; clusters=10; keywords=10 & 1 & 2000 & 3 & 4.00 $\pm$ 0.00 & 4.00 $\pm$ 0.00 & 4.00 $\pm$ 1.00 & 5.00 $\pm$ 0.00 \\
Yelp & DP-SPIN & record; Gaussian sketch + top-$L$; m=100; K=3; L=10 & 1 & 5000 & 3 & 4.67 $\pm$ 0.58 & 3.33 $\pm$ 0.58 & 5.00 $\pm$ 0.00 & 5.00 $\pm$ 0.00 \\
Yelp & DP keyword histogram & fixed public keywords; Laplace histogram; top-10 & 1 & 2000 & 3 & 5.00 $\pm$ 0.00 & 4.00 $\pm$ 0.00 & 5.00 $\pm$ 0.00 & 5.00 $\pm$ 0.00 \\
Yelp & DP category histogram & fixed rating\_bin universe; Laplace histogram; top-3 & 1 & 2000 & 3 & 5.00 $\pm$ 0.00 & 3.33 $\pm$ 0.58 & 5.00 $\pm$ 0.00 & 5.00 $\pm$ 0.00 \\
Yelp & URANIA-style public keywords & DP clustering; DP public-keyword histograms; clusters=10; keywords=10 & 1 & 2000 & 3 & 4.33 $\pm$ 0.58 & 4.00 $\pm$ 0.00 & 3.67 $\pm$ 0.58 & 5.00 $\pm$ 0.00 \\
\bottomrule
\end{tabular}%
}
\end{table*}

\vspace{2pt}

\begin{table*}[b]
\centering
\vspace{5pt}
\caption{
Release-conditioned OpenAI evaluation of selected \texttt{DP-SPIN} summaries across datasets. The decoder receives only the released differentially private semantic plan and public decoding instructions. 
The configuration column specifies the protected unit, release mechanism, atom dictionary size, assignment sparsity, and released plan size. Scores are reported as mean $\pm$ sample standard deviation across the
judged summaries in each row, where each summary corresponds to one experimental seed.
}
\label{tab:openai-main-dpspin}
\scriptsize
\resizebox{0.99\textwidth}{!}{%
\begin{tabular}{llllcccccc}
\toprule
Dataset & Method & Release configuration & $\varepsilon$ & Protected units & Judged summaries & Coverage & Insightfulness & Faithfulness & Clarity \\
\midrule
CFPB & DP-SPIN & record; Gaussian sketch + top-$L$; m=200; K=3; L=10 & 1 & 2000 & 3 & 4.67 $\pm$ 0.58 & 3.67 $\pm$ 0.58 & 4.67 $\pm$ 0.58 & 4.67 $\pm$ 0.58 \\
Amazon & DP-SPIN & record; Gaussian sketch + top-$L$; m=46; K=3; L=10 & 1 & 2000 & 3 & 5.00 $\pm$ 0.00 & 4.00 $\pm$ 0.00 & 5.00 $\pm$ 0.00 & 5.00 $\pm$ 0.00 \\
Yelp & DP-SPIN & record; Gaussian sketch + top-$L$; m=100; K=3; L=10 & 1 & 5000 & 3 & 4.67 $\pm$ 0.58 & 3.33 $\pm$ 0.58 & 5.00 $\pm$ 0.00 & 4.33 $\pm$ 0.58 \\
\bottomrule
\end{tabular}%
}
\end{table*}

\vspace{-2pt}

\subsection{Ablations}
\label{subsec:ablations}

Appendices~\ref{app:additional-results-cfpb}, \ref{app:additional-results-amazon}, and~\ref{app:additional-results-yelp} report dataset-specific ablations over the atom dictionary size $m$, assignment sparsity $K$, released plan size $L$, protected-unit count, and, for user-level runs, the clipping bound $B$. These experiments distinguish exact support recovery from kept semantic mass. Increasing $m$ includes additional atom coordinates but can reduce the mass margin around the non-private top-$L$ boundary. Increasing $L$ changes both the target support and the amount of non-private semantic mass that can be kept. For comparable record distributions, reducing the number of protected units typically reduces aggregate semantic mass, while the relevant global sensitivity bound remains unchanged when the per-unit contribution bound is fixed. We therefore report both $\mathsf{Overlap}@L$ and $\mathsf{Loss}_L$. The former measures exact support recovery, while the latter measures the non-private semantic mass lost by the released atom set.

\vspace{-4pt}

\subsection{Plan verbalization and baseline summaries}
\label{subsec:verbalization-baselines}

We next evaluate whether a language model can verbalize released plans without receiving raw protected records. In \texttt{DP-SPIN}, the decoder receives only admitted atom labels and descriptions fixed independently of the protected target records, released noisy masses or support bins, and public decoding instructions. It does not receive raw target records, target excerpts, nearest neighbors, per-record contribution vectors, non-private representatives, or target-derived labels outside the released plan.

Table~\ref{tab:openai-main-comparison} reports release-conditioned OpenAI evaluations for selected record-level outputs. Each method is evaluated with respect to the object released to its decoder. The table therefore measures verbalization quality conditioned on the released object, rather than plan-level recovery. DP keyword and category histograms can receive high scores when their fixed public vocabularies align with dominant dataset categories. The URANIA-style public-keyword baseline uses a different private intermediate representation, based on preserved clusters and differentially private public-keyword histograms. We include it as a matched public-vocabulary baseline, not as a full reproduction of URANIA.

Table~\ref{tab:openai-main-dpspin} reports selected \texttt{DP-SPIN} verbalization scores across datasets. The release-conditioned evaluations separate faithfulness from insightfulness, which depends on the semantic content available in the released plan. Table~\ref{tab:openai-main-qualitative} gives selected generated summaries. These examples illustrate the decoder interface: the language model produces aggregate text from a differentially private semantic plan, and all dataset-dependent quantities visible to the decoder are contained in the released plan.

\vspace{2pt}

Appendices~\ref{app:additional-results-cfpb}, \ref{app:additional-results-amazon}, and~\ref{app:additional-results-yelp} report the full dataset-specific tables, qualitative examples, private/non-private reference comparisons, and baseline comparisons. The OpenAI scores are evaluation measures and are not part of the privacy mechanism. The formal privacy guarantee follows from the differentially private plan release and post-processing.


\begin{table}[h]
\centering
\caption{
Selected summaries based on OpenAI-generated outputs from \texttt{DP-SPIN} plans across datasets. Each summary is generated only from the released differentially private semantic plan and public decoding instructions. DP-SPIN examples are terminology-normalized to the current semantic-support vocabulary; the admitted themes are unchanged. Summaries are truncated for display; full qualitative examples are reported in the appendices. 
}
\label{tab:openai-main-qualitative}
\scriptsize
\setlength{\tabcolsep}{3pt}
\renewcommand{\arraystretch}{1.04}
\resizebox{0.99\textwidth}{!}{%
\begin{tabular}{p{0.11\textwidth}p{0.26\textwidth}p{0.58\textwidth}}
\toprule
Dataset & Release configuration & Generated summary \\
\midrule
CFPB & record; Gaussian sketch + top-$L$; m=200; $K=3$; $L=10$ &
The aggregate summary identifies moderate semantic support for incorrect
information on credit reports and low semantic support for the admitted
credit-reporting, identity-theft, and debt-collection themes. -
Moderate-support theme: incorrect information on credit reports. -
Low-support themes: credit-reporting, identity-theft, and debt-collection
themes. \\ \hline
\addlinespace[1pt]
Amazon & record; Gaussian sketch + top-$L$; $m=46$; $K=3$; $L=10$ & The aggregate summary identifies high semantic support for easy application, sensitive skin, and waste money. Moderate semantic support is reported for product quality and long hair, while dry hair, natural hair, hair quality, dry skin, and skin feel have low support. - High-support themes: easy application, sensitive skin, waste money. - Moderate-support themes: product quality, long hair. - Low-support themes: dry hair, natural hair, hair quality, dry skin, skin feel.\\\hline
\addlinespace[1pt]
Yelp & record; Gaussian sketch + top-$L$; $m=100$; $K=3$; $L=10$ & The aggregate summary identifies low semantic support across the admitted dining-related themes, including long wait times, friendly staff, customer service, delicious food, pizza places, service staff, food places, Mexican food, lunch specials, and bar restaurants. - Long wait times and customer service have low released support. - Friendly staff has low released support. - Several food-related themes, including Mexican food and... \\
\addlinespace[1pt]
\bottomrule
\end{tabular}%
}
\end{table}

\section{Conclusion}

We introduced \texttt{DP-SPIN}, a framework for differentially private aggregate insight generation from text collections. Each protected record is mapped inside the trusted curator to a bounded sparse semantic vector. For record-level privacy, these vectors are summed directly; for user-level privacy, they are aggregated and clipped per user before summation. The resulting semantic sketch is converted into a differentially private semantic plan containing admitted atoms and noisy masses. Normalized semantic-support values and support bins, when reported, are obtained by post-processing with a public denominator. The language model receives only the released plan and information fixed independently of the protected target records. The final summary is differentially private by post-processing of the released plan. We evaluate plan utility using raw noisy-value error, top-$L$ support recovery, preserved non-private semantic-mass fraction, and support-mass loss. Experiments on CFPB complaint narratives, Amazon All Beauty reviews, and Yelp restaurant reviews use disjoint auxiliary and protected target splits. The experiments characterize plan utility across privacy budgets, evaluate controlled probe-atom admission, and measure consistency of generated summaries with the released plan under the public verifier.

\section*{Acknowledgment}
This work was supported by the Swiss National Science Foundation (SNSF) under Grant No.~222339. The author thanks Prof.~Flavio~P.~Calmon for insightful discussions and helpful suggestions.

\bibliographystyle{plainnat}
\bibliography{references}

\clearpage

\clearpage
\appendix

\AppendixOnlyTOC

\clearpage

\appsection{Utility Analysis}
\label{app:utility-analysis}

In this appendix, we analyze the plan-level utility of \texttt{DP-SPIN}. We evaluate utility at the level of the released semantic plan, not the generated natural-language summary itself. We quantify perturbation of the non-private semantic sketch, recovery of high-mass atoms, and support-mass loss relative to the non-private sketch. We evaluate consistency of the generated text with the released plan separately using the decoder-contract metrics in Appendix~\ref{app:evaluation-metrics}.

\subsection{Full-sketch release}

Let $C=C(D)\in\mathbb R_+^m$ denote the non-private semantic sketch, and let $\widetilde C=C+Z$ be a noisy full-sketch release. Let $C_{(1)}\ge \cdots \ge C_{(m)}$ denote the order statistics of the entries of $C$. For $L<m$, define the top-$L$ margin $\Gamma_L=C_{(L)}-C_{(L+1)}$, and set $\Gamma_m=\infty$. All $\Top_L(\cdot)$ operations use a fixed public deterministic tie-breaking rule. Define
\begin{equation}
S_L^\star=\Top_L(C),
\qquad \widetilde S_L=\Top_L(\widetilde C).    
\end{equation}
For any set $S\subseteq[m]$, define its non-private semantic mass by
\begin{equation}
M(S;C)=\sum_{j\in S}C_j .
\end{equation}
The support-mass loss of the released support is
\begin{equation}
\mathsf{Loss}_L(C,\widetilde S_L) = M(S_L^\star;C)-M(\widetilde S_L;C).
\end{equation}

\subsection{Uniform coordinate error}

We next give high-probability bounds on the maximum coordinate perturbation of the noisy full sketch.

\begin{proposition}[Gaussian uniform coordinate error]
\label{prop:app-gaussian-uniform}
Let $Z\sim\mathcal N(0,\sigma^2 I_m)$. For any $\beta\in(0,1)$,
\begin{equation}
\Pr\!\left[ \|Z\|_\infty
\le \sigma\sqrt{2\log\frac{2m}{\beta}} \right]
\ge 1-\beta .
\end{equation}
\end{proposition}

\begin{proof}
For each coordinate $j\in[m]$, the Gaussian tail bound gives
\begin{equation}
\Pr\{|Z_j|>t\}
\le 2\exp\!\left(-\frac{t^2}{2\sigma^2}\right),
\qquad t\ge0 .
\end{equation}
By the union bound,
\begin{equation}
\Pr\{\|Z\|_\infty>t\}
\le \sum_{j=1}^m \Pr\{|Z_j|>t\}
\le 2m\exp\!\left(-\frac{t^2}{2\sigma^2}\right).
\end{equation}
Taking $t=\sigma\sqrt{2\log\frac{2m}{\beta}}$ gives $\Pr\{\|Z\|_\infty>t\}\le\beta$, and hence $\Pr\{\|Z\|_\infty\le t\}\ge 1-\beta$.
\end{proof}

\begin{proposition}[Laplace uniform coordinate error]
\label{prop:app-laplace-uniform}
Let $Z= (Z_1,\ldots,Z_m)$, where $Z_1,\ldots,Z_m$ are independent Laplace random variables with scale $b>0$. Then, for any $\beta\in(0,1)$, 
\begin{equation}
\Pr\!\left[ \|Z\|_\infty \le b\log\frac{m}{\beta} \right]
\ge 1-\beta .
\end{equation}
\end{proposition}

\begin{proof}
For a Laplace random variable with scale $b$, $\Pr\{|Z_j|>t\}=\exp(-t/b)$ for $t\ge0$. Therefore, by the union bound,
\begin{equation}
\Pr\{\|Z\|_\infty>t\}
\le \sum_{j=1}^m \Pr\{|Z_j|>t\} = m\exp(-t/b).
\end{equation}
Taking $t=b\log\frac{m}{\beta}$ gives $\Pr\{\|Z\|_\infty>t\} \le\beta$, and hence the claim.
\end{proof}

\subsection{Top-\texorpdfstring{$L$}{L} recovery and support-mass loss}

We next show that a uniform coordinate-error bound controls both exact top-$L$ recovery and support-mass loss.

\begin{proposition}[Plan recovery from uniform sketch error]
\label{prop:app-plan-recovery}
Suppose that $\|\widetilde C-C\|_\infty\le \tau$. Then $\mathsf{Loss}_L(C,\widetilde S_L)\le 2L\tau$. If, in addition, $\Gamma_L>2\tau$, then $\widetilde S_L=S_L^\star$.
\end{proposition}

\begin{proof}
For exact recovery, let $j\in S_L^\star$ and $k\notin S_L^\star$. By the definition of $\Gamma_L$, $C_j-C_k\ge \Gamma_L$. On the event $\|\widetilde C-C\|_\infty\le \tau$,
\begin{equation}
\widetilde C_j-\widetilde C_k = (C_j-C_k) + (\widetilde C_j-C_j) - (\widetilde C_k-C_k)
\ge \Gamma_L-2\tau .
\end{equation}
If $\Gamma_L>2\tau$, every coordinate in $S_L^\star$ has strictly larger noisy value than every coordinate outside $S_L^\star$. Hence $\widetilde S_L=S_L^\star$.
For the mass-loss bound, $\widetilde S_L$ is a top-$L$ set for $\widetilde C$,
\begin{equation}
\sum_{j\in S_L^\star}\widetilde C_j
\le \sum_{j\in \widetilde S_L}\widetilde C_j .
\end{equation}
The event $\|\widetilde C-C\|_\infty\le\tau$, implies $C_j\le\widetilde C_j+\tau$ and $\widetilde C_j\le C_j+\tau$ for all $j$. Therefore,
\begin{equation}
\sum_{j\in S_L^\star} C_j
\le \sum_{j\in S_L^\star} \widetilde C_j + L\tau
\le \sum_{j\in \widetilde S_L} \widetilde C_j + L\tau
\le \sum_{j\in \widetilde S_L} C_j + 2L\tau .
\end{equation}
Rearranging gives
\begin{equation}
M(S_L^\star;C)-M(\widetilde S_L;C)\le 2L\tau .
\end{equation}
\end{proof}

\begin{corollary}[High-probability plan recovery]
\label{cor:app-high-prob-plan-recovery}
For Gaussian noise $Z\sim\mathcal N(0,\sigma^2 I_m)$, with probability at least $1-\beta$,
\begin{equation}
\mathsf{Loss}_L(C,\widetilde S_L)
\le 2L\sigma\sqrt{2\log\frac{2m}{\beta}}.
\end{equation}
On the same event, exact top-$L$ recovery holds if
\begin{equation}
\Gamma_L> 2\sigma\sqrt{2\log\frac{2m}{\beta}} .
\end{equation}
For independent Laplace noise with scale $b$, with probability at least $1-\beta$,
\begin{equation}
\mathsf{Loss}_L(C,\widetilde S_L)
\le 2Lb\log\frac{m}{\beta},
\end{equation}
On the same event, exact top-$L$ recovery holds if
\begin{equation}
\Gamma_L>
2b\log\frac{m}{\beta}.
\end{equation}
\end{corollary}

\begin{proof}
Apply Proposition~\ref{prop:app-gaussian-uniform} or Proposition~\ref{prop:app-laplace-uniform} to obtain $\|\widetilde C-C\|_\infty\le\tau$ with probability at least $1-\beta$. Then apply Proposition~\ref{prop:app-plan-recovery}.
\end{proof}

\begin{corollary}[Retained non-private mass]
\label{cor:app-preserved-mass}
Assume $M(S_L^\star;C)>0$. If $\|\widetilde C-C\|_\infty\le\tau$, then
\begin{equation}
\frac{M(\widetilde S_L;C)}{M(S_L^\star;C)}
\ge 1-\frac{2L\tau}{M(S_L^\star;C)} .
\end{equation}
\end{corollary}

\begin{proof}
By Proposition~\ref{prop:app-plan-recovery}, $M(S_L^\star;C)-M(\widetilde S_L;C)\le 2L\tau$. Dividing by $M(S_L^\star;C)>0$ gives the claim.
\end{proof}

\subsection{Sparse private atom admission}

Sparse private atom admission is our alternative release mechanism to full-sketch perturbation. It privately selects $L$ coordinates by sequential applications of the exponential mechanism and then releases noisy values only for the admitted coordinates. Since the mechanism does not release a full noisy sketch, full-vector perturbation errors such as $\|\widetilde C-C\|_\infty$ are not defined for its released object. Instead, support recovery, preserved non-private mass, support-mass loss, and selected-coordinate value error are evaluated directly on the admitted set $\widetilde S_L$, as defined in Appendix~\ref{app:evaluation-metrics}. The sequential admission procedure scores the remaining public atom universe at each round, so we do not claim a computational-scaling advantage over full-sketch perturbation.

\begin{proposition}[Utility certificate for sequential private atom admission]
\label{prop:sequential-em-regret}
At round $t$, let $R_t=[m]\setminus\{J_1,\ldots,J_{t-1}\}$, $M_t=|R_t|=m-t+1$, and let $u_t^\star=\max_{j\in R_t} C_j(D)$. Suppose that round $t$ uses the exponential mechanism with coordinate sensitivity $\Delta_{\mathrm{coord}}$ and privacy parameter $\varepsilon_{\mathrm{sel},t}>0$. Then, for any $\beta\in(0,1)$, with probability at least $1-\beta$, simultaneously for all $t\in[L]$,
\begin{equation}
u_t^\star-C_{J_t}(D) \le \alpha_t
\coloneqq \frac{2\Delta_{\mathrm{coord}}}{\varepsilon_{\mathrm{sel},t}}
\left( \log M_t+\log\frac{L}{\beta} \right).
\end{equation}
On the same event,
\begin{equation}
\mathsf{Loss}_L = M(S_L^\star;C)-M(\widetilde S_L;C)
\le \sum_{t=1}^L \bigl(u_t^\star-C_{J_t}(D)\bigr)
\le \sum_{t=1}^L \alpha_t .
\end{equation}
\end{proposition}

\begin{proof}
Condition on any realized history through round $t-1$. The remaining candidate set $R_t$ is then fixed. The standard exponential-mechanism utility bound gives
\begin{equation}
\Pr\!\left[ u_t^\star-C_{J_t}(D)>\alpha_t \,\middle|\,
J_1,\ldots,J_{t-1} \right] \le \frac{\beta}{L}.
\end{equation}
A union bound over the $L$ adaptive rounds proves the simultaneous bound.

Let $C_{(1)}\ge\cdots\ge C_{(m)}$ denote the ordered non-private coordinate values. Before round $t$, at most $t-1$ coordinates have been removed, so at least one of the $t$ largest coordinates remains. Hence $u_t^\star\ge C_{(t)}$. Therefore,
\begin{align}
\mathsf{Loss}_L = \sum_{t=1}^L C_{(t)} - \sum_{t=1}^L C_{J_t}(D)
\le \sum_{t=1}^L \bigl(u_t^\star-C_{J_t}(D)\bigr),
\end{align}
which gives the result.
\end{proof}

This certificate bounds the utility loss of sequential private admission. Exact top-$L$ recovery requires additional separation conditions on the non-private coordinate masses. In our experiments, the selection budget is allocated uniformly across rounds, $\varepsilon_{\mathrm{sel},t}=\varepsilon_{\mathrm{sel}}/L$.

We next give high-probability bounds on selected-coordinate value errors. Because the conditional failure probability is at most $\beta$ for every realized admitted set, the same bounds hold unconditionally by the law of total probability.

\begin{proposition}[Selected-coordinate Gaussian value error]
\label{prop:app-selected-gaussian}
Let $\widetilde S\subseteq[m]$ be a random admitted set with $|\widetilde S|=L$. Suppose that, conditional on every realization $\widetilde S=S$, the released selected values satisfy $\widetilde c_j=C_j(D)+W_j$, $j\in S$, where $\{W_j:j\in S\}$ are independent $\mathcal N(0,\sigma_{\mathrm{val}}^2)$ random variables. Then, for any $\beta\in(0,1)$,
\begin{equation}
\Pr\!\left[ \max_{j\in\widetilde S} |\widetilde c_j-C_j(D)|
\le \sigma_{\mathrm{val}} \sqrt{2\log\frac{2L}{\beta}} \right] \ge 1-\beta .
\end{equation}
\end{proposition}

\begin{proof}
Condition on the event $\{\widetilde S=S\}$, where $S\subseteq[m]$ is an arbitrary realized admitted set with $|S|=L$. By the Gaussian tail bound and the union bound,
\begin{equation}
\Pr\!\left[ \max_{j\in S}|W_j|>t \,\middle|\, \widetilde S=S \right]
\le 2L\exp\!\left( -\frac{t^2}{2\sigma_{\mathrm{val}}^2} \right).
\end{equation}
Taking $t= \sigma_{\mathrm{val}} \sqrt{2\log(2L/\beta)}$ makes the conditional failure probability at most $\beta$. Since this bound holds for every realized $S$, the unconditional failure probability is also at most $\beta$.
\end{proof}

\begin{proposition}[Selected-coordinate Laplace value error]
\label{prop:app-selected-laplace}
Let $\widetilde S\subseteq[m]$ be a random admitted set with $|\widetilde S|=L$. Suppose that, conditional on every realization $\widetilde S=S$, the released selected values satisfy $\widetilde c_j=C_j(D)+W_j$, $j\in S$, where $\{W_j:j\in S\}$ are independent Laplace random variables with scale $b_{\mathrm{val}}>0$. Then, for any $\beta\in(0,1)$,
\begin{equation}
\Pr\!\left[ \max_{j\in\widetilde S} |\widetilde c_j-C_j(D)|
\le b_{\mathrm{val}}\log\frac{L}{\beta}
\right] \ge 1-\beta .
\end{equation}
\end{proposition}

\begin{proof}
Condition on the event $\{\widetilde S=S\}$. For each $j\in S$, $\Pr\{|W_j|>t\} = \exp(-t/b_{\mathrm{val}})$. The union bound over the $L$ admitted coordinates gives
\begin{equation}
\Pr\!\left[ \max_{j\in S}|W_j|>t \,\middle|\, \widetilde S=S \right]
\le L\exp(-t/b_{\mathrm{val}}).
\end{equation}
Taking $t=b_{\mathrm{val}}\log(L/\beta)$ makes the conditional failure probability at most $\beta$. Since this bound holds for every realized $S$, the same bound holds unconditionally.
\end{proof}

\clearpage
\appsection{Evaluation Metrics}
\label{app:evaluation-metrics}

In this appendix, we specify the empirical metrics used in the experiments. The plan-level utility guarantees are given in Appendix~\ref{app:utility-analysis}. The reported tables and figures use the following metrics to summarize run outputs. Metrics that depend on the non-private sketch $C(D)$ are used only for offline evaluation and are not part of the released differentially private mechanism.

\subsection{Notation recap}

Let $C(D)\in\mathbb R_+^m$ denote the non-private semantic sketch for the protected-unit setting under consideration, as defined in Section~\ref{sec:problem-formulation}, over the fixed semantic atom universe
$\mathcal A=\{a_1,\ldots,a_m\}$. Let $L$ denote the released plan size, and define $S_L^\star=\Top_L(C(D))$ as the non-private top-$L$ atom set under a fixed deterministic tie-breaking rule. Let $\widetilde S_L\subseteq[m]$ denote the atom set admitted by the differentially private mechanism. For the Gaussian and Laplace full-sketch mechanisms, $\widetilde S_L$ is obtained by applying $\Top_L$ to the noisy full sketch. For sparse private atom admission, $\widetilde S_L$ is the set of $L$ atoms admitted by the sequential exponential-mechanism procedure. For any set of atom indices $S\subseteq[m]$, define its non-private semantic
mass as $M(S;C)=\sum_{j\in S}C_j(D)$. Thus, $M(S;C)$ is the total non-private aggregate semantic mass assigned to the atoms indexed by $S$.

\subsection{Plan-level recovery metrics}

\paragraph{Top-$L$ atom overlap.}
The overlap metric measures the fraction of non-private top-$L$ atoms recovered by the released plan
\begin{equation}
\mathsf{Overlap}@L = \frac{|\widetilde S_L\cap S_L^\star|}{L}.
\end{equation}
This metric evaluates support recovery. It equals $1$ when the released plan admits exactly the same atom set as the non-private top-$L$ plan, and it equals $0$ when the two sets are disjoint.

\paragraph{Retained non-private mass fraction.}
The preserved non-private mass fraction measures the fraction of non-private top-$L$ semantic mass captured by the atoms admitted in the released plan
\begin{equation}
\mathsf{MassRetained} = \frac{M(\widetilde S_L;C)}{M(S_L^\star;C)}
= \frac{\sum_{j\in \widetilde S_L}C_j(D)}{\sum_{j\in S_L^\star}C_j(D)}.
\end{equation}
If $M(S_L^\star;C)=0$, this ratio is undefined; our implementation reports $0$ by convention and records the zero-denominator case. This metric is less strict than $\mathsf{Overlap}@L$, since a released plan can miss some non-private top-$L$ atoms while keeping most of the non-private mass when near-boundary atoms have similar masses.

\paragraph{Support-mass loss.}
The support-mass loss measures the non-private semantic mass lost by selecting $\widetilde S_L$ instead of the non-private top-$L$ set
\begin{equation}
\mathsf{Loss}_L = M(S_L^\star;C) - M(\widetilde S_L;C).
\end{equation}
Smaller values indicate lower loss. This metric is reported in the same units as the non-private semantic sketch. Unlike the preserved non-private mass fraction, it is not normalized, so its scale depends on the number of protected units and the contribution bounds used in the experiment.

\subsection{Noisy-value error metrics}

For the full-sketch Gaussian and Laplace releases, the mechanism forms a noisy vector
$\widetilde C=C(D)+Z\in\mathbb{R}^m$. We report the full-vector errors
$\|\widetilde C-C\|_\infty$, $\|\widetilde C-C\|_1$, and
$\|\widetilde C-C\|_2$. These metrics measure perturbation of the full semantic sketch before top-$L$ post-processing.
For sparse private atom admission, \texttt{DP-SPIN} does not release numerical values for all $m$ atoms. It releases noisy values only for atoms in $\widetilde S_L$. Thus, full-vector error is not defined for the released object in this mode. We instead report selected-coordinate value errors 
$\max_{j\in \widetilde S_L}|\widetilde c_j-C_j(D)|$,
$\sum_{j\in \widetilde S_L}|\widetilde c_j-C_j(D)|$,
$(\sum_{j\in \widetilde S_L}(\widetilde c_j-C_j(D))^2)^{1/2}$.
These metrics compare the released noisy masses with their non-private values only on the admitted atom set. When display clipping or support binning is applied after noise addition, these errors are computed from the raw noisy masses before display clipping; displayed masses and support bins are post-processed quantities.

\subsection{Decoder-contract metrics}
\label{app:decoder-contract-metrics}

Let $P(D)$ be a released semantic plan with admitted atom set $\widetilde S\subseteq[m]$. The admitted atoms visible to the decoder are
\begin{equation}
\widetilde{\mathcal A}(P) = \{(\ell(a_j),d(a_j)):j\in\widetilde S\}.
\end{equation}
We compute the following metrics using a fixed public parser and verifier, both defined by the public evaluation protocol. Neither component has access to the protected records.

\paragraph{Unsupported atom rate.}
For strict audit summaries, let $\widehat S(Y)\subseteq[m]$ denote the set of atom indices appearing in canonical markers $\texttt{[atom:}j\texttt{]}$ in $Y$. Canonical markers are the authoritative atom references in this mode. Literal atom-label matching is preserved only as a diagnostic for summaries that do not use the strict marker contract. The unsupported-atom rate is
\begin{equation}
\mathsf{UAR}(Y,P) = \frac{|\widehat S(Y)\setminus \widetilde S|}{\max\{1,|\widehat S(Y)|\}} .
\end{equation}
A value of zero means that every canonical atom reference in $Y$ refers to an atom admitted in the released plan.

\paragraph{Unsupported comparison rate.}
Let $\mathsf{Comp}(Y)$ be the set of comparison records detected in $Y$ by the fixed public parser. In our implementation, these records include simple pairwise semantic-support comparisons, such as ``higher support than'' and ``lower support than'', and semantic-support or rank superlatives, such as ``highest support'' or ``top-ranked theme''. Each pairwise record contains the detected relation and the atom indices detected on the left and right sides of the comparison.
A detected pairwise comparison is valid if each side identifies exactly one admitted atom and the claimed ordering is supported by the released noisy masses in $P(D)$, using the public comparison margin fixed before evaluation. A detected superlative claim is valid if the identified admitted atom has released rank one. Let $\mathsf{ValidComp}(Y,P)\subseteq\mathsf{Comp}(Y)$ denote the subset of detected comparison records that are valid under these rules. The unsupported-comparison rate is
\begin{equation}
\mathsf{UCR}(Y,P) = \frac{|\mathsf{Comp}(Y)\setminus \mathsf{ValidComp}(Y,P)|}{\max\{1,|\mathsf{Comp}(Y)|\}} .
\end{equation}
A value of zero means that every verifier-detected semantic-support or rank comparison in $Y$ is supported by the released plan. Because the denominator is $\max\{1,|\mathsf{Comp}(Y)|\}$, $\mathsf{UCR}=0$ also when no comparison is detected. We therefore interpret $\mathsf{UCR}$ jointly with the number of detected comparison records $|\mathsf{Comp}(Y)|$.

\paragraph{Verifier acceptance rate.}
Let $\mathsf{Accept}(Y,P)\in\{0,1\}$ denote the decision of the fixed public verifier applied to summary $Y$ and released plan $P$. The verifier applies the public decoder-contract checks specified before evaluation. These checks include unsupported atom mentions, unsupported semantic-support or rank comparisons, unsupported support-bin assignments, unsupported numeric claims, raw-record or structured-identifier patterns, and, when strict marker mode is enabled, atom-label mentions without canonical markers. If explicit importance-ordering claims are rejected by the configured verifier, these claims are also treated as violations because the released plan contains no separate importance score.
Under the default acceptance rule, the verifier accepts if and only if no violation is detected. The quantities $\mathsf{UAR}(Y,P)$ and $\mathsf{UCR}(Y,P)$ are computed and preserved as evaluation measures; in the thresholded configuration, acceptance additionally requires $\mathsf{UAR}(Y,P)\le \alpha_{\mathrm{atom}}$ and $\mathsf{UCR}(Y,P)\le \alpha_{\mathrm{comp}}$, with public thresholds fixed before evaluation. Pairwise comparison checks use the public comparison margin fixed before evaluation. For a collection of runs $\mathcal R$, with generated summaries $Y_r$ and released plans $P_r$, the verifier acceptance rate is
\begin{equation}
\mathsf{AcceptRate} =\frac{1}{|\mathcal R|} \sum_{r\in\mathcal R} \mathsf{Accept}(Y_r,P_r).
\end{equation}
The template decoder constructs audit text from admitted plan entries and canonical atom markers and does not introduce pairwise comparison claims. Accordingly, its verifier outputs are preserved as consistency checks but are not reported as parameter-sweep utility results. For free-form audit decoding, the same verifier measures consistency with the released plan under the fixed public parser. The verifier does not assess correctness or completeness relative to the protected corpus.

\subsection{Non-private reference text evaluation}

The primary utility object in \texttt{DP-SPIN} is the released semantic plan. Text-level metrics are used only to evaluate decoded summaries. Let $Y_{\mathrm{DP}}$ denote the summary decoded from the released differentially private plan, and let $Y_{\mathrm{NP}}$ denote the reference summary decoded from the non-private top-$L$ semantic plan
\begin{equation}
P_{\mathrm{NP}} = \left( S_L^\star,\, C_{S_L^\star}(D),\, \{(\ell(a_j),d(a_j)):j\in S_L^\star\} \right).
\end{equation}
The plan $P_{\mathrm{NP}}$ and the summary $Y_{\mathrm{NP}}$ are used only for offline evaluation. They are not given to the decoder that produces $Y_{\mathrm{DP}}$ and are not part of the released differentially private output.

The role of these metrics depends on how the summary is produced. With our template decoder, the plan-to-text map is deterministic. Once the released plan is fixed, the generated text is fixed. The resulting scores mainly reflect overlap between the atom labels, support bins, and fixed template phrases obtained from the differentially private plan and from the non-private top-$L$ plan. In this case, text-level metrics are useful as checks that differences at the plan level are reflected in the generated text, but they should not be read as a separate measure of language-generation quality.

With the OpenAI decoder, the plan-to-text map is treated as stochastic post-processing of the released plan. The generated summary may paraphrase, compress, reorder, or omit admitted themes while following the public decoder rules. In this case, text-level metrics are more informative as secondary verbalization evaluation measures. They measure whether a free-form summary generated from the differentially private plan remains close to the summary generated from the non-private reference plan.

We therefore report text-level metrics with the decoder backend fixed. For the deterministic template decoder, the scores mainly reflect overlap between two fixed plan verbalizations. For the OpenAI decoder, the scores also capture variation from free-form generation, including paraphrasing, compression, and reordering. Our primary utility measure remains recovery of the non-private semantic plan.

\paragraph{Token Jaccard similarity.}
Let $T(Y)$ be the set of normalized non-stopword tokens extracted from summary $Y$. The token Jaccard similarity is
\begin{equation}
J_{\mathrm{tok}} = \frac{|T(Y_{\mathrm{DP}})\cap T(Y_{\mathrm{NP}})|}{|T(Y_{\mathrm{DP}})\cup T(Y_{\mathrm{NP}})|}.
\end{equation}
This metric measures lexical overlap at the token level.

\paragraph{Bigram Jaccard similarity.}
Let $G_2(Y)$ denote the set of normalized non-stopword token bigrams in summary $Y$. We define
\begin{equation}
J_2 = \frac{|G_2(Y_{\mathrm{DP}})\cap G_2(Y_{\mathrm{NP}})|}{|G_2(Y_{\mathrm{DP}})\cup G_2(Y_{\mathrm{NP}})|}.
\end{equation}
This metric measures bigram overlap between the two summaries.

\paragraph{Keyphrase Jaccard similarity.}
Our implementation extracts deterministic lexical keyphrases by forming content-token $n$-grams for $n\le 3$ after stopword removal. Let $\mathsf{KP}(Y)$ denote the resulting keyphrase set. The keyphrase Jaccard similarity is
\begin{equation}
J_{\mathrm{kp}} = \frac{|\mathsf{KP}(Y_{\mathrm{DP}})\cap \mathsf{KP}(Y_{\mathrm{NP}})|}{|\mathsf{KP}(Y_{\mathrm{DP}})\cup \mathsf{KP}(Y_{\mathrm{NP}})|}.
\end{equation}
This is a deterministic lexical-content similarity measure with fixed extraction rules. For all Jaccard metrics, if both sets in the denominator are empty, the implementation reports similarity $1$ by convention and records the empty-denominator case.

\paragraph{TF--IDF cosine similarity.}
Let $q(Y)$ be the term frequency--inverse document frequency (TF--IDF) vector of summary $Y$, constructed from unigram and bigram features over the pair $(Y_{\mathrm{DP}}, Y_{\mathrm{NP}})$. The TF--IDF cosine similarity is
\begin{equation}
\mathsf{Cos}_{\mathrm{tfidf}} = \frac{\langle q(Y_{\mathrm{DP}}),q(Y_{\mathrm{NP}})\rangle}{\|q(Y_{\mathrm{DP}})\|_2\|q(Y_{\mathrm{NP}})\|_2}.
\end{equation}
This metric measures lexical similarity with term weighting rather than set overlap. If no TF--IDF feature is produced, the implementation reports cosine similarity $0$.

\subsection{Embedding-space semantic similarity}

When a fixed sentence-embedding backend is enabled, let $e:\mathcal Y\to\mathbb{R}^d$ denote the corresponding embedding map. We define
\begin{equation}
\mathsf{Cos}_{\mathrm{emb}} = \frac{\langle e(Y_{\mathrm{DP}}),e(Y_{\mathrm{NP}})\rangle}{\|e(Y_{\mathrm{DP}})\|_2\|e(Y_{\mathrm{NP}})\|_2}.
\end{equation}
This is a secondary verbalization measure and does not replace the plan-level utility measures. We do not separately report $\mathsf{Cos}_{\mathrm{emb}}$ for runs configured with the TF--IDF backend, because in that case it coincides with the TF--IDF cosine similarity already reported above.

\subsection{Comparative summary evaluation}

For optional comparative evaluation, a fixed evaluator receives $P_{\mathrm{NP}}$, $Y_{\mathrm{DP}}$, and $Y_{\mathrm{NP}}$. It compares the two summaries as verbalizations of the same non-private reference plan $P_{\mathrm{NP}}$, without receiving raw protected records. This comparison is intended for free-form decoded summaries, especially OpenAI outputs. For deterministic template outputs, the comparison is largely redundant with plan-level recovery because both summaries follow the same fixed verbalization rules.

The evaluator chooses $\mathsf{choice}\in\{\mathrm{DP},\mathrm{NP},\mathrm{tie}\}$ according to coverage of atoms in $P_{\mathrm{NP}}$, faithfulness to $P_{\mathrm{NP}}$, clarity, and concision. For aggregate reporting, we encode the preference as
\begin{equation}
\mathsf{DPPreferred} =
\begin{cases}
1, & \mathsf{choice}=\mathrm{DP},\\
1/2, & \mathsf{choice}=\mathrm{tie},\\
0, & \mathsf{choice}=\mathrm{NP}.
\end{cases}
\end{equation}
The reported comparative preference is the mean of this value across runs.


\subsection{Controlled probe-atom admission and mention measures}

The probe-atom experiments measure how a synthetic public atom with controlled support appears in the released plan and in the decoded summary. For each run, we insert synthetic records expressing a fixed probe pattern into the protected target corpus. The corresponding probe atom is included in the public atom dictionary before the differentially private release. The experiment therefore measures admission of a known public coordinate. It does not evaluate discovery of a new label from protected records.

\paragraph{Plan probe-atom admission.}
Let $\widetilde S_L$ be the admitted atom set and let $\mathcal J_{\mathrm{probe}}\subseteq[m]$ be the set of public probe atom indices. For a single probe atom $j_c$, the plan admission indicator is
\begin{equation}
\mathsf{Adm}_{\mathrm{plan}}(j_c) = \mathbf 1\{j_c\in \widetilde S_L\}.
\end{equation}
For multiple synthetic probe atoms, we report
\begin{equation}
\mathsf{Adm}_{\mathrm{plan}}(\mathcal J_{\mathrm{probe}})
= \frac{1}{|\mathcal J_{\mathrm{probe}}|} \sum_{j\in\mathcal J_{\mathrm{probe}}}
\mathbf 1\{j\in\widetilde S_L\}.
\end{equation}
This metric is computed before decoding. It measures whether the differentially private release admits the public probe atom into the semantic plan.

\paragraph{Summary probe-string mention rate.}
Let $\mathcal C$ be the set of public probe strings and let $Y_{\mathrm{DP}}$ be the summary generated from the differentially private plan. The summary probe atom mention rate is
\begin{equation}
\mathsf{Mention}_{\mathrm{text}} = \frac{|\{c\in\mathcal C: c \text{ is detected in } Y_{\mathrm{DP}}\}|}{|\mathcal C|}.
\end{equation}
Detection uses a fixed public string-matching rule. This metric records whether the public probe string appears in the generated text. It is computed after decoding and does not affect the privacy mechanism. For the template decoder, the metric is largely determined by plan admission. For the OpenAI decoder, it also reflects whether an admitted probe-atom theme is verbalized in the generated summary.

\paragraph{Non-private probe-atom rank.}
The non-private probe-atom rank is computed offline from the non-private sketch $C(D)$ and is used only for evaluation. For a synthetic probe atom $j_c$, define
\begin{equation}
\mathsf{Rank}_{\mathrm{NP}}(j_c)
= 1+ |\{j\in[m]: C_j(D)>C_{j_c}(D)\}| + |\{j\in[m]: C_j(D)=C_{j_c}(D),\ j\prec j_c\}|,
\end{equation}
where $\prec$ is the fixed public tie-breaking order. This rank indicates whether the controlled probe atom is close to the non-private top-$L$ boundary before the differentially private release is applied. It is used only for offline interpretation and is not part of the released mechanism.

\clearpage
\appsection{Experimental Configuration}
\label{app:experimental-configuration}

\vspace{-3pt}

In this appendix, we specify our experimental configuration, dataset-specific run grids, and aggregation rules for repeated runs. We identify the protected unit, the construction of public atom dictionaries, the released-plan parameters, and the grouping keys used in the reported tables and figures. Where relevant, we use the same variable names as in the implementation to simplify reproduction and extension of the experiments.
  
\vspace{-3pt}

\subsection{Protected unit}

Each run is configured with a protected unit $\texttt{protected\_unit}\in\{\texttt{record},\texttt{user}\}$. In record-level experiments, one text record is one protected unit. The corresponding sample-size parameter is $n_{\mathrm{records}}$. Each run selects $n_{\mathrm{records}}$ protected records from the target split and uses record-level add/drop adjacency.
In user-level experiments, one user is one protected unit. The corresponding sample-size parameter is $n_{\mathrm{users}}$. Each run selects $n_{\mathrm{users}}$ protected users from the target split, groups all target records with the same user identifier, and uses user-level add/drop adjacency. Our configuration field \texttt{user\_id\_field} specifies the identifier used for grouping.

\begin{protocolbox}{Record-level and user-level size fields}
\begin{itemize}[leftmargin=14pt]
    \item \textbf{Record-level:} \texttt{protected\_unit=record}, configured by \texttt{n\_records}.
    \item \textbf{User-level:} \texttt{protected\_unit=user}, configured by \texttt{n\_users} and \texttt{user\_id\_field}.
\end{itemize}
\end{protocolbox}

CFPB is evaluated only in the record-level setting because the prepared complaint records do not contain a stable user identifier for grouping multiple records from the same individual. Amazon All Beauty reviews and Yelp Restaurant reviews are evaluated in both record-level and user-level settings because their review records contain stable reviewer identifiers.

\vspace{-3pt}
 
\subsection{Auxiliary and protected target splits}

Each dataset is partitioned into an auxiliary split and a protected target split. The auxiliary split is used to construct decoder-visible atom labels and to fit the text representation used by the assignment map. The protected target split is used only through bounded record-to-atom assignments, aggregation, and the differentially private release mechanism.
For user-level experiments, the auxiliary and target splits are user-disjoint, meaning that all records from the same user are assigned entirely to one split.  This ensures that atom construction and representation fitting do not use records from users whose records are included in the protected target split.

\begin{protocolbox}{Privacy boundary}
\textbf{Auxiliary split.} used for atom construction and representation fitting.

\vspace{1mm}
\textbf{Protected target split.} Used for bounded assignments, aggregation, and DP release.

\vspace{1mm}
\textbf{Not provided to the decoder.} Raw target records, target excerpts, nearest neighbors, per-record contribution vectors, non-private sketch values, atoms outside the released plan, or labels derived non-privately from the protected target split.
\end{protocolbox}

\vspace{-3pt}

\subsection{Public denominator}

The field \texttt{public\_denominator} specifies the denominator used to convert decoder-facing noisy masses into normalized semantic-support values. In record-level runs, $d_{\mathrm{public}}=n_{\mathrm{records}}$, and in user-level runs, $d_{\mathrm{public}}=n_{\mathrm{users}}$, where  $n_{\mathrm{records}}$ and $n_{\mathrm{users}}$ are configured analysis sizes fixed before the protected target data are processed. The denominator is not recomputed from the realized cardinality of a neighboring dataset. For an admitted atom $j$, define
\[
\widetilde c_j^{+}=\max\{0,\widetilde c_j\}, \qquad
\widetilde q_j = \frac{\widetilde c_j^{+}}{d_{\mathrm{public}}}.
\]
The value $\widetilde q_j$ is normalized semantic support and need not equal the fraction of protected units expressing atom $a_j$. Nonnegative clipping, normalization, and support-bin assignment are deterministic post-processing and consume no additional privacy budget.


\vspace{-3pt}

\subsection{Atom and plan parameters}

The parameters $m$, $K$, and $L$ have the following meanings across all experiments:
\begin{align}
m &= \text{number of public atom coordinates}, \nonumber \\
K &= \text{maximum number of nonzero atom assignments per record},\nonumber\\
L &= \text{number of admitted atoms in the released plan}.\nonumber
\end{align}
The atom dictionary is fixed before the protected target split is processed. Each target record is mapped to at most $K$ atoms with total $\ell_1$ mass at most one. The released semantic plan contains $L$ admitted atoms in the configurations considered.

For Yelp, the main dictionary sizes are $m\in\{50,100\}$, with $m=200$ used only as a finer-dictionary ablation. For Amazon All Beauty, we use a more restrictive auxiliary-only construction. We extract two-word review-theme phrases, apply the Amazon-domain phrase filter, canonicalize deterministic variants, merge duplicate labels that denote the same review theme, and keep only candidate phrases with auxiliary document support at least $25$. The value $25$ is the minimum number of auxiliary reviews in which a candidate phrase must appear before it can be preserved. This threshold is configurable; increasing it gives a smaller atom dictionary, while decreasing it admits lower-support phrases. With the threshold set to $25$, the Amazon construction yields a fixed public atom dictionary with $m=46$, used in all Amazon runs. This choice avoids using rare or redundant product-review phrases as atom coordinates in the smaller Amazon All Beauty domain.

\subsection{Privacy mechanism and budgets}

Each run specifies a release mechanism, privacy parameters, and an adjacency relation. The mechanism field takes one of the values $\{\texttt{laplace}, \texttt{gaussian}, \texttt{private\_topL}\}$. The remaining privacy fields specify $\varepsilon$, $\delta$, and the adjacency relation. All reported experiments use add/drop adjacency. In record-level runs, the protected unit is one record. In user-level runs, the protected unit is one user after deterministic $\ell_1$ clipping of that user's aggregate contribution.
We use the following three plan-release mechanisms.
\begin{itemize}[leftmargin=16pt]
    \item 
    \texttt{gaussian}.
    The mechanism forms a differentially private noisy full sketch using Gaussian noise calibrated to the global $\ell_2$ sensitivity. The admitted atom set is selected from this noisy sketch by post-processing, $\widetilde S=\Top_L(\widetilde C)$. The released semantic plan contains the admitted atoms and their corresponding noisy values. In plots and tables, this mechanism is labeled \emph{Gaussian sketch + top-$L$}.
    \item 
    \texttt{laplace}.
    The mechanism forms a differentially private noisy full sketch using Laplace noise calibrated to the global $\ell_1$ sensitivity under the configured protected unit and adjacency relation. The admitted atom set is then selected by post-processing, $\widetilde S=\Top_L(\widetilde C)$. This mechanism is $(\varepsilon,0)$-differentially private, so its effective $\delta$ is zero. In plots and tables, this mechanism is labeled \emph{Laplace sketch + top-$L$}.
    \item 
    \texttt{private\_topL}.
    The mechanism admits $L$ atom indices using sequential exponential-mechanism selection. It then releases noisy masses only for the admitted atoms. The selected-value release uses Gaussian noise calibrated to the $\ell_2$ sensitivity of the restricted map $D\mapsto C_{\widetilde S}(D)$. In plots and tables, this mechanism is labeled \emph{Private top-$L$ admission + Gaussian values}.
\end{itemize}

For the sparse \texttt{private\_topL} mechanism, we split the $\varepsilon$ budget into an atom-admission budget and a selected-value budget:
\begin{equation}
\varepsilon_{\mathrm{sel}}=(1-\rho)\varepsilon,
\qquad \varepsilon_{\mathrm{val}}=\rho\varepsilon,    
\end{equation}
where $\rho\in(0,1)$ is fixed before the run. In our implementation, selection budget is allocated uniformly across the $L$ sequential exponential-mechanism admission rounds
\begin{equation}
\varepsilon_{\mathrm{sel},t} =\varepsilon_{\mathrm{sel}}/L.
\end{equation}
The admitted set has fixed cardinality $L$, but its elements are random. After the admission step, the mechanism releases noisy masses only for the admitted atoms, using a selected-value mechanism with privacy parameters $(\varepsilon_{\mathrm{val}},\delta)$.

For \texttt{Amazon All Beauty} and \texttt{Yelp Restaurants}, the main user-level clipping bounds are $B\in\{1,2\}$, with $B=5$ used only as an ablation. Smaller $B$ lowers the user-level sensitivity and the corresponding noise scale, but clips more per-user semantic mass. Larger $B$ preserves more per-user semantic mass, but requires more noise for the same privacy budget.

\subsection{Selection and random seeds}

Protected units are selected from the target split by a deterministic selection rule controlled by the run seed. The default implementation uses hash-based selection. Changing the run seed changes the selected protected units. When mechanism randomness is also derived from the run seed, changing the seed also changes the differentially private noise realization, while preserving reproducibility.
The field \texttt{strict\_sample\_size} prevents silent changes to the experimental sample size. If the requested number of records or users is unavailable and strict sampling is enabled, the run terminates rather than changing the requested experimental condition.

\vspace{-3pt}

\subsection{Generated configurations and run isolation}

Run scripts construct final per-run YAML configurations from dataset-specific defaults. A final per-run configuration can therefore contain fields not present in the template file, including the atom file used, the effective record count, the effective user count, the selected protected-unit count, the output directory, and the user-level clipping bound.
 
To avoid aggregating runs from incompatible experimental grids, our scripts support a clean-run policy. When \texttt{CLEAN\_RUNS=1}, the target run directory for the selected dataset and protected-unit type is removed before the new grid is executed. This ensures that grouped summaries are computed only from runs produced under the same configuration grid.

\begin{protocolbox}{Configuration construction}
Template configuration files define dataset-specific defaults. Run scripts construct final per-run configurations by setting dataset paths, $n_{\mathrm{records}}$ or $n_{\mathrm{users}}$, $m$, $K$, $L$, $\varepsilon$, $\delta$, the mechanism, the seed, the output directory, and, for user-level runs, the clipping bound $B$.
\end{protocolbox}

The deterministic parameter-sweep experiments use \texttt{decoder\_backend=template}, \texttt{decoder\_contract\_mode=legacy\_text}, \texttt{clean\_reference\_decoder\_backend=template}, and \texttt{summary\_embedding\_backend=none}. The optional structured-claim decoder mode in the released code is not used for these reported sweeps.

\subsection{Run outputs}

\vspace{-3pt}

Each run writes a run directory containing the final per-run configuration and
the artifacts used for analysis:
\begin{itemize}[leftmargin=14pt]
    \item \texttt{config.yaml}: final per-run configuration;
    \item \texttt{private\_plan.json}: released differentially private semantic plan used by the decoder;\vspace{-3pt}
    \item \texttt{summary.txt}: generated aggregate summary or public failure marker;\vspace{-3pt}
    \item \texttt{verifier\_attempts.json}: verifier decisions, feedback, and recorded checks;\vspace{-3pt}
    \item \texttt{metrics.json}: plan-utility metrics, released-value errors, verifier measures, and controlled probe-atom measures.
\end{itemize}
Dataset-level statistics are written under \texttt{outputs/<dataset>/stats/}.
Grouped run summaries are written under \texttt{outputs/<dataset>/summaries/}.

\vspace{-3pt}

\subsection{Grouped run summaries}

The summarization script aggregates repeated runs into grouped tables. Record-level summaries are grouped by\vspace{-3pt}
\[
(\texttt{protected\_unit}, n_{\mathrm{records}}, m, K, L,
\texttt{mechanism}, \varepsilon_{\mathrm{eff}}, \delta_{\mathrm{eff}}).
\]
User-level summaries are grouped by\vspace{-3pt}
\[
(\texttt{protected\_unit}, n_{\mathrm{users}}, B, m, K, L,
\texttt{mechanism}, \varepsilon_{\mathrm{eff}}, \delta_{\mathrm{eff}}).
\]
In our implementation, $\varepsilon_{\mathrm{eff}}$ and $\delta_{\mathrm{eff}}$ denote the mechanism-specific privacy parameters used for the reported release. For Laplace releases, $\delta_{\mathrm{eff}}=0$. For user-level runs, the effective number of selected records can vary across seeds because selected users may contribute different numbers of reviews. Therefore, user-level summaries are not grouped by effective record count. Instead, each group reports effective record-count statistics, including the mean, minimum, maximum, and standard deviation. The implementation additionally includes the adjacency relation, Gaussian calibration, selected-value budget fraction and mechanism, decoder backend, decoder-contract mode, non-private-reference decoder backend, and configured summary-embedding backend in the grouping key; controlled-probe runs also include the probe-support value. Runs that differ in any of these fields are not pooled.

\begin{protocolbox}{Summary grouping rule}
Runs with the same protected unit, sample size, atom dictionary size, assignment sparsity, plan size, mechanism, effective privacy parameters, and clipping bound are grouped across seeds. Effective record counts are reported as group-level statistics and are not used as user-level grouping keys.
\end{protocolbox}

\clearpage

\appsection{CFPB Atom Construction and Assignment Protocol}
\label{app:cfpb-atom-construction}

In this appendix, we specify how the public atom dictionary is constructed for the CFPB experiments and how protected target records are mapped to bounded semantic vectors. This section distinguishes the public or auxiliary objects from the protected target objects and identifies where differential privacy is applied.

\subsection{Raw CFPB record structure}

Each CFPB record contains a free-text complaint narrative and structured metadata fields. In our CFPB experiments, each complaint narrative is treated as one protected record, and its text is the field used in the target assignment step. The fields \texttt{Product}, \texttt{Issue}, and \texttt{Sub-issue} are structured CFPB metadata categories.

\begin{examplebox}{Illustrative CFPB record}
\footnotesize

\textbf{Consumer complaint narrative:}

\emph{I found an account on my credit report that does not belong to me. I disputed the information, but the credit bureau did not remove it.}

\vspace{1mm}
\textbf{Product:}

Credit reporting, credit repair services, or other personal consumer reports

\vspace{1mm}
\textbf{Issue:}

Incorrect information on your report

\vspace{1mm}
\textbf{Sub-issue:}

Information belongs to someone else
\end{examplebox}

The atom labels used in the CFPB experiments are not generated by an LLM and are not derived from protected target narratives. They are constructed from \texttt{Product}, \texttt{Issue}, and \texttt{Sub-issue} values observed in the auxiliary split.

\subsection{Auxiliary and protected target splits}

We partition the CFPB records into two disjoint sets,
\[
D_{\mathrm{aux}} \cap D_{\mathrm{target}} = \varnothing .
\]
The auxiliary split $D_{\mathrm{aux}}$ is used to construct decoder-visible atom labels and to fit the text representation used by the assignment map. The target split $D_{\mathrm{target}}$ is the protected dataset. The CFPB privacy
guarantees are defined with respect to record-level add/drop adjacency on $D_{\mathrm{target}}$, with $D_{\mathrm{aux}}$ fixed.

\begin{protocolbox}{Experimental privacy boundary}
\footnotesize
\textbf{Auxiliary split $D_{\mathrm{aux}}$.}
Used for atom-dictionary construction and representation fitting.

\vspace{1mm}
\textbf{Protected target split $D_{\mathrm{target}}$.}
Used to compute bounded record-to-atom assignments, aggregate them into the semantic sketch, and release the differentially private plan.

\vspace{1mm}
\textbf{Not provided to the decoder.}
Raw target narratives, target excerpts, nearest neighbors, per-record contribution vectors, non-private sketch values, atoms outside the released plan, non-private representatives, or labels derived non-privately from the protected target split.
\end{protocolbox}

Thus, the decoded summary depends on the protected target split only through the released differentially private semantic plan.

\subsection{Public atom labels}

For CFPB, each atom label is formed by concatenating structured metadata values. In all CFPB runs, the atom schema is
\[
\texttt{Product / Issue / Sub-issue}.
\]
Changing $m$ changes the number of preserved public atoms, not the atom schema.

\begin{examplebox}{Example atom label construction}
\footnotesize
From the metadata values

\vspace{1mm}
\textbf{Product:}

Credit reporting, credit repair services, or other personal consumer reports

\textbf{Issue:}

Incorrect information on your report

\textbf{Sub-issue:}

Information belongs to someone else

\vspace{2mm}
we construct the atom label

\vspace{1mm}
\texttt{Credit reporting, credit repair services, or other personal consumer reports /}\\
\texttt{Incorrect information on your report /}\\
\texttt{Information belongs to someone else}
\end{examplebox}

For a requested dictionary size $m$, we compute the auxiliary support of each distinct \texttt{Product / Issue / Sub-issue} combination in $D_{\mathrm{aux}}$. We sort the combinations by auxiliary support, with deterministic tie-breaking, and keep the top $m$ combinations. The resulting public atom universe is $\mathcal A_m=\{a_1,\ldots,a_m\}$. If fewer than $m$ metadata combinations remain after preprocessing, we use all remaining combinations and report the realized dictionary size in the run metadata. In the main CFPB experiments, we use $m\in\{50,100,200\}$.

\begin{examplebox}{Example public atom dictionary}
\footnotesize
Suppose the first rows of the auxiliary-support ordering are

\vspace{1mm}
\begin{tabular}{p{0.15\linewidth}p{0.78\linewidth}}
\toprule
\textbf{Auxiliary support} & \textbf{Metadata combination} \\
\midrule
1520 &
\texttt{Credit reporting / Incorrect information / Information belongs to someone else} \\

940 &
\texttt{Debt collection / Attempts to collect debt not owed / Debt is not yours} \\

710 &
\texttt{Credit reporting / Problem with investigation / Investigation did not fix an error} \\

455 &
\texttt{Checking or savings account / Managing an account / Deposits and withdrawals} \\
\bottomrule
\end{tabular}

\vspace{2mm}

If $m=3$, the public atom dictionary contains the first three combinations in this ordering. If $m=200$, it contains the first 200 combinations, provided that at least 200 combinations remain after preprocessing.
\end{examplebox}

\vspace{-3pt}

\subsection{Record-to-atom assignment}

After the atom dictionary is fixed, each protected target record $x_i\in D_{\mathrm{target}}$ is mapped to a sparse nonnegative vector $v_i=\phi(x_i)\in\mathbb R_+^m$. The assignment map uses the protected target narrative, the fixed public atom dictionary, and a text representation fitted only on auxiliary or public data. In our implementation, we compute similarity scores between the target narrative and the public atom labels or descriptions, preserve the $K$ highest-scoring atoms under deterministic tie-breaking, map the preserved scores to nonnegative weights, and normalize the weights to have total mass at most one. Formally, for each target record $x_i$, let $T_i\subseteq[m]$ be the selected atom indices. Then\vspace{-5pt}
\[
|T_i|\le K,\qquad
v_{ij}=0 \;\, \forall j\notin T_i,\qquad
\sum_{j=1}^m v_{ij}\le 1 .
\]
Equivalently, $\|v_i\|_0\le K$ and $\|v_i\|_1\le 1$.

\begin{center}
\begin{minipage}{\textwidth}
\begin{examplebox}{Example bounded assignment}
\footnotesize

Assume $m=200$ and $K=3$. Consider the protected target narrative

\vspace{1mm}
\emph{I disputed an account on my credit report because it was opened by someone else. The company said it verified the account, but I never received documentation.}

\vspace{2mm}
The assignment map may produce
\[
\begin{array}{lll}
a_{17} &: 0.58 &
\texttt{Credit reporting / Incorrect information / Information belongs to someone else} \\
a_{42} &: 0.27 &
\texttt{Credit reporting / Problem with investigation / Investigation did not fix an error} \\
a_{81} &: 0.15 &
\texttt{Debt collection / Attempts to collect debt not owed / Debt is not yours}.
\end{array}
\]
All other coordinates are zero. The record contributes to at most $K=3$ atoms and has total contribution mass $0.58+0.27+0.15=1.00$.

\end{examplebox}
\end{minipage}
\end{center}

The bounded-contribution condition implies the record-level sensitivity used in the privacy analysis. Under add/drop adjacency, adding or removing one protected record changes the aggregate sketch by a vector with $\ell_1$ norm at most one.

\subsection{Aggregate semantic sketch}

The non-private aggregate sketch is the sum of the target-record assignment vectors
\begin{equation}
C(D_{\mathrm{target}}) = \sum_{x_i\in D_{\mathrm{target}}} v_i \in \mathbb R_+^m .
\end{equation}
Coordinate $C_j(D_{\mathrm{target}})$ is the aggregate semantic mass assigned to atom $a_j$. It need not be an integer count, because a record may split its unit mass across several atoms.

\begin{examplebox}{Example non-private aggregate sketch}
\footnotesize
After aggregation over protected target records, the internal non-private sketch can contain values such as
\begin{equation}
\begin{array}{lll}
C_{17}(D_{\mathrm{target}}) &=& 482.3,\\
C_{42}(D_{\mathrm{target}}) &=& 391.8,\\
C_{81}(D_{\mathrm{target}}) &=& 210.5.
\end{array}
\end{equation}
These values are internal to the trusted curator and are not released without a differential-privacy guarantee.

\end{examplebox}

\subsection{Differentially private semantic-plan release}

\texttt{DP-SPIN} releases a differentially private semantic plan rather than the non-private sketch. In the full-sketch release mode, the curator samples $\widetilde C = C(D_{\mathrm{target}})+Z$, where $Z$ is Laplace or Gaussian noise calibrated to the sensitivity of $C$ under the chosen norm. Under record-level add/drop adjacency and the bound $\|v_i\|_1\le 1$, the sketch satisfies $\Delta_1(C)\le 1$, $\Delta_2(C)\le 1$.
The admitted atom set is then selected by post-processing, $\widetilde S=\mathsf{Top}_L(\widetilde C)$. The released semantic plan is
\begin{equation}
P(D_{\mathrm{target}}) = \left( \widetilde S, \widetilde C_{\widetilde S}, \{(\ell(a_j),d(a_j)):j\in \widetilde S\} \right).   
\end{equation}
Normalized semantic-support values and support bins, when reported, are deterministic post-processing of the released noisy values and the public denominator.

\begin{center}
\begin{minipage}{1.05\textwidth}
\begin{examplebox}{Example released differentially private plan}
\scriptsize
For $L=3$, the released plan may be

{\scriptsize
\[
\begin{array}{lll}
\textbf{Atom label} & \textbf{Noisy mass} & \textbf{Support bin} \\
\midrule
\texttt{Credit reporting / Incorrect information / Information belongs to someone else}
& 479.6 & \texttt{high} \\

\texttt{Credit reporting / Problem with investigation / Investigation did not fix an error}
& 386.2 & \texttt{moderate} \\

\texttt{Debt collection / Attempts to collect debt not owed / Debt is not yours}
& 214.7 & \texttt{low}
\end{array}
\]
}
\end{examplebox}
\end{minipage}
\end{center}

\subsection{CFPB protocol summary}

\begin{protocolbox}{CFPB record-level protocol}
\footnotesize
\begin{enumerate}[leftmargin=16pt]
    \item Split CFPB records into $D_{\mathrm{aux}}$ and $D_{\mathrm{target}}$.
    \item Construct public atom labels from \texttt{Product / Issue / Sub-issue} combinations in $D_{\mathrm{aux}}$.
    \item Keep the $m$ combinations with highest auxiliary support as the public atom universe.
    \item Fit the assignment representation using auxiliary or public data.
    \item Map each protected target narrative to at most $K$ atoms with total mass at most one.
    \item Aggregate the bounded assignment vectors into $C(D_{\mathrm{target}})$.
    \item Release a differentially private semantic plan containing $L$ admitted atoms and noisy masses.
    \item Provide the decoder only the released plan and public decoding rules.
\end{enumerate}
\end{protocolbox}

\clearpage

\appsection{Amazon Reviews Atom Construction and Assignment Protocol}
\label{app:amazon-atom-construction}

In this appendix, we specify the dataset-specific atom construction and assignment protocol used for the Amazon experiments. We use the Amazon Reviews 2023 \texttt{All\_Beauty} category. Each record contains review text, a rating, item identifiers, category metadata, and a reviewer identifier. The reviewer identifier allows all target reviews written by the same reviewer to be treated as one protected unit in the user-level setting.

Amazon \texttt{All\_Beauty} contains many reviewers with one review and a smaller set of reviewers with multiple reviews. We therefore use it for both record-level and user-level differential-privacy evaluation. We report dataset statistics, including the number of records, the number of reviewers, the number of reviewers with multiple reviews, and the distribution of per-reviewer contribution sizes, together with the experimental outputs.

\subsection{Raw Amazon review structure}

An Amazon review record contains review text and structured fields such as a reviewer identifier, rating, item identifier, and category. In our experiments, the review text is the protected text field used for assignment. For record-level privacy, one protected unit is one review. For user-level privacy, one protected unit is the collection of all target reviews associated with the same reviewer identifier.

\begin{examplebox}{Illustrative Amazon \texttt{All\_Beauty} review record}
\footnotesize
\textbf{Review text:}

\emph{The skin cream absorbed quickly and worked well on dry skin, but the pump dispenser broke after a few uses.}

\vspace{1mm}
\textbf{Rating:}

2.0

\vspace{1mm}
\textbf{User identifier:}

\texttt{user\_A17Q9}

\vspace{1mm}
\textbf{Item identifier:}

\texttt{parent\_asin = B000EXAMPLE}

\vspace{1mm}
\textbf{Category:}

\texttt{All\_Beauty}
\end{examplebox}

\subsection{Auxiliary and target splits}

For record-level experiments, Amazon reviews are partitioned into auxiliary and protected target records. For user-level experiments, the partition is performed at the reviewer level. All reviews from the same reviewer are assigned entirely to either the auxiliary split or the protected target split.

\begin{protocolbox}{Amazon split protocol}
\footnotesize
\textbf{Record-level experiments.}
Reviews are split into auxiliary and protected target records. Each target review is one protected unit.

\vspace{1mm}
\textbf{User-level experiments.}
Reviewers are split into auxiliary and protected target reviewers. All reviews from the same reviewer are placed in one split.

\vspace{1mm}
\textbf{User-disjointness.}
The user-disjoint split ensures that atom construction and representation fitting do not use reviews from reviewers whose records are included in the protected target split.

\vspace{1mm}
\textbf{Protected user contribution.}
For a protected target reviewer $u$, all target reviews written by $u$ are aggregated and clipped before the differentially private release.
\end{protocolbox}

\subsection{Amazon atom labels}

Unlike CFPB, Amazon Reviews does not provide a fine-grained issue taxonomy analogous to \texttt{Product / Issue / Sub-issue}. The \texttt{All\_Beauty} category restricts the domain, but it does not provide a fine-grained set of issue labels from which to form atom coordinates. We therefore construct decoder-visible atom labels as review-theme phrases from the auxiliary split.

We use deterministic phrase extraction, filtering, canonicalization, duplicate merging, and support-based ranking within the \texttt{All\_Beauty} domain. Candidate labels are two-word review-theme phrases extracted only from auxiliary review text. The filtering step removes generic review phrases, sentiment-only phrases, uninformative word-order variants, and phrases outside the \texttt{All\_Beauty} domain. Examples of preserved phrases include recurring themes related to skin concerns, hair concerns, product use, scent, packaging, delivery, price, quality, customer support, and product failure.

\begin{center}
\begin{minipage}{0.55\textwidth}
\begin{examplebox}{Example Amazon \texttt{All\_Beauty} atom labels}
\footnotesize
\begin{itemize}[leftmargin=14pt]
    \item \texttt{sensitive skin}
    \item \texttt{dry skin}
    \item \texttt{skin irritation}
    \item \texttt{hair loss}
    \item \texttt{dry hair}
    \item \texttt{pump dispenser}
    \item \texttt{strong smell}
    \item \texttt{absorbs quickly}
    \item \texttt{fast shipping}
    \item \texttt{waste money}
\end{itemize}
\end{examplebox}
\end{minipage}
\end{center}

These atom labels are fixed before the protected target split is processed. The protected target reviews affect only bounded record-to-atom assignments, aggregation, and the differentially private release. If decoder-visible atom labels are constructed from protected target reviews in a deployment, that construction is a data-dependent release and its privacy cost must be included in the accounting.

\subsection{Amazon dictionary construction}
 
The Amazon atom dictionary is constructed only from the auxiliary split. The procedure is as follows.
\begin{enumerate}[leftmargin=16pt]
    \item Load auxiliary Amazon \texttt{All\_Beauty} reviews.
    \item Normalize auxiliary review text by lowercasing and deterministic token filtering.\vspace{-1pt}
    \item Extract candidate two-word review-theme phrases from auxiliary review text.\vspace{-1pt}
    \item Remove generic review phrases, sentiment-only phrases, uninformative word-order variants, and phrases outside the \texttt{All\_Beauty} domain.\vspace{-1pt}
    \item Canonicalize phrase variants deterministically, such as \texttt{easy apply} to \texttt{easy application}, \texttt{quality hair} to \texttt{hair quality}, and \texttt{wasted money} to \texttt{waste money}.\vspace{-1pt}
    \item Merge duplicate labels that denote the same review theme.\vspace{-1pt}
    \item Retain only candidates with auxiliary document support at least $25$ after filtering, canonicalization, and duplicate merging.\vspace{-1pt}
    \item Rank the remaining candidates by auxiliary document support, with deterministic tie-breaking.\vspace{-1pt}
    \item Use the resulting ordered list as the fixed public atom dictionary for the protected target split.
\end{enumerate}

Duplicate merging is restricted to labels treated as the same review theme; distinct concerns remain separate. It merges only labels that denote the same review theme and would otherwise occupy multiple atom coordinates. It does not merge distinct concerns such as \texttt{dry skin}, \texttt{sensitive skin}, and \texttt{skin irritation}. Similarly, \texttt{hair loss}, \texttt{dry hair}, and \texttt{hair growth} remain distinct when they describe different product effects.

In the Amazon experiments, filtering, canonicalization, duplicate merging, and the auxiliary document-support threshold keep $46$ atom labels. Thus, the resulting Amazon atom universe has size $m=46$, and this fixed public dictionary is used for all Amazon \texttt{All\_Beauty} runs.

\begin{examplebox}{Example Amazon auxiliary construction}
\footnotesize

Suppose the first rows of the auxiliary-support ordering are

\vspace{1mm}
\begin{center}
\begin{tabular}{p{0.3\linewidth}p{0.3\linewidth}}
\toprule
\textbf{Auxiliary support} & \textbf{Candidate atom label} \\
\midrule
1840 & \texttt{sensitive skin} \\
1390 & \texttt{dry skin} \\
1215 & \texttt{pump dispenser} \\
980  & \texttt{strong smell} \\
875  & \texttt{customer service} \\
\bottomrule
\end{tabular}    
\end{center}

If this illustrative ordered dictionary were truncated to three atoms, the atom dictionary would be $\{\texttt{sensitive skin},~\texttt{dry skin},~\texttt{pump dispenser}\}$.
\end{examplebox}

The resulting public atom dictionary is fixed before the protected target split is used for assignment, aggregation, and differentially private release.

\vspace{-3pt}

\subsection{Record-level Amazon assignment}

\vspace{-3pt}

For record-level differential privacy, each target review $x_i$ is mapped to a bounded sparse vector over the public atom dictionary, $v_i=\phi(x_i)\in\mathbb R_+^m$, with $\|v_i\|_0\le K$ and $\|v_i\|_1\le 1$. The assignment map is fixed before the protected target split is processed. It uses the public atom dictionary and a text representation fitted only on public or auxiliary text.

\begin{examplebox}{Example Amazon record-level assignment}
\footnotesize
\textbf{Protected target review:}

\emph{The skin cream absorbed quickly and helped with dry skin, but the pump dispenser stopped working.}

\vspace{1mm}
\textbf{Assigned atoms for $K=3$:}
\[
\begin{array}{lll}
a_{12} &: 0.52 & \texttt{dry skin} \\
a_{38} &: 0.31 & \texttt{pump dispenser} \\
a_{91} &: 0.17 & \texttt{absorbs quickly}
\end{array}
\]
All other coordinates are zero. The review contributes total mass $0.52+0.17+0.31=1.00$ to at most three atoms.
\end{examplebox}

\vspace{-3pt}

\subsection{User-level Amazon assignment}

\vspace{-3pt}

For user-level differential privacy, all target reviews written by the same reviewer form one protected unit. Let $D_u^{\mathrm{target}}$ denote the set of target reviews written by reviewer $u$. We first compute record-level assignment vectors and aggregate them as $U_u=\sum_{x_i\in D_u^{\mathrm{target}}}\phi(x_i)$. We then clip the aggregate reviewer vector to a fixed $\ell_1$ contribution bound $B$, obtaining $V_u=\mathsf{Clip}_B(U_u)$, $\|V_u\|_1\le B$. The user-level aggregate sketch is $C(D)=\sum_u V_u$.
Under user-level add/drop adjacency, adding or removing all target reviews from one reviewer changes the clipped aggregate sketch by a vector with $\ell_1$ norm at most $B$. The noise scale is therefore calibrated to the user-level sensitivity $B$, rather than to the record-level sensitivity $1$. In the main Amazon user-level experiments, we use $B\in\{1,2\}$, with $B=5$ included as a clipping-bound ablation. Smaller $B$ reduces the noise required for a fixed privacy budget but clips more per-reviewer semantic mass. Larger $B$ preserves more semantic mass from reviewers with multiple target reviews but requires more noise for the same privacy budget.

\vspace{-3pt}


\subsection{Dataset and run evaluation measures}

For each Amazon run, the implementation records dataset statistics and grouped run summaries. The dataset statistics include the number of records, the number of reviewers, the number of records with a valid reviewer identifier, the number of reviewers with multiple reviews, the per-reviewer contribution-size distribution, text-length statistics, and category counts. Grouped run summaries aggregate plan-utility and verifier metrics by protected-unit type, number of selected records or reviewers, user contribution bound $B$, release mechanism, dictionary size $m=46$, and effective privacy parameters.

For user-level experiments, the effective number of selected records can vary across random seeds because a fixed number of selected reviewers may contribute different numbers of reviews. Therefore, grouped user-level summaries are not grouped by effective record count. Instead, each group reports the mean, minimum, maximum, and standard deviation of the effective record count. This keeps runs with the same user-level experimental condition in the same group while reporting variation in the number of underlying reviews.

\begin{protocolbox}{Amazon run outputs}
\footnotesize
For Amazon \texttt{All\_Beauty}, the implementation records 
\begin{itemize}[leftmargin=14pt]
    \item processed-record statistics for the prepared category;\vspace{-1pt}
    \item auxiliary and protected-target split statistics;\vspace{-1pt}
    \item user-contribution statistics for user-level experiments;\vspace{-1pt}
    \item atom-construction manifests specifying the domain profile,
    canonicalization rules, refinement mode, candidate counts, auxiliary-support
    threshold, preserved dictionary size, and auxiliary-support values;\vspace{-1pt}
    \item grouped run summaries keyed by $n_{\mathrm{records}}$ or
    $n_{\mathrm{users}}$, $B$ when applicable, mechanism, $m=46$, and
    $\varepsilon$;\vspace{-1pt}
    \item verifier measures, including unsupported-atom rate and
    unsupported-comparison rate.
\end{itemize}
\end{protocolbox}

\vspace{-3pt}

\subsection{Fixed public Amazon atom dictionary}
\label{app:amazon-fixed-atoms}

\vspace{-3pt}

Table~\ref{tab:amazon-fixed-atoms} gives the fixed public atom dictionary used in the Amazon \texttt{All\_Beauty} experiments. The support column reports auxiliary document support after filtering, canonicalization, duplicate merging, and application of the auxiliary-support threshold. These support values are computed only from the auxiliary split and are not statistics of the protected target split.

\vspace{-5pt}

\subsection{Amazon protocol summary}

\vspace{-3pt}

\begin{protocolbox}{Amazon \texttt{All\_Beauty} record-level and user-level protocol}
\footnotesize
\begin{itemize}[leftmargin=14pt]
    \item \textbf{Dataset subset.} Amazon Reviews 2023 \texttt{All\_Beauty}.\vspace{-1pt}
    \item \textbf{Record-level DP.} One review is one protected unit; each review contributes a vector with $\ell_1$ mass at most $1$.\vspace{-1pt}
    \item \textbf{User-level DP.} All target reviews written by one reviewer form one protected unit.\vspace{-1pt}
    \item \textbf{Auxiliary and target splits.} For user-level experiments, the split is user-disjoint.\vspace{-1pt}
    \item \textbf{Atom labels.} Atom labels are two-word review-theme phrases extracted only from auxiliary \texttt{All\_Beauty} reviews, then filtered, canonicalized, and duplicate-merged.\vspace{-1pt}
    \item \textbf{Dictionary size.} The fixed public Amazon atom dictionary has size $m=46$.\vspace{-1pt}
    \item \textbf{Assignment.} Each review is mapped to at most $K$ atoms with total mass at most one.\vspace{-1pt}
    \item \textbf{User clipping.} Each reviewer aggregate vector is clipped to $\ell_1$ norm at most $B$; the main Amazon user-level runs use $B\in\{1,2\}$, with $B=5$ used only as an ablation.\vspace{-1pt}
    \item \textbf{Differentially private release.} The released semantic plan contains admitted atom labels and noisy masses or support bins.\vspace{-1pt}
    \item \textbf{Recorded outputs.} Dataset statistics, atom-construction manifests, grouped run summaries, and verifier measures are recorded with the experiment outputs.
\end{itemize}
\end{protocolbox}

\clearpage

\begin{table}[h]
\centering
\caption{Fixed Public Amazon \texttt{All\_Beauty} atom dictionary with $m=46$.}
\label{tab:amazon-fixed-atoms}
\tiny
\resizebox{0.60\textwidth}{!}{
\begin{tabular}{rll}
\toprule
\textbf{Atom id} & \textbf{Atom label} & \textbf{Auxiliary support} \\
\midrule
0 & \texttt{waste money} & 3781 \\
1 & \texttt{sensitive skin} & 2560 \\
2 & \texttt{easy application} & 2426 \\
3 & \texttt{curly hair} & 2033 \\
4 & \texttt{dry skin} & 1675 \\
5 & \texttt{product quality} & 1589 \\
6 & \texttt{long hair} & 1530 \\
7 & \texttt{skin feel} & 1526 \\
8 & \texttt{fast shipping} & 1207 \\
9 & \texttt{customer service} & 1197 \\
10 & \texttt{human hair} & 984 \\
11 & \texttt{hair quality} & 919 \\
12 & \texttt{natural hair} & 825 \\
13 & \texttt{dry hair} & 751 \\
14 & \texttt{fast delivery} & 666 \\
15 & \texttt{short hair} & 643 \\
16 & \texttt{skin care} & 617 \\
17 & \texttt{oily skin} & 587 \\
18 & \texttt{strong smell} & 550 \\
19 & \texttt{dead skin} & 483 \\
20 & \texttt{straight hair} & 453 \\
21 & \texttt{hair length} & 451 \\
22 & \texttt{wavy hair} & 427 \\
23 & \texttt{reasonable price} & 426 \\
24 & \texttt{skin irritation} & 399 \\
25 & \texttt{wet hair} & 368 \\
26 & \texttt{battery life} & 361 \\
27 & \texttt{low price} & 339 \\
28 & \texttt{synthetic hair} & 326 \\
29 & \texttt{low quality} & 309 \\
30 & \texttt{soft hair} & 292 \\
31 & \texttt{absorbs quickly} & 283 \\
32 & \texttt{chemical smell} & 277 \\
33 & \texttt{hair growth} & 252 \\
34 & \texttt{hair loss} & 190 \\
35 & \texttt{soft skin} & 186 \\
36 & \texttt{smooth skin} & 184 \\
37 & \texttt{high price} & 109 \\
38 & \texttt{natural scent} & 93 \\
39 & \texttt{pump dispenser} & 90 \\
40 & \texttt{smooth application} & 81 \\
41 & \texttt{oily hair} & 78 \\
42 & \texttt{smooth hair} & 63 \\
43 & \texttt{skin cream} & 47 \\
44 & \texttt{customer support} & 30 \\
45 & \texttt{skin rash} & 26 \\
\bottomrule
\end{tabular}}
\end{table}

\clearpage

\appsection{Yelp Reviews Atom Construction and Assignment Protocol}
\label{app:yelp-atom-construction}

In this appendix, we specify the dataset-specific atom construction and assignment protocol used for the Yelp experiments. Yelp review records contain review text, business identifiers, ratings, timestamps, and reviewer identifiers. We join each review to Yelp business metadata by \texttt{business\_id} and restrict our experiments to reviews whose joined business categories include \texttt{Restaurants}. The resulting Yelp Restaurants subset is used for both record-level and user-level differential-privacy evaluation. The \texttt{user\_id} field allows all target reviews written by the same reviewer to be treated as one protected unit in the user-level setting.

Yelp Restaurants contains many reviewers with one review and a smaller set of reviewers with multiple reviews. We report dataset statistics, including the number of records, the number of reviewers, the number of reviewers with
multiple reviews, and the distribution of per-reviewer contribution sizes, together with the experimental outputs.

\subsection{Raw Yelp review structure}

Yelp review record contains review text and structured fields such as a reviewer identifier, business identifier, rating, and date. The Yelp business metadata file contains business-level fields such as business name, city, state, and categories. In preprocessing, we join review records to business metadata by \texttt{business\_id} and keep reviews whose joined category string contains \texttt{Restaurants}.

\begin{examplebox}{Illustrative Yelp Restaurants record after metadata join}
\footnotesize

\textbf{Review text:}

\emph{The food was excellent and the staff were friendly, but we waited almost forty minutes even though we had a reservation. The restaurant was also very noisy.}

\vspace{1mm}
\textbf{Stars:}

3.0

\vspace{1mm}
\textbf{User identifier:}

\texttt{user\_8K42}

\vspace{1mm}
\textbf{Business identifier:}

\texttt{business\_P19Z}

\vspace{1mm}
\textbf{Business categories:}

\texttt{Restaurants | Italian | Nightlife}
\end{examplebox}

For record-level differential privacy, one protected unit is one Yelp restaurant review. For user-level differential privacy, one protected unit is the collection of all target reviews written by the same \texttt{user\_id}.

\subsection{Auxiliary and target splits}

For record-level experiments, Yelp restaurant reviews are partitioned into auxiliary and target records. For user-level experiments, the partition is performed at the reviewer level, so all reviews from the same \texttt{user\_id} are assigned entirely to either the auxiliary split or the protected target split.

\begin{protocolbox}{Yelp Restaurants split protocol}
\footnotesize
\textbf{Record-level experiments.}
Restaurant reviews are split into auxiliary and protected target records. Each target review is one protected unit.

\vspace{1mm}
\textbf{User-level experiments.}
Reviewers are split into auxiliary and protected target reviewers. All reviews written by the same reviewer are placed entirely in one split.

\vspace{1mm}
\textbf{User-disjointness.}
The user-disjoint split ensures that atom construction and representation fitting do not use reviews written by reviewers whose records are included in the protected target split.
\end{protocolbox}

\subsection{Yelp atom labels}

Yelp business categories are used to restrict the dataset to restaurant reviews, but they do not define fine-grained review-theme labels. We therefore construct Yelp atom labels only from auxiliary restaurant-review text. Our implementation uses deterministic restaurant-domain phrase extraction, filtering, canonicalization, duplicate merging, and auxiliary-support ranking.

The preserved atoms include restaurant-experience themes, such as service, waiting time, food quality, seating, price, and portion size, as well as restaurant-category and menu themes, such as cuisine type, breakfast items, coffee shops, sushi, and food trucks. Category and menu atoms are preserved when they occur as recurring review themes in the auxiliary restaurant-review text.

These atom labels are fixed before the protected target split is processed. The protected target reviews affect only bounded record-to-atom assignments, aggregation, and the differentially private release. If decoder-visible atom labels are constructed from protected target reviews in a deployment, that construction is a data-dependent release and its privacy cost must be included in the accounting.

\subsection{Yelp dictionary construction}

The Yelp atom dictionary is constructed only from the auxiliary split. The construction proceeds as follows.
\begin{enumerate}[leftmargin=16pt]
    \item Join Yelp reviews to Yelp business metadata by \texttt{business\_id}.
    \item Retain reviews for businesses whose category string contains
    \texttt{Restaurants}.
    \item Partition the resulting Yelp Restaurants records into auxiliary and protected target sets. For user-level experiments, the partition is user-disjoint.
    \item Use only the auxiliary Yelp Restaurants reviews for atom construction.
    \item Normalize auxiliary review text by lowercasing and deterministic token filtering.
    \item Extract candidate two-word restaurant-review phrases from auxiliary review text.
    \item Remove generic review phrases, phrases specific to non-restaurant product-review domains, sentiment-only phrases, and ill-formed or uninformative word orders.
    \item Canonicalize phrase variants deterministically, such as
    \texttt{staff friendly} to \texttt{friendly staff} and
    \texttt{food delicious} to \texttt{delicious food}.
    \item Merge duplicate labels after canonicalization.
    \item Rank the remaining candidates by auxiliary document support, with deterministic tie-breaking.
    \item Retain the top $m$ labels as the fixed public atom dictionary for the protected target split.
\end{enumerate}

In the Yelp Restaurants experiments, we use dictionary sizes $m\in\{50,100\}$, with $m=200$ used only as a finer-dictionary ablation. The dictionaries are ordered by auxiliary document support. The $m=50$ dictionary is reported in Table~\ref{tab:yelp-fixed-atoms-m50}. The $m=100$ and $m=200$ dictionaries are provided with the experiment artifacts.

\subsection{Fixed Public Yelp Restaurants atom dictionary}
\label{app:yelp-fixed-atoms}

Table~\ref{tab:yelp-fixed-atoms-m50} reports the fixed $m=50$ public atom dictionary used in the main Yelp Restaurants experiments. The support column reports auxiliary document support after filtering, canonicalization, and duplicate merging. These support values are computed only from the auxiliary split and are not statistics of the protected target split.

\begin{table}[h]
\centering
\caption{Fixed Public Yelp Restaurants $m=50$ atom dictionary.}
\label{tab:yelp-fixed-atoms-m50}
\footnotesize
\resizebox{0.50\textwidth}{!}{
\begin{tabular}{rll}
\toprule
\textbf{Atom id} & \textbf{Atom label} & \textbf{Auxiliary support} \\
\midrule
0 & \texttt{friendly staff} & 62502 \\
1 & \texttt{customer service} & 58087 \\
2 & \texttt{long wait} & 38604 \\
3 & \texttt{service speed} & 34677 \\
4 & \texttt{delicious food} & 31452 \\
5 & \texttt{food quality} & 25145 \\
6 & \texttt{mexican food} & 24201 \\
7 & \texttt{outdoor seating} & 19223 \\
8 & \texttt{fast service} & 14345 \\
9 & \texttt{chinese food} & 13887 \\
10 & \texttt{fresh food} & 11137 \\
11 & \texttt{high quality} & 10288 \\
12 & \texttt{portion size} & 10182 \\
13 & \texttt{indian food} & 9421 \\
14 & \texttt{mexican restaurant} & 8988 \\
15 & \texttt{thai food} & 8852 \\
16 & \texttt{italian food} & 8811 \\
17 & \texttt{coffee shop} & 8665 \\
18 & \texttt{pizza place} & 8618 \\
19 & \texttt{lunch special} & 7813 \\
20 & \texttt{seating} & 6933 \\
21 & \texttt{slow service} & 6418 \\
22 & \texttt{reasonable price} & 6230 \\
23 & \texttt{italian restaurant} & 6195 \\
24 & \texttt{terrible service} & 5979 \\
25 & \texttt{lunch menu} & 5585 \\
26 & \texttt{breakfast sandwich} & 5566 \\
27 & \texttt{food truck} & 5397 \\
28 & \texttt{sushi place} & 5251 \\
29 & \texttt{drink menu} & 5211 \\
30 & \texttt{price} & 5209 \\
31 & \texttt{chinese restaurant} & 5090 \\
32 & \texttt{quick lunch} & 4936 \\
33 & \texttt{sushi bar} & 4610 \\
34 & \texttt{breakfast burrito} & 4495 \\
35 & \texttt{breakfast place} & 4118 \\
36 & \texttt{dinner menu} & 3893 \\
37 & \texttt{large portion} & 3774 \\
38 & \texttt{food menu} & 3745 \\
39 & \texttt{small portion} & 3650 \\
40 & \texttt{vietnamese food} & 3509 \\
41 & \texttt{brunch menu} & 3298 \\
42 & \texttt{korean food} & 3230 \\
43 & \texttt{food place} & 3077 \\
44 & \texttt{greek food} & 3060 \\
45 & \texttt{thai restaurant} & 3021 \\
46 & \texttt{sushi restaurant} & 2956 \\
47 & \texttt{indian restaurant} & 2941 \\
48 & \texttt{breakfast food} & 2705 \\
49 & \texttt{hot food} & 2666 \\
\bottomrule
\end{tabular}}
\end{table}

\subsection{Record-level Yelp assignment}

For record-level differential privacy, each target review $x_i$ is mapped to a bounded sparse vector over the fixed public atom dictionary, $v_i=\phi(x_i)\in\mathbb R_+^m$, with $\|v_i\|_0\le K$ and $\|v_i\|_1\le 1$.

\begin{examplebox}{Example Yelp record-level assignment}
\footnotesize

\textbf{Protected target review:}

\emph{The food was excellent and the staff were friendly, but we waited almost forty minutes even though we had a reservation.}

\vspace{1mm}
\textbf{Assigned atoms for $K=3$:}
\[
\begin{array}{lll}
a_{4}  &: 0.45 & \texttt{long wait} \\
a_{11} &: 0.35 & \texttt{friendly staff} \\
a_{29} &: 0.20 & \texttt{food quality}
\end{array}
\]

All other coordinates are zero. The review contributes total mass $1.00$ to at most three atoms.

\end{examplebox}

\subsection{User-level Yelp assignment}

For user-level differential privacy, all target reviews written by the same reviewer form one protected unit. Let $D_u^{\mathrm{target}}$ denote the set of target reviews written by reviewer $u$. We first compute record-level assignment vectors and aggregate them as $U_u=\sum_{x_i\in D_u^{\mathrm{target}}}\phi(x_i)$. We then clip the aggregate reviewer vector to a fixed $\ell_1$ contribution bound $B$, obtaining $V_u=\mathsf{Clip}_B(U_u)$, $\|V_u\|_1\le B$. The user-level semantic sketch is $C(D)=\sum_u V_u$. Under user-level add/drop adjacency, adding or removing all target reviews from one reviewer changes the clipped aggregate sketch by a vector with $\ell_1$ norm at most $B$. The noise scale is therefore calibrated to the user-level sensitivity $B$, rather than to the record-level sensitivity $1$. In the main Yelp user-level experiments, we use $B\in\{1,2\}$, with $B=5$ included as a clipping-bound ablation. Smaller $B$ reduces the noise required for a fixed privacy budget but clips more per-reviewer semantic mass. Larger $B$ preserves more semantic mass from reviewers with multiple target reviews but requires more noise for the same privacy budget.

\subsection{Dataset and run statistics}

For each Yelp run, the implementation records dataset statistics and grouped run summaries. The dataset statistics include the number of records, the number of reviewers, the number of records with a valid \texttt{user\_id}, the number of reviewers with multiple reviews, the per-reviewer contribution-size distribution, text-length statistics, and the most frequent business categories. Grouped summaries aggregate plan-utility and verifier metrics by protected-unit type, number of selected records or reviewers, contribution bound $B$, release mechanism, dictionary size $m$, and effective privacy parameters. This grouping keeps runs with different user counts, dictionary sizes, mechanisms, privacy parameters, or clipping bounds in separate summary tables.

\begin{protocolbox}{Yelp run outputs}
\footnotesize
For Yelp Restaurants, our implementation records
\begin{itemize}[leftmargin=14pt]
    \item processed-record statistics for the restaurant subset;
    \item auxiliary and protected-target split statistics;
    \item reviewer-contribution statistics for user-level experiments;
    \item atom-construction manifests specifying the domain profile, canonicalization rules, refinement mode, candidate counts, preserved dictionary size, and auxiliary-support values;
    \item grouped run summaries keyed by $n_{\mathrm{records}}$ or $n_{\mathrm{users}}$, $B$ when applicable, mechanism, $m$, and effective privacy parameters;
    \item verifier measures, including unsupported-atom rate and unsupported-comparison rate.
\end{itemize}
\end{protocolbox}

\subsection{Yelp protocol summary}

\begin{protocolbox}{Yelp Restaurants record-level and user-level protocol}
\footnotesize
\begin{itemize}[leftmargin=14pt]
    \item \textbf{Dataset subset.} Yelp reviews are joined to business metadata by \texttt{business\_id}. The experiments use reviews whose business categories include \texttt{Restaurants}.
    
    \item \textbf{Record-level DP.} One Yelp restaurant review is one protected unit.
    
    \item \textbf{User-level DP.} All target reviews written by the same reviewer form one protected unit, with reviewers identified by
    \texttt{user\_id}.
    
    \item \textbf{Auxiliary and target splits.} For user-level experiments, the split is user-disjoint.
    
    \item \textbf{Atom labels.} Atom labels are two-word restaurant-review phrases extracted only from auxiliary text, then filtered, canonicalized, and duplicate-merged. Yelp atoms include restaurant-experience themes and recurring restaurant-category or menu themes.
    
    \item \textbf{Dictionary sizes.} Main experiments use $m\in\{50,100\}$; $m=200$ is used as a finer-dictionary ablation.
    
    \item \textbf{Assignment.} Each review is mapped to at most $K$ atoms with total mass at most one.
    
    \item \textbf{User clipping.} Each reviewer aggregate vector is clipped to $\ell_1$ norm at most $B$. Main Yelp user-level runs use $B\in\{1,2\}$, with $B=5$ used only as an ablation.
    
    \item \textbf{Differentially private release.} The released semantic plan contains admitted atom labels and noisy masses or support bins.
    
    \item \textbf{Recorded outputs.} Dataset statistics, atom-construction manifests, grouped run summaries, and verifier measures are recorded with the experiment outputs.
\end{itemize}
\end{protocolbox}

\clearpage

\appsection{Contract-Constrained Decoder and Public Plan Verifier}
\label{app:decoder-verifier}

In this appendix, we specify the decoder interface and the public verifier used in the experiments. The decoder is not part of the privacy mechanism. It receives only the released semantic plan and public decoding instructions. The verifier is a public post-processing procedure that checks whether the generated summary is consistent with the released plan.

\subsection{Decoder-visible plan}

For a released plan $P(D)= \left( \widetilde S, \widetilde c_{\widetilde S}, \{(\ell(a_j),d(a_j)):j\in\widetilde S\} \right)$, the decoder receives only the following objects:
\begin{itemize}[leftmargin=16pt]
    \item admitted atom identifiers $j\in\widetilde S$;\vspace{-3pt}
    \item admitted atom labels $\ell(a_j)$;\vspace{-3pt}
    \item optional public atom descriptions $d(a_j)$;\vspace{-3pt}
    \item noisy masses or support bins derived from $\widetilde c_j$;\vspace{-3pt}
    \item public formatting and decoding instructions.
\end{itemize}

The decoder does not receive raw protected records, target excerpts, nearest neighbors, per-record contribution vectors, non-private sketch values, atoms outside the released plan, or labels derived non-privately from the protected target split.

\begin{protocolbox}{Decoder-visible input}
\footnotesize
The decoder receives a released plan table of the following form:
\[
\begin{array}{llll}
\textbf{Marker} & \textbf{Atom label} & \textbf{Noisy mass} & \textbf{Support bin} \\
\midrule
\texttt{[atom:17]} & \texttt{long wait} & 318.4 & \texttt{high} \\
\texttt{[atom:42]} & \texttt{friendly staff} & 271.9 & \texttt{moderate} \\
\texttt{[atom:81]} & \texttt{food quality} & 244.7 & \texttt{moderate}
\end{array}
\]
The decoder is instructed to generate the summary using only these released entries and public instructions.
\end{protocolbox}

\subsection{Decoder modes}

We use two decoder modes for \texttt{DP-SPIN} outputs.

\paragraph{Audit decoder.}
The audit decoder is used for verification. It is instructed to attach a canonical marker of the form $\texttt{[atom:}j\texttt{]}$ to each substantive theme mention, where $j\in\widetilde S$. It may report released noisy masses and support bins. This makes each substantive claim directly checkable by the public verifier.
 
\paragraph{Display decoder.}
The display decoder produces summaries for qualitative examples and text-level evaluation. It receives the same released plan as the audit decoder, including admitted atom labels, public descriptions, released noisy masses, and support bins. Unlike the audit decoder, it is not required to display canonical atom markers or exact noisy masses. It may express released quantitative information through the corresponding support bins, such as $\texttt{high}$, $\texttt{moderate}$, and $\texttt{low}$. It remains constrained to the admitted atoms and public descriptions in the released plan.

Both decoder modes are post-processing of the same released differentially private plan. The display decoder changes only the presentation format and does not receive additional private information.

\subsection{Non-private-reference decoding}

For offline utility evaluation, we also construct a non-private reference plan $P_{\mathrm{NP}}$ from the non-private top-$L$ semantic sketch. This plan is not released by the differentially private mechanism and is used only for evaluation. The corresponding non-private reference display summary $Y_{\mathrm{NP}}$ is used for text-level comparison with the differentially private display summary.

\begin{protocolbox}{Evaluation use of decoder outputs}
\footnotesize
The audit summary is used for verifier checks and verifier measures. The display summary is used for qualitative examples and text-level metrics. The non-private reference display summary is used only as an offline reference for utility evaluation.
\end{protocolbox}

\subsection{Canonical atom references and decoder contract}

The audit decoder uses canonical markers to make the link to the released plan explicit. For example, instead of writing only ``long wait time has high semantic support,'' the audit decoder writes ``long wait time \texttt{[atom:17]} has high semantic support.'' This format gives the verifier an explicit link between each generated theme mention and a released plan entry.
A generated audit summary $Y$ satisfies the decoder contract if every substantive claim is supported by the released plan. Allowed claims are:
\begin{enumerate}[leftmargin=16pt]
    \item an admitted atom label, referenced with its canonical marker;\vspace{-3pt}
    \item a paraphrase directly supported by an admitted atom description,
    referenced with its canonical marker;\vspace{-3pt}
    \item a released noisy mass or support bin for an admitted atom;\vspace{-3pt}
    \item a pairwise semantic-support comparison between admitted atoms, supported by released noisy masses;\vspace{-3pt}
    \item a superlative or rank claim for an admitted atom, supported by the released rank;\vspace{-3pt}
    \item public, data-independent text stating that the output is an aggregate summary or is based on a differentially private plan.
\end{enumerate}
Disallowed claims are:
\begin{enumerate}[leftmargin=16pt]
    \item raw examples or quotations from protected records;\vspace{-3pt}
    \item record-specific details, including names, dates, addresses, account numbers, case identifiers, or unique events;\vspace{-3pt}
    \item non-admitted atom labels;\vspace{-3pt}
    \item substantive themes without canonical support in the released plan;\vspace{-3pt}
    \item unsupported support-bin assignments;\vspace{-3pt}
    \item unsupported numeric claims;\vspace{-3pt}
    \item unsupported pairwise semantic-support comparisons;\vspace{-3pt}
    \item unsupported superlative or rank claims;\vspace{-3pt}
    \item explicit importance-ordering claims, because the released plan does not contain a separate importance score;\vspace{-3pt}
    \item claims about causes, demographics, subgroups, locations, or entities not present in the released plan.
\end{enumerate}

The display summary is not used as the strict verifier target because it may omit canonical markers for readability. It remains constrained to the released plan and is post-processing of the same released differentially private plan.

\begin{examplebox}{Allowed and disallowed decoder behavior}
\footnotesize

\textbf{Released atoms:}
\[
\{\texttt{[atom:17] long wait},\;
  \texttt{[atom:42] friendly staff},\;
  \texttt{[atom:81] food quality}\}.
\]

\vspace{1mm}
\textbf{Allowed sentence:}

\emph{The released plan identifies long wait \texttt{[atom:17]}, friendly staff \texttt{[atom:42]}, and food quality \texttt{[atom:81]} as recurring themes.}

\vspace{2mm}
\textbf{Disallowed sentence:}

\emph{Many users complained that a specific waiter named John ignored them on Friday night.}

\vspace{1mm}
The disallowed sentence introduces a record-specific example and a named person that are not present in the released plan.
\end{examplebox}

\subsection{Decoder implementations}

We use two decoder implementations.

\paragraph{Template decoder.}
The template decoder deterministically converts the released plan into a short natural-language summary. In audit mode, it emits canonical atom markers by construction. It uses only admitted atom labels, public descriptions, and released support bins from the released plan. In display mode, it uses the same released plan entries but omits audit markers when producing qualitative examples.

\paragraph{LLM decoder.}
The LLM decoder receives a prompt containing the released plan and a public decoder policy. In audit mode, the prompt requires canonical atom markers and restricts quantitative claims to released values or public post-processing of released values. In display mode, the prompt requests a concise aggregate summary using only admitted themes and released support bins, without requiring atom identifiers in the displayed text. In both modes, the LLM is used only as a verbalizer of the released plan.

The prompt templates for the audit decoder, display decoder, non-private-reference decoder, baselines, and OpenAI judge are reported in Appendix~\ref{app:decoder-prompts}.

\subsection{Public verifier}

After audit decoding, we apply a public verifier. The verifier receives only $(P(D),Y_{\mathrm{audit}})$, where $P(D)$ is the released differentially private plan and $Y_{\mathrm{audit}}$ is the audit summary. The verifier does not receive the protected dataset, the non-private sketch, the display summary as a verification target, or any per-record information.

Our implemented verifier is a contract verifier, not a semantic oracle. It checks whether the generated text is consistent with the released plan under fixed public rules. The verifier uses:
\begin{itemize}[leftmargin=16pt]
    \item the admitted atom identifiers $\widetilde S$;\vspace{-3pt}
    \item admitted atom labels and optional public descriptions;\vspace{-3pt}
    \item released noisy masses, support bins, and released ranks;\vspace{-3pt}
    \item public comparison rules for pairwise semantic-support comparisons and superlative or rank claims;\vspace{-3pt}
    \item public regular-expression filters for structured identifiers and personal-contact patterns, such as email addresses, phone numbers, Social-Security-number formats, account or case identifiers, long digit sequences, long quotations, and titled person-name patterns.
\end{itemize}

\subsection{Verifier checks}

The verifier performs the following checks.

\paragraph{Canonical marker check.}
Every marker of the form $\texttt{[atom:}j\texttt{]}$ must refer to an admitted atom $j\in\widetilde S$. A marker referring to a non-admitted atom is rejected.

\vspace{-3pt}

\paragraph{Literal label support check.}
If the summary contains a literal atom label from the public dictionary, the corresponding atom must be admitted in the released plan. This check detects unsupported atom mentions even when the decoder omits the canonical marker.

\vspace{-3pt}

\paragraph{Support-bin check.}
If the summary assigns a support bin to an admitted atom, the bin must match the released bin for that atom, or a public coarsening allowed by the released plan or decoding policy.

\vspace{-3pt}

\paragraph{Numeric-claim check.}
If the summary contains a numeric mass, percentage, or normalized semantic-support value, the number must be a released noisy value or a value obtained by public post-processing of released values. Unsupported numeric claims are rejected.

\vspace{-3pt}

\paragraph{Pairwise semantic-support check.}
If the summary states that atom $a$ has higher semantic support than atom $b$, both atoms must be admitted, and the claimed ordering must be supported by the released noisy masses under the fixed public comparison margin.

\vspace{-3pt}

\paragraph{Superlative and rank check.}
If the summary uses a superlative or rank expression such as ``highest support,'' ``highest released rank,'' or ``top-ranked theme,'' the referenced atom must have released rank one.

\vspace{-3pt}

\paragraph{Importance-ordering check.}
Explicit importance-ordering claims, such as ``more important than'' or ``most important,'' are rejected because the released plan contains no separate importance score.

\vspace{-3pt}

\paragraph{Raw-record pattern check.}
The verifier rejects summaries containing raw-record patterns or structured identifier patterns, including email addresses, phone numbers, Social-Security-number formats, account or case identifiers, long digit sequences, long quotations, and titled person-name patterns. These checks are public regular-expression filters and do not use the protected dataset.

\begin{examplebox}{Accepted audit summary}
\footnotesize

\textbf{Released plan:}
\[
\begin{array}{llll}
\texttt{[atom:17]} & \texttt{long wait} & 318.4 & \texttt{high} \\
\texttt{[atom:42]} & \texttt{friendly staff} & 271.9 & \texttt{moderate} \\
\texttt{[atom:81]} & \texttt{food quality} & 244.7 & \texttt{moderate}
\end{array}
\]

\vspace{1mm}
\textbf{Generated audit summary:}

\emph{The released plan highlights long wait \texttt{[atom:17]}, friendly staff \texttt{[atom:42]}, and food quality \texttt{[atom:81]} as recurring themes. Long wait \texttt{[atom:17]} has a higher released noisy mass than food quality \texttt{[atom:81]}.}

\vspace{1mm}
\textbf{Verifier result:}

All atom markers refer to admitted atoms. The comparison is supported because the released noisy mass for \texttt{[atom:17]} is larger than the released noisy mass for \texttt{[atom:81]}. The summary is accepted.

\end{examplebox}

\begin{examplebox}{Rejected audit summary: non-admitted theme}
\footnotesize

\textbf{Released plan:}
\[
\{\texttt{[atom:17] long wait},\;
  \texttt{[atom:42] friendly staff},\;
  \texttt{[atom:81] food quality}\}.
\]

\vspace{1mm}
\textbf{Generated audit summary:}

\emph{The main themes are long wait \texttt{[atom:17]}, friendly staff \texttt{[atom:42]}, food quality \texttt{[atom:81]}, and parking problems \texttt{[atom:99]}.}

\vspace{1mm}
\textbf{Verifier result:}

The marker \texttt{[atom:99]} does not refer to an admitted atom. The summary is rejected.
\end{examplebox}

\begin{examplebox}{Rejected audit summary: unsupported numeric claim}
\footnotesize

\textbf{Released plan:}
\[
\texttt{[atom:17] long wait}: \text{noisy mass }318.4.
\]

\vspace{1mm}
\textbf{Generated audit summary:}

\emph{Exactly 612 records mention long wait \texttt{[atom:17]}.}

\vspace{1mm}
\textbf{Verifier result:}

The value $612$ is not a released noisy value and is not obtained by permitted public post-processing of the released plan. The summary is rejected.
\end{examplebox}

\subsection{Regeneration loop}

If the verifier rejects an audit summary, the decoder may generate a new audit summary using the same released plan $P(D)$, public rules, and verifier feedback. The protected dataset is not accessed again. The implemented loop is:
\begin{enumerate}[leftmargin=16pt]
    \item release the differentially private semantic plan $P(D)$;\vspace{-4pt}
    \item generate an audit summary $Y_{\mathrm{audit}}^{(1)}$ from $P(D)$;\vspace{-4pt}
    \item verify $Y_{\mathrm{audit}}^{(1)}$ using only $P(D)$ and public rules;
    \item if accepted, record $Y_{\mathrm{audit}}^{(1)}$;\vspace{-4pt}
    \item if rejected, provide public verifier feedback and generate
    $Y_{\mathrm{audit}}^{(2)}$ from the same $P(D)$;\vspace{-4pt}
    \item repeat up to a fixed public maximum number of attempts;\vspace{-4pt}
    \item if all attempts fail, output a deterministic audit-template summary
    computed from the same released plan and record the failure status and
    verifier measures.
\end{enumerate}

The display summary may be generated from the same released plan for qualitative presentation. Every candidate summary, verifier decision, feedback message, regenerated summary, and display summary is computed only from $P(D)$, public rules, and, when applicable, decoder randomness. Therefore, the entire procedure is post-processing of the released differentially private plan.

\begin{protocolbox}{Post-processing property}
\footnotesize
The decoder, verifier, regeneration loop, and display formatting do not consume additional privacy budget. They never access protected target records, raw excerpts, per-record assignment vectors, or the non-private sketch. They only transform the released differentially private semantic plan into text and recorded verifier measures.
\end{protocolbox}

\subsection{Scope and Limitations of Verification}
\label{app:verifier-scope}

The verifier establishes consistency between an audit summary and the released differentially private plan. It does not establish completeness of the released plan, equality between noisy and non-private masses, or completeness of the summary with respect to the protected dataset. In particular:
\begin{itemize}[leftmargin=16pt]
    \item a verifier-consistent summary may omit a theme that is present in the protected dataset but not admitted into the released plan;
    \item a verifier-consistent summary reflects released noisy values, not necessarily the corresponding non-private values;
    \item the verifier does not evaluate semantic correctness against raw protected records, because it never receives those records;
    \item the verifier can check paraphrases only to the extent that they are linked to admitted atom markers and supported by public atom labels or descriptions;
    \item the verifier does not provide the privacy guarantee. Privacy follows from the differentially private plan-release mechanism and the closure of differential privacy under post-processing.
\end{itemize}

Thus, the verifier enforces the public decoder contract and reports plan-consistency measures. It is not a semantic audit of the protected corpus.

\clearpage

\appsection{Decoder and Evaluation Prompts}
\label{app:decoder-prompts}

\vspace{-2pt}

In this appendix, we report the public prompt templates used for LLM decoding and evaluation.  In \texttt{DP-SPIN} decoding, the model receives only the released semantic plan and public instructions. It does not receive raw protected records, target excerpts, nearest neighbors, per-record assignments, non-private sketch values, or non-admitted atoms.

\vspace{-2pt}

\subsection{DP-SPIN audit decoder prompt}
\label{app:prompt-audit-decoder}

\vspace{-2pt}

The audit decoder is used for verifier checks and verifier measures. It requires canonical atom markers and permits quantitative or semantic-support claims only when they are supported by released noisy masses or released support bins. The public verifier is applied to the generated audit summary.

\begin{promptbox}{Audit decoder prompt}
\footnotesize

You are given a differentially private aggregate summary plan.

Write a concise aggregate summary using only the listed plan entries.

Every substantive theme mention must include its canonical marker \texttt{[atom:<id>]} exactly as shown.

Rules:
\begin{enumerate}[leftmargin=16pt]
    \item Do not introduce any theme not listed in the plan.
    \item Do not invent examples, names, organizations, dates, quotes, or events.
    \item Do not mention individual records or claim to have seen raw records.
    \item Do not infer causes, demographics, locations, product identities, businesses, user groups, or subgroups beyond the atom labels and descriptions.
    \item Do not interpret semantic support as the number, fraction, or frequency of records expressing an atom.
    \item Use semantic-support language only according to the released support bins and released noisy masses.
    \item Comparisons such as `higher support than' are allowed only when supported by the released noisy masses.
    \item If the plan is low-support, noisy, or ambiguous, state uncertainty rather than adding unsupported content.
\end{enumerate}

Released differentially private summary plan:

\texttt{<PLAN\_ENTRIES>}

Output format: one short paragraph followed by up to five bullets. Use canonical markers in both the paragraph and the bullets.

\end{promptbox}

Each plan entry has the form
{\small
\[
\texttt{- [atom:}j\texttt{] label=<label> | description=<description> | bin=<bin> | noisy\_mass=<value>}.
\]
}

\subsection{DP-SPIN display decoder prompt}
\label{app:prompt-display-decoder}

The display decoder is used for qualitative examples and text-level evaluation. It receives the same released plan as the audit decoder, including admitted atom labels, public descriptions, released noisy masses, and support bins. It is not required to display atom markers or exact noisy masses.

\begin{promptbox}{Display decoder prompt}
\footnotesize

You are given a differentially private aggregate summary plan.

Write a concise aggregate summary using only the listed plan entries.

Rules:
\begin{enumerate}[leftmargin=16pt]
    \item Do not introduce any theme not listed in the plan.
    \item Do not invent examples, names, organizations, dates, quotes, events, locations, or subgroups.
    \item Do not mention individual records or claim to have seen raw records.
    \item Do not infer causes, demographics, product identities, businesses, user groups, or subgroups beyond the atom labels and descriptions.
    \item  Do not interpret semantic support as the number, fraction, or frequency of records expressing an atom.
    \item Use semantic-support language only from the released support bins. Do not report exact noisy masses.
    \item Do not include atom identifiers such as \texttt{[atom:<id>]} in the output.
    \item If the plan is low-support, noisy, or ambiguous, state uncertainty rather than adding unsupported content.
    \item Write one concise paragraph followed by at most three short bullets only if bullets improve clarity.
\end{enumerate}

Released differentially private summary plan:

\texttt{<PLAN\_ENTRIES>}

Output format: one concise paragraph, optionally followed by at most three short bullets. Do not include atom IDs or exact noisy masses.

\end{promptbox}

\subsection{Non-private reference decoder prompt substitution}
\label{app:prompt-non-private-reference}

For offline utility evaluation, we also decode a non-private reference plan $P_{\mathrm{NP}}$. The non-private reference decoder uses the same audit and display prompt templates as the differentially private plan decoder, with two public substitutions. First, the opening sentence identifies the input as a non-private reference aggregate summary plan. Second, the plan block is headed by \texttt{Non-private reference summary plan:} and contains entries from the non-private top-$L$ plan rather than entries from the released differentially private plan. Thus, the audit-mode prompt begins with
\[
\texttt{You are given a non-private reference aggregate summary plan.}
\]
and its plan block is headed by 
\[
\texttt{Non-private reference summary plan:}
\]
followed by entries of the form
\[
\texttt{- [atom:}j\texttt{] label=<label> | description=<description> | bin=<bin> | noisy\_mass=<value>}.
\]
The display-mode prompt uses the same substitutions and omits atom identifiers in the requested output, as in the display decoder prompt above.

The resulting non-private reference summaries are used only for offline evaluation. They are not part of the released differentially private output and are not given to the decoder that produces the differentially private summary. The field name \texttt{noisy\_mass} is preserved only to use the same plan-entry schema; for the non-private reference plan, this value is the corresponding non-private top-$L$ mass.

\subsection{Non-private baseline prompts}
\label{app:prompt-openai-baselines}

For non-private LLM baselines, the model receives either sampled target records or heuristically redacted sampled target records. These baselines are not \texttt{DP-SPIN} outputs and are reported only as non-private utility references.

\begin{promptbox}{Raw-record baseline prompt}
\footnotesize

You are given a bounded sample of records.

Write a concise aggregate summary of recurring themes in the records.

Rules:
\begin{enumerate}[leftmargin=16pt]
    \item Summarize only aggregate themes.
    \item Do not quote records.
    \item Do not include names, identifiers, addresses, dates, or other
    record-specific details.
    \item Do not describe individual users, reviewers, businesses, or organizations.
    \item Use one concise paragraph, optionally followed by at most three short bullets.
\end{enumerate}

Records:

\texttt{<RECORD\_TEXT\_SAMPLE>}
\end{promptbox}

\begin{promptbox}{Heuristic-redaction baseline prompt}
\footnotesize
You are given a bounded sample of heuristically redacted records.

Write a concise aggregate summary of recurring themes in the records.

Rules:
\begin{enumerate}[leftmargin=16pt]
    \item Summarize only aggregate themes.
    \item Do not quote records.
    \item Do not include names, identifiers, addresses, dates, or other record-specific details.
    \item Do not describe individual users, reviewers, businesses, or organizations.
    \item Use one concise paragraph, optionally followed by at most three short bullets.
\end{enumerate}

Redacted records:

\texttt{<REDACTED\_RECORD\_TEXT\_SAMPLE>}
\end{promptbox}

\subsection{URANIA-style public-keyword baseline prompt}
\label{app:prompt-urania-style}

The URANIA-style public-keyword baseline generates cluster summaries from released cluster-level keyword information. The decoder receives selected public keywords and noisy cluster-size information, but not raw protected records.

\begin{promptbox}{URANIA-style cluster-summary prompt}
\footnotesize

You are given a released cluster description consisting of selected public keywords and a noisy cluster size.

Write a concise topic name and a brief aggregate description using only the provided keywords and the noisy cluster size. Rules:
\begin{enumerate}[leftmargin=16pt]
    \item Use only the provided keywords and the noisy cluster size.\vspace{-3pt}
    \item Do not invent examples, names, organizations, dates, quotes, locations, or record-specific details.\vspace{-3pt}
    \item Do not claim to have read raw records.\vspace{-3pt}
    \item Do not include XML tags or Markdown headings.\vspace{-3pt}
    \item Use the format \texttt{Topic: <short topic name>} followed by \texttt{Description: <one or two sentences>}.
\end{enumerate}

Noisy cluster size:

\texttt{<NOISY\_CLUSTER\_SIZE>}

Selected keywords:

\texttt{<KEYWORDS>}
\end{promptbox}

\subsection{LLM judge prompt}
\label{app:prompt-openai-judge}

The LLM judge is used only for offline evaluation. Its outputs do not affect the released differentially private summary. The judge is evaluated at temperature zero. 




\begin{promptbox}{OpenAI judge prompt}
\footnotesize

You are evaluating a generated aggregate insight summary.
You must judge the summary only with respect to the released object shown below. Do not assume access to the protected corpus. Do not reward or penalize a method for information that is absent from its released object. Do not treat this as a formal differential-privacy audit.

Return strict valid JSON with integer scores from 1 to 5 for exactly these
keys:

\begin{itemize}[leftmargin=16pt]
    \item \texttt{coverage}: Does the summary cover the main high-salience entries in the released object?\vspace{-3pt}

    \item \texttt{specificity}: Does the summary give concrete theme-level detail permitted by the released object, rather than only generic wording or a list of isolated keywords?\vspace{-3pt}

    \item \texttt{insightfulness}: Does the summary synthesize the released entries into useful collection-level insight, beyond merely repeating labels or counts?\vspace{-3pt}

    \item \texttt{faithfulness}: Does the summary avoid adding themes, examples, causes, entities, comparisons, or record-level details unsupported by the released object?\vspace{-3pt}

    \item \texttt{privacy\_safety}: Does the generated text avoid individual-level, identifying, rare-record, raw-example, or quote-like details? This is a text-quality score, not a formal DP guarantee.\vspace{-3pt}

    \item \texttt{clarity}: Is the summary readable, concise, and organized?
\end{itemize}

Scoring calibration:

\begin{enumerate}[leftmargin=16pt]
    \item A summary that faithfully repeats a short keyword histogram may receive high \texttt{faithfulness} and \texttt{privacy\_safety}, but should receive high \texttt{specificity} or \texttt{insightfulness} only if it turns released keywords into clearly supported aggregate themes.\vspace{-3pt}

    \item A public category histogram can be highly useful when the public categories are themselves semantically informative. Do not automatically penalize it for being a histogram.\vspace{-3pt}

    \item A \texttt{DP-SPIN} semantic plan can receive high  \texttt{specificity} or \texttt{insightfulness} only when the generated summary uses the admitted atom labels, descriptions, and released bins coherently.\vspace{-3pt}

    \item Penalize any method for unsupported causal claims, invented examples, exact private details, or unsupported comparisons.\vspace{-3pt}

    \item Return only valid JSON with the six integer score keys and a short rationale string.
\end{enumerate}

Dataset:

\texttt{<DATASET>}

Method:

\texttt{<METHOD>}

Decoder input:

\texttt{<DECODER\_INPUT>}

Released object type:

\texttt{<RELEASED\_OBJECT\_TYPE>}

Mechanism:

\texttt{<MECHANISM>}

Epsilon:

\texttt{<EPSILON>}

Protected unit:

\texttt{<PROTECTED\_UNIT>}

Protected units:

\texttt{<N\_PROTECTED\_UNITS>}

Released object:

\texttt{<RELEASED\_OBJECT>}

Generated summary:

\texttt{<SUMMARY>}

\end{promptbox}

\clearpage

\appsection{Controlled Probe Atom Insertion Experiment}
\label{app:controlled-probe-atom-experiment}

In this appendix, we describe a controlled probe-atom insertion experiment. The experiment measures whether a synthetic theme with known support is admitted into the released differentially private semantic plan. It is used only for evaluation. The inserted probe records are not used to construct the ordinary atom dictionary, fit the assignment representation, or define the main utility benchmark.

\subsection{Objective}

The experiment measures the admission probability of a known public probe atom as a function of its support and the privacy parameters. Starting from a fixed base corpus, we construct protected target corpora by adding controlled numbers of synthetic probe records. We then run the record-level \texttt{DP-SPIN} release mechanism and record whether the probe atom is admitted into the released plan. The experiment asks:
\begin{center}
\textit{At probe support $s$, how often is the probe atom admitted into
$\widetilde S$?}
\end{center}
This measures the response of the private atom-admission procedure to a controlled synthetic theme.

\subsection{Probe record construction}

Let $D_0$ be a fixed base corpus of size $n_0$. For a support value $s$, we construct a probe-augmented target corpus
\begin{equation}
D_s = D_0 \cup \{x^{\mathrm{probe}}_1,\ldots,x^{\mathrm{probe}}_s\},   
\end{equation}
where each $x^{\mathrm{probe}}_r$ is a synthetic record expressing the same controlled probe theme. The probe records are inserted only into the protected target corpus used by the differentially private release. They are not used to fit the assignment representation or to construct atom labels.

The probe atom is added explicitly to the public atom dictionary before the release. Therefore, the experiment evaluates differentially private admission of a known public coordinate, not data-dependent discovery of a new label from protected records.

\begin{examplebox}{Example probe setup}
\footnotesize

Suppose the base corpus contains $n_0=2000$ target records. For support $s=5$, we create $|D_s| = 2005$. The public atom dictionary contains all ordinary public atoms plus one probe atom, for example $\texttt{synthetic probe billing dispute}$. The five inserted probe records are constructed so that their bounded assignment vectors place mass on this probe atom.
\end{examplebox}

\subsection{Per-support configuration}

For each support value $s$, our experiment constructs a separate probe-augmented target file containing exactly $n_0+s$ records. The run configuration is set to
\begin{equation}
n_{\mathrm{records}} = n_0+s.
\end{equation}
This avoids a sample-size mismatch in which a run requests more records than are available in the corresponding probe-augmented target file.
The public denominator used for support-bin computation is set to the same corpus size:
\begin{equation}
d_{\mathrm{public}} = n_0+s.
\end{equation}
Thus, a noisy probe mass $\widetilde c_{\mathrm{probe}}$ is converted to displayed semantic support by
\begin{equation}
\widetilde q_{\mathrm{probe}} = \frac{ \max \{ 0, \widetilde c_{\mathrm{probe}}\} }{n_0+s}.
\end{equation}
Note that this quantity is normalized assignment support for a known public probe coordinate, not private discovery of a new label and not a population prevalence estimate.

\begin{protocolbox}{Probe denominator policy}
\footnotesize
For the probe experiment, each support value $s$ uses
\begin{equation}
n_{\mathrm{records}} = n_0+s, \qquad d_{\mathrm{public}} = n_0+s. \nonumber 
\end{equation}
The denominator is therefore the configured public size of the probe-augmented protected target corpus for that run.
\end{protocolbox}

This choice keeps support bins tied to the corpus used in each run. It also prevents support values with different corpus sizes from being evaluated with a shared denominator.

\subsection{Probe release and admission indicator}

For each support value $s$, \texttt{DP-SPIN} computes bounded record-to-atom assignment vectors $v_i=\phi(x_i)$, with $\|v_i\|_1\le 1$, aggregates them into the non-private sketch $C(D_s)=\sum_{x_i\in D_s}v_i$, and releases a differentially private semantic plan. Let $\widetilde S_s$ denote the admitted atom set. The probe admission indicator for one run is $\mathsf{Adm}(s)=\mathbf 1\{a_{\mathrm{probe}}\in \widetilde S_s\}$. Across $R$ repeated runs, we report the probe admission rate
\begin{equation}
\widehat{\pi}_{\mathrm{adm}}(s) = \frac{1}{R} \sum_{r=1}^R
\mathbf 1\{a_{\mathrm{probe}}\in\widetilde S_s^{(r)}\}.
\end{equation}

\subsection{Interpretation}

The probe experiment evaluates atom admission under the full differentially private release pipeline. For fixed assignment and release parameters, increasing the probe support $s$ increases the non-private mass assigned to the probe atom. Larger noise, as required by a smaller privacy budget, can reduce the probability that the probe atom is admitted. Admission remains random because the admitted set is produced by a randomized differentially private release.

Probe admission does not imply that raw probe text is released. The decoder receives only the released atom label, noisy mass or support bin, and public decoding rules. The raw probe records are not provided to the decoder. Conversely, non-admission does not imply that the probe has zero non-private support. It only means that the probe atom was not selected into the released differentially private plan in that run.

\begin{protocolbox}{Probe experiment summary}
\footnotesize
\begin{enumerate}[leftmargin=16pt]
    \item Fix a base protected corpus $D_0$ of size $n_0$.
    \item Add $s$ synthetic probe records to obtain $D_s$.
    \item Add the probe atom to the public atom dictionary.
    \item Set $n_{\mathrm{records}}=n_0+s$.
    \item Set the semantic-support denominator to $d_{\mathrm{public}}=n_0+s$.
    \item Run the record-level \texttt{DP-SPIN} release mechanism.
    \item Record whether the probe atom is admitted into the released plan.
\end{enumerate}
\end{protocolbox}

\clearpage

\appsection{Our URANIA-Style Public-Keyword Baseline}
\label{app:urania-baseline}

In this appendix, we describe our independently implemented \texttt{URANIA-PublicKeywords} baseline. The baseline preserves the main decoder-facing structure of URANIA---differentially private clustering, noisy cluster-size filtering, differentially private per-cluster keyword histograms, and language generation from the resulting private cluster--keyword object---but it is not an exact reproduction of the original URANIA experiments. We use a keyword universe and text representation fixed from public or disjoint auxiliary data so that the comparison with \texttt{DP-SPIN} uses the same public-vocabulary boundary.

\subsection{Scope of our implementation}

The baseline receives
\begin{itemize}[leftmargin=16pt]
    \item protected target records $D_{\mathrm{target}}=\{x_1,\ldots,x_n\}$;
    \item disjoint auxiliary records and public atom text;
    \item a fixed public keyword vocabulary $\mathcal K=\{w_1,\ldots,w_q\}$;
    \item a fixed text-representation map $g$;
    \item the public numbers of clusters $k$ and released keywords per preserved cluster $t$;
    \item a public noisy-size threshold $\tau$; and
    \item privacy parameters $(\varepsilon,\delta)$ with budgets $\varepsilon_{\mathrm{clust}}$, $\varepsilon_{\mathrm{size}}$, and $\varepsilon_{\mathrm{hist}}$ satisfying
    \begin{equation}
    \varepsilon_{\mathrm{clust}}+\varepsilon_{\mathrm{size}}+\varepsilon_{\mathrm{hist}}=\varepsilon,
    \qquad \delta_{\mathrm{clust}}=\delta.
    \end{equation}
\end{itemize}

In the reported configuration, $k=10$, $t=10$, and
\begin{equation}
\varepsilon_{\mathrm{clust}}=0.45\varepsilon,
\qquad \varepsilon_{\mathrm{size}}=0.15\varepsilon,
\qquad \varepsilon_{\mathrm{hist}}=0.40\varepsilon.
\end{equation}
The public keyword vocabulary is constructed from public atom labels and descriptions before the protected target split is processed. Specifically, our implementation extracts non-stopword alphabetic or hyphenated tokens of length at least four and preserves the $q\leq 200$ most frequent tokens under deterministic tie-breaking. This construction is specific to our matched baseline and is not the \textsc{URANIA} \textsc{KwSet-Public} procedure.

\subsection{Text representation}

For each protected record $x_i$, our \textsc{URANIA-PublicKeywords} implementation computes a vector representation
\begin{equation}
e_i = g(x_i).
\end{equation}
In our experiments, $g$ uses a lowercase TF--IDF vectorizer with English stop-word removal, unigram and bigram features, and at most $50{,}000$ features. The vectorizer is fitted on public atom text and disjoint auxiliary records. If the fitted TF--IDF dimension exceeds $64$, truncated SVD with output dimension $64$ is fitted on the same public-or-auxiliary text. The transformed vectors are then $\ell_2$-normalized, so $\|e_i\|_2\leq 1$. All parameters of $g$, including the vocabulary, inverse-document-frequency weights, and SVD components, are fixed before $D_{\mathrm{target}}$ is processed.

Note that our representation differs from the original \textsc{URANIA} experiments, in which an LLM first maps each conversation to a structured record summary and the summary field is embedded using the pretrained \texttt{all-mpnet-base-v2} sentence encoder. Our TF--IDF--SVD map is an implementation-specific choice for a reproducible matched public-or-auxiliary representation protocol.

\subsection{Differentially private clustering}

Our implementation uses a Gaussian noisy Lloyd procedure rather than the Google DP-KMeans implementation used in the original \textsc{URANIA} experiments. Initial centers are obtained by non-private $k$-means on the public-or-auxiliary representations and are therefore fixed independently of $D_{\mathrm{target}}$.

Let $R_{\mathrm{clust}}=3$ denote the number of Lloyd iterations. At iteration $r$, each protected representation is assigned to its nearest current center. The curator forms the cluster sums and counts
\begin{equation}
S_j^{(r)}=\sum_{i:a_r(i)=j}e_i,
\qquad N_j^{(r)}=\left|\{i:a_r(i)=j\}\right|.
\end{equation}
Because each $e_i$ has $\ell_2$ norm at most one and contributes to exactly one cluster, the flattened sum vector $(S_1^{(r)},\ldots,S_k^{(r)})$ and the count vector $(N_1^{(r)},\ldots,N_k^{(r)})$ each have add/drop $\ell_2$-sensitivity at most one. Our implementation applies analytically calibrated Gaussian mechanisms to these two vectors. The per-iteration budget is $(\varepsilon_{\mathrm{clust}}/R_{\mathrm{clust}},\delta_{\mathrm{clust}}/R_{\mathrm{clust}})$, with $75\%$ assigned to the sum release and $25\%$ to the count release. The next center is obtained by public post-processing of the noisy statistics and is subsequently $\ell_2$-normalized. Basic adaptive composition over the three iterations gives $(\varepsilon_{\mathrm{clust}},\delta_{\mathrm{clust}})$-DP for the released final centers $\widetilde c_1,\ldots,\widetilde c_k$. Each protected record is then assigned internally to its nearest released center,
\begin{equation}
a(i)=\arg\min_{j\in[k]}\|e_i-\widetilde c_j\|_2,
\end{equation}
using deterministic tie-breaking. These record-to-cluster assignments are not released.

\subsection{Noisy cluster-size filtering}

For each cluster $j$, define the internal cluster size
\begin{equation}
n_j=\left|\{i:a(i)=j\}\right|, \qquad \boldsymbol n=(n_1,\ldots,n_k).
\end{equation}
Since each protected record contributes to exactly one coordinate, $\boldsymbol n$ has add/drop $\ell_1$-sensitivity one. Our implementation applies the private histogram mechanism
\begin{equation}
\widetilde{\boldsymbol n} = \operatorname{PHR}_{1,\varepsilon_{\mathrm{size}}}(\boldsymbol n),
\end{equation}
which adds independent symmetric discrete-Laplace noise,
\begin{equation}
\widetilde n_j=n_j+\eta_j, \qquad
\Pr(\eta_j=z) = \frac{1-e^{-\varepsilon_{\mathrm{size}}}}{1+e^{-\varepsilon_{\mathrm{size}}}} e^{-\varepsilon_{\mathrm{size}}|z|}, \qquad z\in\mathbb Z.
\end{equation}
A cluster is preserved only if $\widetilde n_j\geq\tau$. The threshold $\tau$ is fixed independently of the protected target records. In the reported runs, the public analysis size is $N_0=2000$, $k=10$, and we fix $\tau=50$ before processing the target records. Clusters that fail the noisy-size test produce no released keyword list or summary.

\subsection{Per-record public-keyword indicators}

For each protected record $x_i$, our implementation constructs a set $h_i\subseteq\mathcal K$ containing at most $b_{\mathrm{kw}}=5$ public keywords. In the reported configuration, keyword selection is deterministic lexical matching: our implementation counts case-insensitive, boundary-delimited occurrences of every public keyword in the record text, orders matching keywords by decreasing occurrence count with deterministic public tie-breaking, and preserves at most five. This differs from the original \textsc{URANIA} experiments, which use an LLM to select up to five keywords from the supplied keyword set.
All records assigned to a cluster are used in the reported keyword histogram. For $j\in[k]$ and $w\in\mathcal K$, define
\begin{equation}
r_{j,w} = \sum_{i=1}^{n} \mathbf 1\{a(i)=j\}\mathbf 1\{w\in h_i\}.
\end{equation}
Each protected record contributes to at most $b_{\mathrm{kw}}$ coordinates of the flattened cluster--keyword table $(r_{j,w})_{j,w}$.

\subsection{Differentially private keyword histogram release}

Our implementation applies
\begin{equation}
(\widetilde r_{j,w})_{j,w} = \operatorname{PHR}_{b_{\mathrm{kw}},\varepsilon_{\mathrm{hist}}} \bigl((r_{j,w})_{j,w}\bigr).
\end{equation}
Equivalently, it adds independent discrete-Laplace noise
\begin{equation}
\widetilde r_{j,w}=r_{j,w}+\xi_{j,w},
\qquad \Pr(\xi_{j,w}=z) \propto \exp \left(-\frac{\varepsilon_{\mathrm{hist}}}{b_{\mathrm{kw}}}|z|\right),
\qquad z\in\mathbb Z.
\end{equation}
For each cluster passing the noisy-size test, our implementation selects at most $t$ keywords with the largest positive noisy counts,
\begin{equation}
\widetilde K_j = \operatorname{Top}^{+}_{t} \bigl(\{\widetilde r_{j,w}:w\in\mathcal K\}\bigr),
\end{equation}
using deterministic public tie-breaking. This selection is post-processing of the private keyword histogram. The same rule is applied to every preserved cluster; the released keyword list does not branch on the non-private true cluster size.

\subsection{Cluster-level and aggregate verbalization}

For each preserved cluster, the first decoder call receives only its selected public keywords and released noisy size. It is instructed to produce a concise keyword-supported topic label and description, not to repeat the numerical noisy size, and not to introduce record-specific details or claims unsupported by the released keywords. The preserved clusters are then ordered by released noisy size. A second decoder call forms the paper-facing aggregate summary from the selected keywords and cluster-level descriptions of at most the first four preserved clusters. Both calls use \texttt{gpt-4o-mini} in the reported experiments. The decoder does not receive raw protected records, target excerpts, non-private cluster sizes, non-private keyword counts, record-to-cluster assignments, or any keyword outside the fixed public vocabulary. All language generation is therefore post-processing of the released differentially private cluster--keyword object.

\begin{protocolbox}{Decoder-visible information in URANIA-PublicKeywords}
\footnotesize

For each preserved cluster, the decoder receives a noisy cluster size and a selected keyword list, for example
\[
\widetilde n_j = 185, \qquad
\{\texttt{credit report},\texttt{incorrect information}, \texttt{identity theft},\texttt{dispute},\texttt{investigation}\}.
\]
The decoder does not receive the records assigned to the cluster.
\end{protocolbox}

\subsection{Privacy accounting and release boundary}

The decoder-facing mechanism composes
\begin{equation}
(\varepsilon_{\mathrm{clust}},\delta_{\mathrm{clust}}) \quad\text{for final cluster centers},
\end{equation}
\begin{equation}
(\varepsilon_{\mathrm{size}},0) \quad\text{for the noisy cluster-size histogram},
\end{equation}
and
\begin{equation}
(\varepsilon_{\mathrm{hist}},0) \quad\text{for the flattened cluster--keyword histogram}.
\end{equation}
Thus, by adaptive sequential composition, the released centers, preserved-cluster decisions, selected public keywords, noisy sizes, cluster-level descriptions, and final aggregate summary are
\begin{equation}
\left( \varepsilon_{\mathrm{clust}} + \varepsilon_{\mathrm{size}} + \varepsilon_{\mathrm{hist}}, \delta_{\mathrm{clust}} \right)
= (\varepsilon,\delta)
\end{equation}
differentially private under record-level add/drop adjacency, conditional on the public-or-auxiliary representation map, fixed public keyword vocabulary, public hyperparameters, and the bounded per-record keyword contribution.
Non-private quantities, including true cluster sizes, exact sample-membership information, and statistics computed directly from unnoised keyword matches, are used only for internal evaluation and are not included in the released baseline artifact or provided to the decoder.

\subsection{Relationship to the original URANIA system}

Our baseline preserves the main released-object structure of \textsc{URANIA}: differentially private cluster centers, noisy cluster-size filtering, differentially private per-cluster keyword histograms, and language generation from selected released keywords. It differs from the original system in the following respects:
\begin{itemize}[leftmargin=16pt]
    \item the fixed vocabulary is extracted from the public \texttt{DP-SPIN}
    atom text and is not an implementation of \textsc{URANIA}'s
    \textsc{KwSet-TFIDF}, \textsc{KwSet-LLM}, \textsc{KwSet-Public}, or
    \textsc{KwSet-Hybrid} procedures;
    
    \item protected records are represented using a public-or-auxiliary
    TF--IDF--SVD map rather than LLM-generated record summaries encoded by
    \texttt{all-mpnet-base-v2};
    
    \item clustering is performed using our Gaussian noisy Lloyd
    implementation rather than the Google DP-KMeans implementation used by
    \textsc{URANIA};
    
    \item per-record keywords are selected by deterministic lexical matching
    rather than an LLM keyword-selection prompt;
    
    \item verbalization uses our OpenAI prompts and \texttt{gpt-4o-mini},
    rather than the LLM configuration used in the original experiments; and
    
    \item the implementation does not reproduce the original
    hierarchy-construction and visualization stages.
\end{itemize}
Our comparison evaluates this release structure under the matched fixed-vocabulary protocol rather than reproducing the original \textsc{URANIA} results.

\begin{table}[h]
\centering
\scriptsize
\renewcommand{\arraystretch}{1.15}
\rowcolors{2}{gray!9}{white}
\begin{tabular}{p{0.21\linewidth}p{0.36\linewidth}p{0.36\linewidth}}
\toprule
\rowcolor{white}
\textbf{Aspect} & \textbf{\texttt{DP-SPIN}} & \textbf{\texttt{URANIA-PublicKeywords}} \\
\midrule
Protected unit
& One text record; optionally one user after deterministic contribution clipping
& One target text record \\
Per-record contribution
& Bounded nonnegative assignment vector over public semantic atoms
& One norm-bounded representation for clustering and at most five indicators over a fixed public keyword vocabulary \\
Internal aggregate statistic
& Atom-mass sketch over an atom universe fixed from public or auxiliary data
& Cluster sums and counts, final cluster-size histogram, and a flattened cluster--keyword count table \\
Generation-visible semantic strings
& Admitted public atom labels and descriptions
& Selected terms from a fixed public keyword vocabulary \\
Data-dependent admitted objects
& Semantic atoms admitted from a private sketch or private selection mechanism
& Clusters preserved by noisy size and keywords selected from private per-cluster histograms \\
Released differentially private object
& Admitted atoms with noisy masses or support bins
& Retained clusters with noisy sizes and selected public keywords \\
LLM input
& Released atom labels, public descriptions, and noisy masses or support bins
& Selected public keywords and released noisy-size information for preserved clusters \\
LLM output
& Aggregate summary constrained to the released semantic plan
& Cluster-level descriptions followed by an aggregate summary over at most four preserved clusters \\
Main differentially private mechanisms
& Gaussian or Laplace full-sketch release with top-$L$ post-processing, or
sequential exponential-mechanism atom admission followed by noisy
selected-mass release
& Gaussian noisy Lloyd clustering, discrete-Laplace cluster-size release, and discrete-Laplace cluster--keyword histogram release \\
\bottomrule
\end{tabular}
\caption{Mechanism-level comparison between \texttt{DP-SPIN} and the \texttt{URANIA-PublicKeywords} baseline.}
\label{tab:dpspin-urania-detailed}
\end{table}

\clearpage
\appsection{Additional Experimental Results: CFPB Complaint}
\label{app:additional-results-cfpb}

In this appendix, we report additional experimental results on the \texttt{CFPB} consumer complaint narratives.

\begin{figure}[h]
    \centering
    \includegraphics[width=0.99\linewidth]{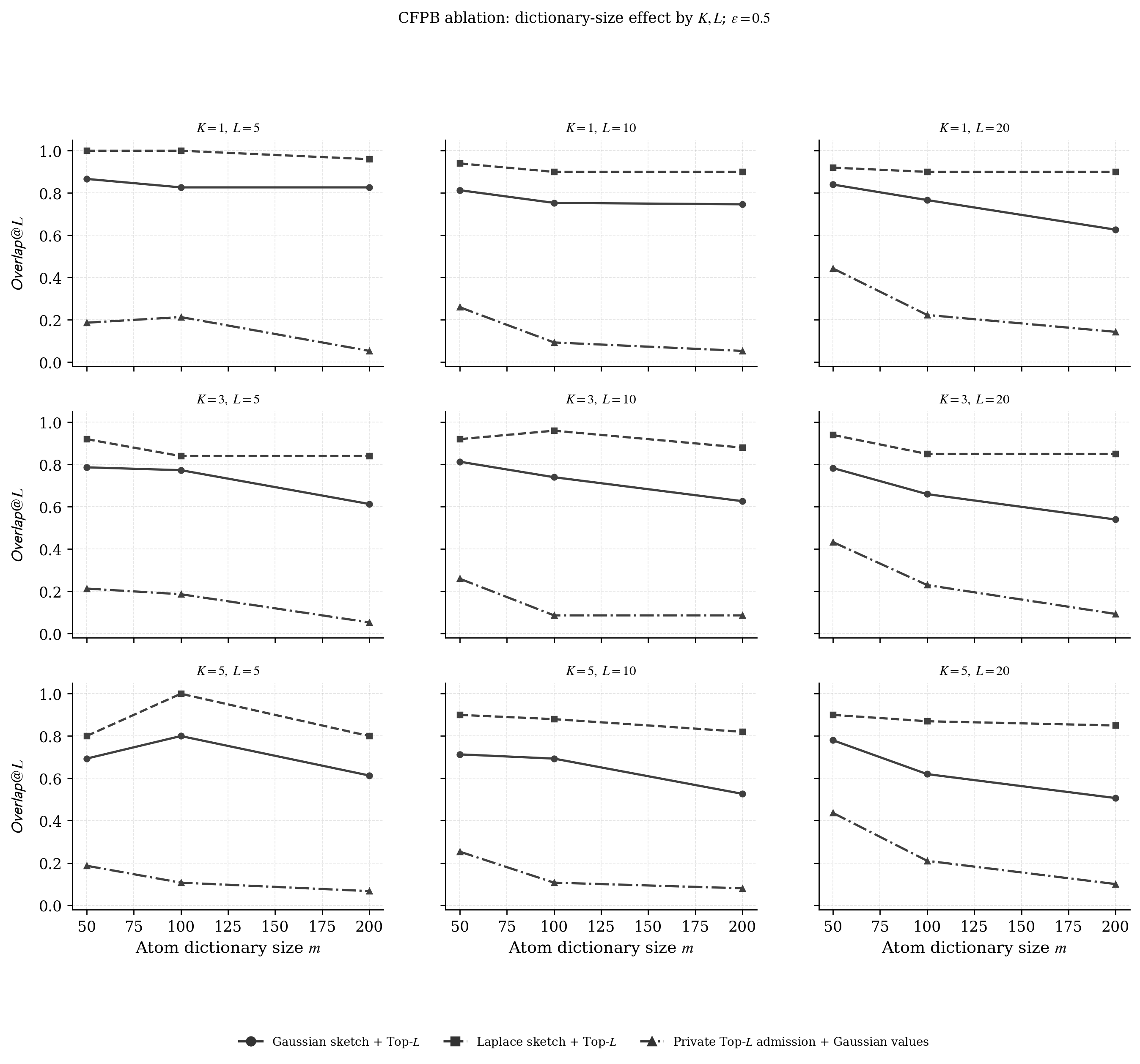}
    \caption{
    CFPB ablation of $\mathsf{Overlap}@L$ as the atom dictionary size $m$ varies at $\varepsilon=0.5$. Panels vary the assignment sparsity $K$ and released plan size $L$; curves compare the release mechanisms.
    }
    \label{fig:app-cfpb-ablation-overlap-atL-vs-m}
\end{figure}

\begin{figure}[h]
    \centering
    \includegraphics[width=0.99\linewidth]{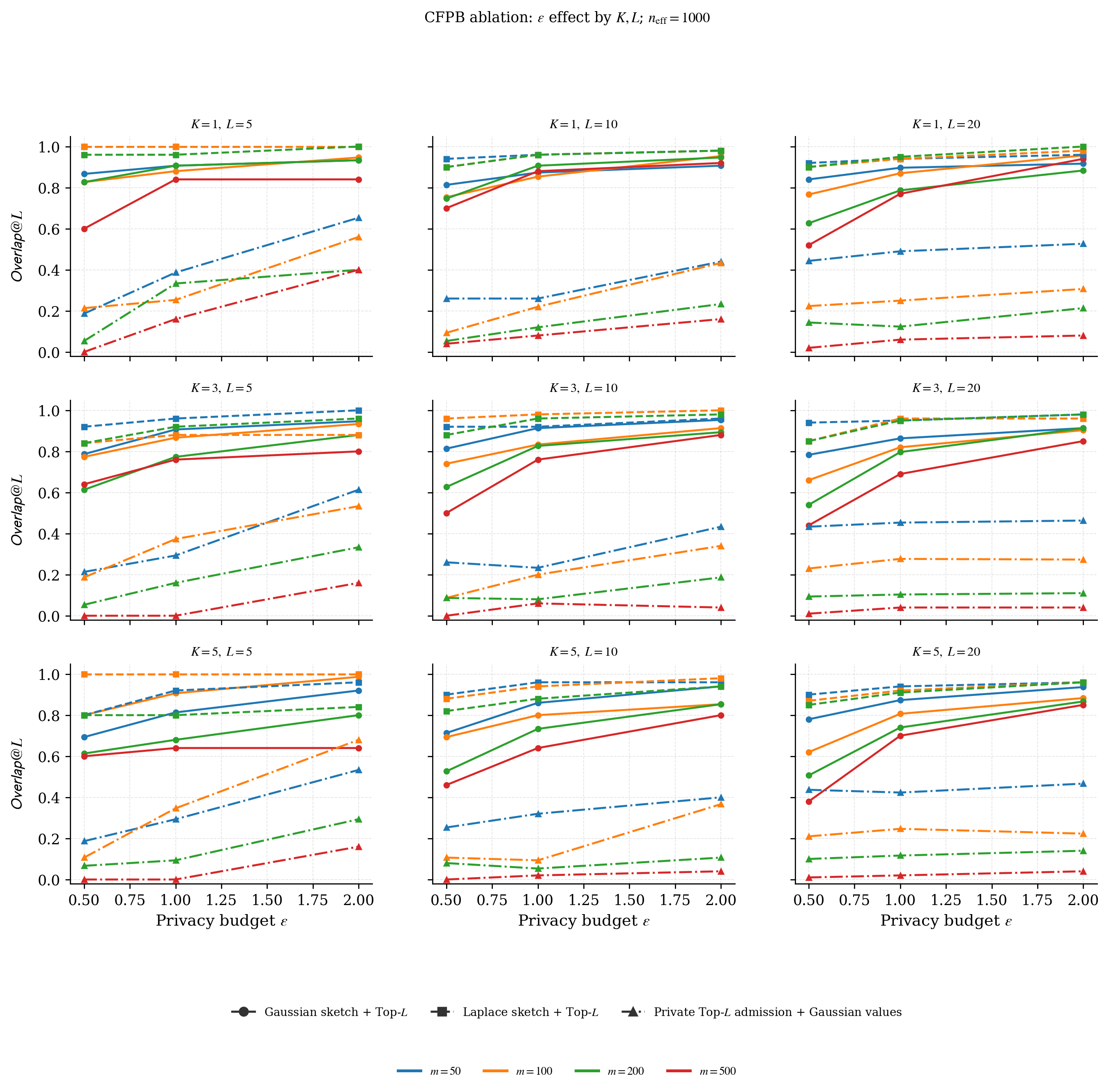}
    \caption{
    CFPB ablation of $\mathsf{Overlap}@L$ as the privacy budget $\varepsilon$ varies, with $n_{\mathrm{eff}}=1000$. Panels vary $K$ and $L$; colors indicate $m$, and line styles indicate the release mechanism.
    }
    \label{fig:app-cfpb-ablation-overlap-atL-vs-epsilon}
\end{figure}

\begin{figure}[h]
    \centering
    \includegraphics[width=0.99\linewidth]{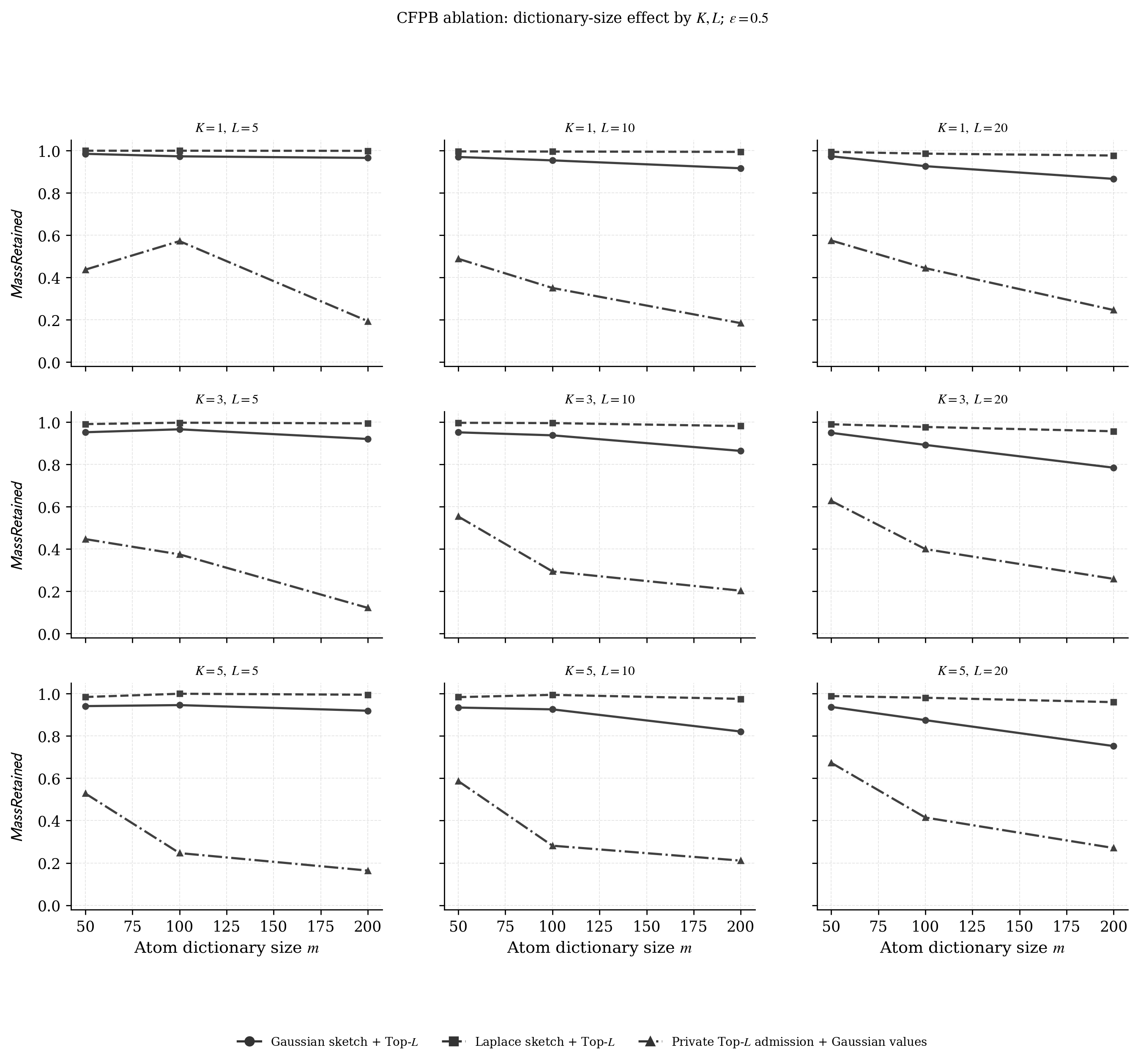}
    \caption{
    CFPB ablation of preserved non-private mass as the atom dictionary size $m$ varies at $\varepsilon=0.5$. Panels vary $K$ and $L$; curves compare the release mechanisms.
    }
    \label{fig:app-cfpb-ablation-preserved-mass-vs-m}
\end{figure}

\begin{figure}[h]
    \centering
    \includegraphics[width=0.99\linewidth]{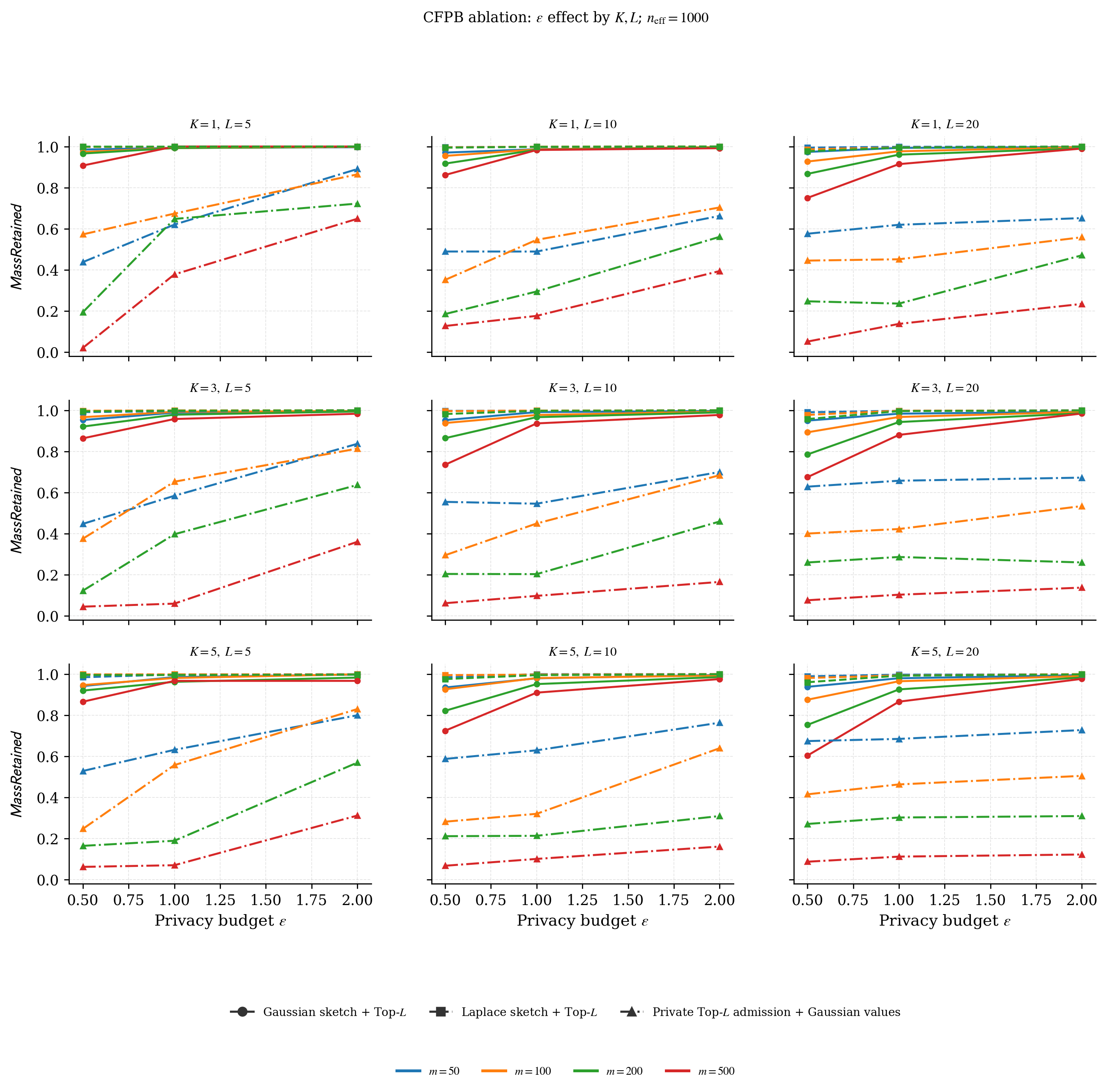}
    \caption{
    CFPB ablation of preserved non-private mass as the privacy budget $\varepsilon$ varies, with $n_{\mathrm{eff}}=1000$. Panels vary $K$ and $L$; colors indicate $m$, and line styles indicate the release mechanism.
    }
    \label{fig:app-cfpb-ablation-preserved-mass-vs-epsilon}
\end{figure}

\begin{figure}[h]
    \centering
    \includegraphics[width=0.99\linewidth]{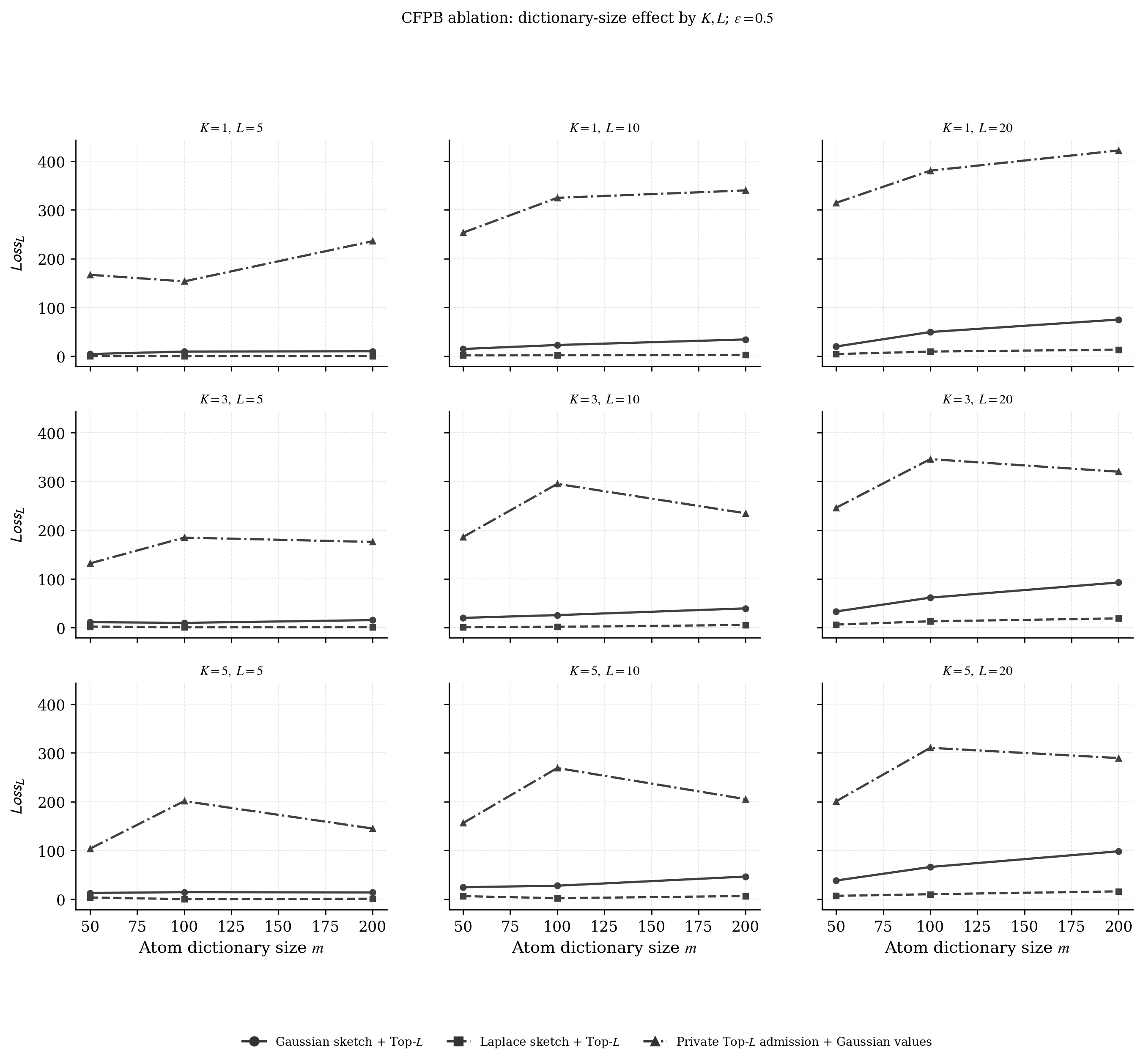}
    \caption{CFPB ablation of support-mass loss $\mathsf{Loss}_L$ as the atom dictionary size $m$ varies at $\varepsilon=0.5$. Panels vary $K$ and $L$; curves compare the release mechanisms. Lower values indicate less non-private semantic mass lost by the released atom set.}
    \label{fig:app-ablation-cfpb-support-loss-vs-m}
\end{figure}

\begin{figure}[h]
    \centering
    \includegraphics[width=0.99\linewidth]{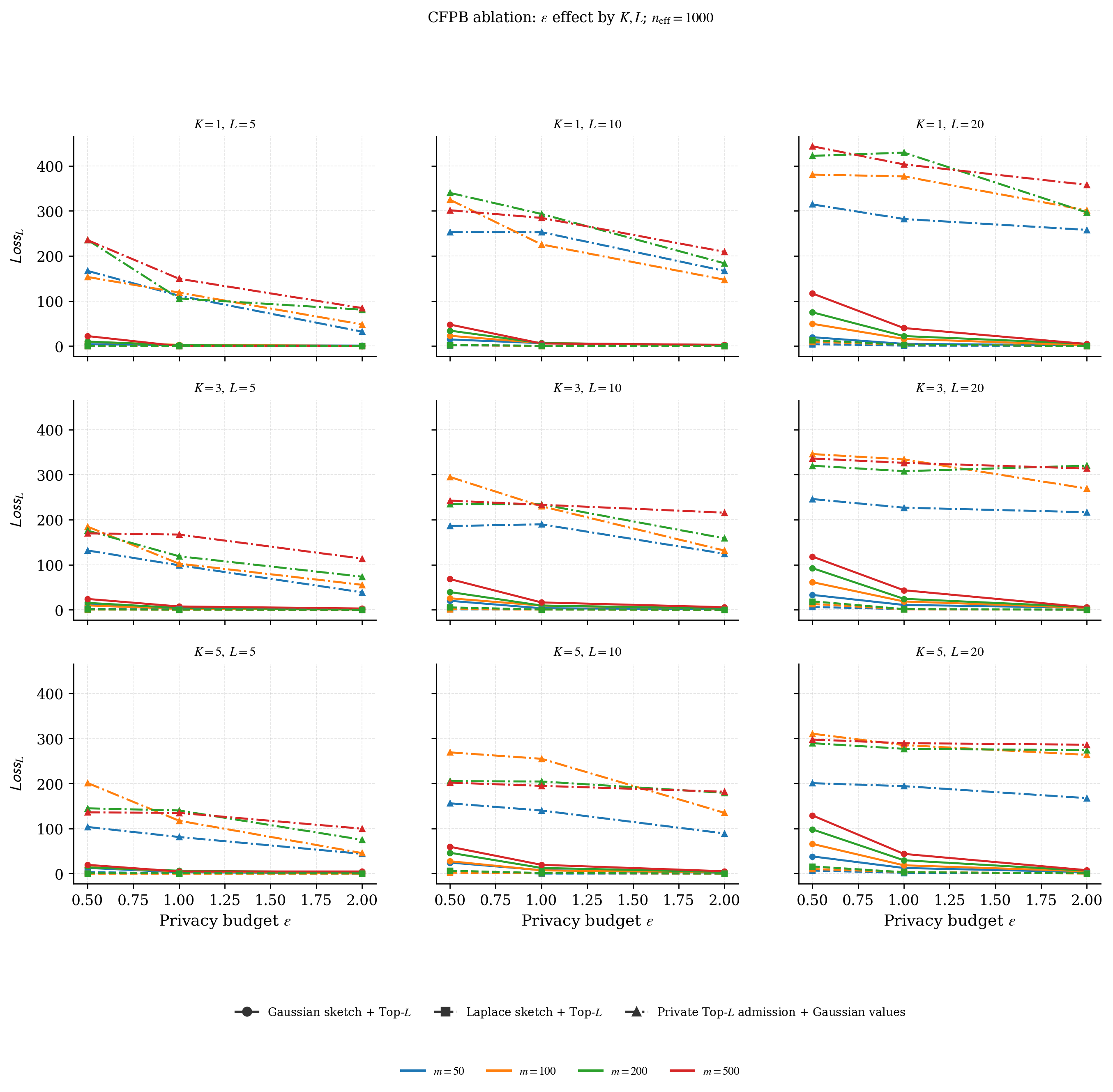}
    \caption{CFPB ablation of support-mass loss $\mathsf{Loss}_L$ as the privacy budget $\varepsilon$ varies, with $n_{\mathrm{eff}}=1000$. Panels vary $K$ and $L$; colors indicate $m$, and line styles indicate the release mechanism. Lower values indicate less non-private semantic mass lost by the released atom set.
    }
    \label{fig:app-cfpb-ablation-support-loss-vs-epsilon}
\end{figure}

\begin{figure}[h]
    \centering
    \includegraphics[width=0.99\linewidth]{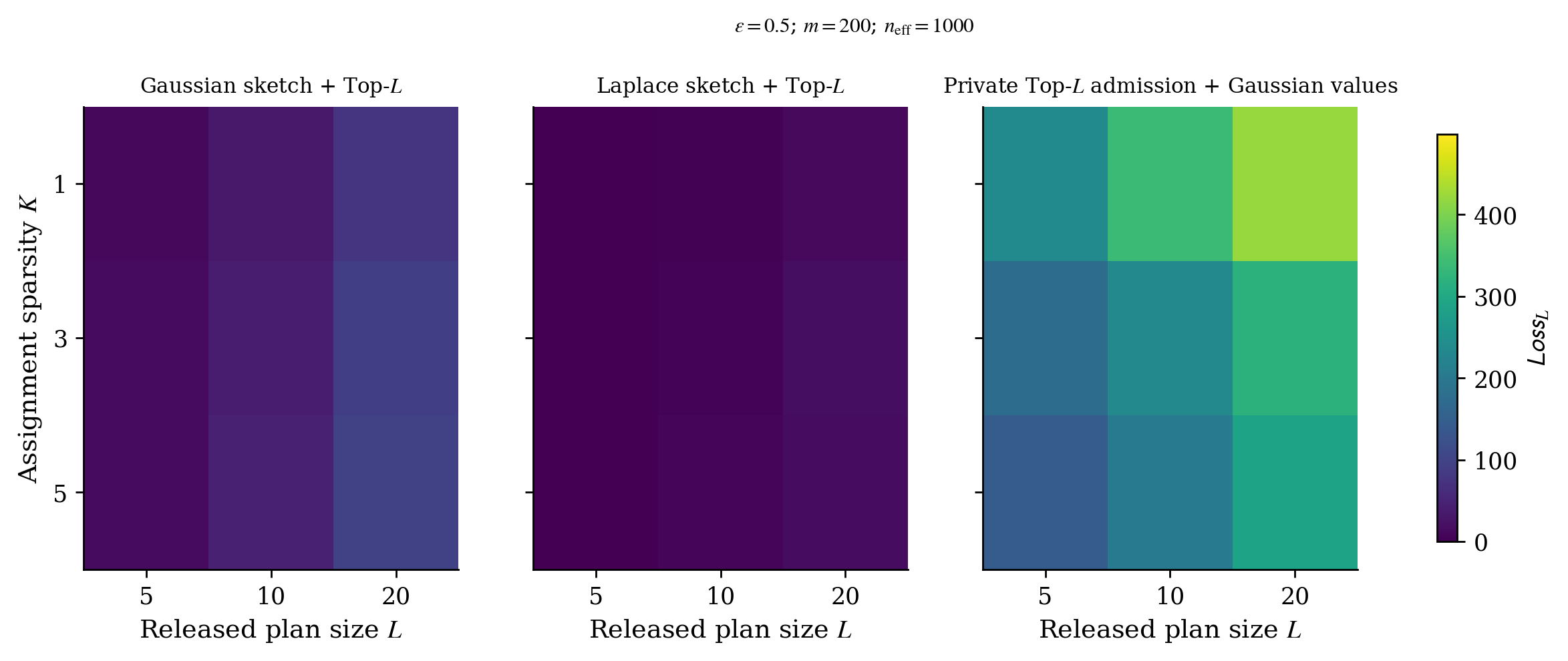}
    \caption{CFPB heatmap of support-mass loss $\mathsf{Loss}_L$ over assignment sparsity $K$ and released plan size $L$, with $\varepsilon=0.5$, $m=200$, and $n_{\mathrm{eff}}=1000$. Each panel corresponds to one release mechanism. Lower values indicate less non-private semantic mass lost by the released atom set.}
    \label{fig:app-cfpb-ablation-support-loss-heatmap-K-L}
\end{figure}


\begin{figure}[h]
    \centering
    \begin{minipage}{0.49\linewidth}
        \centering
        \includegraphics[width=\linewidth]{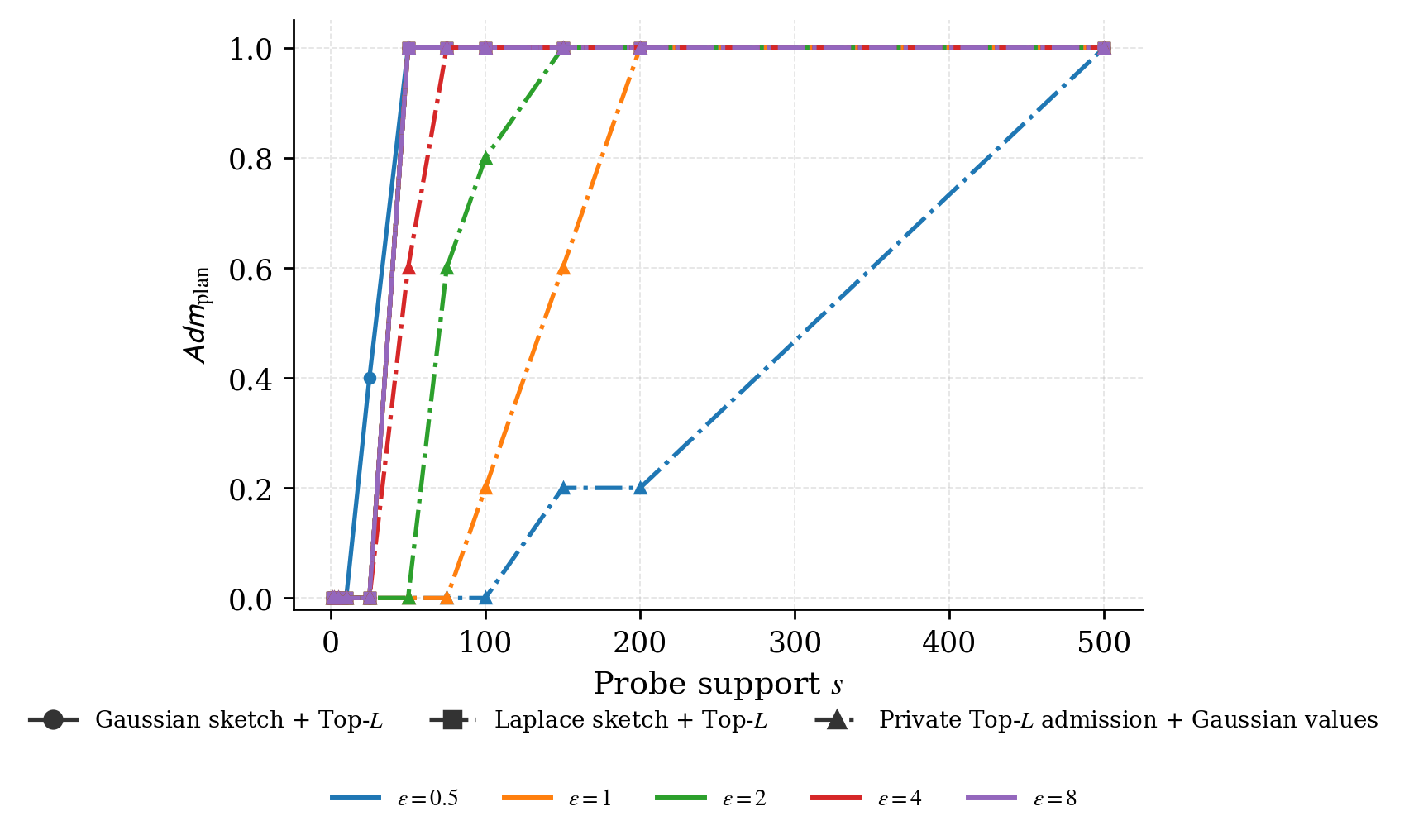}
        \vspace{2pt}
        \textbf{(a)} Plan admission.
    \end{minipage}
    \hfill
    \begin{minipage}{0.49\linewidth}
        \centering
        \includegraphics[width=\linewidth]{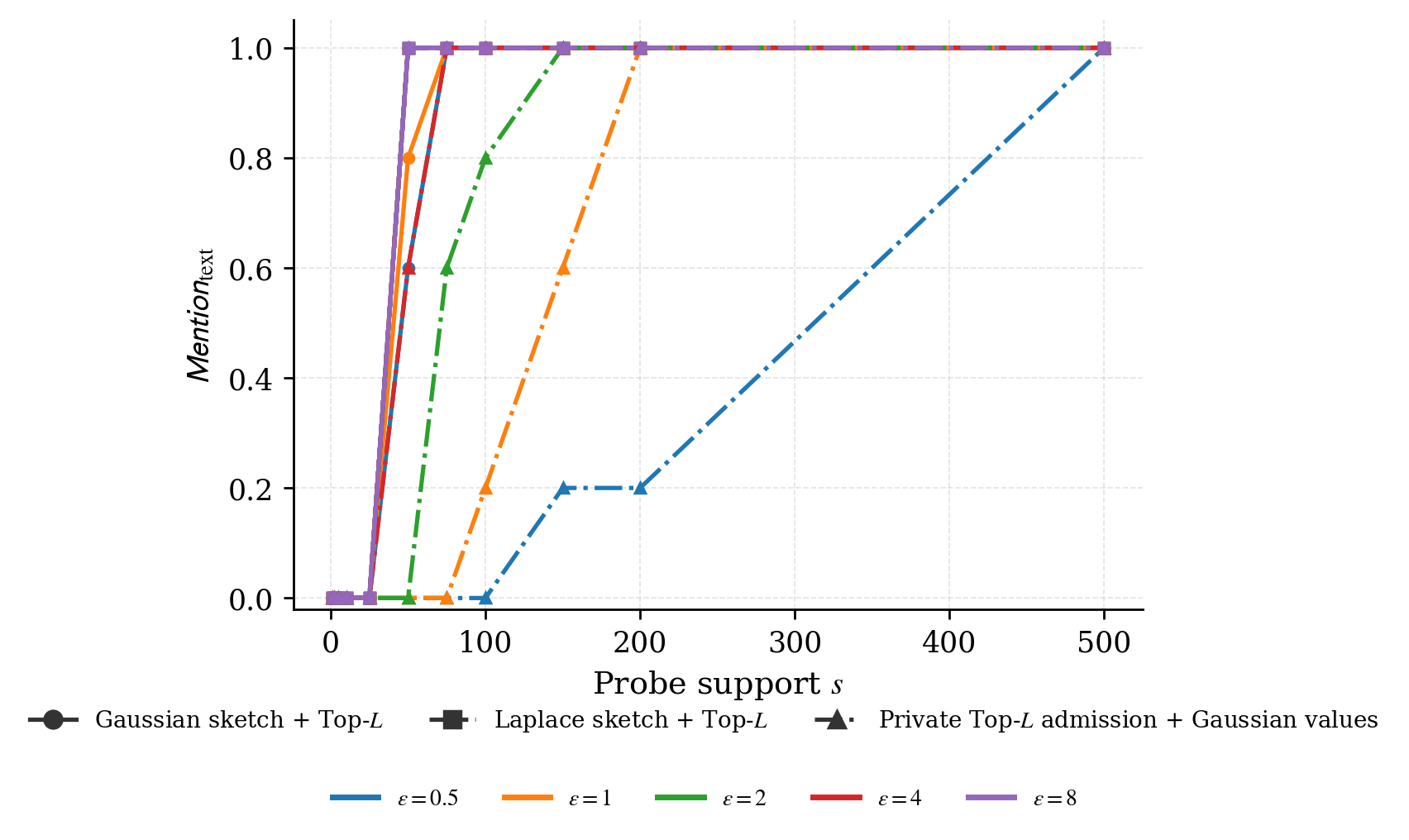}
        \vspace{2pt}
        \textbf{(b)} Summary mention.
    \end{minipage}
    \caption{CFPB controlled probe-atom experiment as the injected probe support $s$ varies, with $K=3$ and $L=10$. Panel~(a) reports the plan-admission rate $\mathsf{Adm}_{\mathrm{plan}}$, which measures whether the differentially
    private release admits the public probe atom into the released semantic plan. Panel~(b) reports the summary mention rate $\mathsf{Mention}_{\mathrm{text}}$, which measures whether the public probe string is detected in the decoded summary. Curves compare privacy budgets and release mechanisms under the configured probe-sweep setting.
    }
    \label{fig:app-cfpb-probe-admission-mention}
\end{figure}

\begin{figure}[h]
    \centering
    \includegraphics[width=0.5\linewidth]{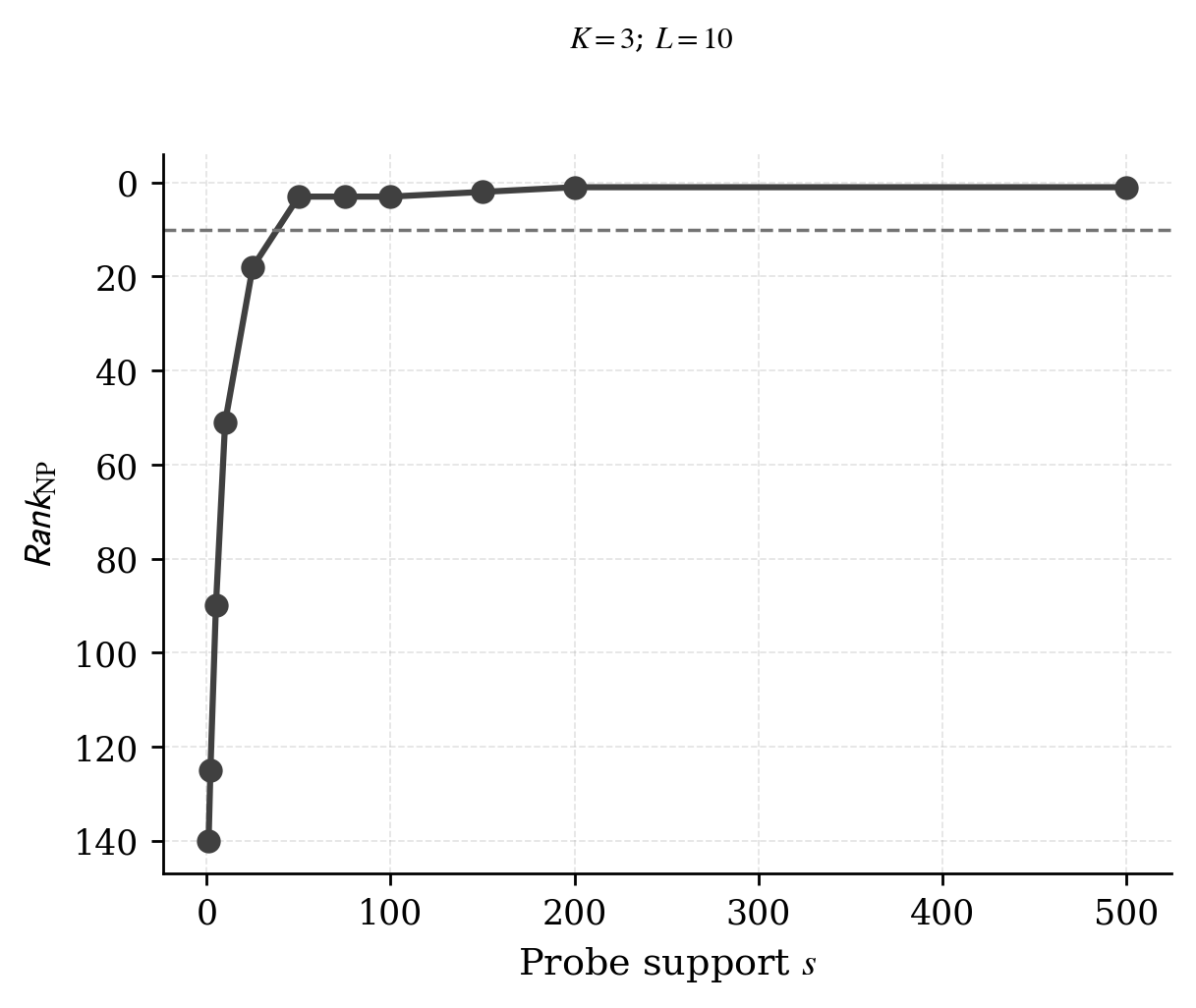}
    \caption{CFPB offline non-private rank of the controlled probe atom as the injected probe support $s$ varies, with $K=3$ and $L=10$. The rank $\mathsf{Rank}_{\mathrm{NP}}$ is computed from the non-private semantic sketch before the differentially private release and is used only for interpretation. Smaller rank values indicate that the probe atom is closer to the top of the non-private sketch. The dashed line marks the top-$L$ cutoff.
    }
    \label{fig:app-cfpb-probe-non-private-rank}
\end{figure}

\clearpage

\begin{table*}[t]
\centering
\caption{Release-conditioned OpenAI evaluation of selected generated summaries on CFPB consumer complaints. Each summary is evaluated only with respect to the object released to its decoder. Scores are assigned on a 1--5 scale and report verbalization quality. The judge privacy-safety column evaluates only the generated text and is not a formal differential-privacy guarantee. We report the grouped values as mean $\pm$ sample standard deviation across the summaries in each row.}
\label{tab:openai-baseline-comparison-cfpb}
\small
\resizebox{0.99\textwidth}{!}{%
\begin{tabular}{lllccccccccc}
\toprule
Method & Release object & Configuration & $\varepsilon$ & Protected units & $N$ & Coverage & Specificity & Insightfulness & Faithfulness & Text safety & Clarity \\
\midrule
DP-SPIN private plan & DP semantic plan & record; Gaussian sketch + top-$L$; m=200; K=3; L=10 & 1.0 & 2000 & 3 & 4.67 $\pm$ 0.58 & 4.00 $\pm$ 0.00 & 3.67 $\pm$ 0.58 & 4.67 $\pm$ 0.58 & 5.00 $\pm$ 0.00 & 4.67 $\pm$ 0.58 \\
DP keyword histogram & DP keyword histogram & fixed public keyword vocabulary; Laplace histogram; released top-10 keywords & 1.0 & 2000 & 3 & 5.00 $\pm$ 0.00 & 4.00 $\pm$ 0.00 & 3.00 $\pm$ 0.00 & 5.00 $\pm$ 0.00 & 5.00 $\pm$ 0.00 & 4.00 $\pm$ 0.00 \\
DP category histogram & DP category histogram & fixed Issue universe; Laplace histogram; released top-10 categories & 1.0 & 2000 & 3 & 5.00 $\pm$ 0.00 & 4.00 $\pm$ 0.00 & 4.00 $\pm$ 0.00 & 5.00 $\pm$ 0.00 & 5.00 $\pm$ 0.00 & 5.00 $\pm$ 0.00 \\
URANIA-style public keywords & DP public keywords & DP clustering; DP public-keyword histograms; clusters=10; keywords=10 & 1.0 & 2000 & 3 & 4.33 $\pm$ 0.58 & 3.33 $\pm$ 0.58 & 4.00 $\pm$ 0.00 & 3.67 $\pm$ 0.58 & 5.00 $\pm$ 0.00 & 4.33 $\pm$ 0.58 \\
\bottomrule
\end{tabular}%
}
\end{table*}

\begin{table*}[t]
\centering
\caption{Release-conditioned OpenAI evaluation of selected \texttt{DP-SPIN} summaries on CFPB consumer complaints. The decoder receives only the released differentially private semantic plan and public decoding instructions. Scores are assigned on the same 1--5 rubric used for the baseline-comparison table.}
\label{tab:openai-dpspin-cfpb}
\small
\resizebox{0.99\textwidth}{!}{%
\begin{tabular}{lllccccccccc}
\toprule
Unit & Mechanism & Configuration & $\varepsilon$ & Protected units & $N$ & Coverage & Specificity & Insightfulness & Faithfulness & Text safety & Clarity \\
\midrule
record & Gaussian sketch + top-$L$ & record; Gaussian sketch + top-$L$; m=200; K=3; L=10 & 1.0 & 2000 & 3 & 4.67 $\pm$ 0.58 & 4.00 $\pm$ 0.00 & 3.67 $\pm$ 0.58 & 4.67 $\pm$ 0.58 & 5.00 $\pm$ 0.00 & 4.67 $\pm$ 0.58 \\
record & Gaussian sketch + top-$L$ & record; Gaussian sketch + top-$L$; m=200; K=3; L=10 & 4.0 & 2000 & 3 & 4.67 $\pm$ 0.58 & 4.33 $\pm$ 0.58 & 4.00 $\pm$ 0.00 & 5.00 $\pm$ 0.00 & 5.00 $\pm$ 0.00 & 5.00 $\pm$ 0.00 \\
record & Private top-$L$ admission + Gaussian values & record; Private top-$L$ admission + Gaussian values; m=200; K=3; L=10 & 1.0 & 5000 & 3 & 5.00 $\pm$ 0.00 & 5.00 $\pm$ 0.00 & 4.00 $\pm$ 0.00 & 5.00 $\pm$ 0.00 & 5.00 $\pm$ 0.00 & 5.00 $\pm$ 0.00 \\
record & Private top-$L$ admission + Gaussian values & record; Private top-$L$ admission + Gaussian values; m=200; K=3; L=10 & 4.0 & 5000 & 3 & 5.00 $\pm$ 0.00 & 4.67 $\pm$ 0.58 & 4.00 $\pm$ 0.00 & 5.00 $\pm$ 0.00 & 5.00 $\pm$ 0.00 & 5.00 $\pm$ 0.00 \\
\bottomrule
\end{tabular}%
}
\end{table*}

\begin{table*}[t]
\centering
\caption{Similarity between OpenAI summaries generated from \texttt{DP-SPIN} private plans and summaries generated from the corresponding non-private top-$L$ plans on CFPB consumer complaints. Metrics compare generated texts and do not evaluate the underlying semantic sketches. Values are reported as mean $\pm$ sample standard deviation across runs.}
\label{tab:openai-private-non-private-cfpb-consumer-complaints}
\small
\resizebox{0.99\textwidth}{!}{%
\begin{tabular}{lllccccccc}
\toprule
Unit & Mechanism & Configuration & $\varepsilon$ & Protected units & Runs & $J_{\mathrm{tok}}$ & $J_2$ & $J_{\mathrm{kp}}$ & $\mathsf{Cos}_{\mathrm{tfidf}}$ \\
\midrule
record & Gaussian sketch + top-$L$ & record; Gaussian sketch + top-$L$; m=200; K=3; L=10 & 1.0 & 2000 & 3 & 0.49 $\pm$ 0.08 & 0.21 $\pm$ 0.05 & 0.23 $\pm$ 0.05 & 0.57 $\pm$ 0.04  \\
record & Gaussian sketch + top-$L$ & record; Gaussian sketch + top-$L$; m=200; K=3; L=10 & 4.0 & 2000 & 3 & 0.55 $\pm$ 0.08 & 0.22 $\pm$ 0.06 & 0.25 $\pm$ 0.06 & 0.60 $\pm$ 0.08 \\
record & Private top-$L$ admission + Gaussian values & record; Private top-$L$ admission + Gaussian values; m=200; K=3; L=10 & 1.0 & 5000 & 3 & 0.37 $\pm$ 0.10 & 0.21 $\pm$ 0.08 & 0.21 $\pm$ 0.08 & 0.54 $\pm$ 0.10  \\
record & Private top-$L$ admission + Gaussian values & record; Private top-$L$ admission + Gaussian values; m=200; K=3; L=10 & 4.0 & 5000 & 3 & 0.67 $\pm$ 0.13 & 0.40 $\pm$ 0.13 & 0.41 $\pm$ 0.14 & 0.73 $\pm$ 0.08  \\
\bottomrule
\end{tabular}%
}
\end{table*}

\begin{table*}[t]
\centering
\caption{Examples of OpenAI-generated aggregate summaries on CFPB consumer complaints. Each row reports the generated summary for one selected \texttt{DP-SPIN} or baseline release object.}
\label{tab:openai-qualitative-cfpb-consumer-complaints}
\scriptsize
\setlength{\tabcolsep}{3pt}
\renewcommand{\arraystretch}{1.04}
\resizebox{0.99\textwidth}{!}{%
\begin{tabular}{p{0.27\textwidth}p{0.68\textwidth}}
\toprule
Method and release configuration & Generated summary \\
\midrule
DP-SPIN private plan (record; Gaussian sketch + top-$L$; m=200; K=3; L=10; $\varepsilon=1$; $n=2000$) & The aggregate summary reveals moderate concerns regarding incorrect information on credit reports, with notable entries related to information that does not belong to the individual and issues with account status. Additionally, there are low-level reports concerning problems with dispute statements during credit reporting investigations, identity theft related to credit cards, and various debt collection issues, including attempts to collect debts that are not owed and insufficient information for debt verification. - Moderate support of incorrect information on credit reports. - Low support of issues related to credit reporting investigations and identity theft. - Low support of debt collection disputes and inaccuracies. \\\hline
\addlinespace[1pt]
DP keyword histogram (fixed public keyword vocabulary; vocab=10; Laplace histogram; released top-10 keywords; $\varepsilon=1$; $n=2000$) & The aggregate summary reveals various counts of keywords related to financial topics, with the highest support observed for "credit," followed by "account" and "report." Other notable keywords include "information," "reporting," and "debt," while "payment," "card," "loan," and "balance" show lower counts. - Highest keyword count: credit - Other significant keywords: account, report, information, reporting, debt - Lower support keywords: payment, card, loan, balance \\\hline
\addlinespace[1pt]
DP category histogram (fixed Issue universe; Laplace histogram; released top-10 categories; $\varepsilon=1$; $n=2000$) & The aggregate summary of reported issues reveals a range of concerns related to public category fields. The most frequently reported issue is incorrect information on reports, followed by improper use of reports. Other notable issues include problems with investigations by companies, attempts to collect debts not owed, and managing accounts. Additionally, there are reports of written notifications about debt and problems with purchases shown on statements. - Incorrect information on reports is the most prevalent issue. - Improper use of reports is also a significant concern. - Other issues include problems with investigations and debt collection attempts. \\\hline
\addlinespace[1pt]
URANIA-style public keywords (DP clustering; DP public-keyword histograms; clusters=10; keywords=10; $\varepsilon=1$; $n=2000$) & The preserved cluster entries highlight various aspects of financial services, particularly focusing on credit and loan information, prepaid transfer reporting, customer communication regarding credit features, and security in vehicle rewards. These clusters suggest a landscape where users are engaged in actions related to credit and loans, while also navigating challenges in communication and service features. The emphasis on consent and reporting indicates a potential concern for transparency and user awareness in financial transactions. - The clusters reflect a diverse range of topics within financial services, emphasizing user interactions and experiences. - There is a notable focus on communication challenges and the importance of consent in financial dealings. - Security and reporting mechanisms appear to be critical areas of interest, particularly in relation to rewards and prepaid services. \\
\addlinespace[1pt]
\bottomrule
\end{tabular}%
}
\end{table*}

\clearpage
\appsection{Additional Experimental Results: Amazon Reviews}
\label{app:additional-results-amazon}

In this appendix, we report additional experimental results on the \texttt{Amazon All Beauty} reviews.

\subsection{Record-level results}

\begin{figure}[h]
    \centering
    \includegraphics[width=0.99\linewidth]{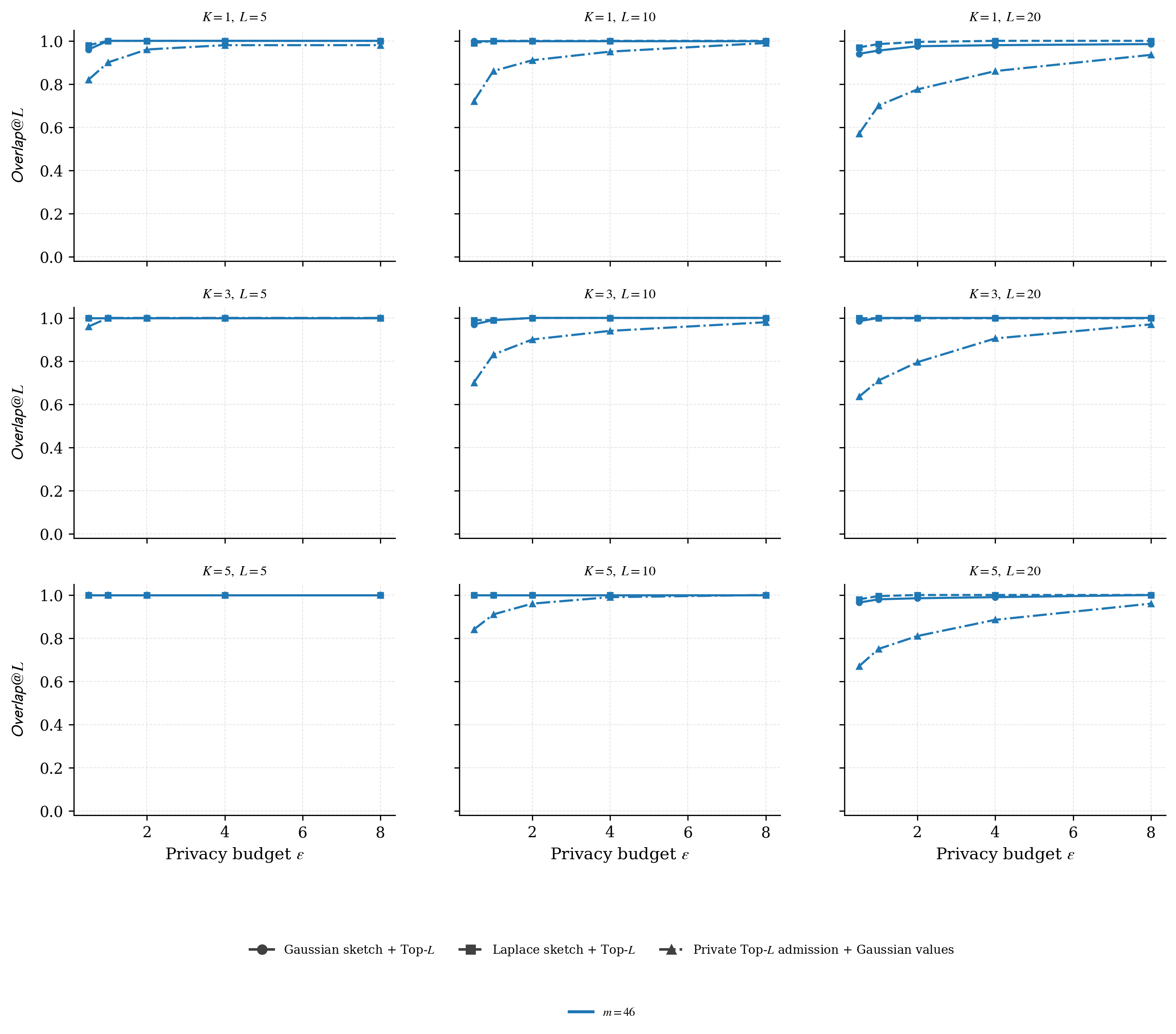}
    \caption{Amazon All Beauty record-level ablation of $\mathsf{Overlap}@L$ as the privacy budget $\varepsilon$ varies, with $m=46$. Panels vary the assignment sparsity $K$ and released plan size $L$; curves compare the release mechanisms.}
    \label{fig:amazon-record-overlap-vs-epsilon}
\end{figure}

\begin{figure}[h]
    \centering
    \includegraphics[width=0.99\linewidth]{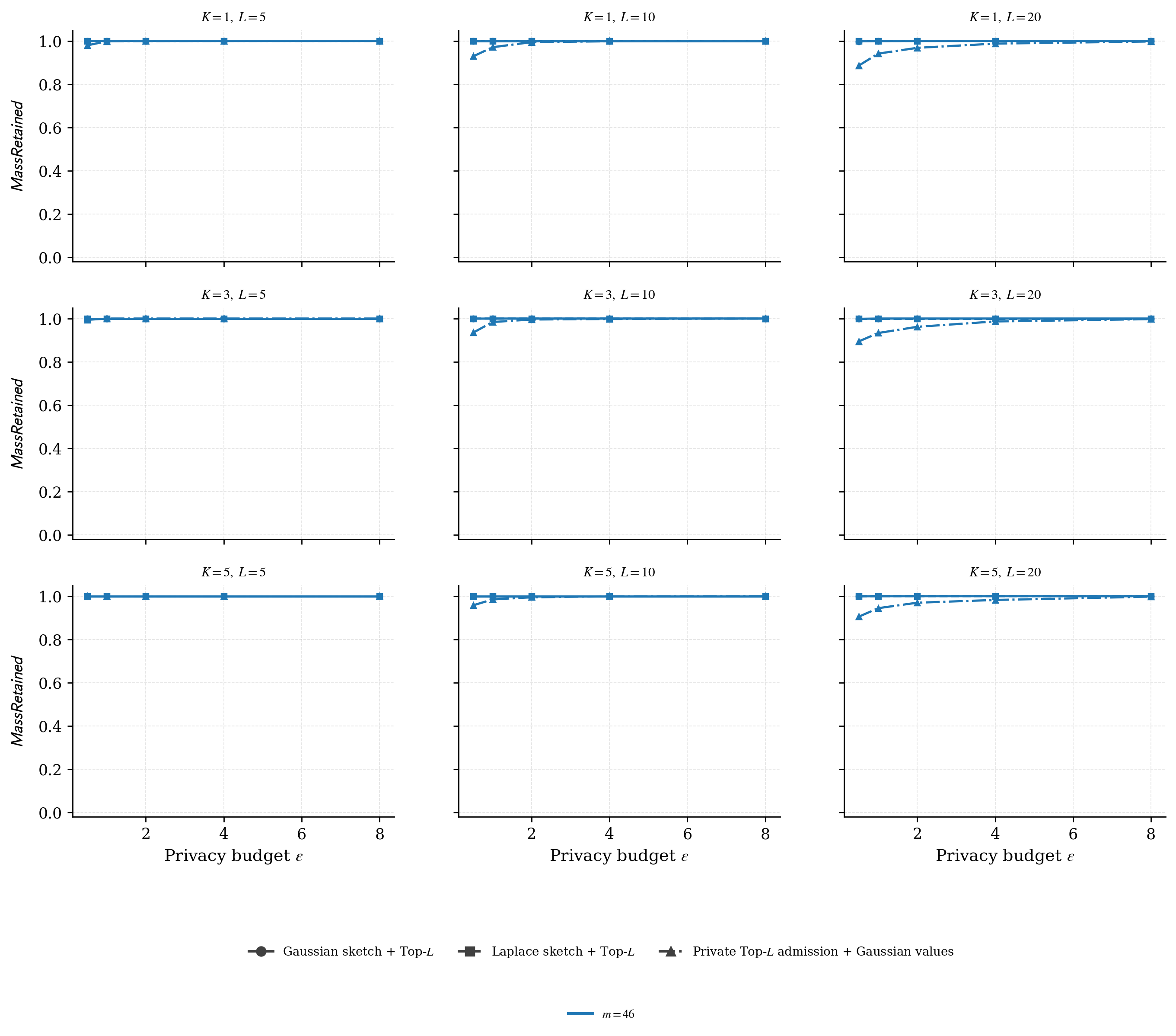}
    \caption{Amazon All Beauty record-level ablation of the preserved non-private mass fraction $\mathsf{MassRetained}$ as the privacy budget $\varepsilon$ varies, with $m=46$. Panels vary $K$ and $L$; curves compare the release mechanisms.}
    \label{fig:amazon-record-mass-preserved-vs-epsilon}
\end{figure}

\begin{figure}[h]
    \centering
    \includegraphics[width=0.99\linewidth]{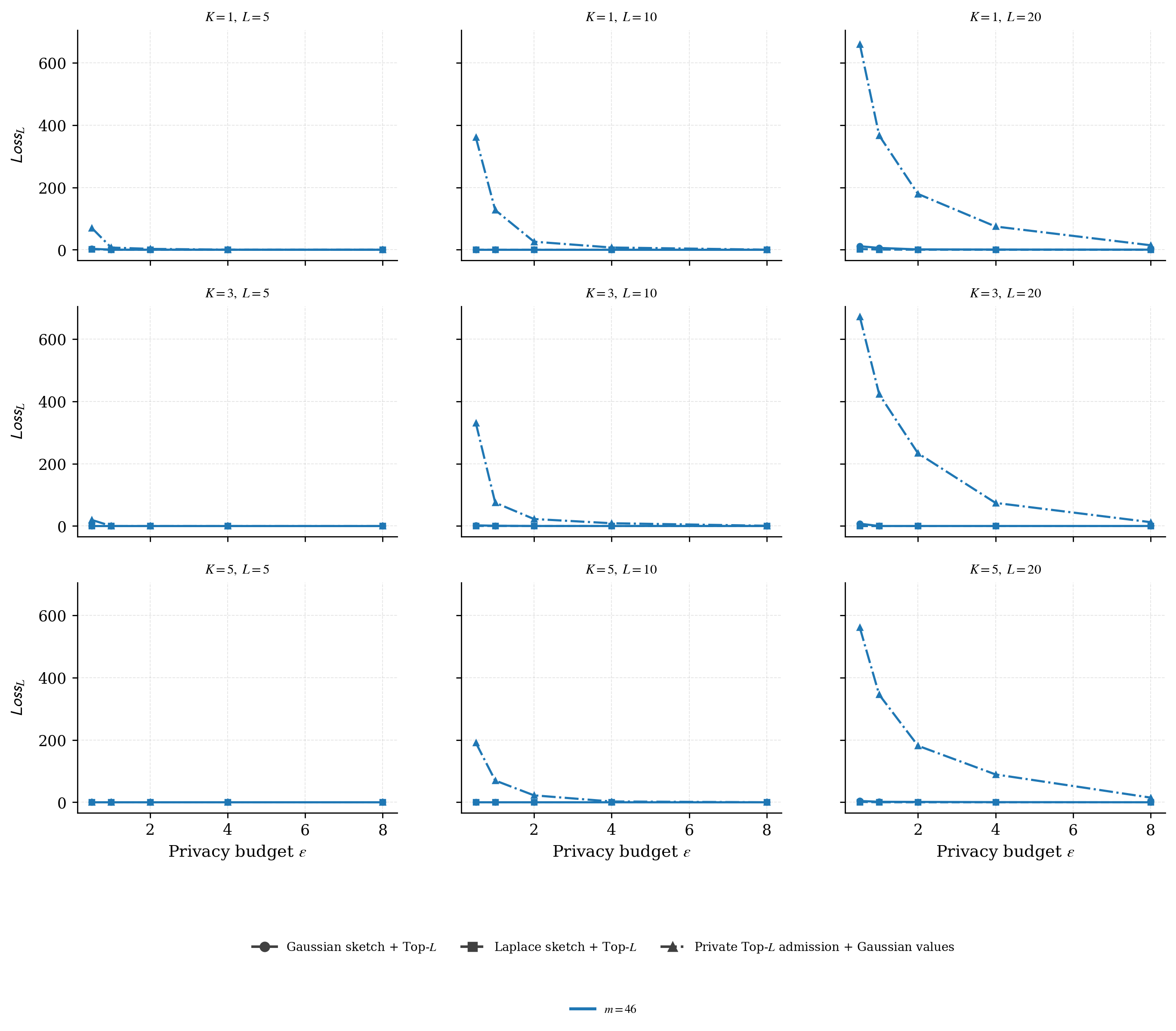}
    \caption{Amazon All Beauty record-level ablation of support-mass loss $\mathsf{Loss}_L$ as the privacy budget $\varepsilon$ varies, with $m=46$. Panels vary $K$ and $L$; curves compare the release mechanisms. Lower values indicate less non-private semantic mass lost by the released atom set.}
    \label{fig:amazon-record-loss-vs-epsilon}
\end{figure}

\begin{figure}[h]
    \centering
    \includegraphics[width=0.99\linewidth]{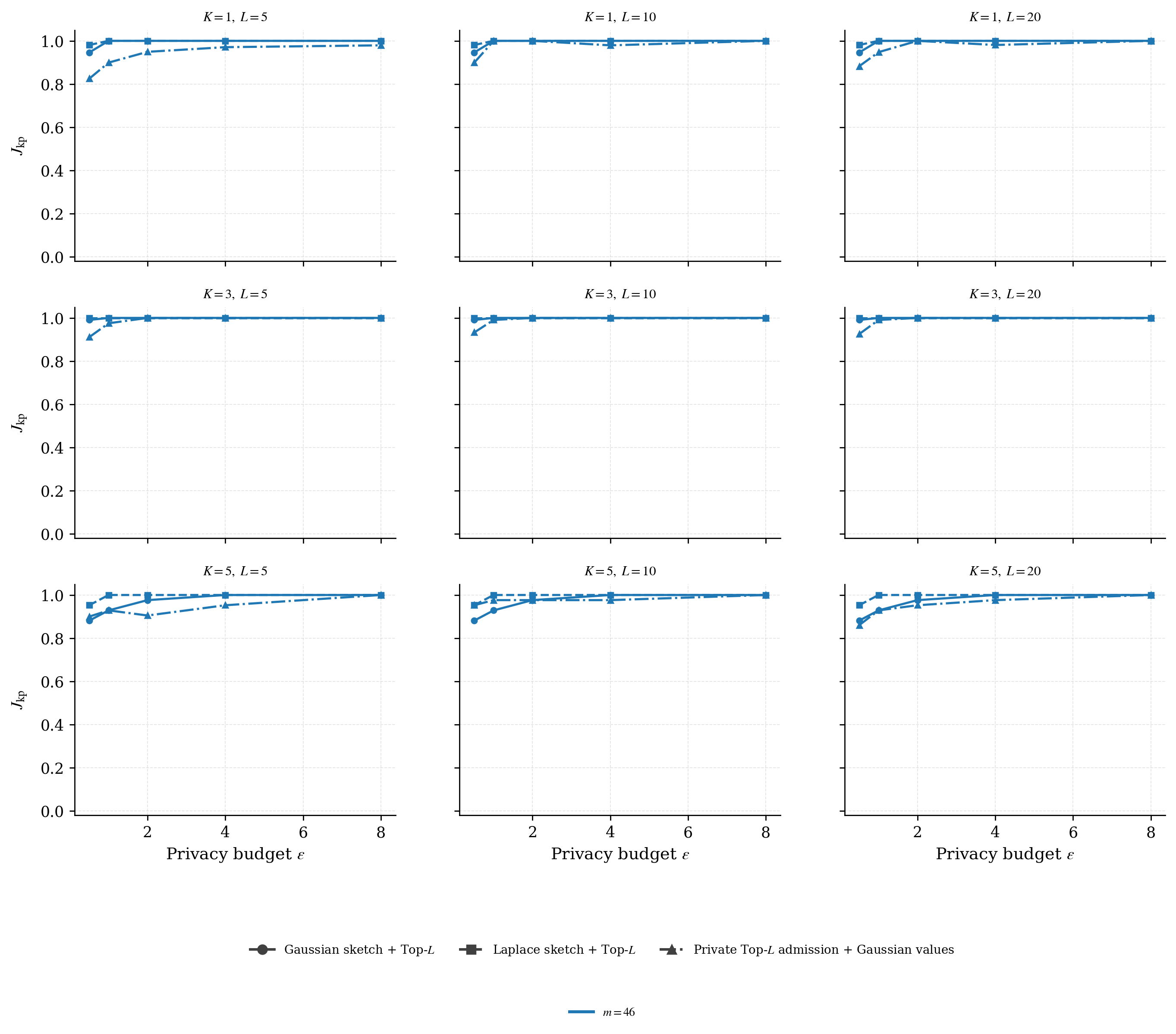}
    \caption{Amazon All Beauty record-level keyphrase Jaccard similarity $J_{\mathrm{kp}}$ between the DP-plan summary and the non-private-plan reference summary as the privacy budget $\varepsilon$ varies, with $m=46$. Panels vary $K$ and $L$; curves compare the release mechanisms.}
    \label{fig:amazon-record-keyphrase-jaccard-vs-epsilon}
\end{figure}

\begin{figure}[h]
    \centering
    \includegraphics[width=0.99\linewidth]{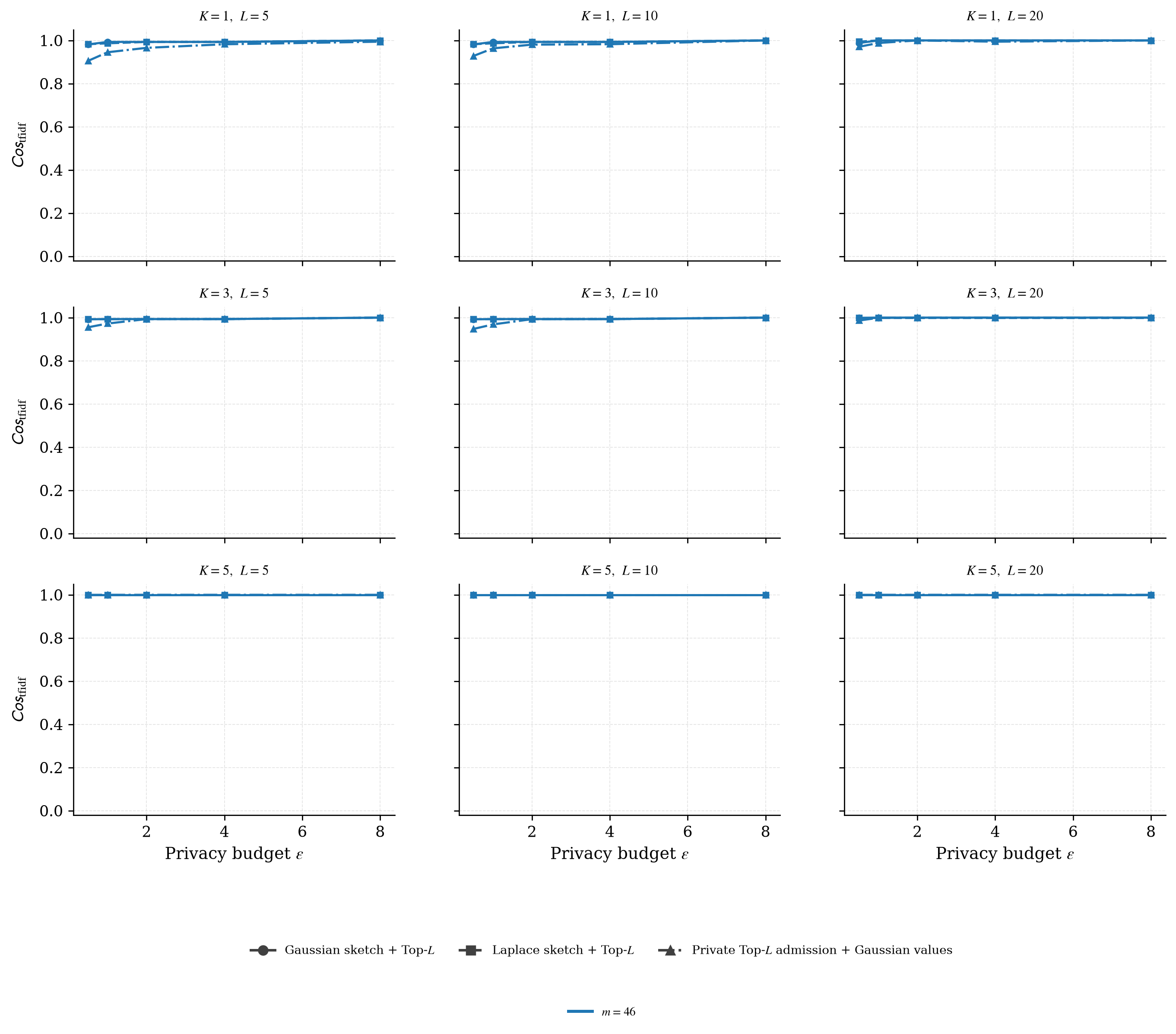}
    \caption{Amazon All Beauty record-level TF--IDF cosine similarity $\mathsf{Cos}_{\mathrm{tfidf}}$ between the DP-plan summary and the non-private-plan reference summary as the privacy budget $\varepsilon$ varies, with $m=46$. Panels vary $K$ and $L$; curves compare the release mechanisms.}
    \label{fig:amazon-record-tfidf-cosine-vs-epsilon}
\end{figure}

\begin{figure}[h]
    \centering
    \includegraphics[width=0.99\linewidth]{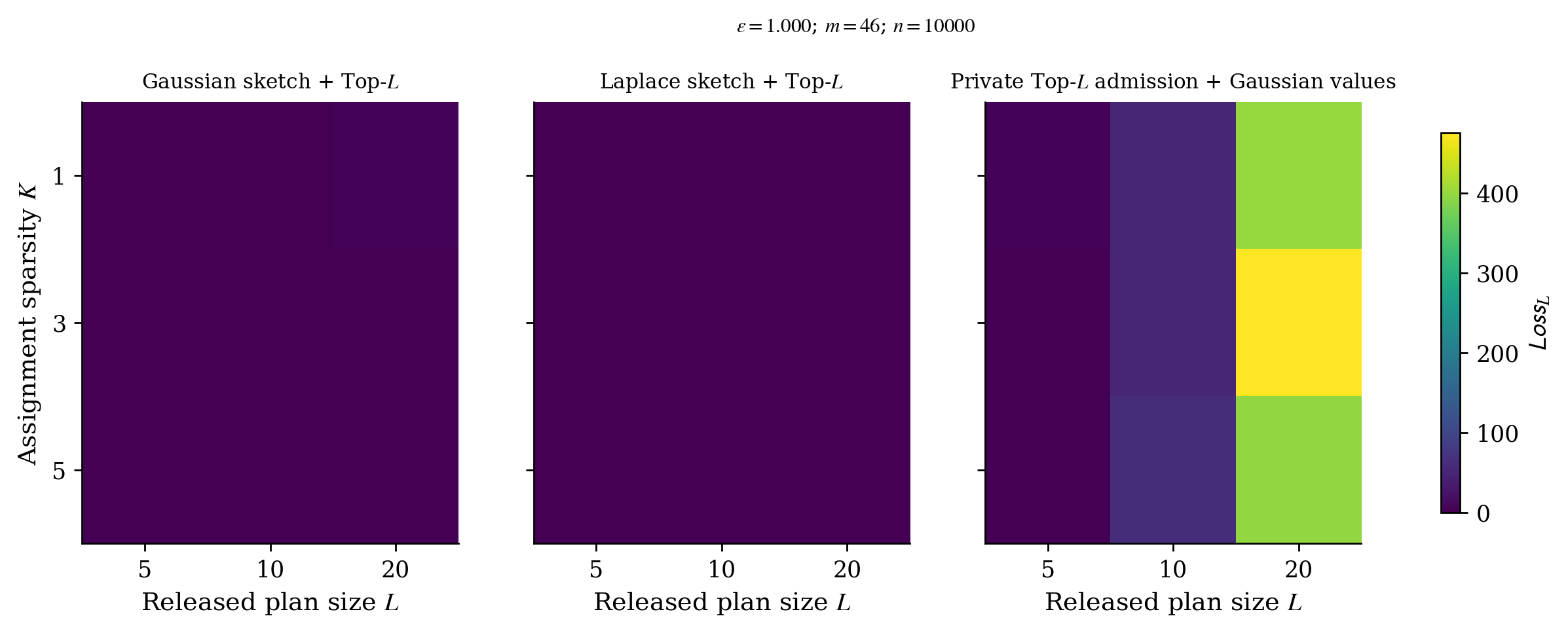}
    \caption{Amazon All Beauty record-level heatmap of support-mass loss $\mathsf{Loss}_L$ over assignment sparsity $K$ and released plan size $L$, with $\varepsilon=1.0$, $m=46$, and $n=10000$. Each panel corresponds to one release mechanism. Lower values indicate less non-private semantic mass lost by the released atom set.}
    \label{fig:amazon-record-loss-heatmap}
\end{figure}

\begin{figure}[h]
    \centering
    \begin{minipage}{0.49\linewidth}
        \centering
        \includegraphics[width=\linewidth]{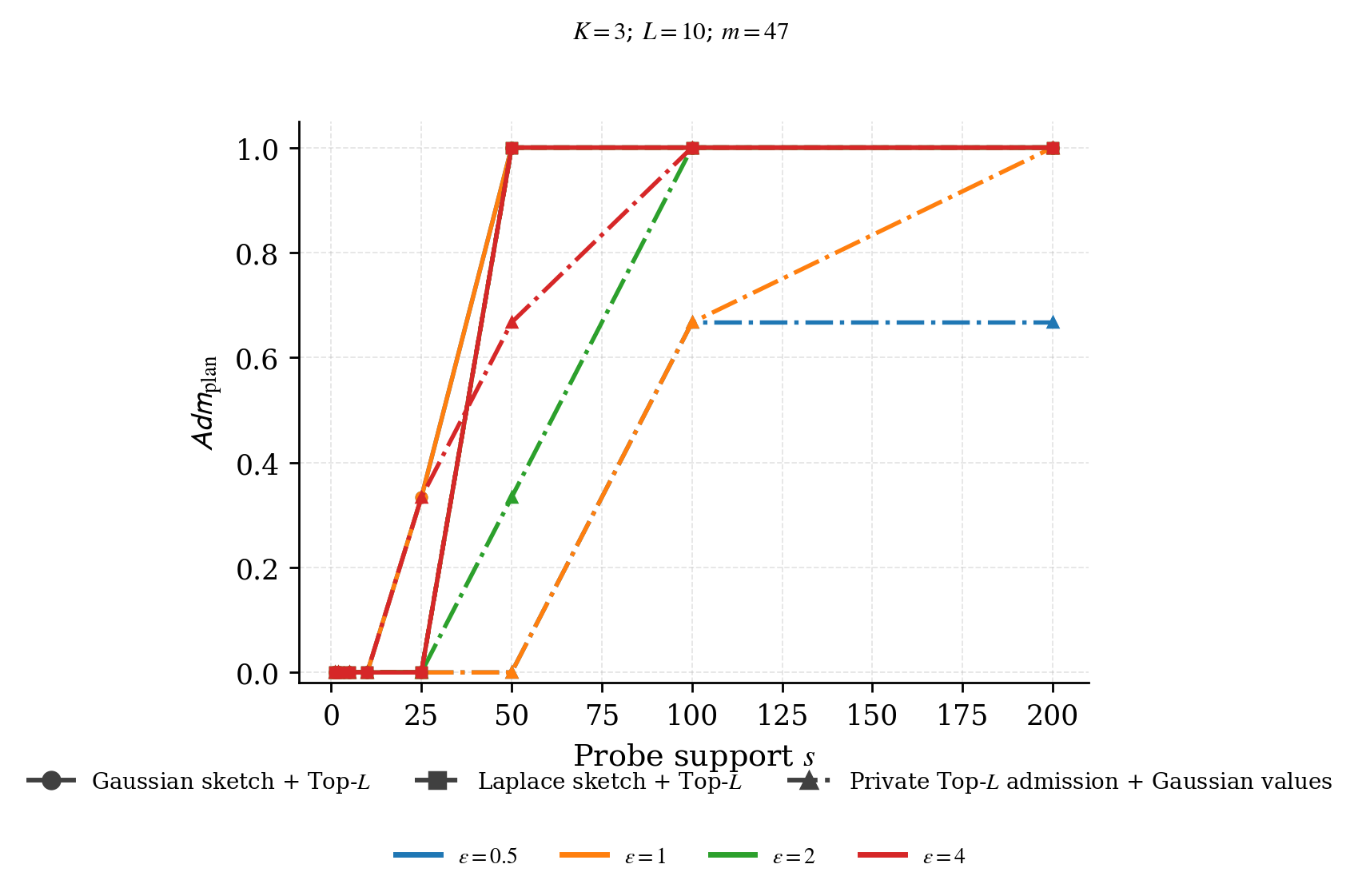}
        \vspace{2pt}
        \textbf{(a)} Plan admission.
    \end{minipage}
    \hfill
    \begin{minipage}{0.49\linewidth}
        \centering
        \includegraphics[width=\linewidth]{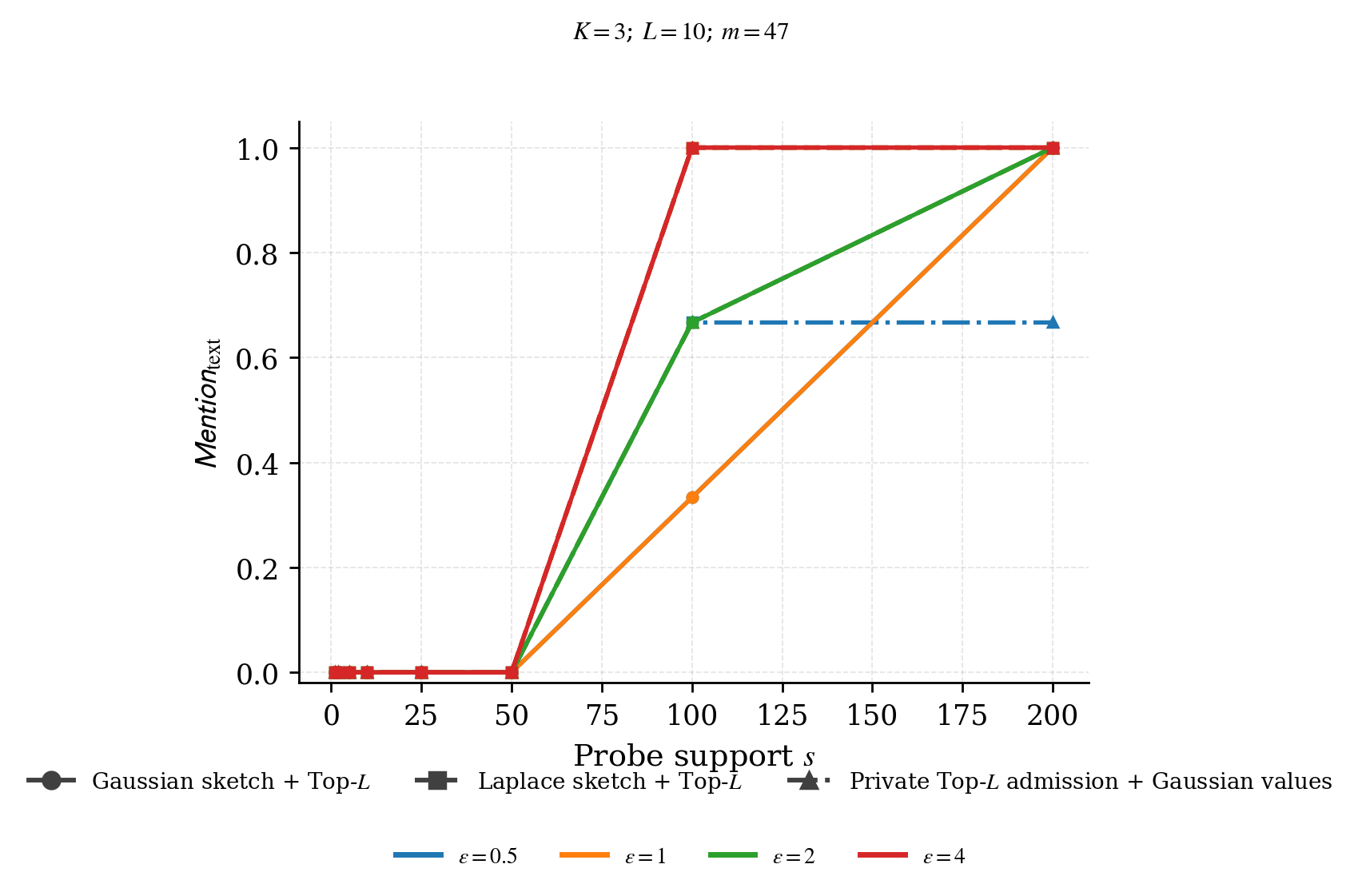}
        \vspace{2pt}
        \textbf{(b)} Summary mention.
    \end{minipage}
    \caption{
    Amazon All Beauty record-level controlled probe-atom experiment as the injected probe support $s$ varies, with $K=3$, $L=10$, and $m=47$ public atoms including the probe atom. Panel~(a) reports the plan-admission rate $\mathsf{Adm}_{\mathrm{plan}}$, which measures whether the differentially private release admits the public probe atom into the released semantic plan. Panel~(b) reports the summary mention rate $\mathsf{Mention}_{\mathrm{text}}$, which measures whether the public probe
    string is detected in the decoded summary. Curves compare privacy budgets and release mechanisms under the configured Amazon record-level probe-sweep setting.}
    \label{fig:amazon-record-probe-admission-mention}
\end{figure}

\clearpage
\subsection{User-level results}

\begin{figure}[h]
    \centering
    \includegraphics[width=0.99\linewidth]{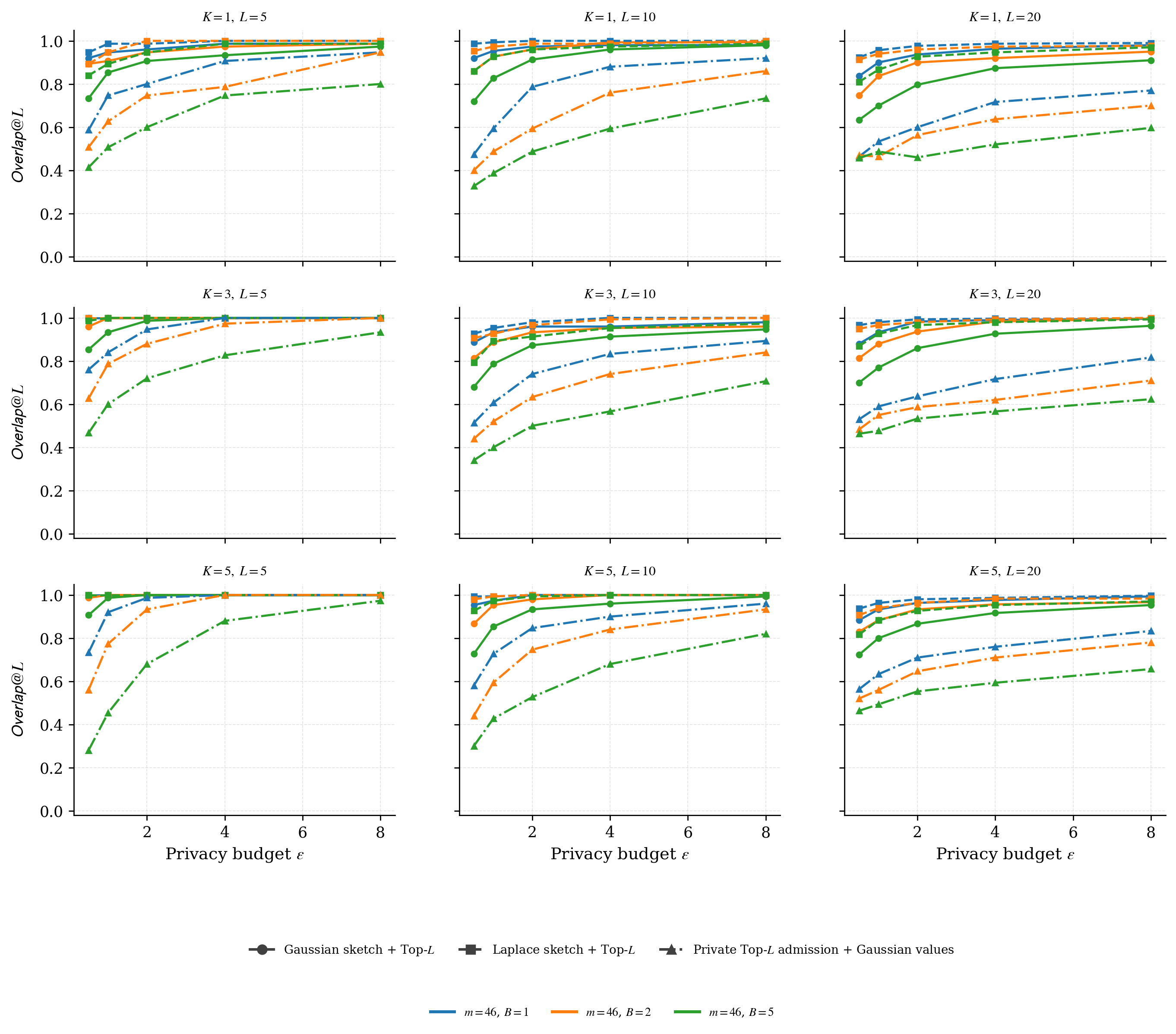}
    \caption{Amazon All Beauty user-level ablation of $\mathsf{Overlap}@L$ as the privacy budget $\varepsilon$ varies. Panels vary the assignment sparsity $K$ and released plan size $L$; curves compare the release mechanisms and user-level configurations.
    }
    \label{fig:amazon-user-overlap-vs-epsilon}
\end{figure}

\begin{figure}[h]
    \centering
    \includegraphics[width=0.99\linewidth]{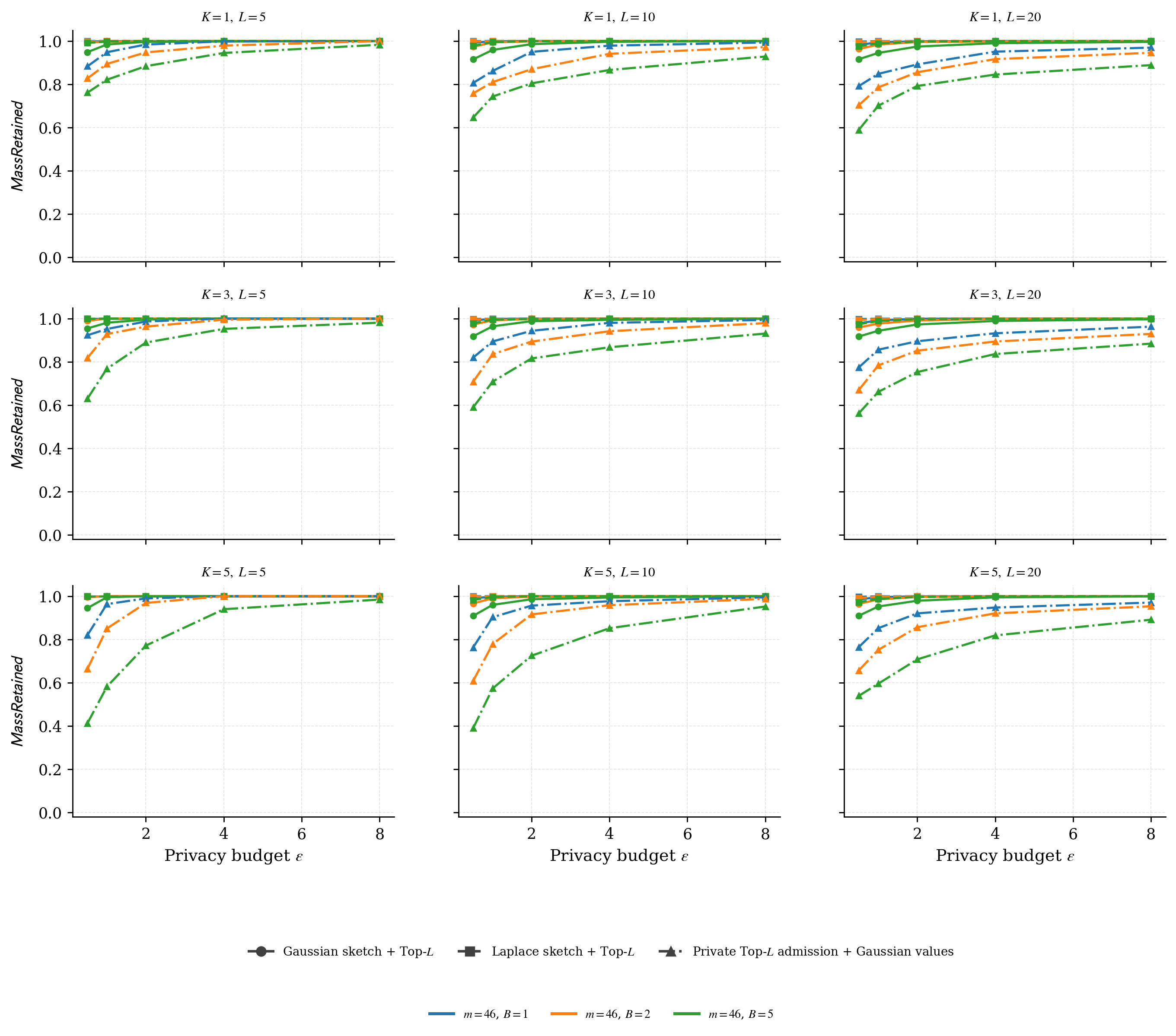}
    \caption{Amazon All Beauty user-level ablation of the preserved non-private mass fraction $\mathsf{MassRetained}$ as the privacy budget $\varepsilon$ varies. Panels vary $K$ and $L$; curves compare the release mechanisms and user-level configurations.}
    \label{fig:amazon-user-mass-preserved-vs-epsilon}
\end{figure}

\begin{figure}[h]
    \centering
    \includegraphics[width=0.99\linewidth]{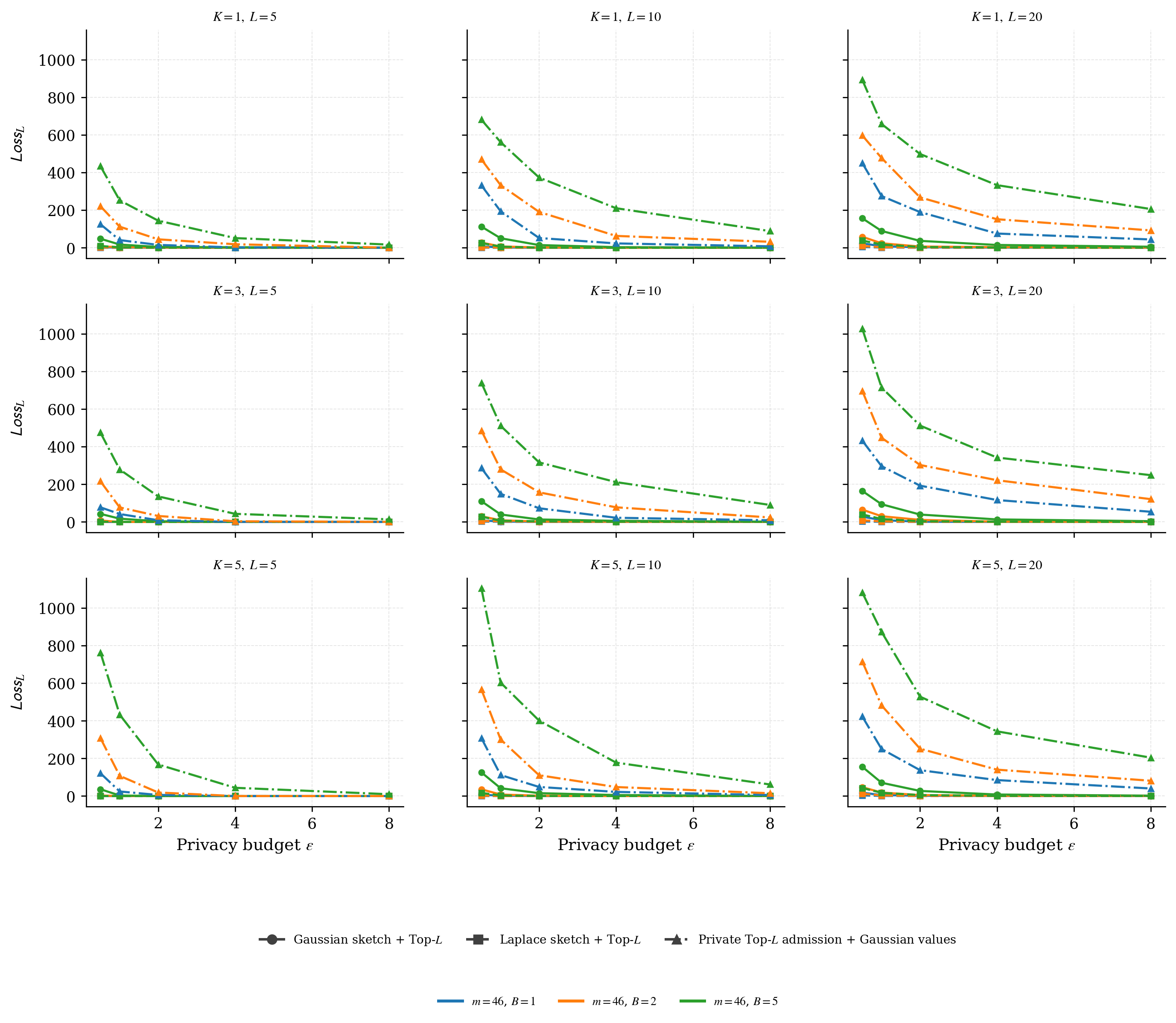}
    \caption{Amazon All Beauty user-level ablation of support-mass loss $\mathsf{Loss}_L$ as the privacy budget $\varepsilon$ varies. Panels vary $K$ and $L$; curves compare the release mechanisms and user-level configurations. Lower values indicate less non-private semantic mass lost by the released atom set.}
    \label{fig:amazon-user-loss-vs-epsilon}
\end{figure}

\begin{figure}[h]
    \centering
    \includegraphics[width=0.99\linewidth]{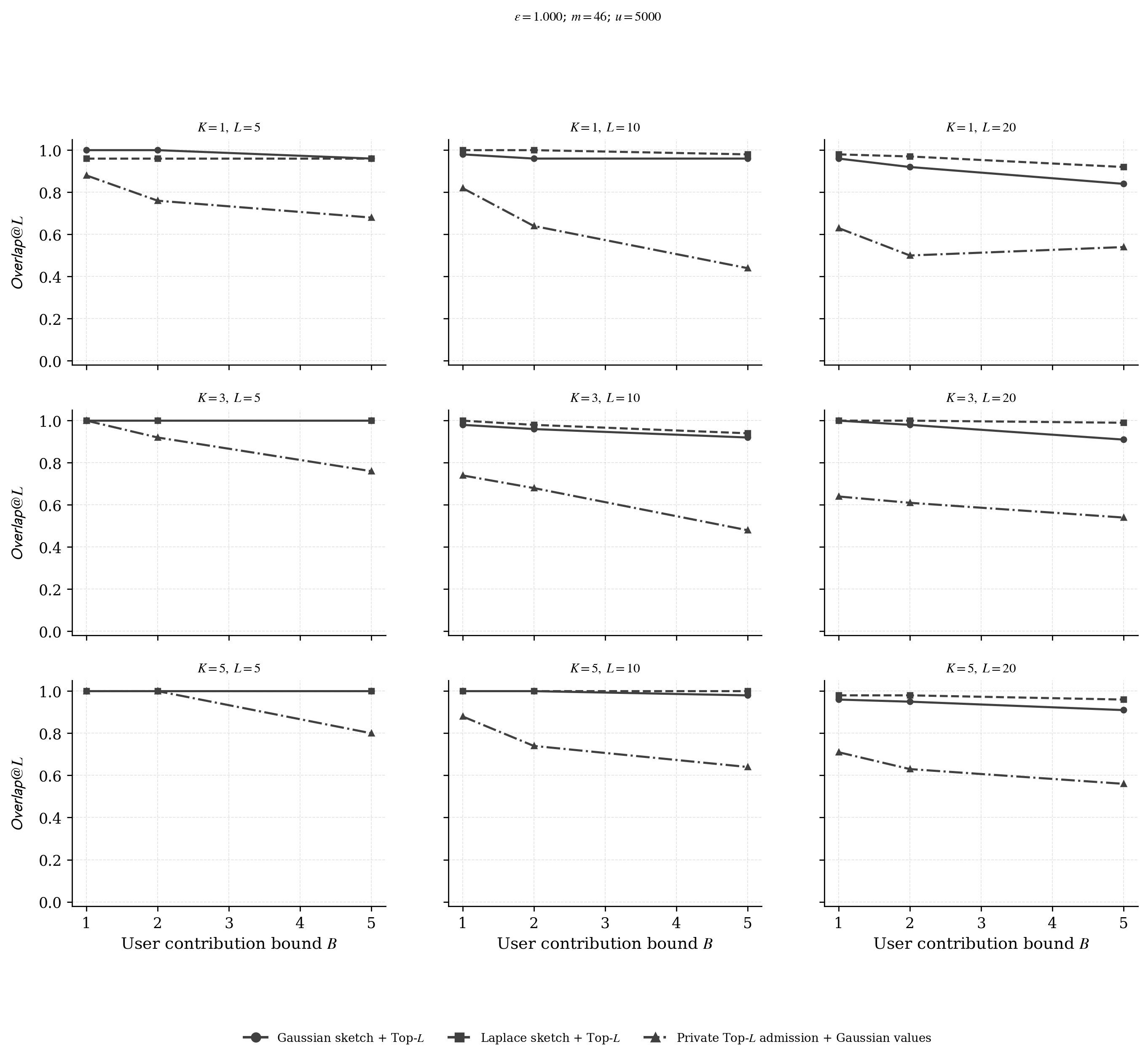}
    \caption{
    Amazon All Beauty user-level ablation of $\mathsf{Overlap}@L$ as the user contribution bound $B$ varies, with $\varepsilon=1.0$, $m=46$, and $u=5000$ selected users. Panels vary $K$ and $L$; curves compare the release mechanisms.}
    \label{fig:amazon-user-overlap-vs-B}
\end{figure}

\begin{figure}[h]
    \centering
    \includegraphics[width=0.99\linewidth]{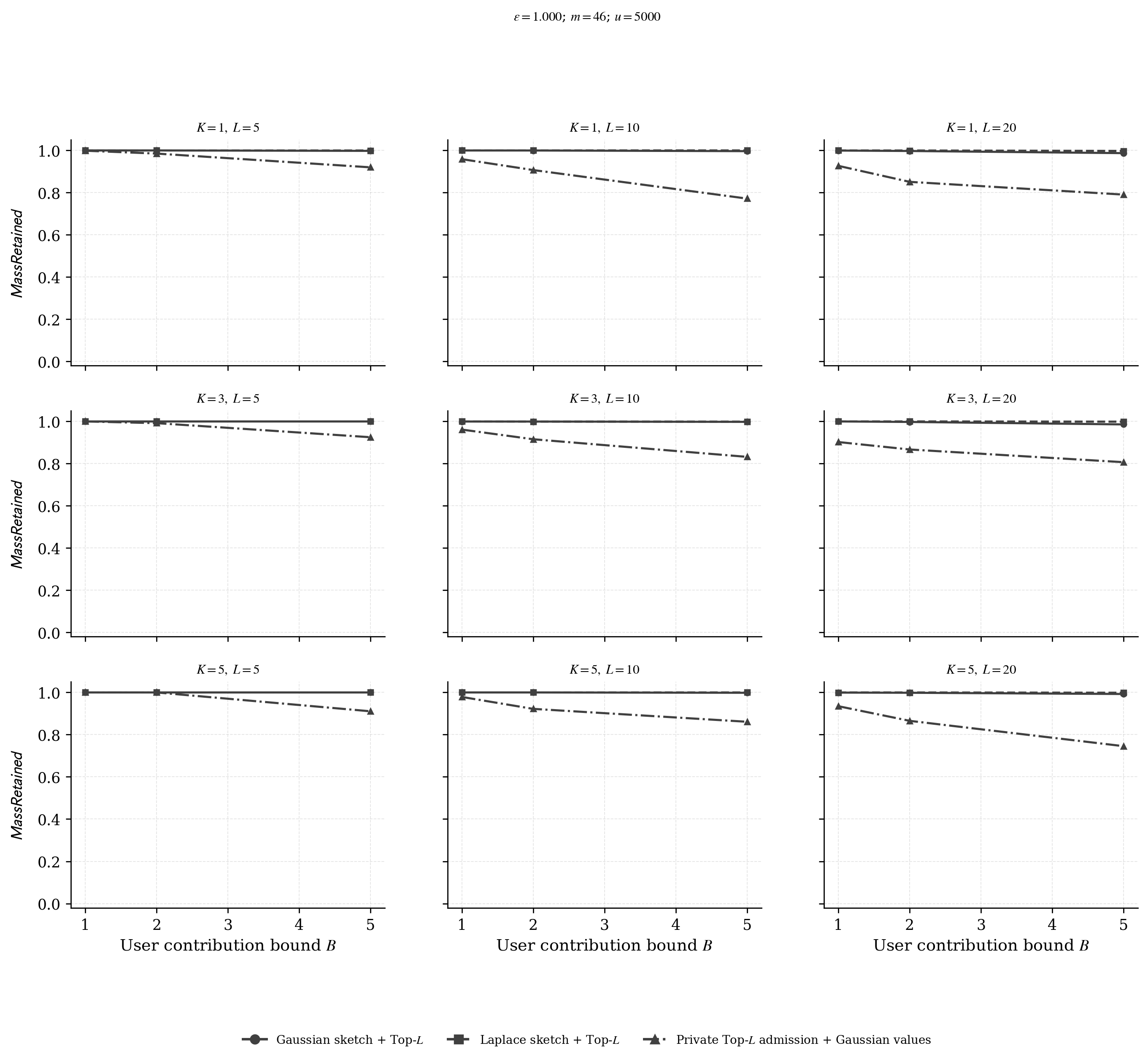}
    \caption{
    Amazon All Beauty user-level ablation of the preserved non-private mass fraction $\mathsf{MassRetained}$ as the user contribution bound $B$ varies, with $\varepsilon=1.0$, $m=46$, and $u=5000$ selected users. Panels vary $K$ and $L$; curves compare the release mechanisms.}
    \label{fig:amazon-user-mass-preserved-vs-B}
\end{figure}

\begin{figure}[h]
    \centering
    \includegraphics[width=0.99\linewidth]{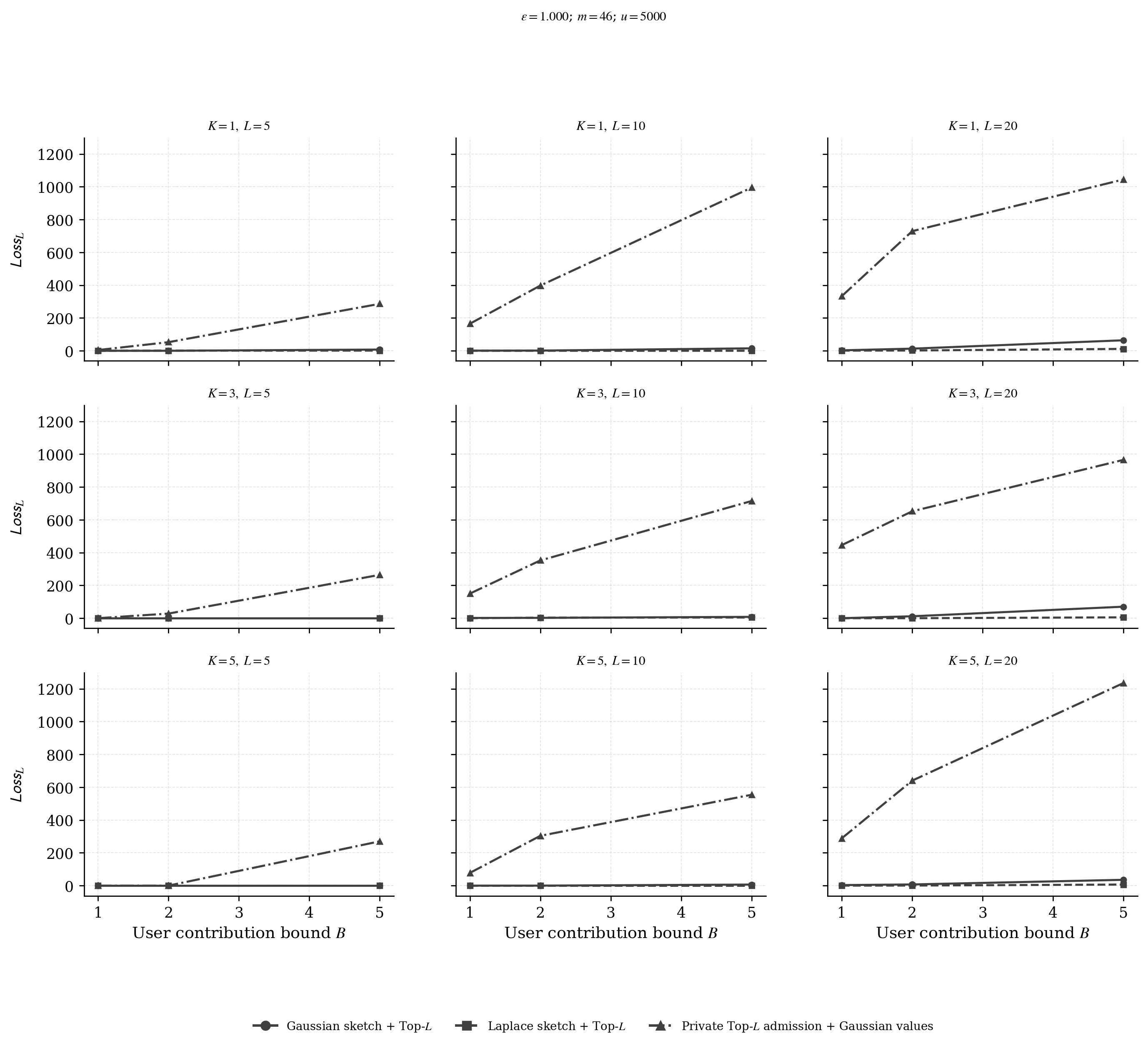}
    \caption{Amazon All Beauty user-level ablation of support-mass loss $\mathsf{Loss}_L$ as the user contribution bound $B$ varies, with $\varepsilon=1.0$, $m=46$, and $u=5000$ selected users. Panels vary $K$ and $L$; curves compare the release mechanisms. Lower values indicate less non-private semantic mass lost by the released atom set.}
    \label{fig:amazon-user-loss-vs-B}
\end{figure}

\begin{figure}[h]
    \centering
    \includegraphics[width=0.99\linewidth]{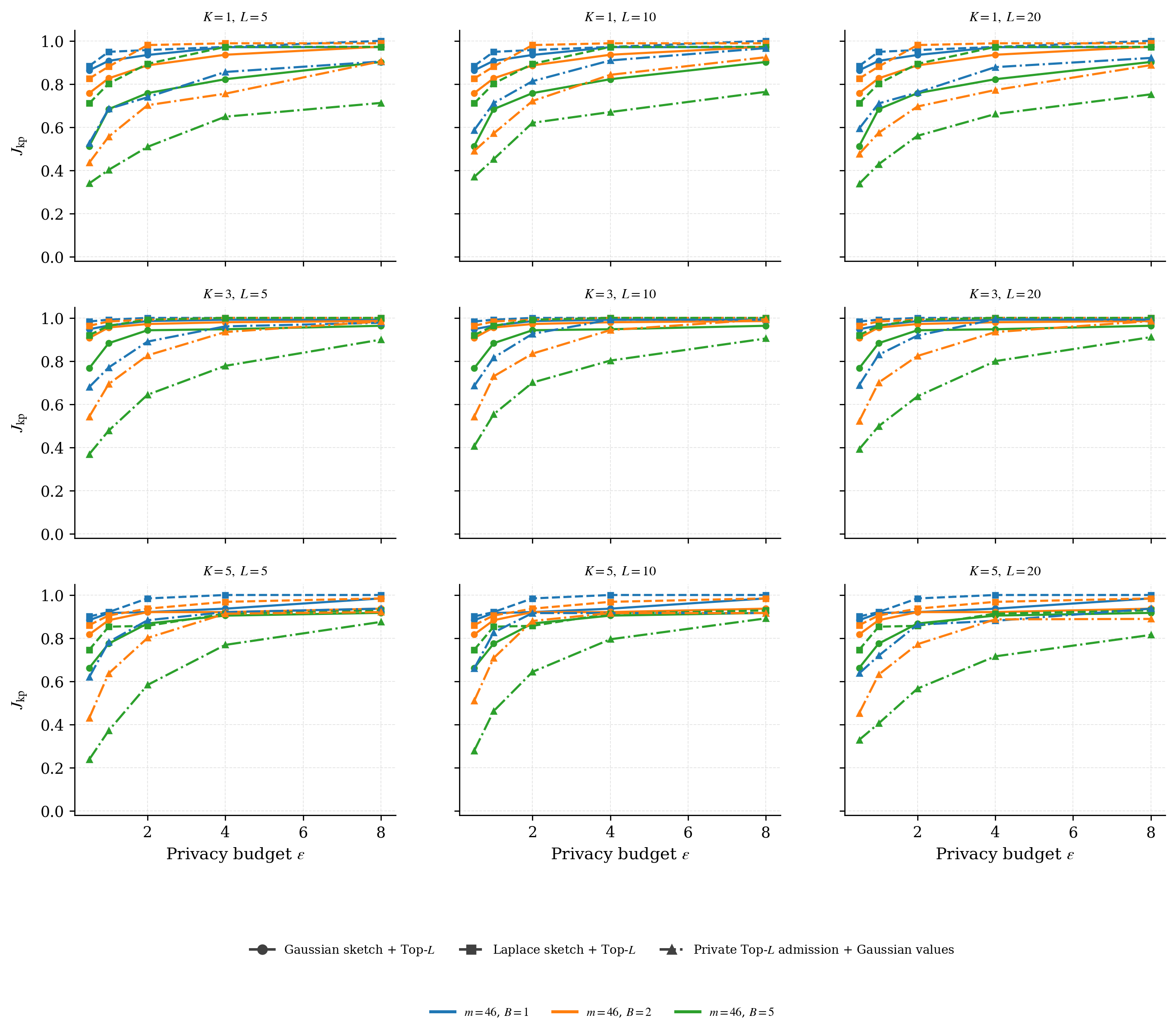}
    \caption{
    Amazon All Beauty user-level keyphrase Jaccard similarity $J_{\mathrm{kp}}$ between the DP-plan summary and the non-private-plan reference summary as the privacy budget $\varepsilon$ varies. Panels vary $K$ and $L$; curves compare the release mechanisms and user-level configurations.
    }
    \label{fig:amazon-user-keyphrase-jaccard-vs-epsilon}
\end{figure}

\begin{figure}[h]
    \centering
    \includegraphics[width=0.99\linewidth]{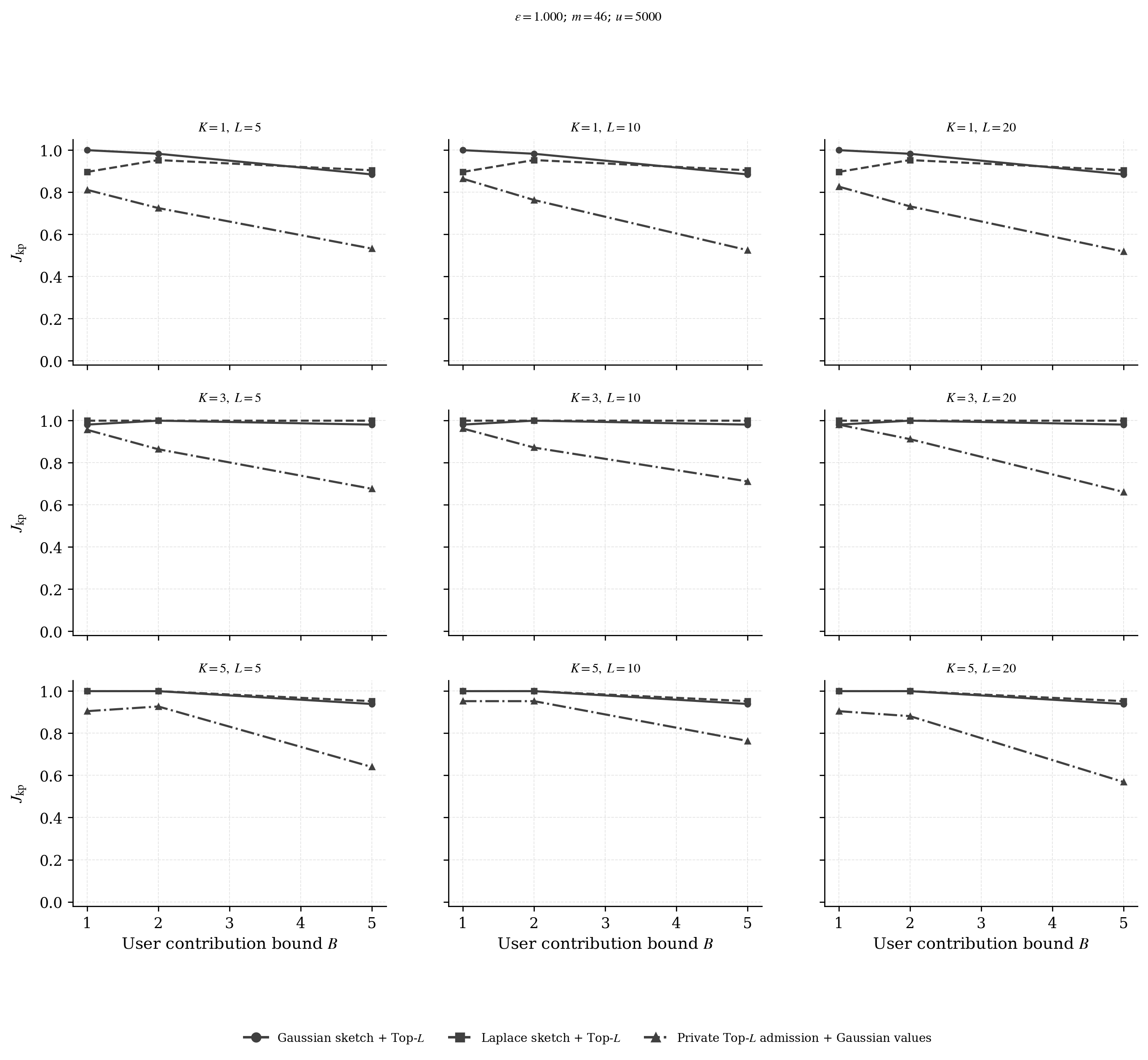}
    \caption{
    Amazon All Beauty user-level keyphrase Jaccard similarity $J_{\mathrm{kp}}$ between the DP-plan summary and the non-private-plan reference summary as the user contribution bound $B$ varies, with fixed $\varepsilon$ and $m$ as specified in the figure. Panels vary $K$ and $L$; curves compare the release mechanisms.
    }
    \label{fig:amazon-user-keyphrase-jaccard-vs-B}
\end{figure}

\begin{figure}[h]
    \centering
    \includegraphics[width=0.99\linewidth]{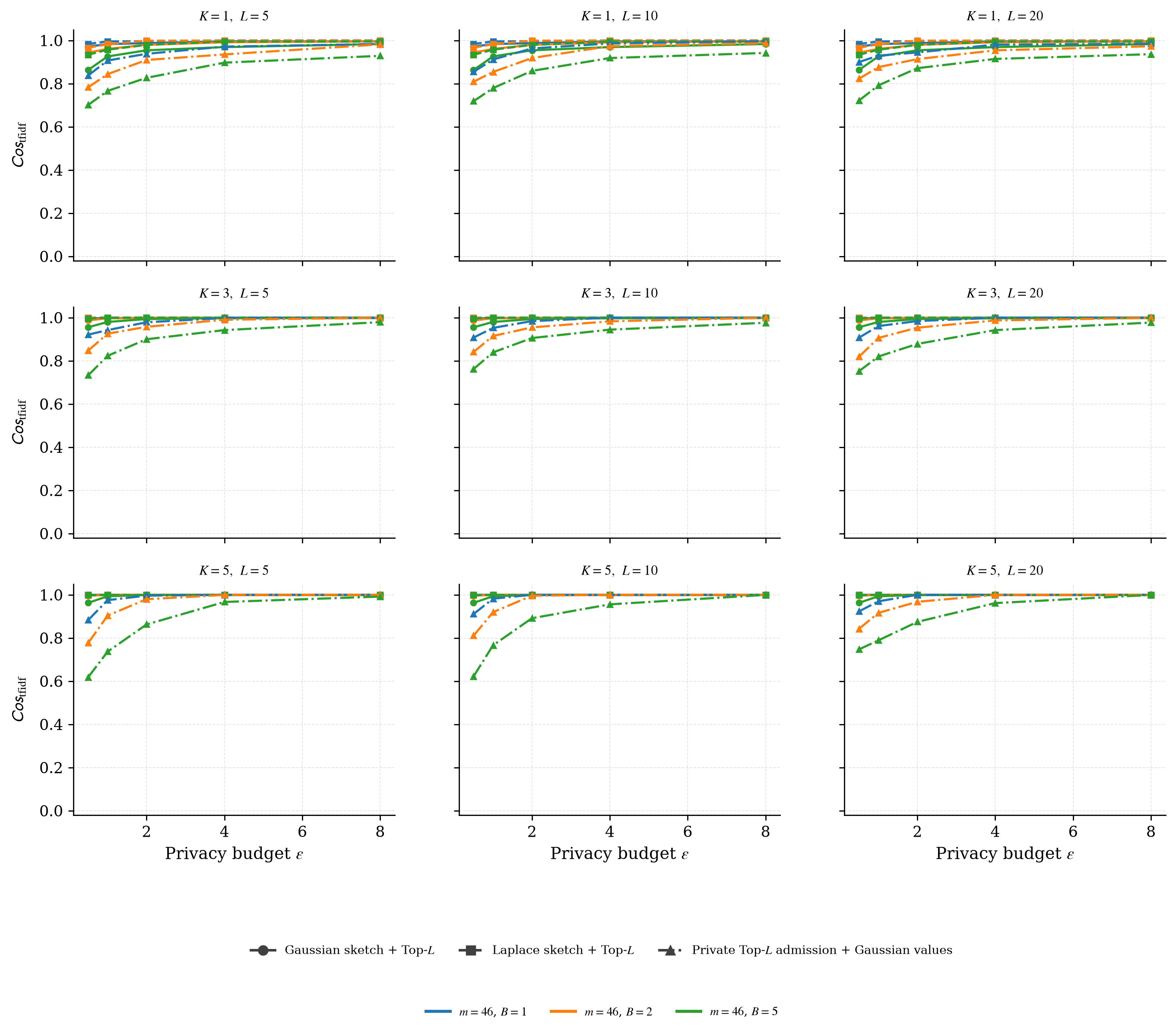}
    \caption{Amazon All Beauty user-level TF--IDF cosine similarity $\mathsf{Cos}_{\mathrm{tfidf}}$ between the DP-plan summary and the non-private-plan reference summary as the privacy budget $\varepsilon$ varies. Panels vary $K$ and $L$; curves compare the release mechanisms and user-level configurations.}
    \label{fig:amazon-user-tfidf-cosine-vs-epsilon}
\end{figure}

\begin{figure}[h]
    \centering
    \includegraphics[width=0.99\linewidth]{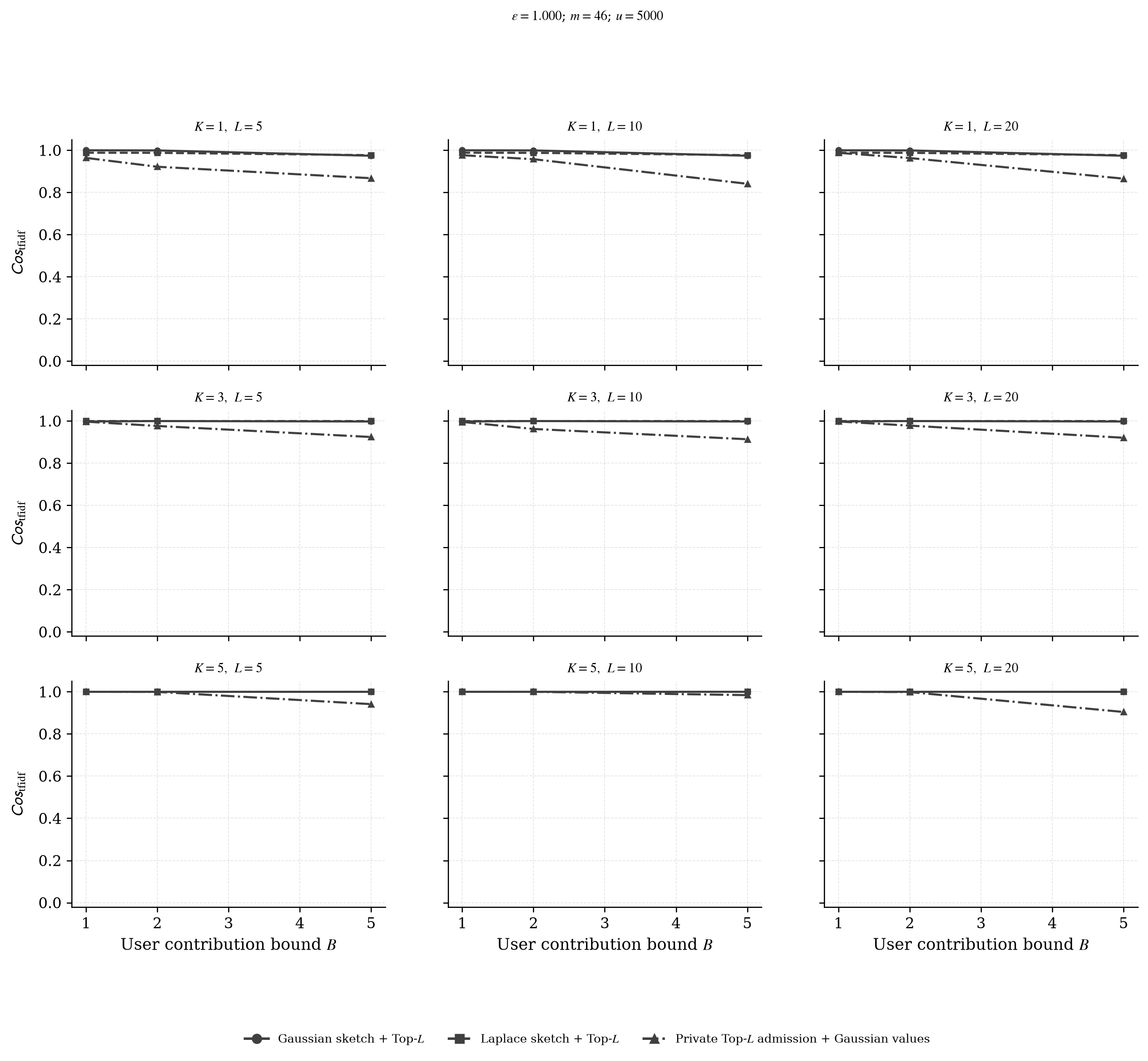}
    \caption{Amazon All Beauty user-level TF--IDF cosine similarity $\mathsf{Cos}_{\mathrm{tfidf}}$ between the DP-plan summary and the non-private-plan reference summary as the user contribution bound $B$ varies, with fixed $\varepsilon$ and $m$ as specified in the figure. Panels vary $K$ and $L$; curves compare the release mechanisms.}
    \label{fig:amazon-user-tfidf-cosine-vs-B}
\end{figure}

\begin{figure}[h]
    \centering
    \includegraphics[width=0.99\linewidth]{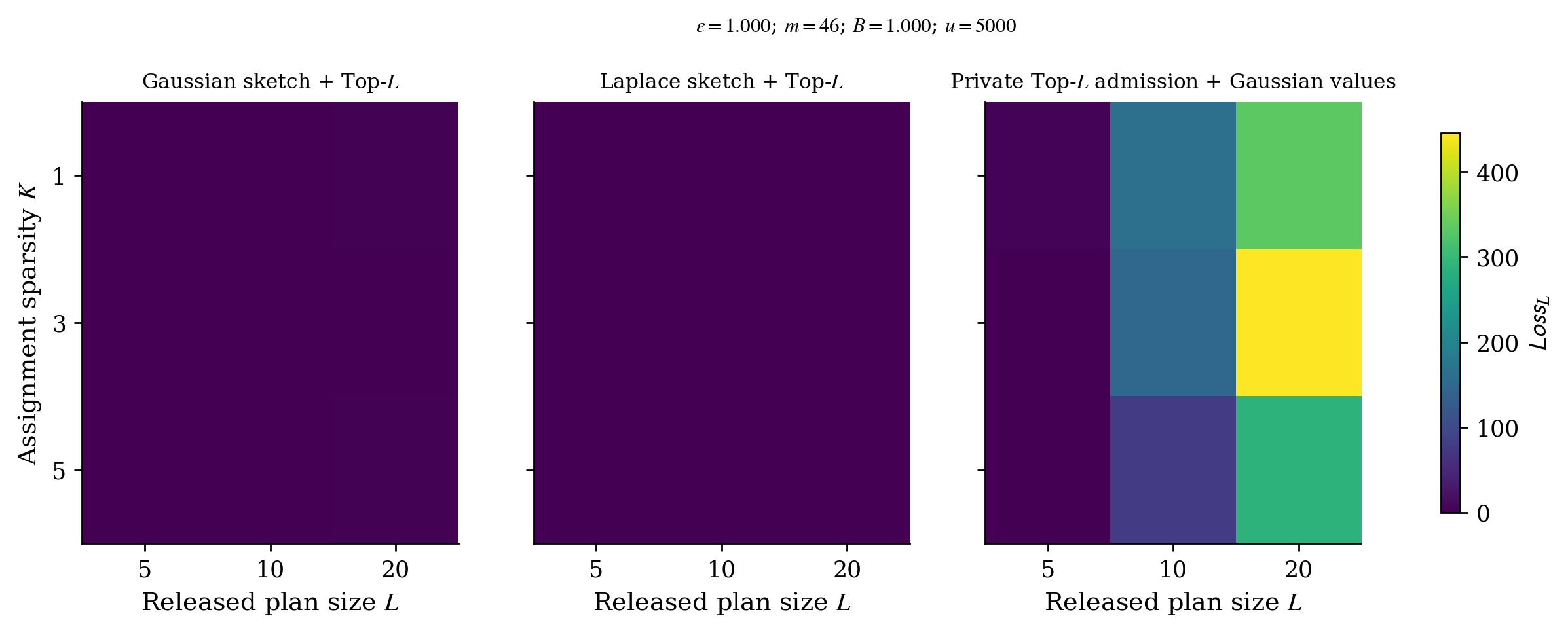}
    \caption{
    Amazon All Beauty user-level heatmap of support-mass loss $\mathsf{Loss}_L$ over assignment sparsity $K$ and released plan size $L$, with fixed $\varepsilon$, atom dictionary size $m$, and user contribution bound $B$ as specified in the figure. Each panel corresponds to one release mechanism. Lower values indicate less non-private semantic mass lost by the released atom set.
    }
    \label{fig:amazon-user-loss-heatmap}
\end{figure}

\begin{figure}[h]
    \centering
    \begin{minipage}{0.49\linewidth}
        \centering
        \includegraphics[width=\linewidth]{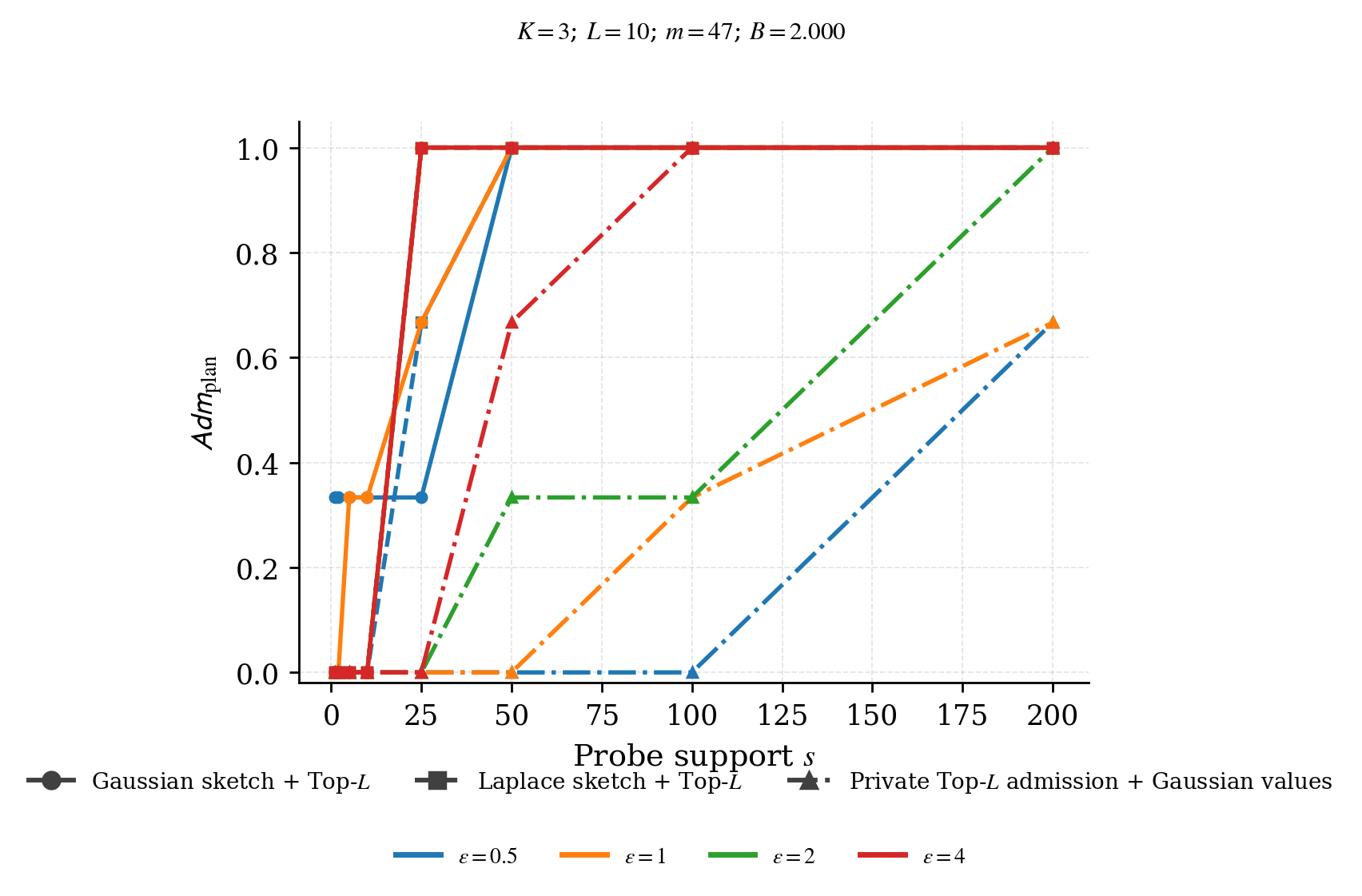}
        \vspace{2pt}
        \textbf{(a)} Plan admission.
    \end{minipage}
    \hfill
    \begin{minipage}{0.49\linewidth}
        \centering
        \includegraphics[width=\linewidth]{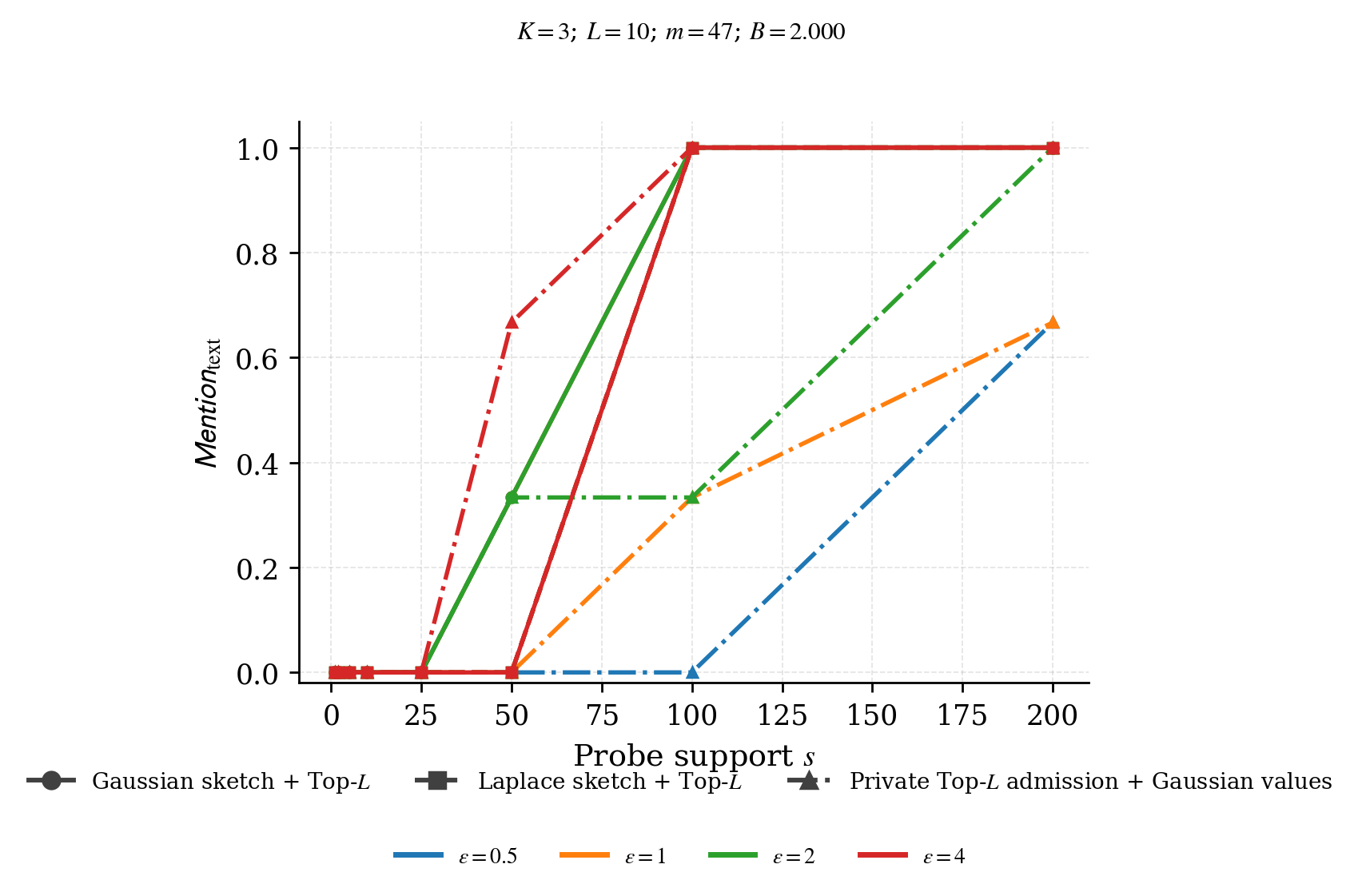}
        \vspace{2pt}
        \textbf{(b)} Summary mention.
    \end{minipage}
    \caption{
    Amazon All Beauty user-level controlled probe-atom experiment as the injected probe support $s$ varies, with $K=3$, $L=10$, $m=47$ public atoms including the probe atom, and user contribution bound $B=2$. Panel~(a) reports the plan-admission rate $\mathsf{Adm}_{\mathrm{plan}}$, which measures whether the differentially private release admits the public probe atom into the released semantic plan. Panel~(b) reports the summary mention rate $\mathsf{Mention}_{\mathrm{text}}$, which measures whether the public probe string is detected in the decoded summary. Curves compare privacy budgets and release mechanisms under the configured Amazon user-level probe-sweep setting.
    }
    \label{fig:amazon-user-probe-admission-mention}
\end{figure}

\clearpage

\begin{table*}[t]
\centering
\caption{Release-conditioned OpenAI evaluation of selected generated summaries on Amazon All Beauty reviews. Each summary is evaluated only with respect to the object released to its decoder. Scores are assigned on a 1--5 scale and report verbalization quality. The judge privacy-safety column evaluates only the generated text and is not a formal differential-privacy guarantee. Grouped values are reported as mean $\pm$ sample standard deviation across
the summaries in each row.}
\label{tab:openai-baseline-comparison-amazon-all-beauty}
\small
\resizebox{0.99\textwidth}{!}{%
\begin{tabular}{lllccccccccc}
\toprule
Method & Release object & Configuration & $\varepsilon$ & Protected units & Judged summaries  & Coverage & Specificity & Insightfulness & Faithfulness & Text safety & Clarity \\
\midrule
DP-SPIN private plan & DP semantic plan & record; Gaussian sketch + top-$L$; m=46; K=3; L=10 & 1.0 & 2000 & 3 & 5.00 $\pm$ 0.00 & 5.00 $\pm$ 0.00 & 4.00 $\pm$ 0.00 & 5.00 $\pm$ 0.00 & 5.00 $\pm$ 0.00 & 5.00 $\pm$ 0.00 \\
DP-SPIN private plan & DP semantic plan & user; Gaussian sketch + top-$L$; m=46; K=3; L=10; B=1 & 1.0 & 1000 & 3 & 5.00 $\pm$ 0.00 & 4.67 $\pm$ 0.58 & 4.00 $\pm$ 0.00 & 5.00 $\pm$ 0.00 & 5.00 $\pm$ 0.00 & 5.00 $\pm$ 0.00 \\
DP keyword histogram & DP keyword histogram & fixed public keyword vocabulary; Laplace histogram; released top-10 keywords & 1.0 & 2000 & 3 & 5.00 $\pm$ 0.00 & 4.00 $\pm$ 0.00 & 3.67 $\pm$ 0.58 & 5.00 $\pm$ 0.00 & 5.00 $\pm$ 0.00 & 4.67 $\pm$ 0.58 \\
DP category histogram & DP category histogram & fixed rating\_bin universe; Laplace histogram; released top-3 categories & 1.0 & 2000 & 3 & 5.00 $\pm$ 0.00 & 4.00 $\pm$ 0.00 & 3.33 $\pm$ 0.58 & 5.00 $\pm$ 0.00 & 5.00 $\pm$ 0.00 & 4.33 $\pm$ 0.58 \\
URANIA-style public keywords & DP public keywords & DP clustering; DP public-keyword histograms; clusters=10; keywords=10 & 1.0 & 2000 & 3 & 4.00 $\pm$ 0.00 & 3.67 $\pm$ 0.58 & 4.00 $\pm$ 0.00 & 4.00 $\pm$ 1.00 & 5.00 $\pm$ 0.00 & 4.33 $\pm$ 0.58 \\
\bottomrule
\end{tabular}%
}
\end{table*}

\begin{table*}[t]
\centering
\caption{Release-conditioned OpenAI evaluation of selected \texttt{DP-SPIN} summaries on Amazon All Beauty reviews. The decoder receives only the released differentially private semantic plan and public decoding instructions. Scores are assigned on the same 1--5 rubric used for the baseline-comparison table.}
\label{tab:openai-dpspin-amazon-all-beauty}
\small
\resizebox{0.99\textwidth}{!}{%
\begin{tabular}{lllccccccccc}
\toprule
Unit & Mechanism & Configuration & $\varepsilon$ & Protected units & Judged summaries  & Coverage & Specificity & Insightfulness & Faithfulness & Text safety & Clarity \\
\midrule
record & Gaussian sketch + top-$L$ & record; Gaussian sketch + top-$L$; m=46; K=3; L=10 & 1.0 & 2000 & 3 & 5.00 $\pm$ 0.00 & 5.00 $\pm$ 0.00 & 4.00 $\pm$ 0.00 & 5.00 $\pm$ 0.00 & 5.00 $\pm$ 0.00 & 5.00 $\pm$ 0.00 \\
record & Gaussian sketch + top-$L$ & record; Gaussian sketch + top-$L$; m=46; K=3; L=10 & 4.0 & 2000 & 3 & 5.00 $\pm$ 0.00 & 5.00 $\pm$ 0.00 & 4.00 $\pm$ 0.00 & 5.00 $\pm$ 0.00 & 5.00 $\pm$ 0.00 & 5.00 $\pm$ 0.00 \\
record & Private top-$L$ admission + Gaussian values & record; Private top-$L$ admission + Gaussian values; m=46; K=3; L=10 & 1.0 & 2000 & 3 & 5.00 $\pm$ 0.00 & 4.67 $\pm$ 0.58 & 4.00 $\pm$ 0.00 & 5.00 $\pm$ 0.00 & 5.00 $\pm$ 0.00 & 5.00 $\pm$ 0.00 \\
record & Private top-$L$ admission + Gaussian values & record; Private top-$L$ admission + Gaussian values; m=46; K=3; L=10 & 4.0 & 2000 & 3 & 5.00 $\pm$ 0.00 & 5.00 $\pm$ 0.00 & 4.00 $\pm$ 0.00 & 5.00 $\pm$ 0.00 & 5.00 $\pm$ 0.00 & 5.00 $\pm$ 0.00 \\
user & Gaussian sketch + top-$L$ & user; Gaussian sketch + top-$L$; m=46; K=3; L=10; B=1 & 1.0 & 1000 & 3 & 5.00 $\pm$ 0.00 & 4.67 $\pm$ 0.58 & 4.00 $\pm$ 0.00 & 5.00 $\pm$ 0.00 & 5.00 $\pm$ 0.00 & 5.00 $\pm$ 0.00 \\
user & Gaussian sketch + top-$L$ & user; Gaussian sketch + top-$L$; m=46; K=3; L=10; B=1 & 4.0 & 1000 & 3 & 5.00 $\pm$ 0.00 & 4.67 $\pm$ 0.58 & 4.00 $\pm$ 0.00 & 5.00 $\pm$ 0.00 & 5.00 $\pm$ 0.00 & 5.00 $\pm$ 0.00 \\
user & Private top-$L$ admission + Gaussian values & user; Private top-$L$ admission + Gaussian values; m=46; K=3; L=10; B=1 & 1.0 & 1000 & 3 & 5.00 $\pm$ 0.00 & 5.00 $\pm$ 0.00 & 4.00 $\pm$ 0.00 & 5.00 $\pm$ 0.00 & 5.00 $\pm$ 0.00 & 5.00 $\pm$ 0.00 \\
user & Private top-$L$ admission + Gaussian values & user; Private top-$L$ admission + Gaussian values; m=46; K=3; L=10; B=1 & 4.0 & 1000 & 3 & 5.00 $\pm$ 0.00 & 5.00 $\pm$ 0.00 & 4.00 $\pm$ 0.00 & 5.00 $\pm$ 0.00 & 5.00 $\pm$ 0.00 & 5.00 $\pm$ 0.00 \\
\bottomrule
\end{tabular}%
}
\end{table*}

\begin{table*}[t]
\centering
\caption{Similarity between OpenAI summaries generated from \texttt{DP-SPIN} private plans and summaries generated from the corresponding non-private top-$L$ plans on Amazon All Beauty reviews. Metrics compare generated texts and do not evaluate the underlying semantic sketches. Values are reported as mean $\pm$ sample standard deviation across runs.}
\label{tab:openai-private-non-private-amazon-all-beauty-reviews}
\small
\resizebox{0.99\textwidth}{!}{%
\begin{tabular}{lllccccccc}
\toprule
Unit & Mechanism & Configuration & $\varepsilon$ & Protected units & Runs & $J_{\mathrm{tok}}$ & $J_2$ & $J_{\mathrm{kp}}$ & $\mathsf{Cos}_{\mathrm{tfidf}}$ \\
\midrule
record & Gaussian sketch + top-$L$ & record; Gaussian sketch + top-$L$; m=46; K=3; L=10 & 1.0 & 2000 & 3 & 0.37 $\pm$ 0.01 & 0.22 $\pm$ 0.05 & 0.22 $\pm$ 0.03 & 0.62 $\pm$ 0.08 \\
record & Gaussian sketch + top-$L$ & record; Gaussian sketch + top-$L$; m=46; K=3; L=10 & 4.0 & 2000 & 3 & 0.45 $\pm$ 0.06 & 0.24 $\pm$ 0.03 & 0.25 $\pm$ 0.02 & 0.61 $\pm$ 0.04 \\
record & Private top-$L$ admission + Gaussian values & record; Private top-$L$ admission + Gaussian values; m=46; K=3; L=10 & 1.0 & 2000 & 3 & 0.34 $\pm$ 0.13 & 0.13 $\pm$ 0.08 & 0.14 $\pm$ 0.08 & 0.44 $\pm$ 0.10 \\
record & Private top-$L$ admission + Gaussian values & record; Private top-$L$ admission + Gaussian values; m=46; K=3; L=10 & 4.0 & 2000 & 3 & 0.44 $\pm$ 0.06 & 0.17 $\pm$ 0.04 & 0.20 $\pm$ 0.04 & 0.55 $\pm$ 0.10 \\
user & Gaussian sketch + top-$L$ & user; Gaussian sketch + top-$L$; m=46; K=3; L=10; B=1 & 1.0 & 1000 & 3 & 0.48 $\pm$ 0.08 & 0.18 $\pm$ 0.03 & 0.20 $\pm$ 0.02 & 0.49 $\pm$ 0.06 \\
user & Gaussian sketch + top-$L$ & user; Gaussian sketch + top-$L$; m=46; K=3; L=10; B=1 & 4.0 & 1000 & 3 & 0.45 $\pm$ 0.06 & 0.27 $\pm$ 0.06 & 0.27 $\pm$ 0.07 & 0.61 $\pm$ 0.09 \\
user & Private top-$L$ admission + Gaussian values & user; Private top-$L$ admission + Gaussian values; m=46; K=3; L=10; B=1 & 1.0 & 1000 & 3 & 0.33 $\pm$ 0.15 & 0.09 $\pm$ 0.04 & 0.11 $\pm$ 0.04 & 0.38 $\pm$ 0.11 \\
user & Private top-$L$ admission + Gaussian values & user; Private top-$L$ admission + Gaussian values; m=46; K=3; L=10; B=1 & 4.0 & 1000 & 3 & 0.41 $\pm$ 0.02 & 0.20 $\pm$ 0.03 & 0.22 $\pm$ 0.02 & 0.53 $\pm$ 0.02 \\
\bottomrule
\end{tabular}%
}
\end{table*}

\begin{table*}[t]
\centering
\caption{Examples of  aggregate summaries based on OpenAI-generated outputs on Amazon All Beauty reviews. Each row reports one selected \texttt{DP-SPIN} or baseline release object. The DP-SPIN row is terminology-normalized to the current semantic-support vocabulary; the admitted themes and released numeric values are unchanged.}
\label{tab:openai-qualitative-amazon-all-beauty-reviews}
\scriptsize
\setlength{\tabcolsep}{3pt}
\renewcommand{\arraystretch}{1.04}
\resizebox{0.99\textwidth}{!}{%
\begin{tabular}{p{0.27\textwidth}p{0.68\textwidth}}
\toprule
Method and release configuration & Generated summary \\
\midrule
DP-SPIN private plan (record; Gaussian sketch + top-$L$; m=46; K=3; L=10; $\varepsilon=1$; $n=2000$) & The aggregate summary identifies high semantic support for easy application, sensitive skin, and waste money. Moderate semantic support is reported for product quality and long hair, while dry hair, natural hair, hair quality, dry skin, and skin feel have low support. - High-support themes: easy application, sensitive skin, waste money. - Moderate-support themes: product quality, long hair. - Low-support themes: dry hair, natural hair, hair quality, dry skin, skin feel. \\ \hline
\addlinespace[1pt]
DP keyword histogram (fixed public keyword vocabulary; vocab=200; Laplace histogram; released top-10 keywords; $\varepsilon=1$; $n=2000$) & The aggregate summary reveals various keyword counts related to hair, products, skin, quality, and price, indicating a diverse range of interests. The most prominent focus is on hair and product keywords, followed by skin and quality. Other notable mentions include aspects related to ease of use, smell, and texture, with financial considerations also being present. - Hair and product keywords are the most frequently mentioned. - Skin and quality keywords show significant interest. - Additional mentions include ease of use, smell, and texture. \\ \hline
\addlinespace[1pt]
DP category histogram (fixed rating\_bin universe; categories=3; Laplace histogram; released top-3 categories; $\varepsilon=1$; $n=2000$) & The aggregate summary indicates a distribution of ratings within the public category field, with a notable support of high ratings, followed by low and medium ratings. - High ratings are the most common. - Low ratings are present but less frequent. - Medium ratings are the least common among the three categories. \\ \hline
\addlinespace[1pt]
URANIA-style public keywords (DP clustering; DP public-keyword histograms; clusters=10; keywords=10; $\varepsilon=1$; $n=2000$) & The preserved cluster entries highlight various aspects of product offerings and customer experiences, particularly in the realm of hair care and natural products. There is a notable emphasis on synthetic hair care products designed for sensitive skin, as well as natural product scents that consider factors like pricing and waste. Additionally, customer feedback indicates a preference for fast shipping and high-quality products that effectively address irritation and chemical concerns. - Synthetic hair care products are tailored for sensitive skin, emphasizing softness and short styles. - Natural product scents are linked to considerations of price and environmental impact. - Customers value quick delivery of high-quality products that minimize waste and address irritation. \\
\addlinespace[1pt]
\bottomrule
\end{tabular}%
}
\end{table*}

\clearpage
\appsection{Additional Experimental Results: Yelp Reviews}
\label{app:additional-results-yelp}

\vspace{-7pt}

In this appendix, we report additional experimental results on the \texttt{Yelp restaurant} reviews.

\vspace{-7pt}

\subsection{Record-level results}

\begin{figure}[h]
    \centering
    \includegraphics[width=0.49\linewidth]{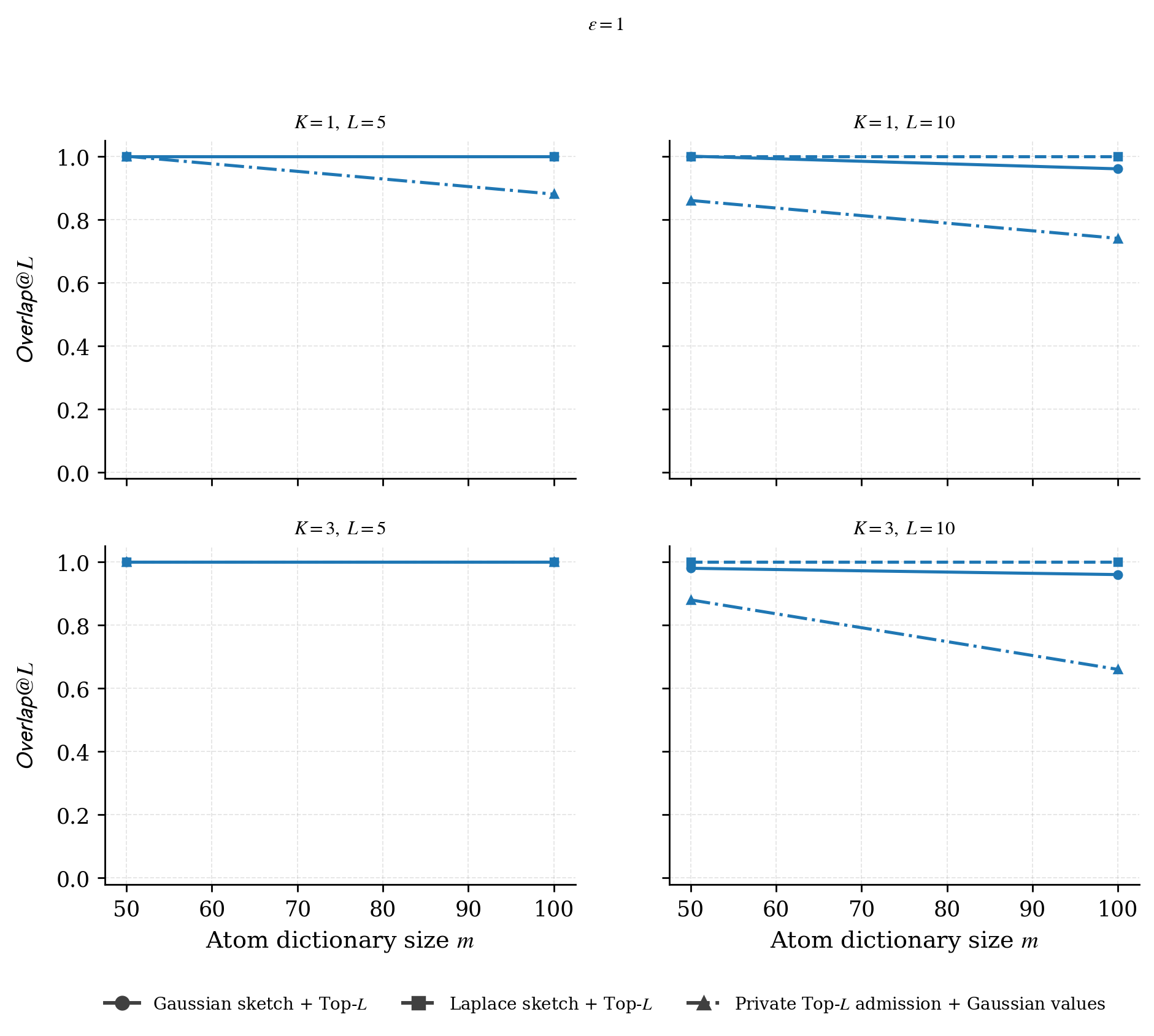}
    \vspace{-5pt}
    \caption{
    Yelp record-level ablation of $\mathsf{Overlap}@L$ as the atom dictionary size $m$ varies at $\varepsilon=1.0$. Panels vary the assignment sparsity $K$ and released plan size $L$; curves compare the release mechanisms.
    }
    \vspace{-10pt}
    \label{fig:yelp-record-overlap-vs-m}
\end{figure}

\begin{figure}[h]
    \centering
    \includegraphics[width=0.49\linewidth]{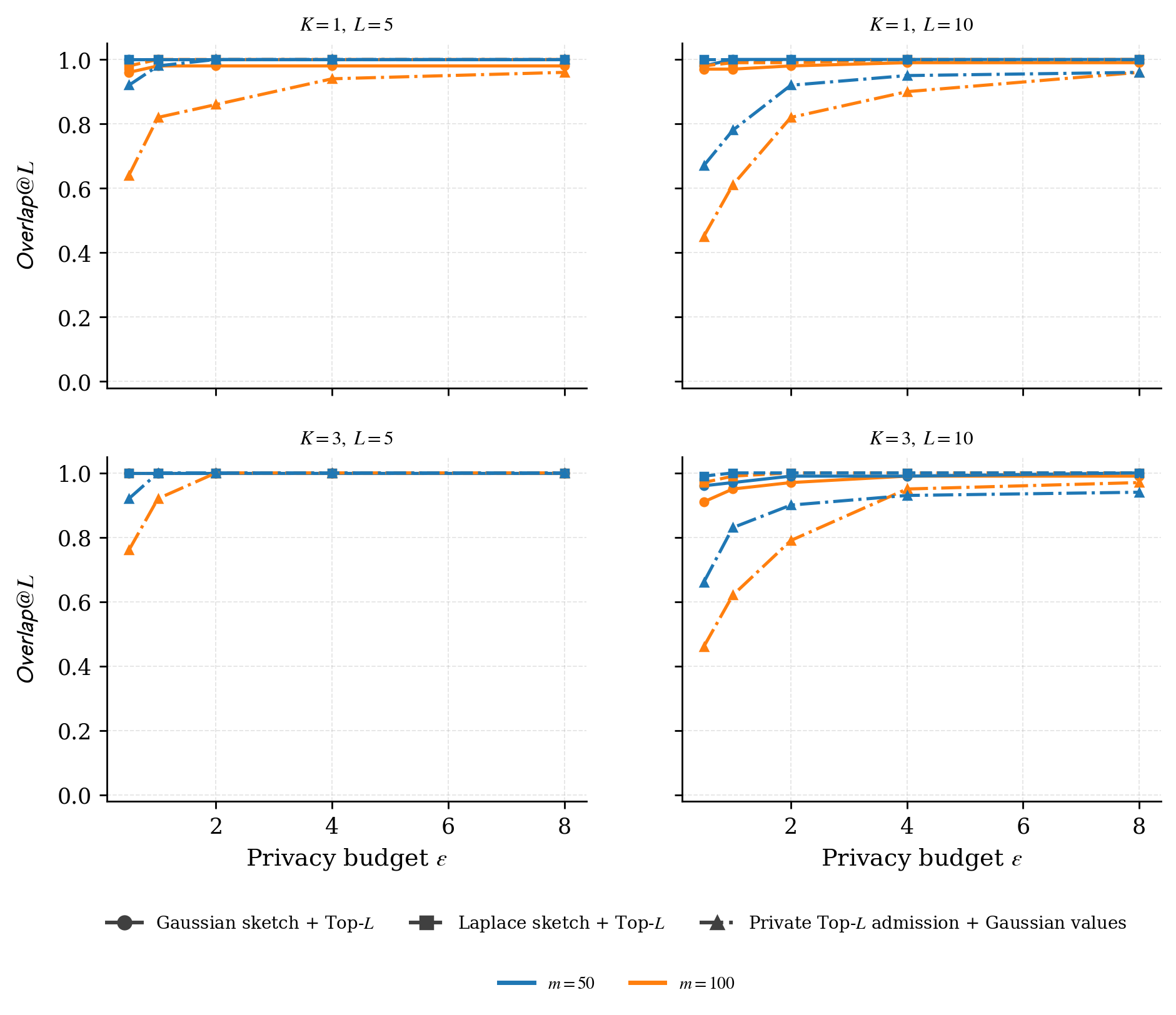}
    \vspace{-5pt}
    \caption{
    Yelp record-level ablation of $\mathsf{Overlap}@L$ as the privacy budget $\varepsilon$ varies. Panels vary the assignment sparsity $K$ and released plan size $L$; colors indicate the atom dictionary size $m$, and line styles indicate the release mechanism.
    }
    \label{fig:yelp-record-overlap-vs-epsilon}
\end{figure}

\begin{figure}[h]
    \centering
    \includegraphics[width=0.63\linewidth]{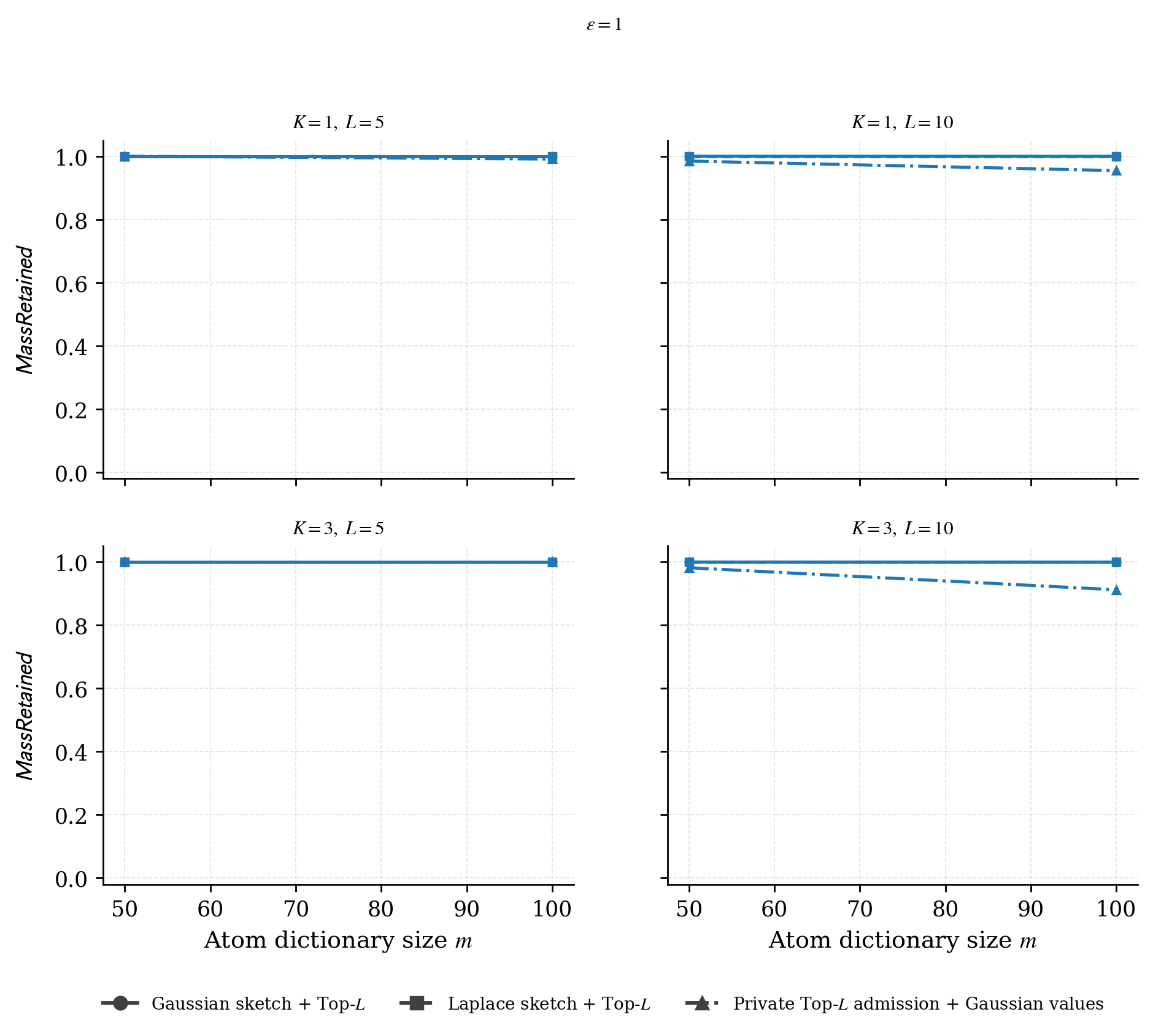}
    \caption{
    Yelp record-level ablation of the preserved non-private mass fraction $\mathsf{MassRetained}$ as the atom dictionary size $m$ varies at $\varepsilon=1.0$. Panels vary $K$ and $L$; curves compare the release mechanisms.
    }
    \label{fig:yelp-record-mass-preserved-vs-m}
\end{figure}

\begin{figure}[h]
    \centering
    \includegraphics[width=0.63\linewidth]{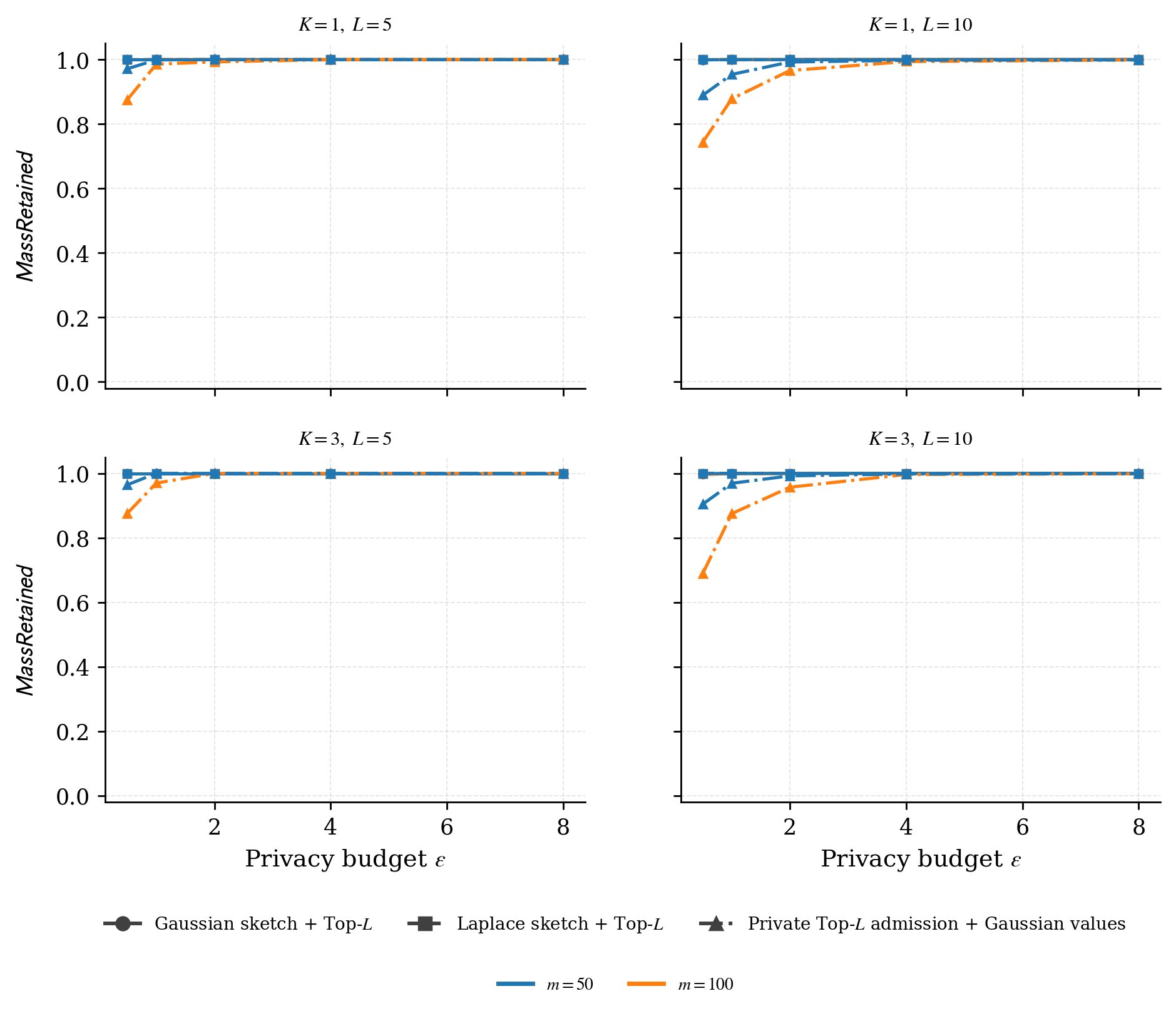}
    \caption{
    Yelp record-level ablation of the preserved non-private mass fraction $\mathsf{MassRetained}$ as the privacy budget $\varepsilon$ varies. Panels vary $K$ and $L$; colors indicate the atom dictionary size $m$, and line styles indicate the release mechanism.
    }
    \label{fig:yelp-record-mass-preserved-vs-epsilon}
\end{figure}

\begin{figure}[h]
    \centering
    \includegraphics[width=0.63\linewidth]{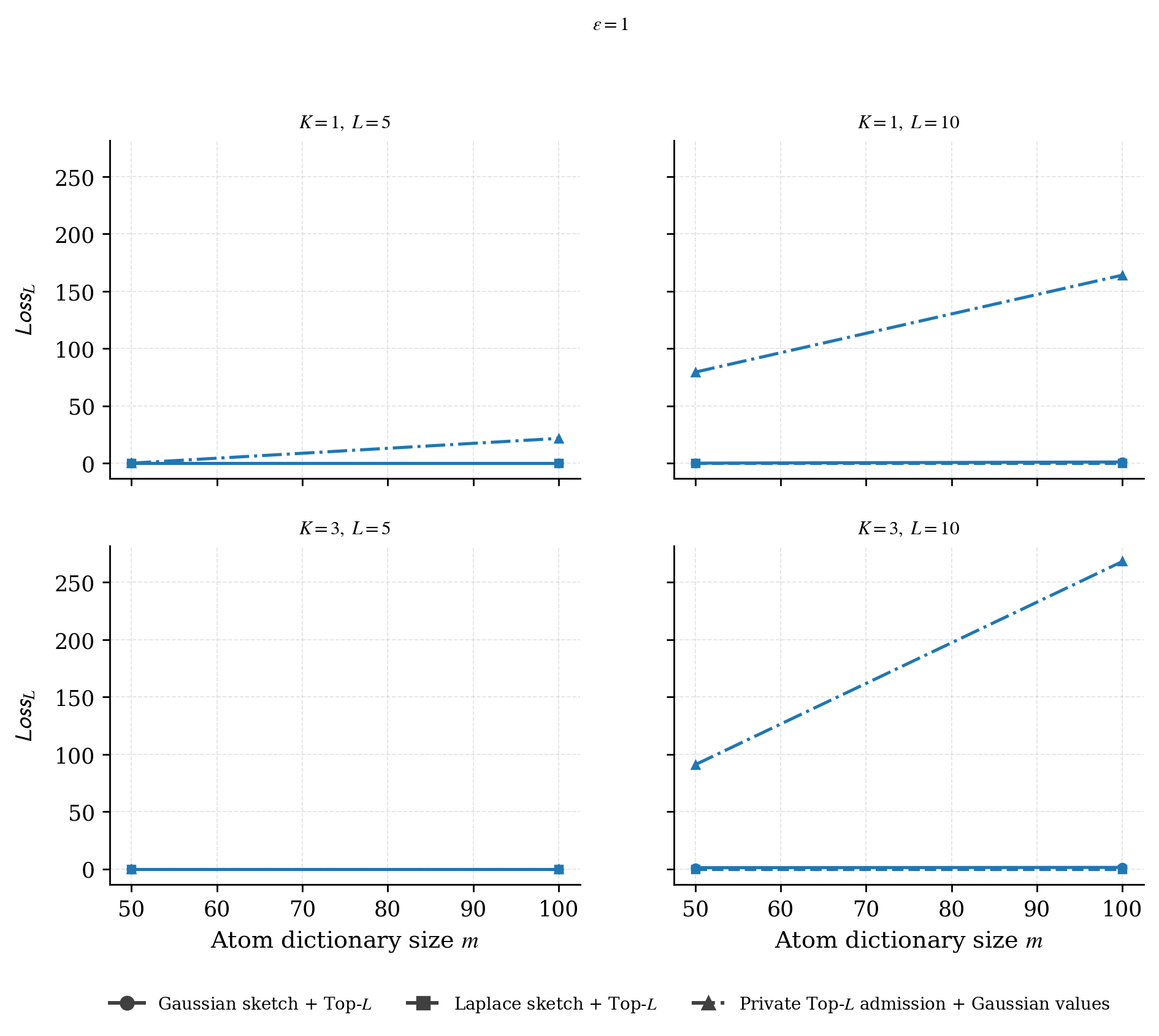}
    \caption{
    Yelp record-level ablation of support-mass loss $\mathsf{Loss}_L$ as the atom dictionary size $m$ varies at $\varepsilon=1.0$. Panels vary $K$ and $L$; curves compare the release mechanisms. Lower values indicate less non-private semantic mass lost by the released atom set.
    }
    \label{fig:yelp-record-loss-vs-m}
\end{figure}

\begin{figure}[h]
    \centering
    \includegraphics[width=0.63\linewidth]{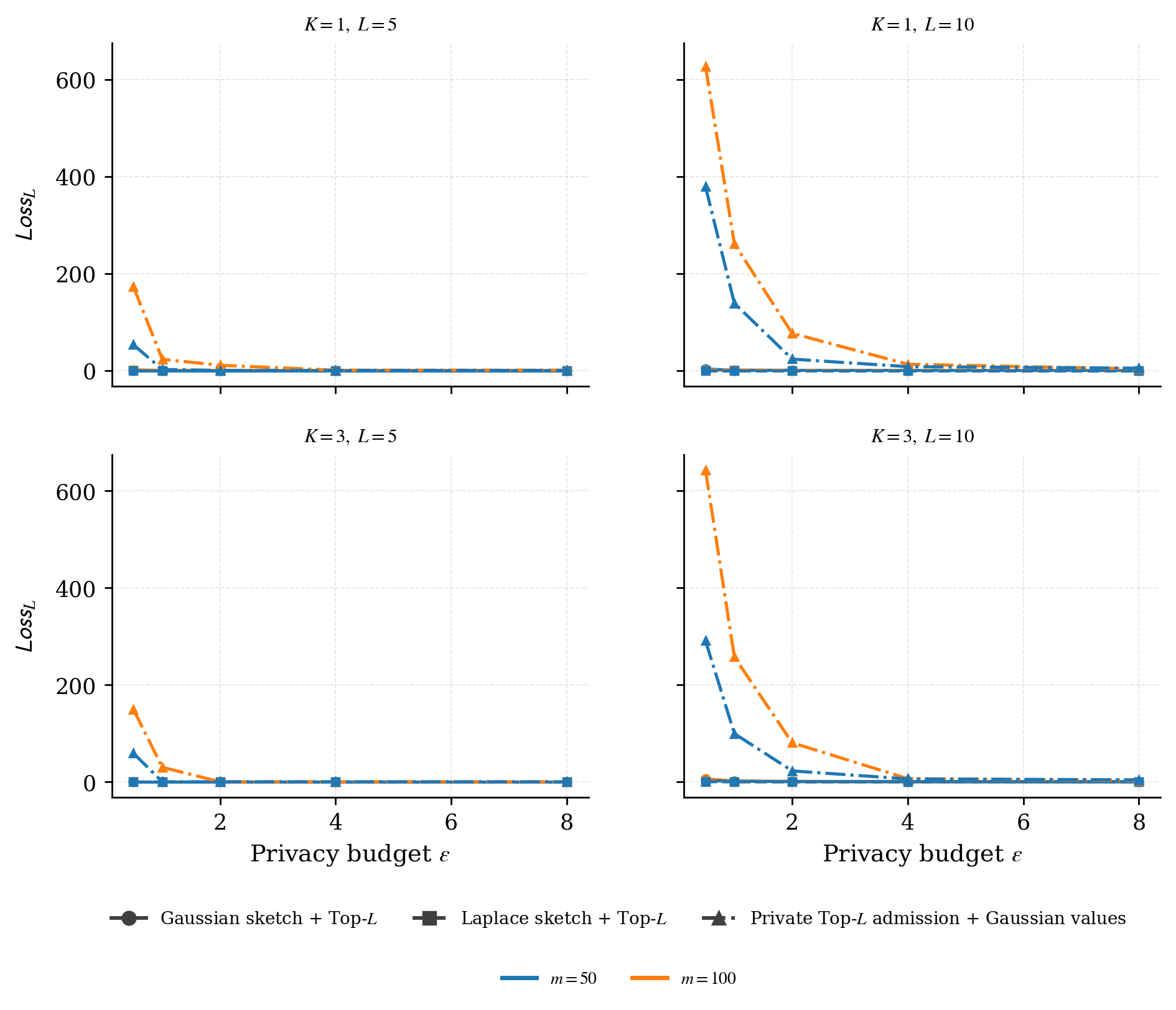}
    \caption{
    Yelp record-level ablation of support-mass loss $\mathsf{Loss}_L$ as the privacy budget $\varepsilon$ varies. Panels vary $K$ and $L$; colors indicate the atom dictionary size $m$, and line styles indicate the release
    mechanism. Lower values indicate less non-private semantic mass lost by the released atom set.
    }
    \label{fig:yelp-record-loss-vs-epsilon}
\end{figure}

\begin{figure}[h]
    \centering
    \includegraphics[width=0.63\linewidth]{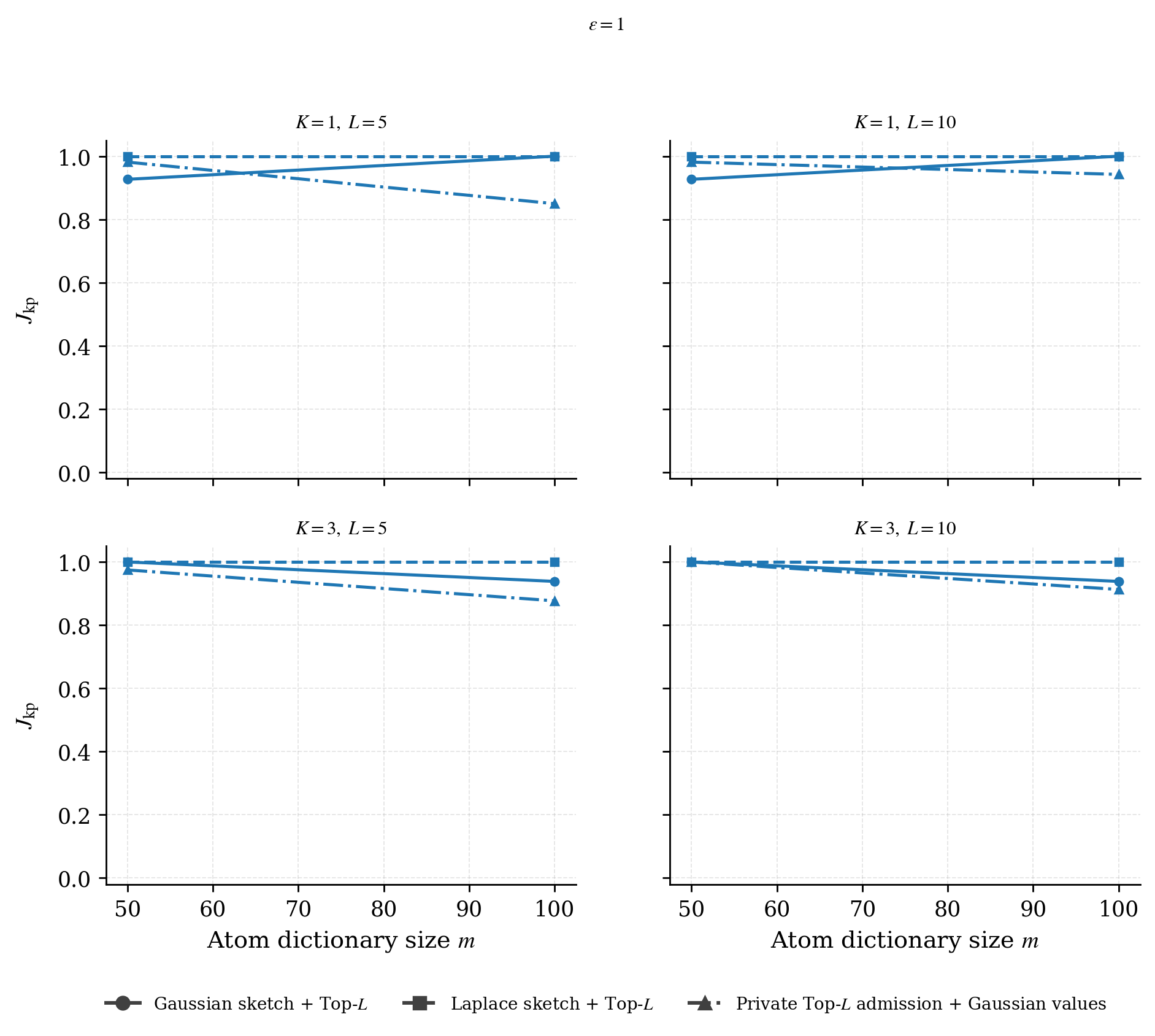}
    \caption{
    Yelp record-level keyphrase Jaccard similarity $J_{\mathrm{kp}}$ between the DP-plan summary and the non-private-plan reference summary as the atom dictionary size $m$ varies at $\varepsilon=1.0$. Panels vary $K$ and $L$; curves compare the release mechanisms.
    }
    \label{fig:yelp-record-keyphrase-jaccard-vs-m}
\end{figure}

\begin{figure}[h]
    \centering
    \includegraphics[width=0.63\linewidth]{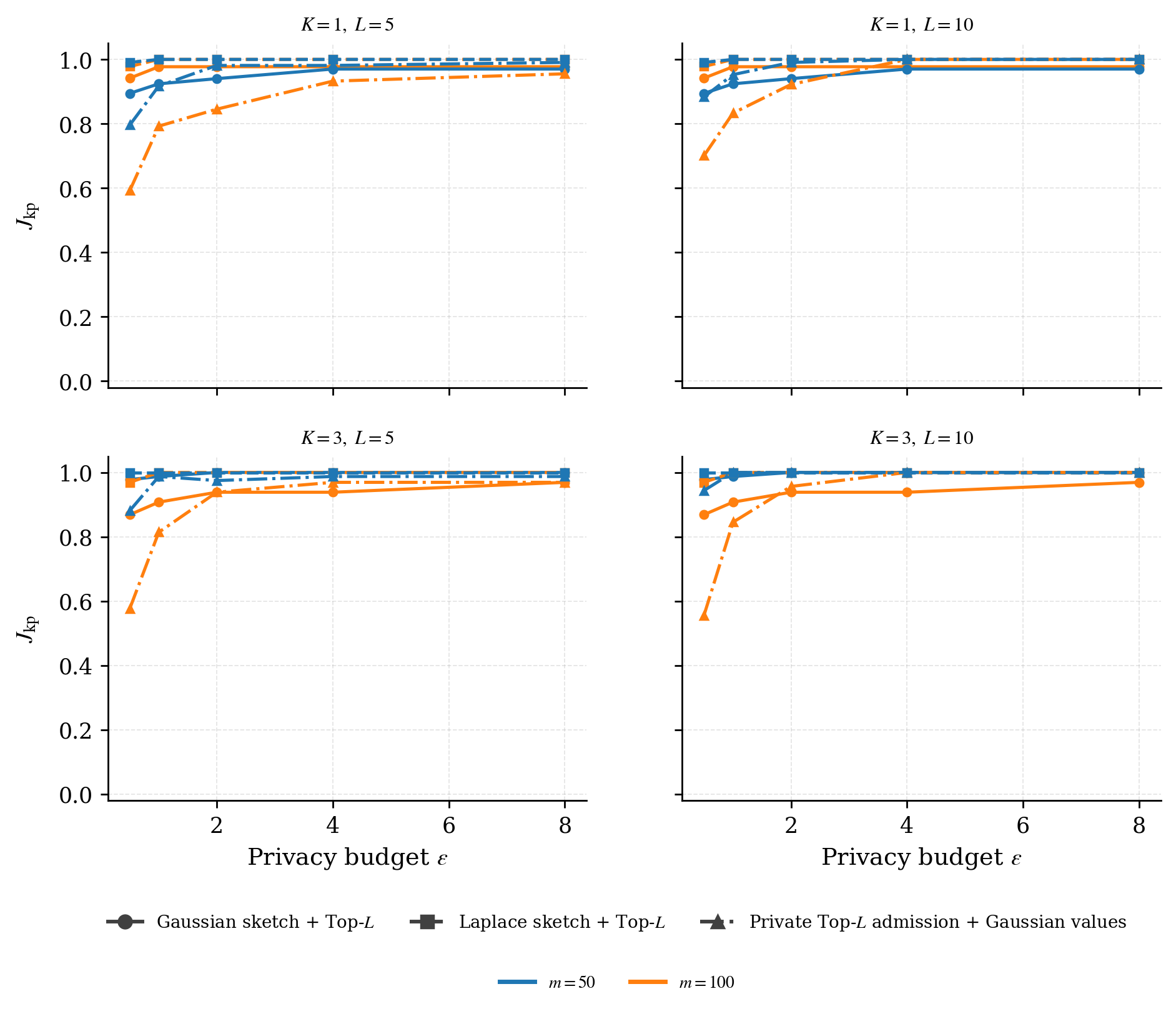}
    \caption{
    Yelp record-level keyphrase Jaccard similarity $J_{\mathrm{kp}}$ between the DP-plan summary and the non-private-plan reference summary as the privacy budget $\varepsilon$ varies. Panels vary $K$ and $L$; colors indicate the atom dictionary size $m$, and line styles indicate the release mechanism.
    }
    \label{fig:yelp-record-keyphrase-jaccard-vs-epsilon}
\end{figure}

\begin{figure}[h]
    \centering
    \includegraphics[width=0.63\linewidth]{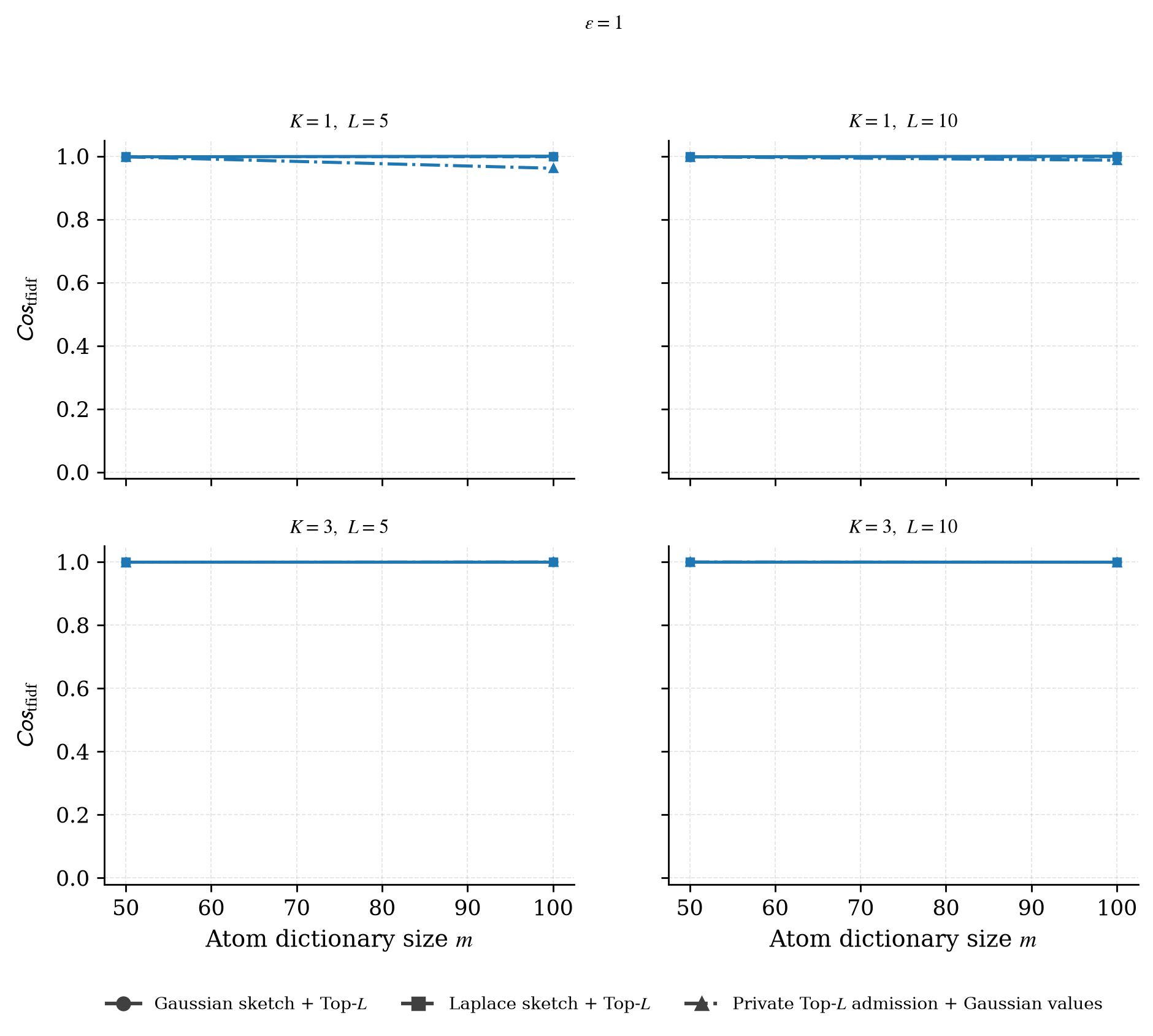}
    \caption{Yelp record-level TF--IDF cosine similarity $\mathsf{Cos}_{\mathrm{tfidf}}$ between the DP-plan summary and the non-private-plan reference summary as the atom dictionary size $m$ varies at $\varepsilon=1.0$. Panels vary $K$ and $L$; curves compare the release mechanisms.}
    \label{fig:yelp-record-tfidf-cosine-vs-m}
\end{figure}

\begin{figure}[h]
    \centering
    \includegraphics[width=0.63\linewidth]{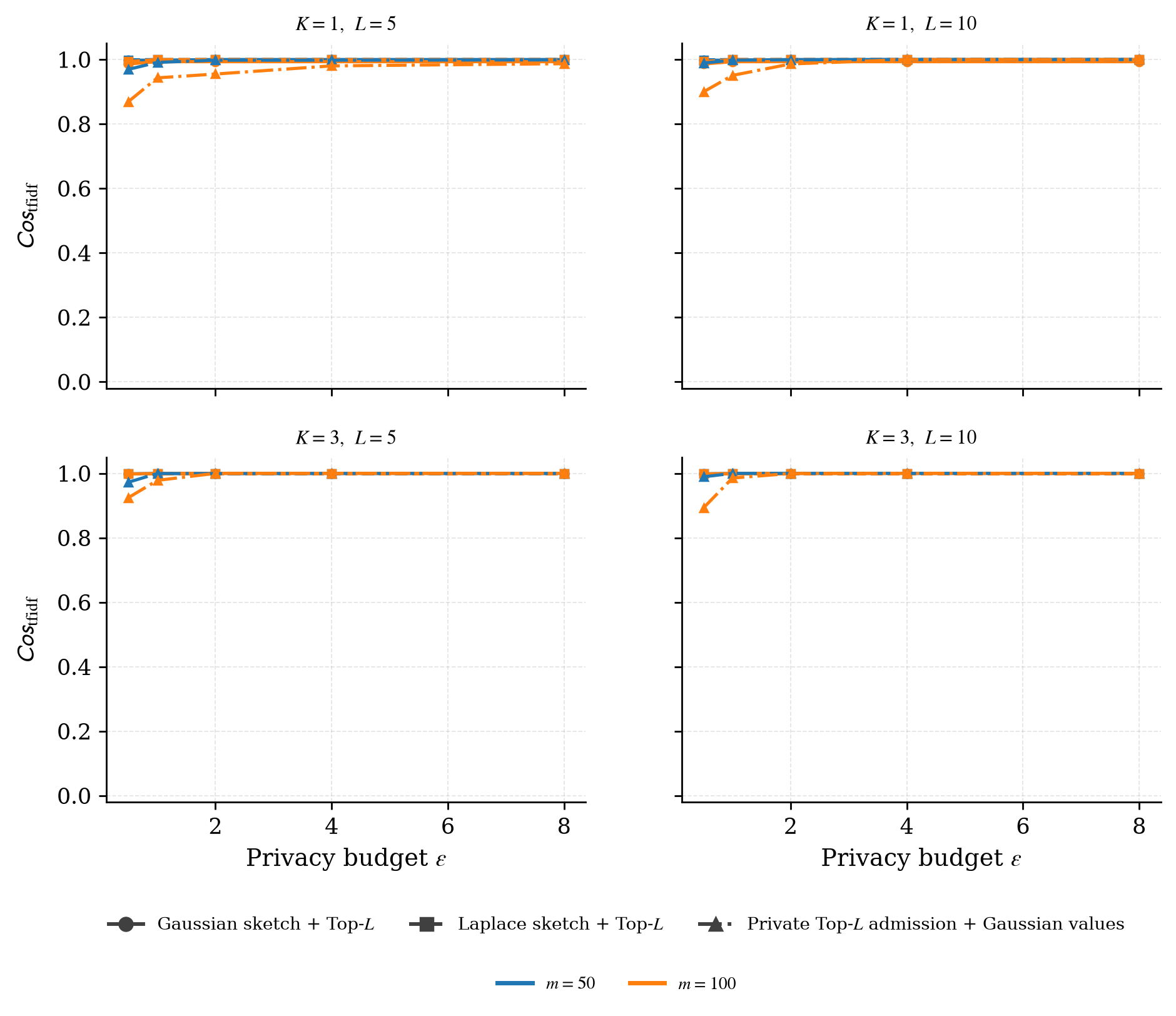}
    \caption{Yelp record-level TF--IDF cosine similarity $\mathsf{Cos}_{\mathrm{tfidf}}$ between the DP-plan summary and the non-private-plan reference summary as the privacy budget $\varepsilon$ varies. Panels vary $K$ and $L$; colors indicate the atom dictionary size $m$, and line styles indicate the release mechanism.}
    \label{fig:yelp-record-tfidf-cosine-vs-epsilon}
\end{figure}

\begin{figure}[h]
    \centering
    \includegraphics[width=0.7\linewidth]{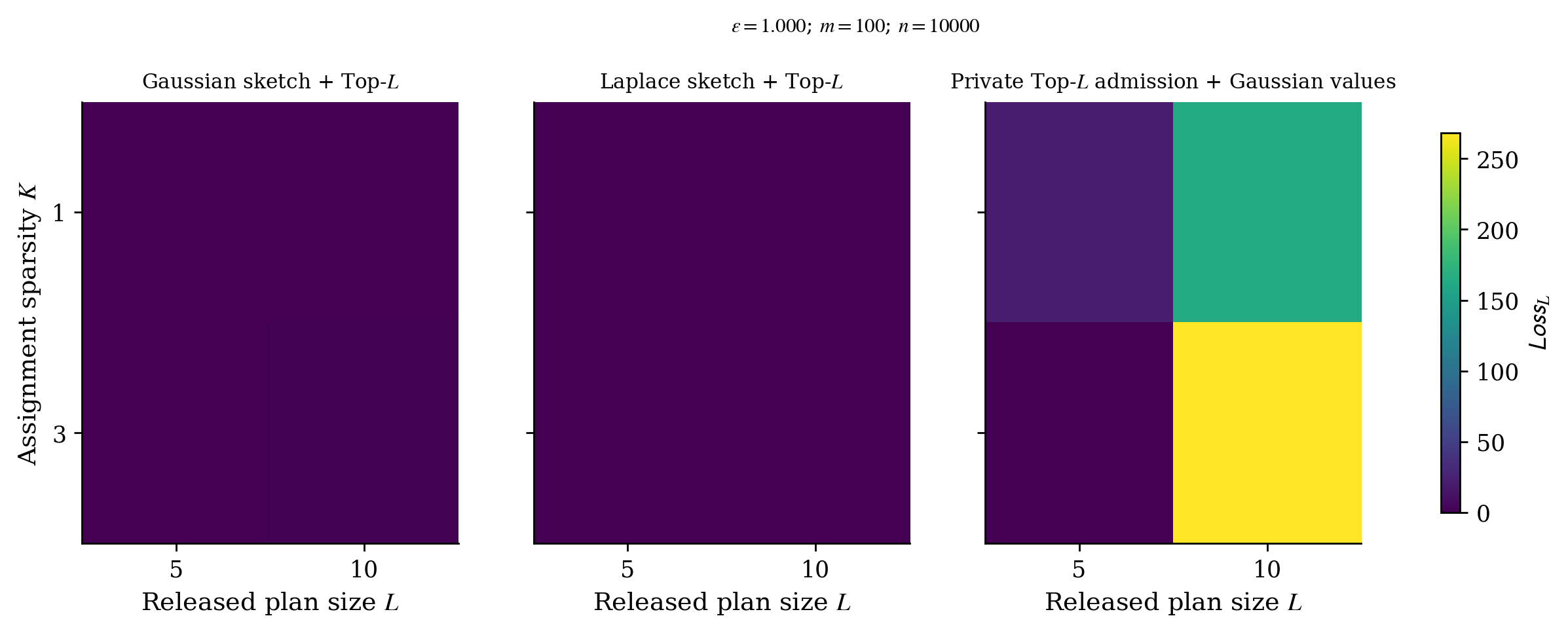}
    \caption{
    Yelp record-level heatmap of support-mass loss $\mathsf{Loss}_L$ over assignment sparsity $K$ and released plan size $L$, with $\varepsilon=1.0$, $m=100$, and $n=10000$. Each panel corresponds to one release mechanism. Lower values indicate less non-private semantic mass lost by the released atom set.
    }
    \label{fig:yelp-record-loss-heatmap}
\end{figure}

\begin{figure}[h]
    \centering
    \begin{minipage}{0.49\linewidth}
        \centering
        \includegraphics[width=\linewidth]{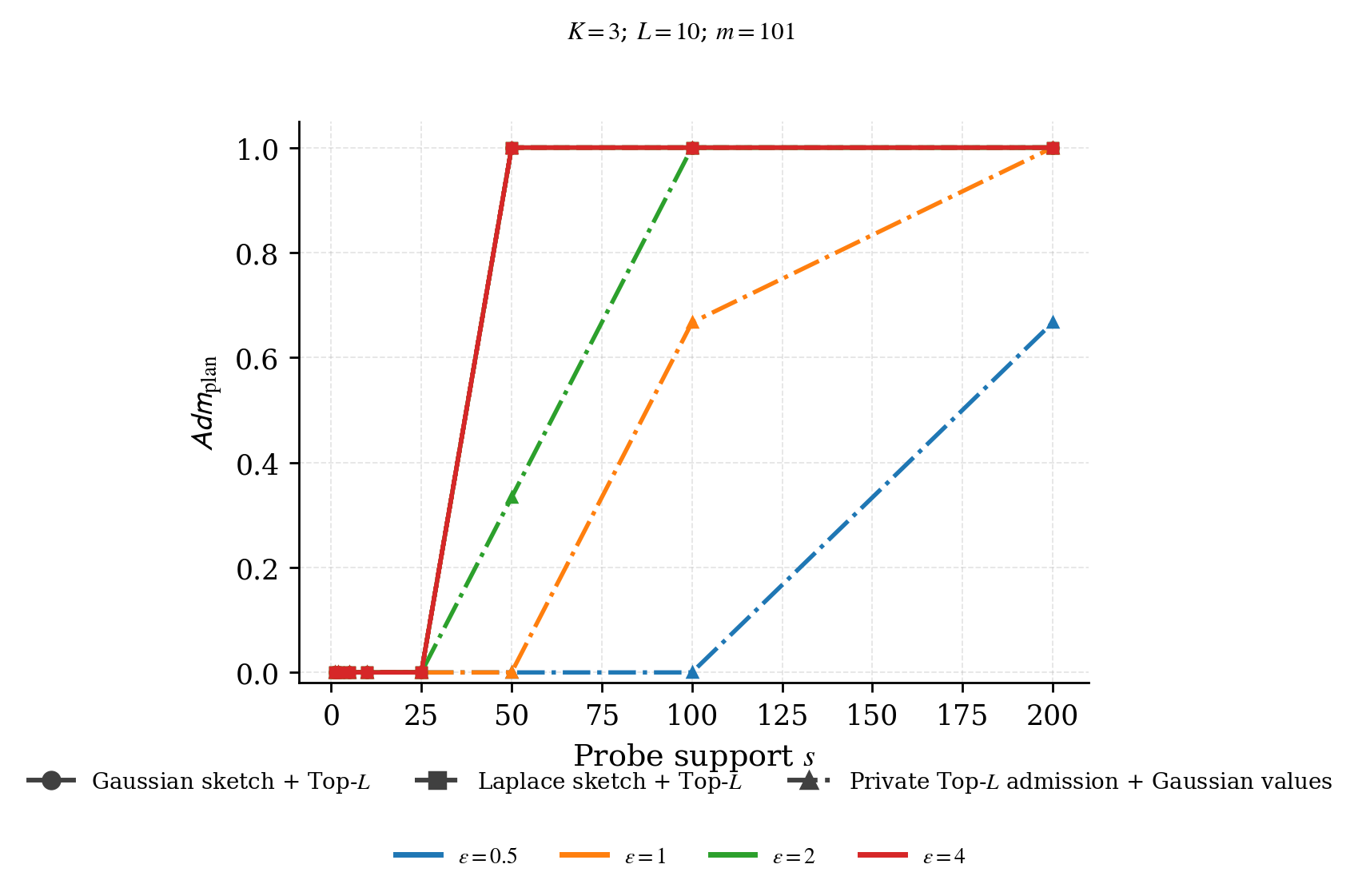}
        \vspace{2pt}
        \textbf{(a)} Plan admission.
    \end{minipage}
    \hfill
    \begin{minipage}{0.49\linewidth}
        \centering
        \includegraphics[width=\linewidth]{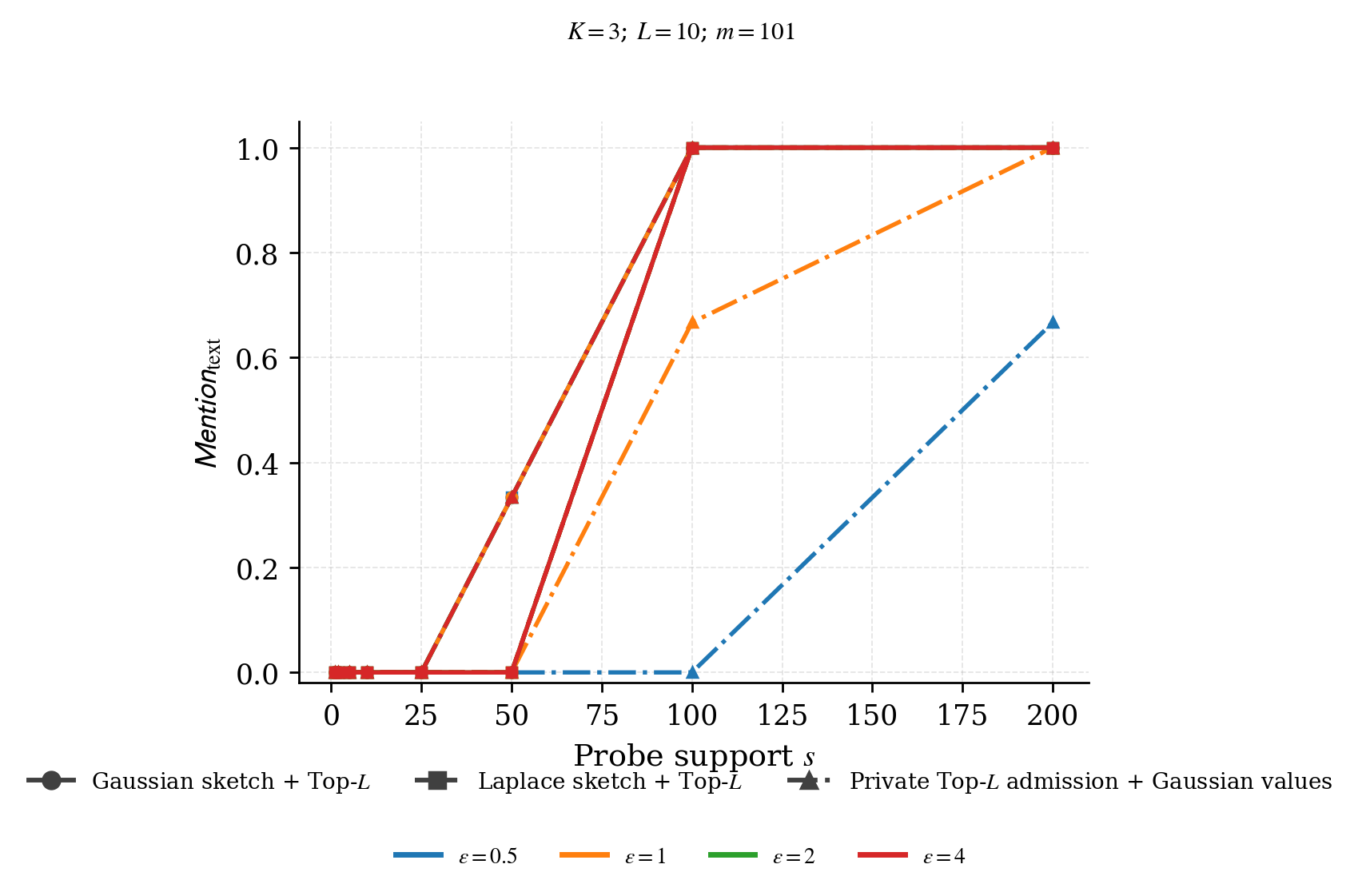}
        \vspace{2pt}
        \textbf{(b)} Summary mention.
    \end{minipage}
    \caption{
    Yelp record-level controlled probe-atom experiment as the injected probe support $s$ varies, with $K=3$, $L=10$, and $m=101$ public atoms including the probe atom. Panel~(a) reports the plan-admission rate $\mathsf{Adm}_{\mathrm{plan}}$, which measures whether the differentially private release admits the public probe atom into the released semantic plan. Panel~(b) reports the summary mention rate $\mathsf{Mention}_{\mathrm{text}}$, which measures whether the public probe string is detected in the decoded summary. Curves compare privacy budgets and release mechanisms under the configured Yelp record-level probe-sweep setting.
    }
    \label{fig:yelp-record-probe-admission-mention}
\end{figure}

\clearpage

\subsection{User-level results}

\vspace{-20pt}

\begin{figure}[h]
    \centering
    \includegraphics[width=0.52\linewidth]{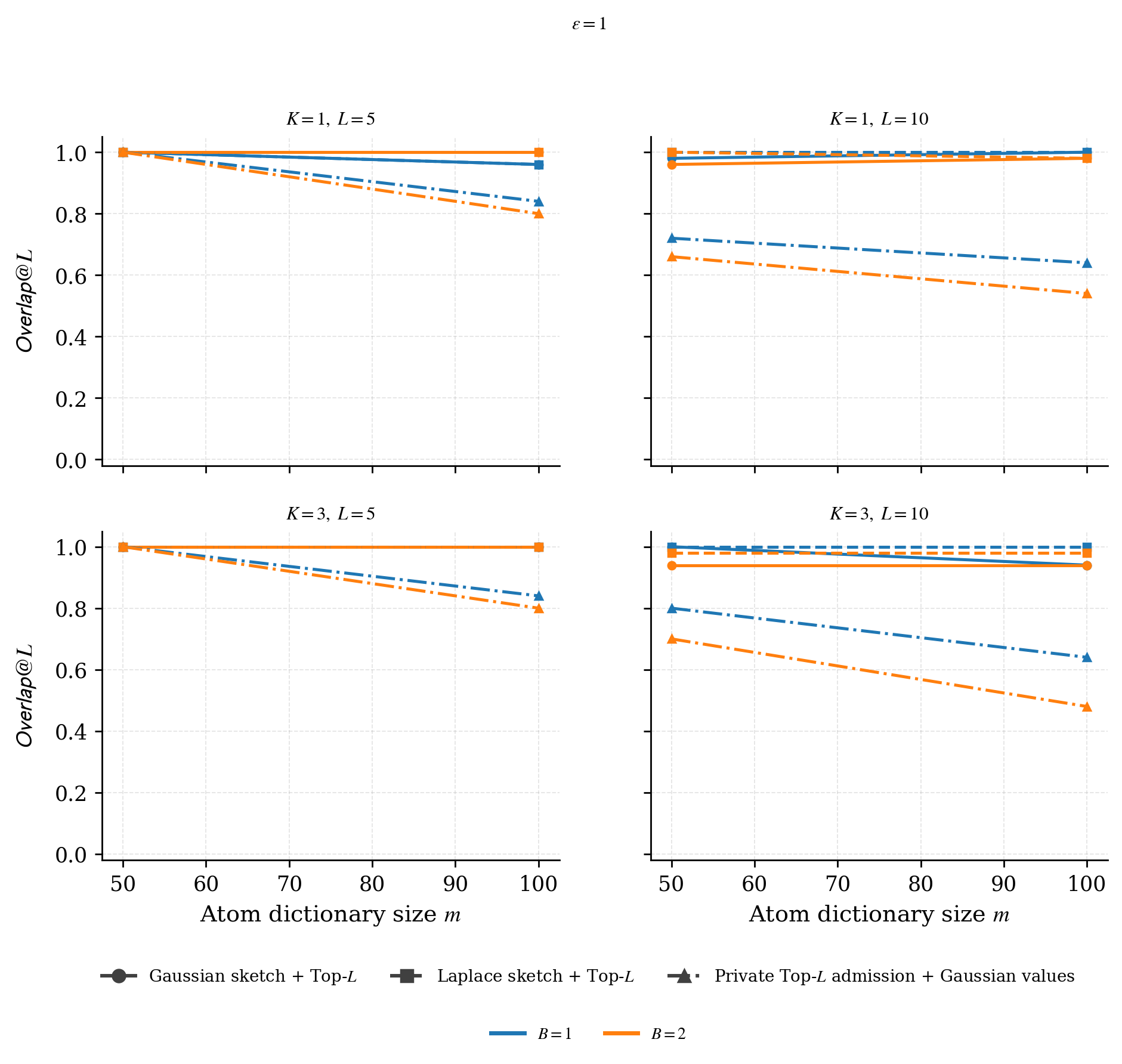}
    \vspace{-8pt}
    \caption{
    Yelp user-level ablation of $\mathsf{Overlap}@L$ as the atom dictionary size $m$ varies at $\varepsilon=1.0$. Panels vary the assignment sparsity $K$ and released plan size $L$; curves compare the release mechanisms.
    }
    \vspace{-20pt}
    \label{fig:yelp-user-overlap-vs-m}
\end{figure}

\begin{figure}[h]
    \centering
    \includegraphics[width=0.52\linewidth]{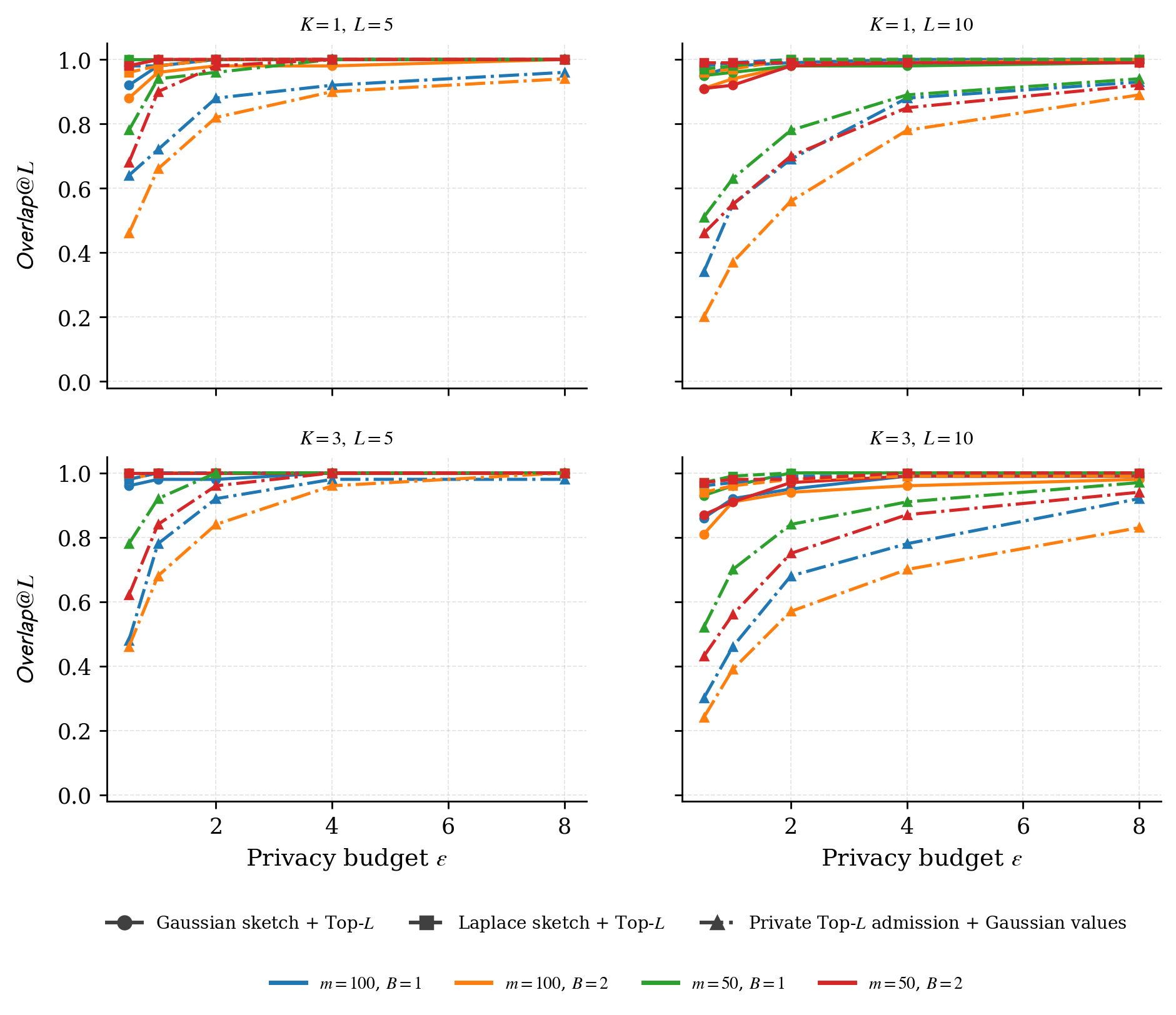}
    \vspace{-8pt}
    \caption{
    Yelp user-level ablation of $\mathsf{Overlap}@L$ as the privacy budget $\varepsilon$ varies. Panels vary the assignment sparsity $K$ and released plan size $L$; curves compare the release mechanisms and user-level configurations.
    }
    \label{fig:yelp-user-overlap-vs-epsilon}
\end{figure}

\begin{figure}[h]
    \centering
    \includegraphics[width=0.61\linewidth]{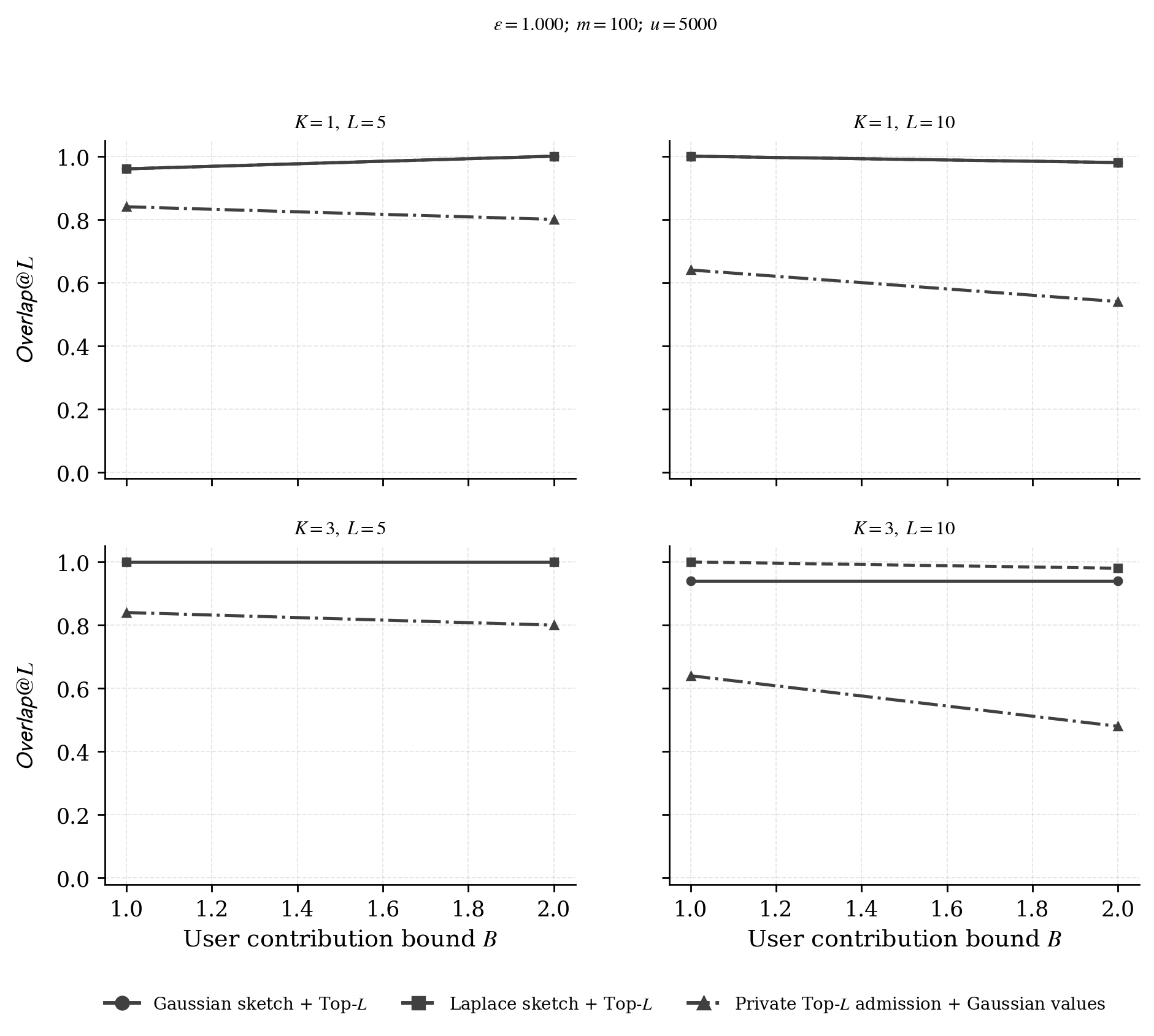}
    \caption{
    Yelp user-level ablation of $\mathsf{Overlap}@L$ as the user contribution bound $B$ varies, with $\varepsilon=1.0$, $m=100$, and $u=5000$ selected users. Panels vary $K$ and $L$; curves compare the release mechanisms.
    }
    \label{fig:yelp-user-overlap-vs-B}
\end{figure}

\begin{figure}[h]
    \centering
    \includegraphics[width=0.62\linewidth]{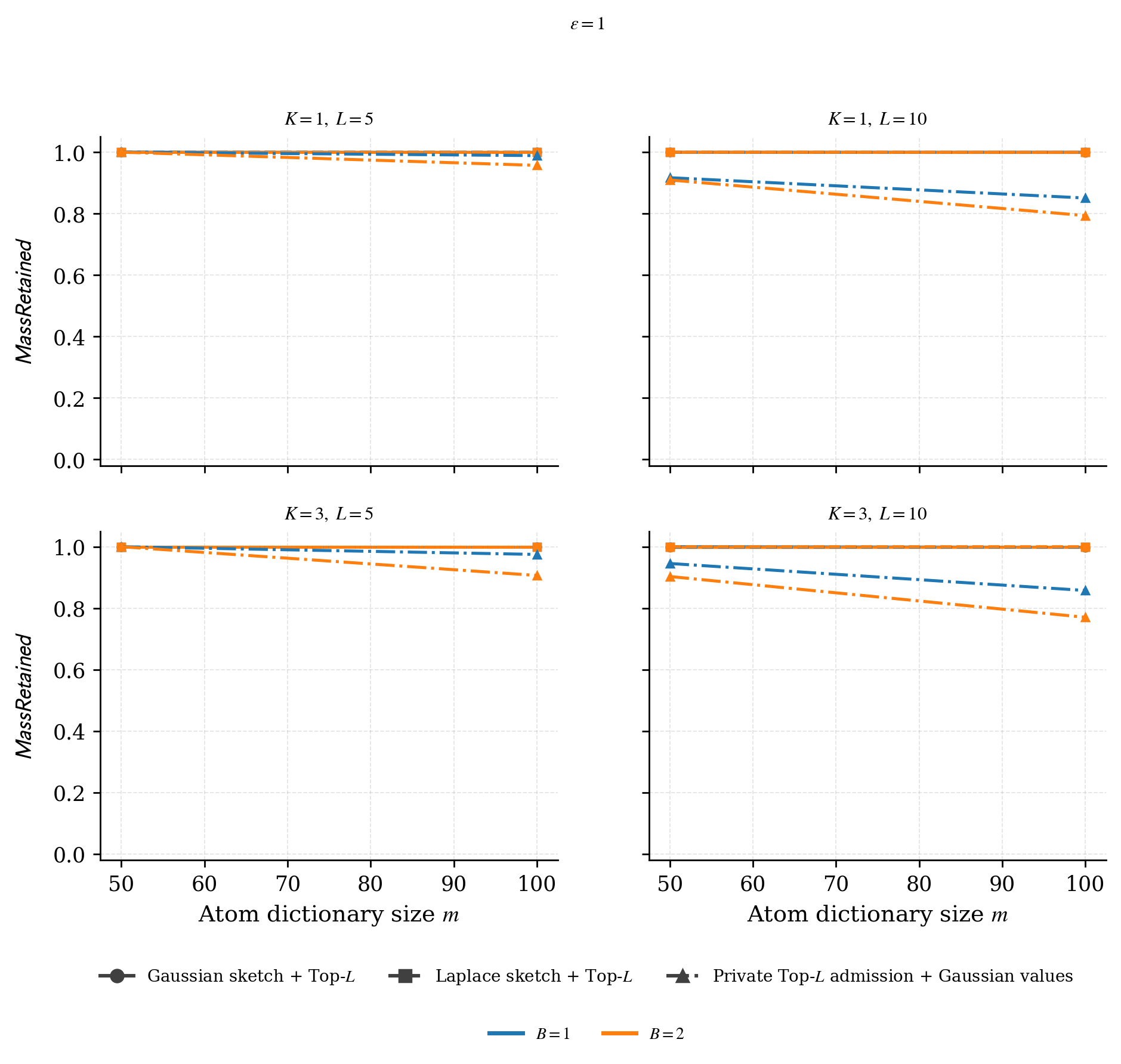}
    \caption{
    Yelp user-level ablation of the preserved non-private mass fraction $\mathsf{MassRetained}$ as the atom dictionary size $m$ varies at
    $\varepsilon=1.0$. Panels vary $K$ and $L$; curves compare the release mechanisms.
    }
    \label{fig:yelp-user-mass-preserved-vs-m}
\end{figure}

\begin{figure}[h]
    \centering
    \includegraphics[width=0.63\linewidth]{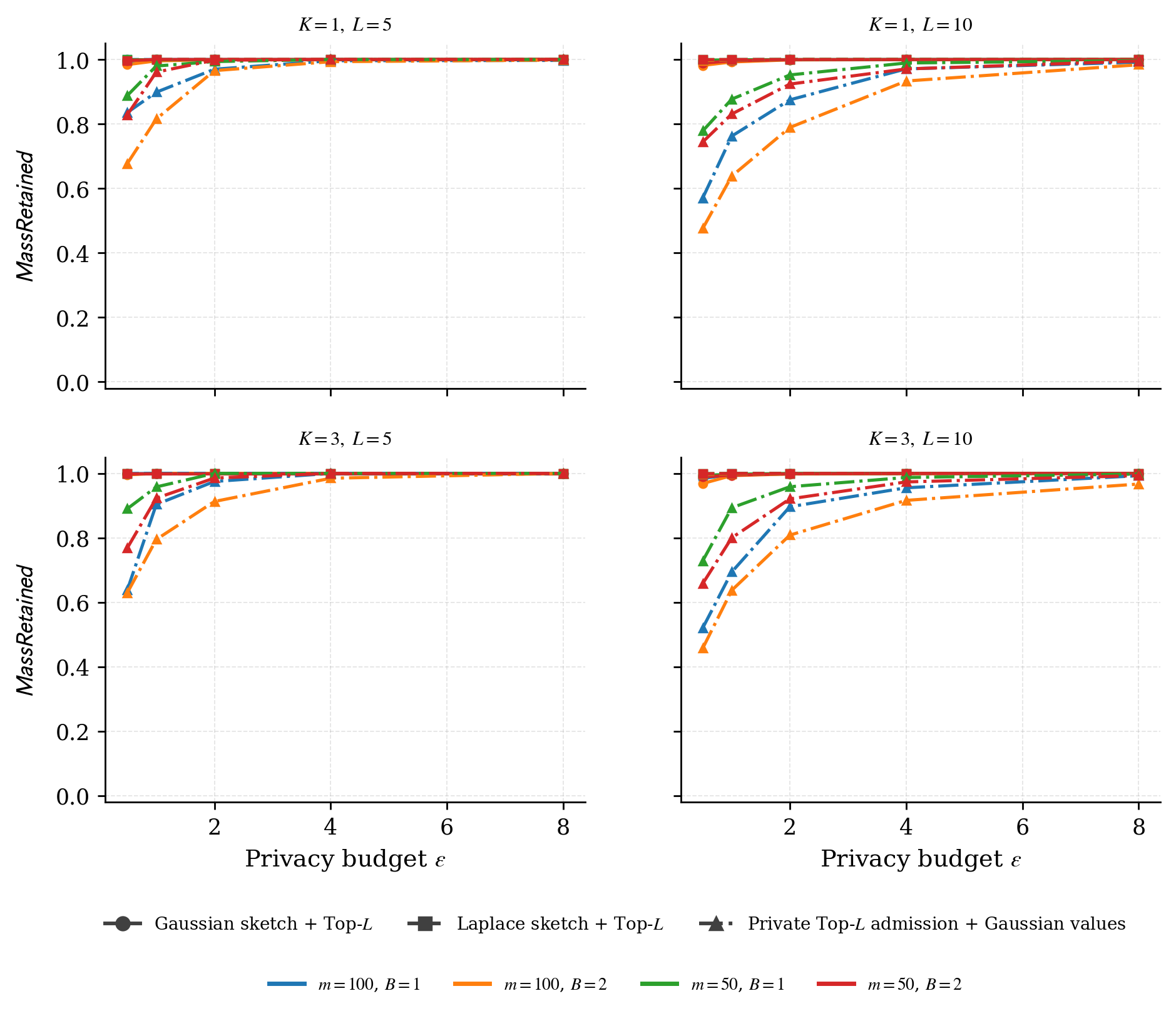}
    \caption{
    Yelp user-level ablation of the preserved non-private mass fraction $\mathsf{MassRetained}$ as the privacy budget $\varepsilon$ varies. Panels vary $K$ and $L$; curves compare the release mechanisms and user-level configurations.
    }
    \label{fig:yelp-user-mass-preserved-vs-epsilon}
\end{figure}

\begin{figure}[h]
    \centering
    \includegraphics[width=0.62\linewidth]{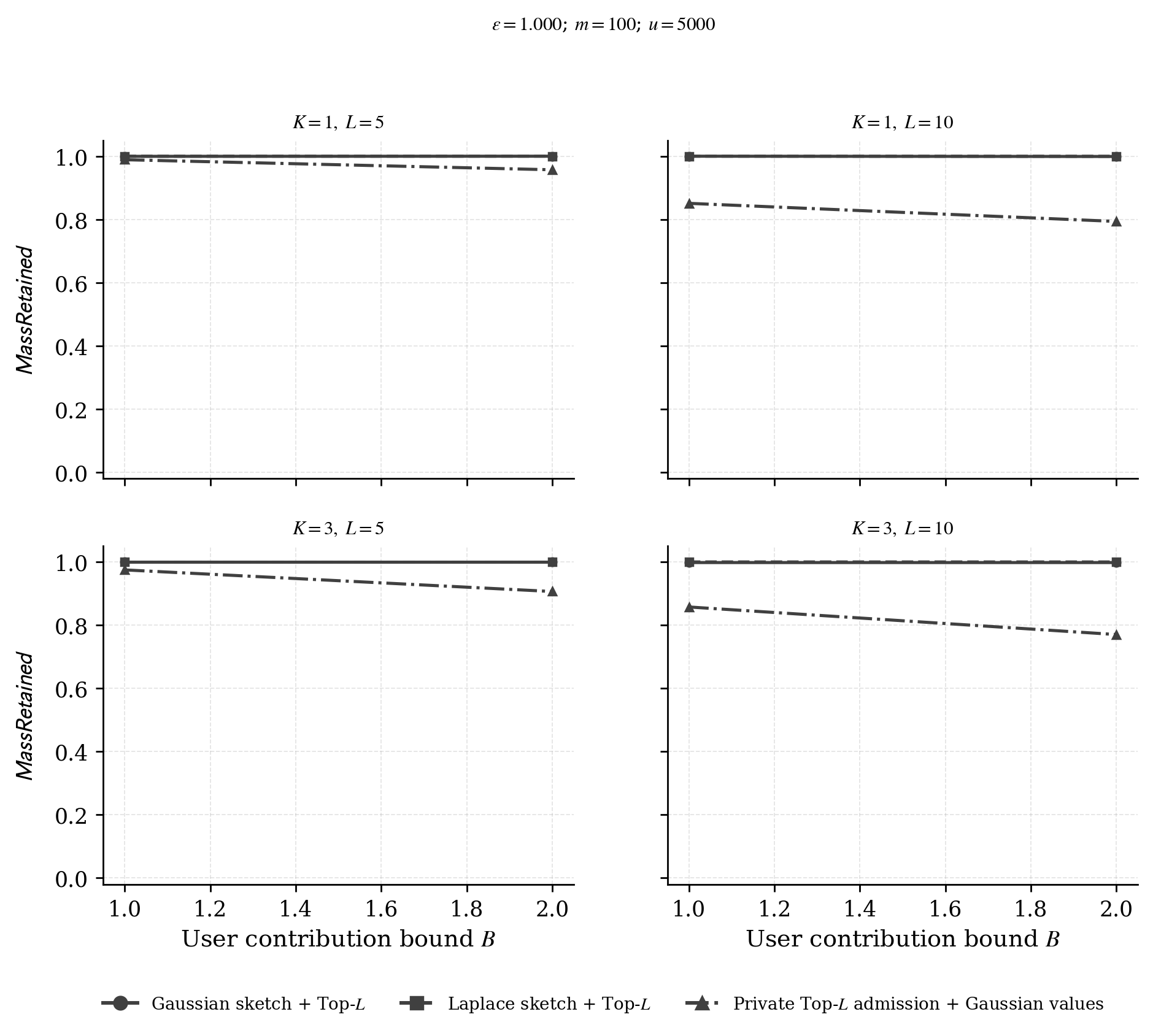}
    \caption{
    Yelp user-level ablation of the preserved non-private mass fraction $\mathsf{MassRetained}$ as the user contribution bound $B$ varies, with $\varepsilon=1.0$, $m=100$, and $u=5000$ selected users. Panels vary $K$ and $L$; curves compare the release mechanisms.
    }
    \label{fig:yelp-user-mass-preserved-vs-B}
\end{figure}

\begin{figure}[h]
    \centering
    \includegraphics[width=0.63\linewidth]{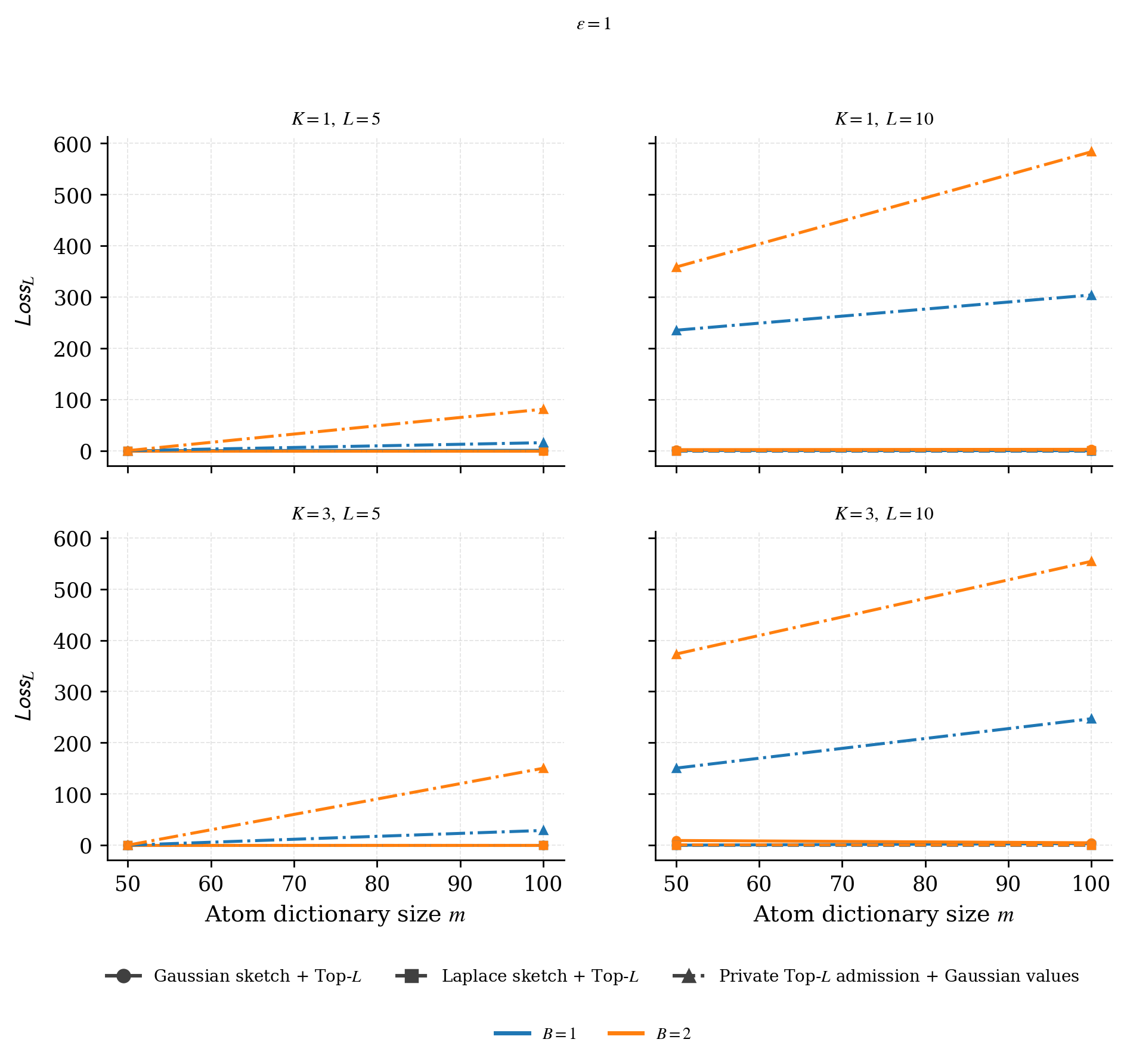}
    \caption{
    Yelp user-level ablation of support-mass loss $\mathsf{Loss}_L$ as the atom dictionary size $m$ varies at $\varepsilon=1.0$. Panels vary $K$ and $L$; curves compare the release mechanisms. Lower values indicate less non-private semantic mass lost by the released atom set.
    }
    \label{fig:yelp-user-loss-vs-m}
\end{figure}

\begin{figure}[h]
    \centering
    \includegraphics[width=0.63\linewidth]{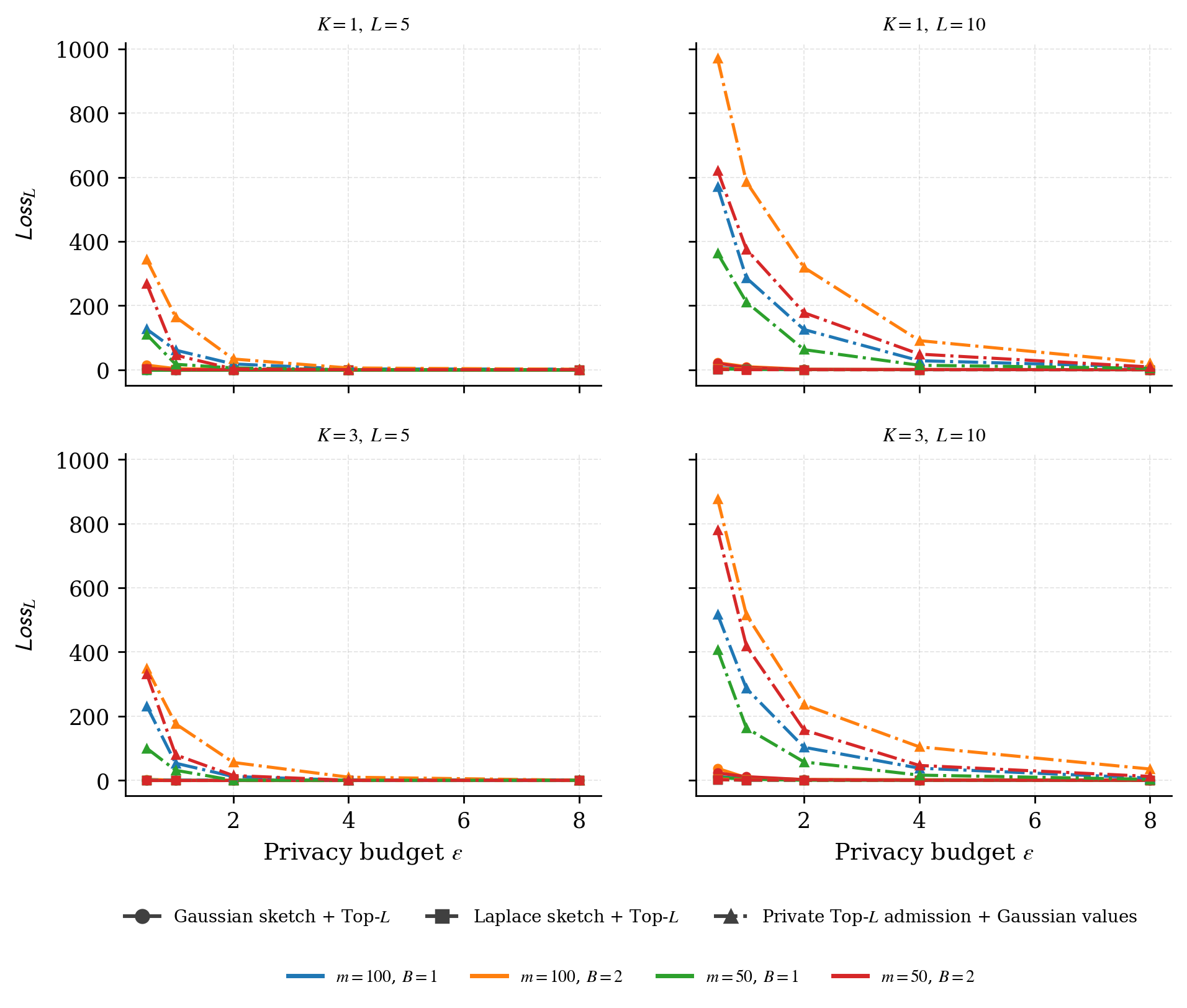}
    \caption{
    Yelp user-level ablation of support-mass loss $\mathsf{Loss}_L$ as the privacy budget $\varepsilon$ varies. Panels vary $K$ and $L$; curves compare the release mechanisms and user-level configurations. Lower values indicate less non-private semantic mass lost by the released atom set.
    }
    \label{fig:yelp-user-loss-vs-epsilon}
\end{figure}

\begin{figure}[h]
    \centering
    \includegraphics[width=0.6\linewidth]{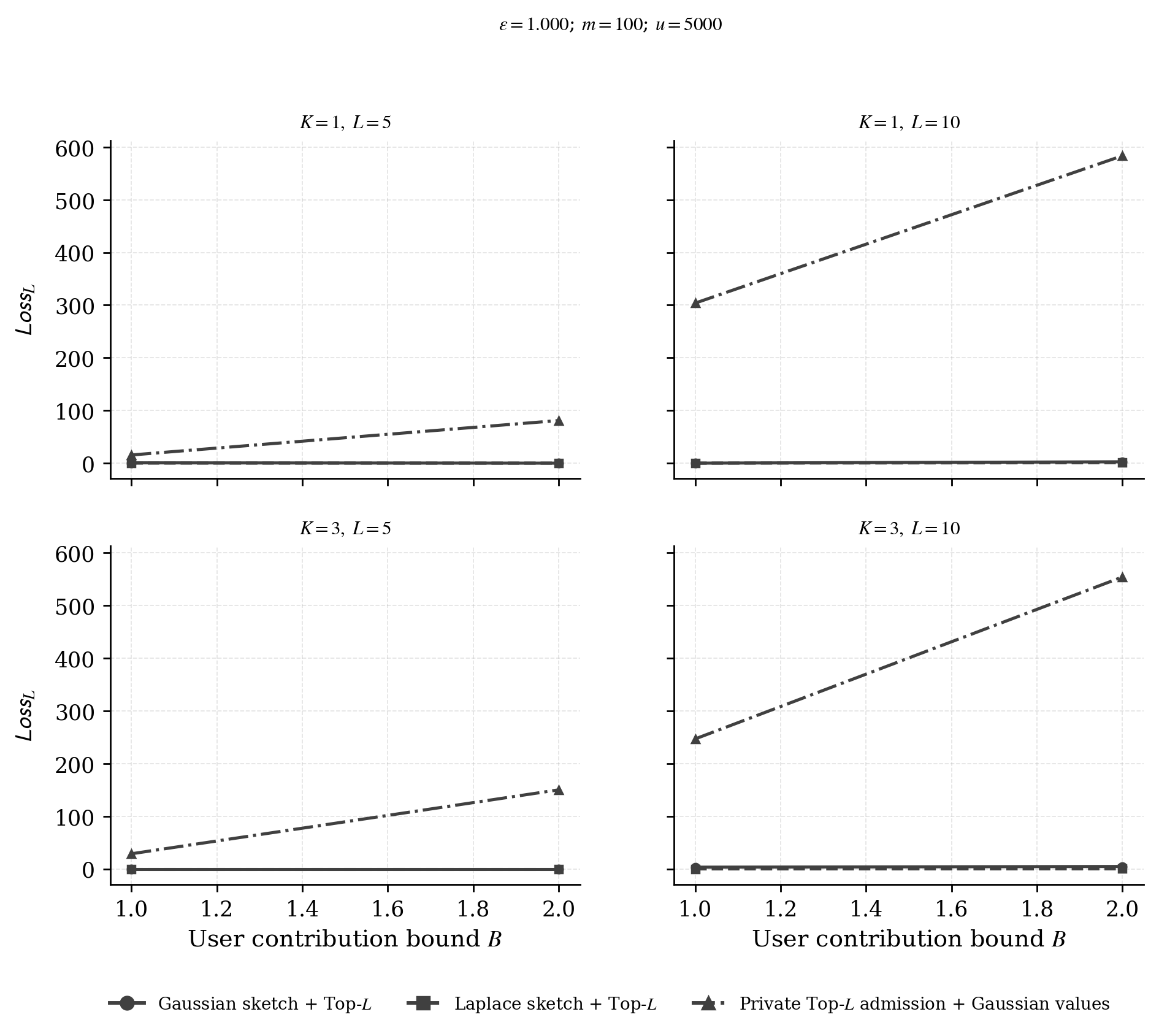}
    \caption{
    Yelp user-level ablation of support-mass loss $\mathsf{Loss}_L$ as the user contribution bound $B$ varies, with $\varepsilon=1.0$, $m=100$, and $u=5000$ selected users. Panels vary $K$ and $L$; curves compare the release mechanisms. Lower values indicate less non-private semantic mass lost by the released atom set.
    }
    \label{fig:yelp-user-loss-vs-B}
\end{figure}

\begin{figure}[h]
    \centering
    \includegraphics[width=0.61\linewidth]{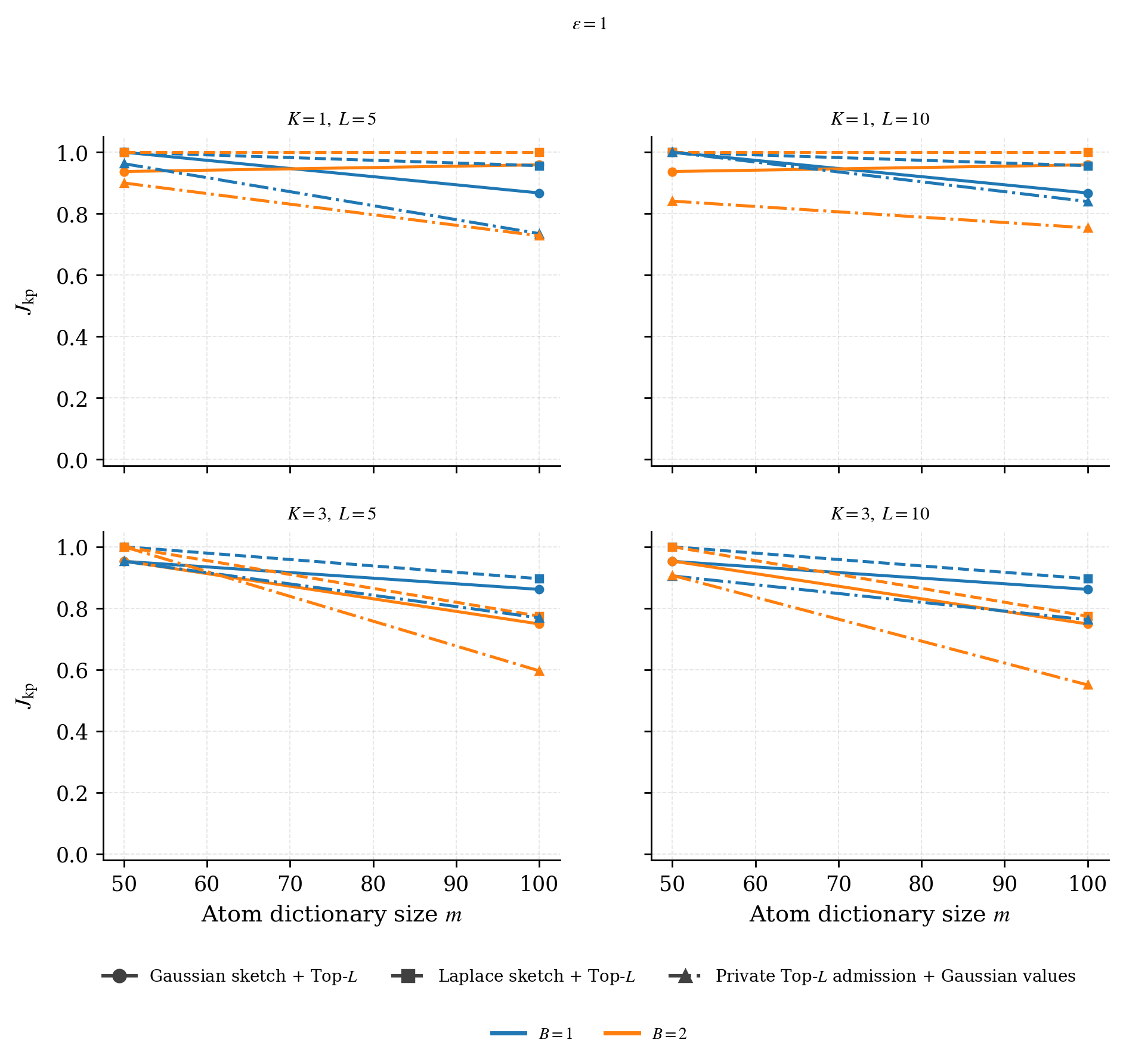}
    \caption{
    Yelp user-level keyphrase Jaccard similarity $J_{\mathrm{kp}}$ between the  DP-plan summary and the non-private-plan reference summary as the atom dictionary size $m$ varies at $\varepsilon=1.0$. Panels vary $K$ and $L$; curves compare the release mechanisms.
    }
    \label{fig:yelp-user-keyphrase-jaccard-vs-m}
\end{figure}

\begin{figure}[h]
    \centering
    \includegraphics[width=0.61\linewidth]{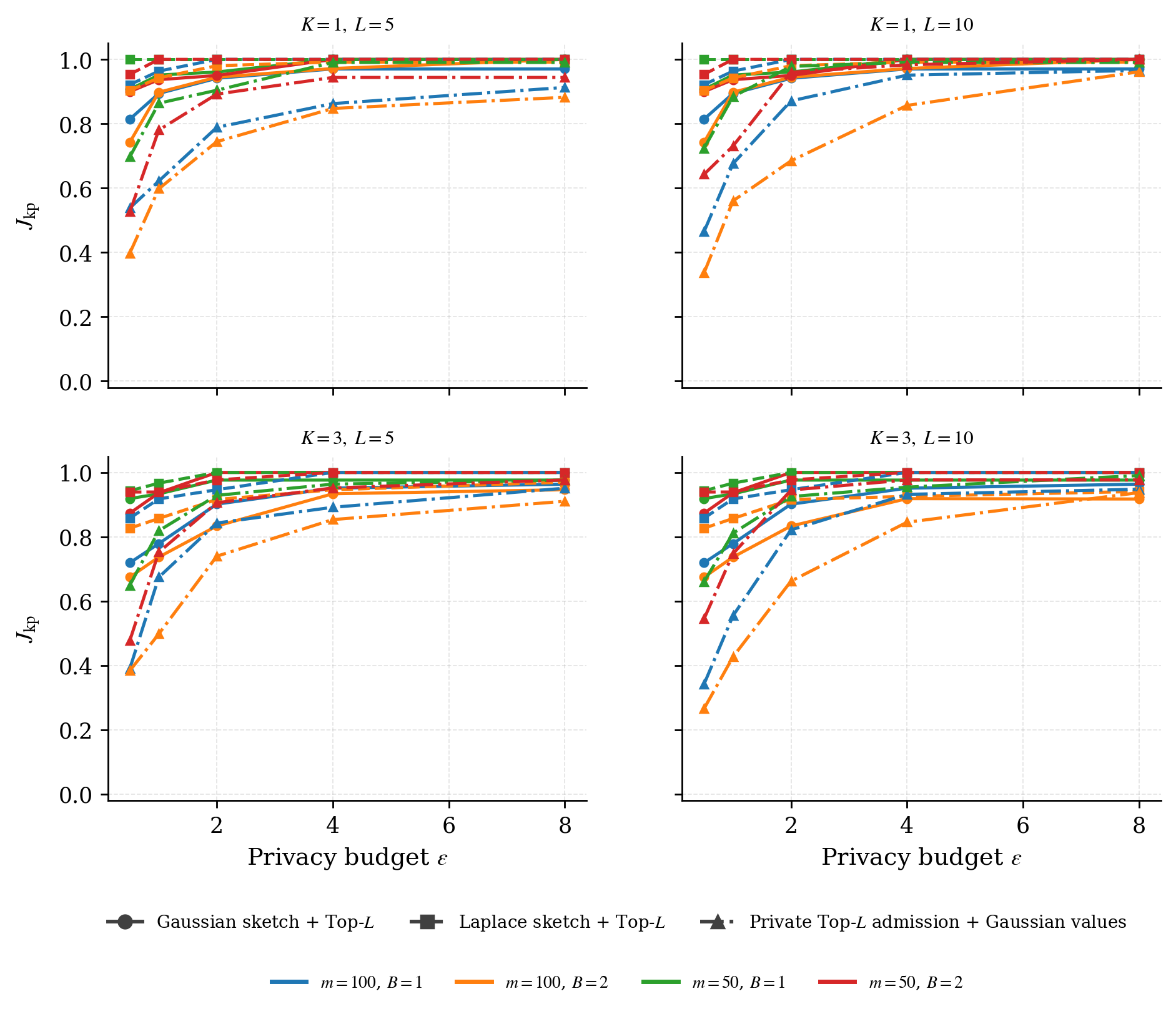}
    \caption{
    Yelp user-level keyphrase Jaccard similarity $J_{\mathrm{kp}}$ between the DP-plan summary and the non-private-plan reference summary as the privacy budget $\varepsilon$ varies. Panels vary $K$ and $L$; curves compare the release mechanisms and user-level configurations.
    }
    \label{fig:yelp-user-keyphrase-jaccard-vs-epsilon}
\end{figure}

\begin{figure}[h]
    \centering
    \includegraphics[width=0.61\linewidth]{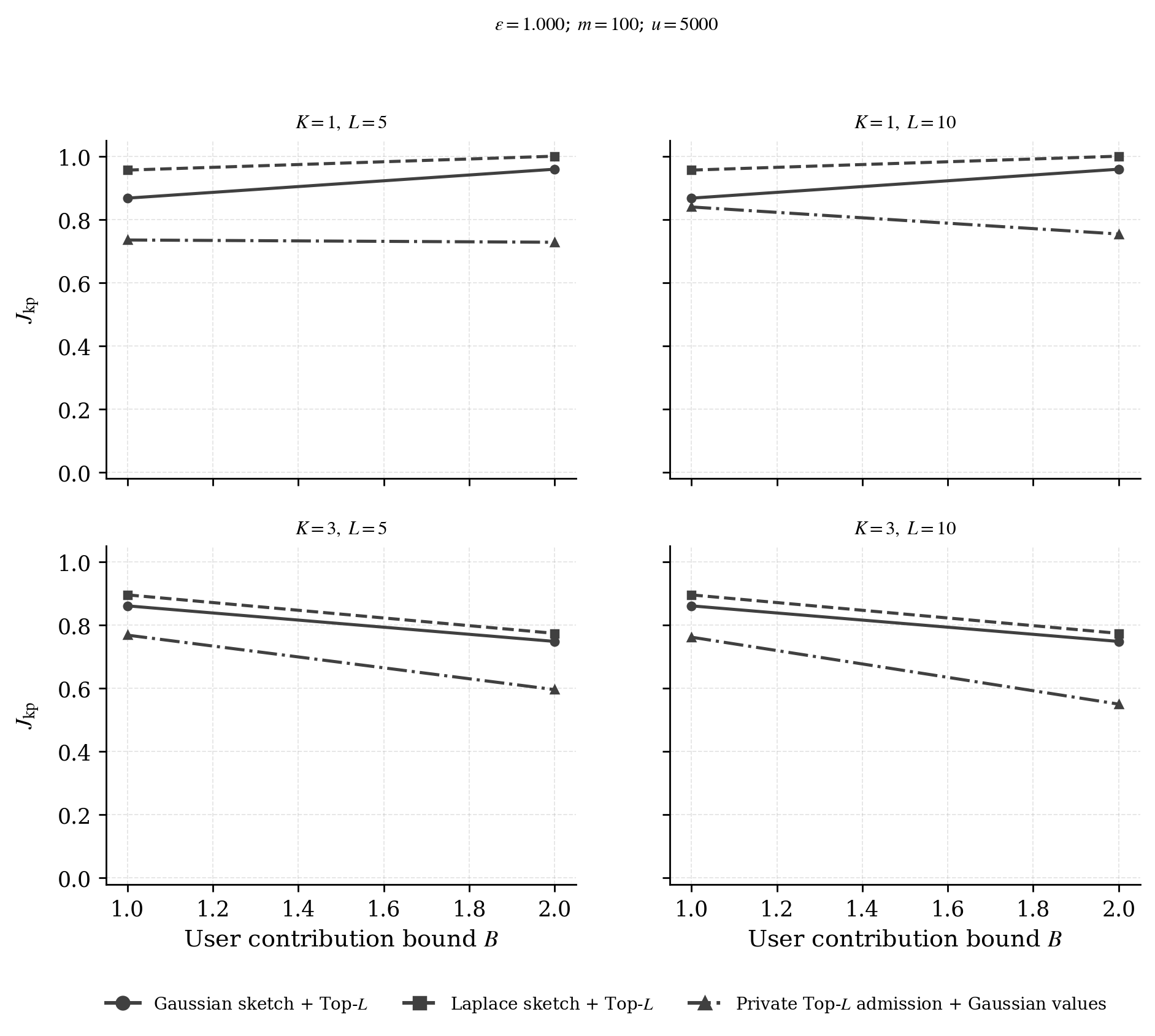}
    \caption{
    Yelp user-level keyphrase Jaccard similarity $J_{\mathrm{kp}}$ between the DP-plan summary and the non-private-plan reference summary as the user contribution bound $B$ varies, with $\varepsilon=1.0$, $m=100$, and $u=5000$ selected users. Panels vary $K$ and $L$; curves compare the release mechanisms.
    }
    \label{fig:yelp-user-keyphrase-jaccard-vs-B}
\end{figure}

\begin{figure}[h]
    \centering
    \includegraphics[width=0.61\linewidth]{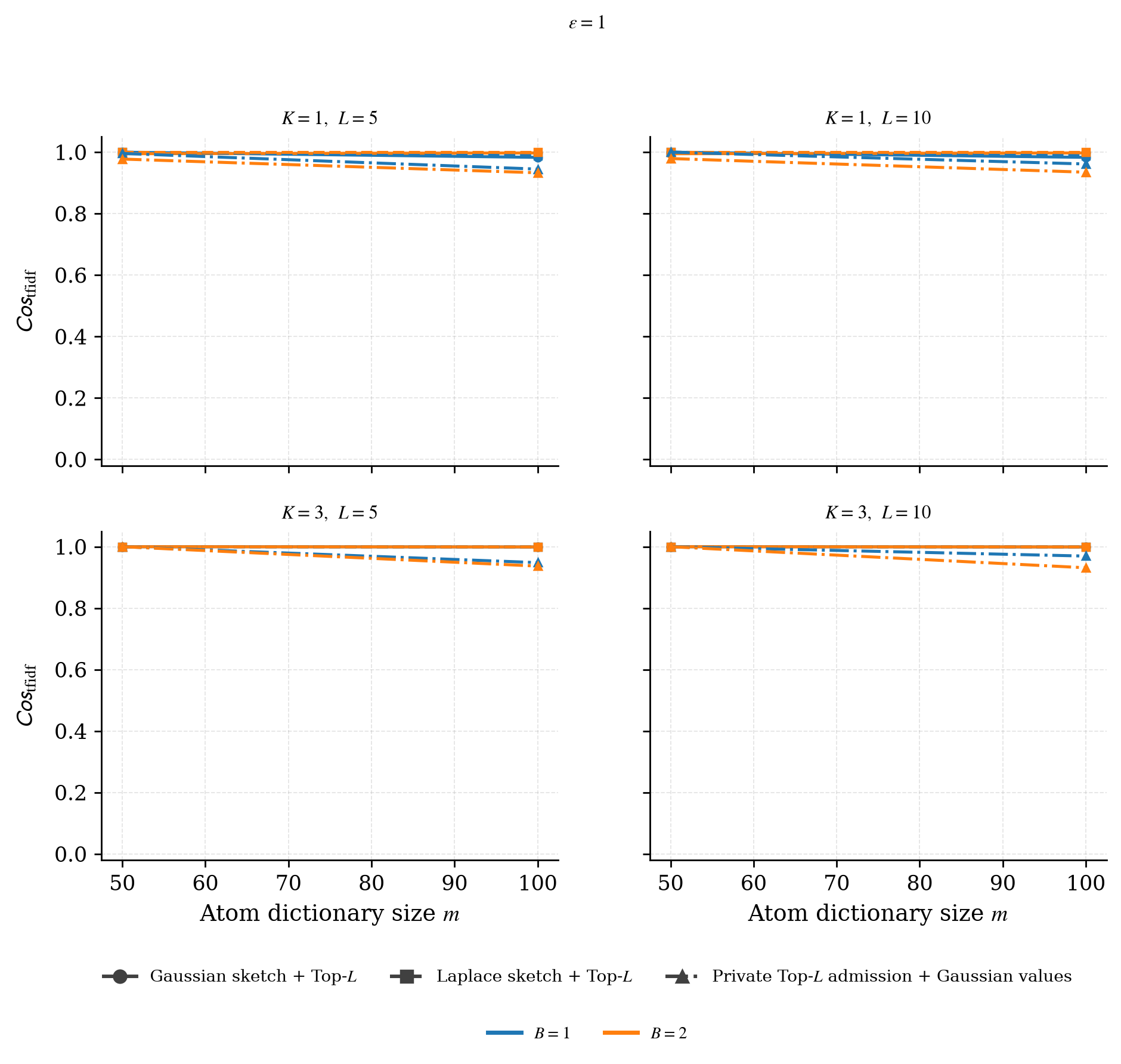}
    \caption{Yelp user-level TF--IDF cosine similarity $\mathsf{Cos}_{\mathrm{tfidf}}$ between the DP-plan summary and the non-private-plan reference summary as the atom dictionary size $m$ varies at $\varepsilon=1.0$. Panels vary $K$ and $L$; curves compare the release mechanisms.}
    \label{fig:yelp-user-tfidf-cosine-vs-m}
\end{figure}

\begin{figure}[h]
    \centering
    \includegraphics[width=0.63\linewidth]{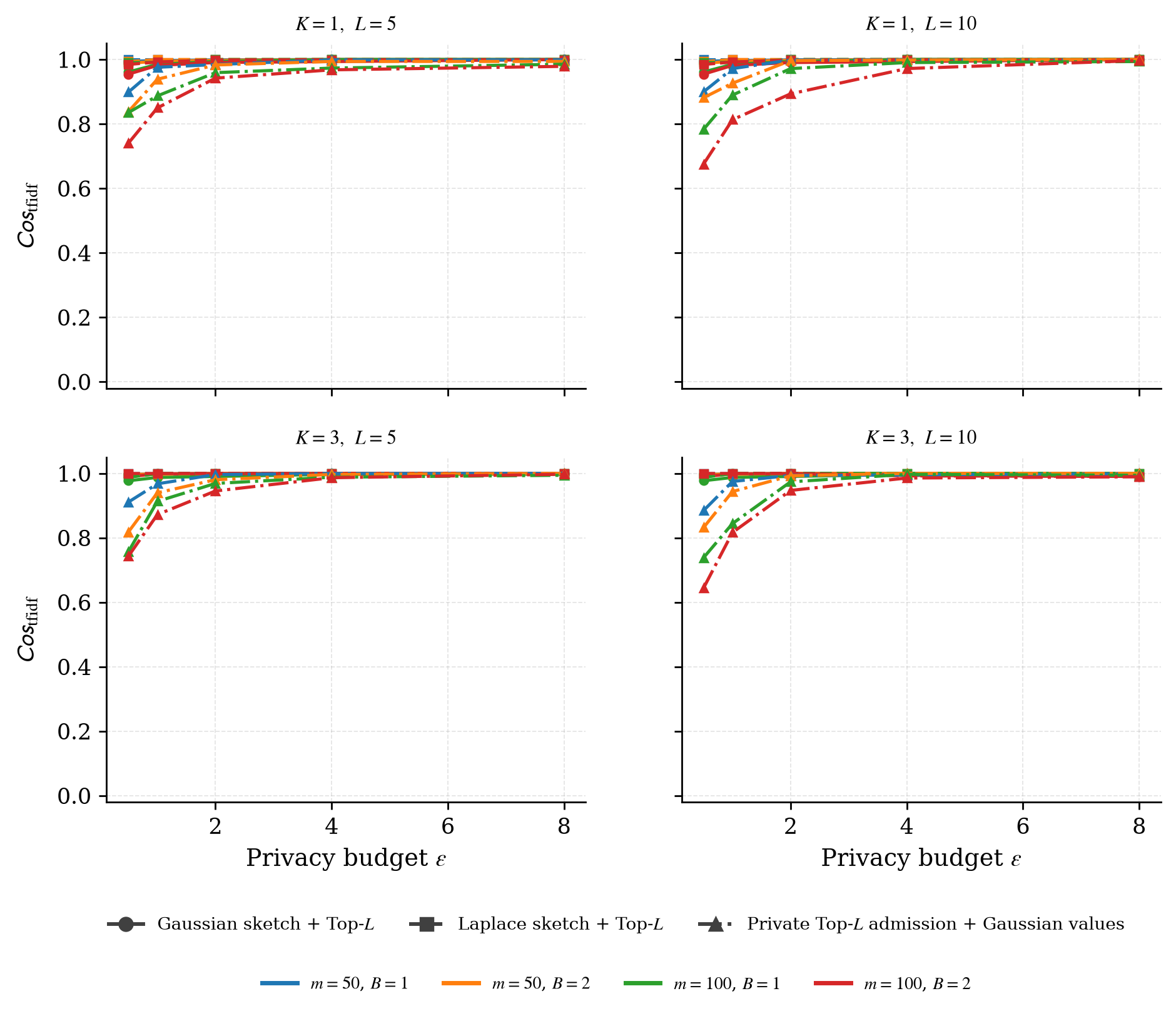}
    \caption{Yelp user-level TF--IDF cosine similarity $\mathsf{Cos}_{\mathrm{tfidf}}$ between the DP-plan summary and the non-private-plan reference summary as the privacy budget $\varepsilon$ varies. Panels vary $K$ and $L$; curves compare the release mechanisms and user-level configurations.}
    \label{fig:yelp-user-tfidf-cosine-vs-epsilon}
\end{figure}

\begin{figure}[h]
    \centering
    \includegraphics[width=0.63\linewidth]{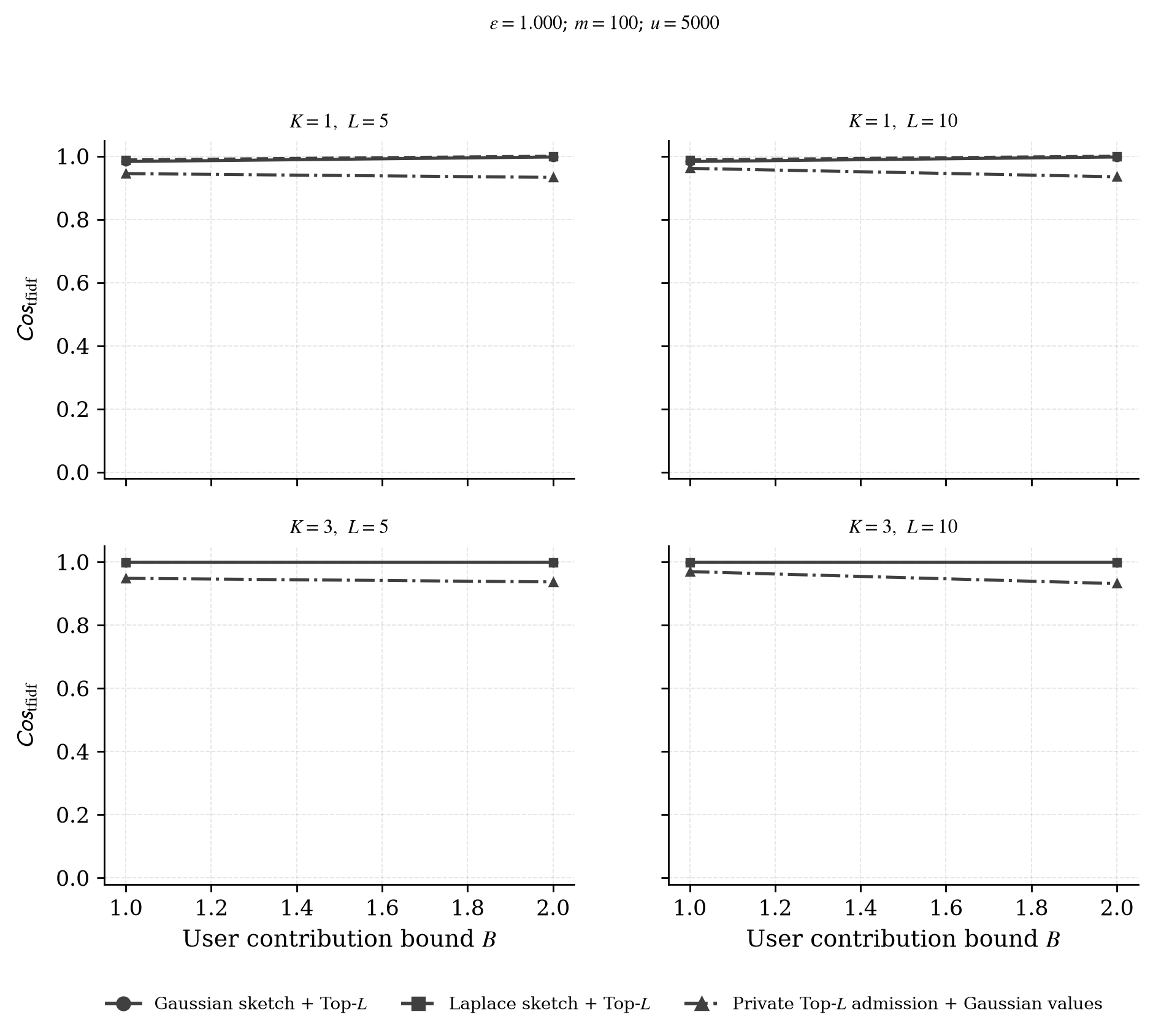}
    \caption{Yelp user-level TF--IDF cosine similarity $\mathsf{Cos}_{\mathrm{tfidf}}$ between the DP-plan summary and the non-private-plan reference summary as the user contribution bound $B$ varies, with $\varepsilon=1.0$, $m=100$, and $u=5000$ selected users. Panels vary $K$ and $L$; curves compare the release mechanisms.}
    \label{fig:yelp-user-tfidf-cosine-vs-B}
\end{figure}


\begin{figure}[h]
    \centering
    \includegraphics[width=0.7\linewidth]{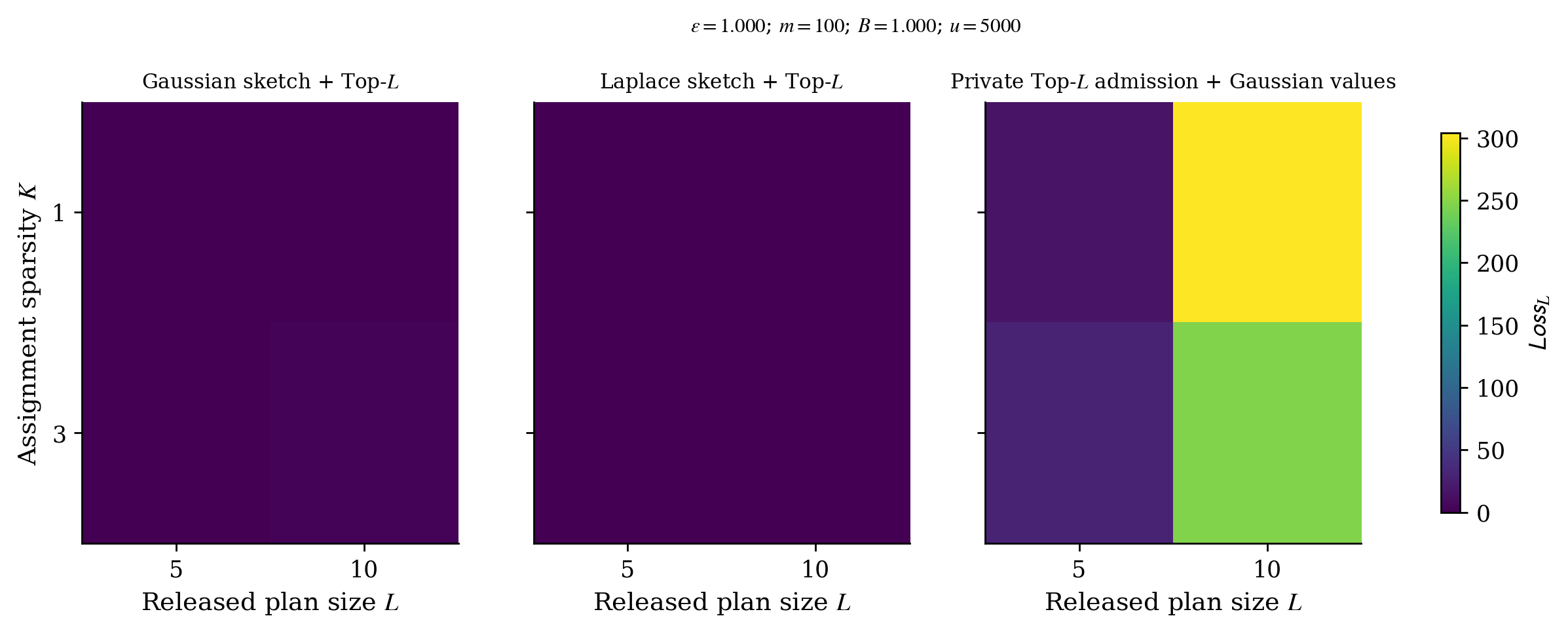}
    \caption{
    Yelp user-level heatmap of support-mass loss $\mathsf{Loss}_L$ over assignment sparsity $K$ and released plan size $L$, with $\varepsilon=1.0$, $m=100$, $B=1.0$, and $u=5000$ selected users. Each panel corresponds to one release mechanism. Lower values indicate less
    non-private semantic mass lost by the released atom set.
    }
    \label{fig:yelp-user-loss-heatmap}
\end{figure}

\begin{figure}[h]
    \centering
    \begin{minipage}{0.49\linewidth}
        \centering
        \includegraphics[width=\linewidth]{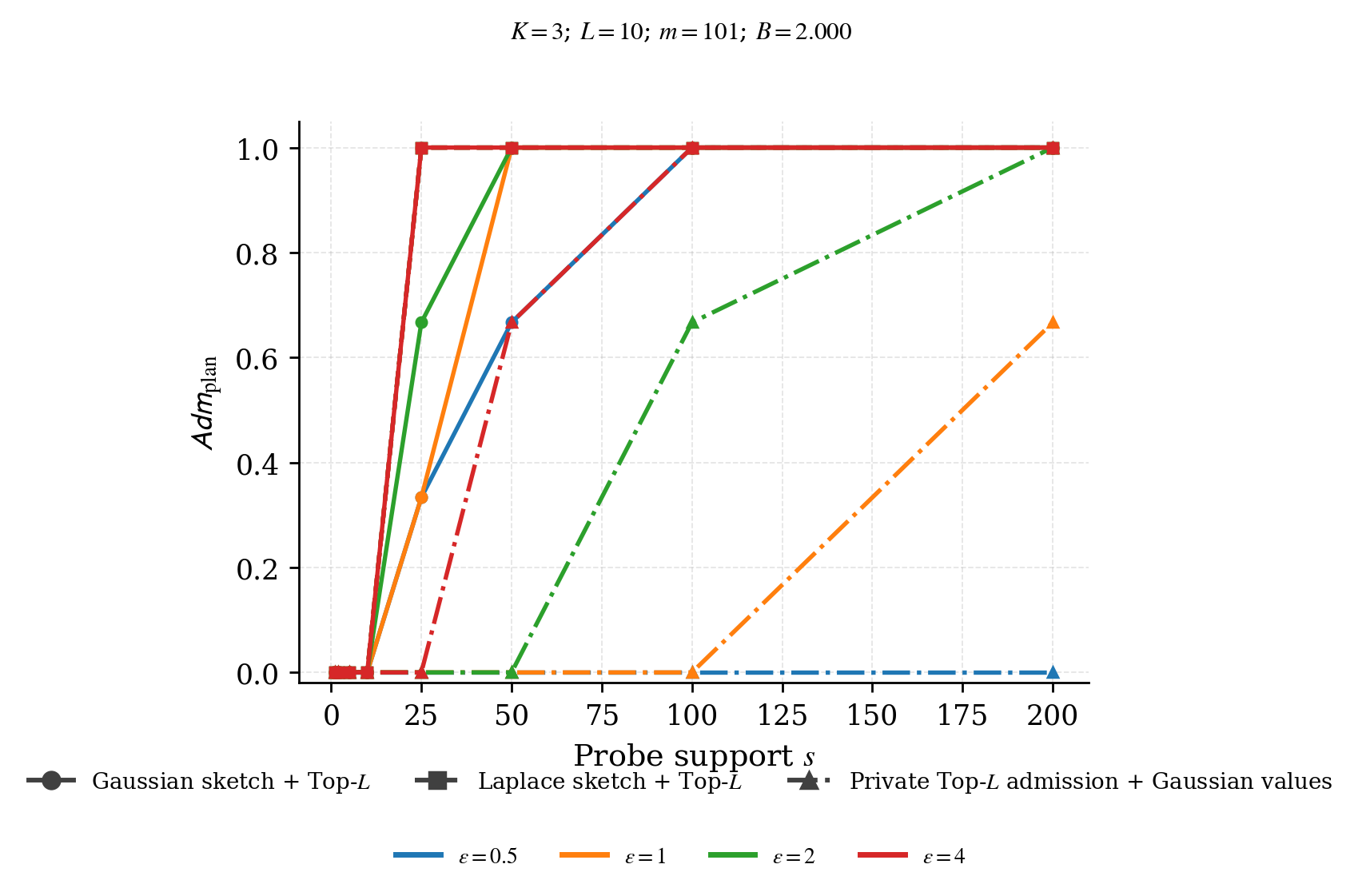}
        \vspace{2pt}
        \textbf{(a)} Plan admission.
    \end{minipage}
    \hfill
    \begin{minipage}{0.49\linewidth}
        \centering
        \includegraphics[width=\linewidth]{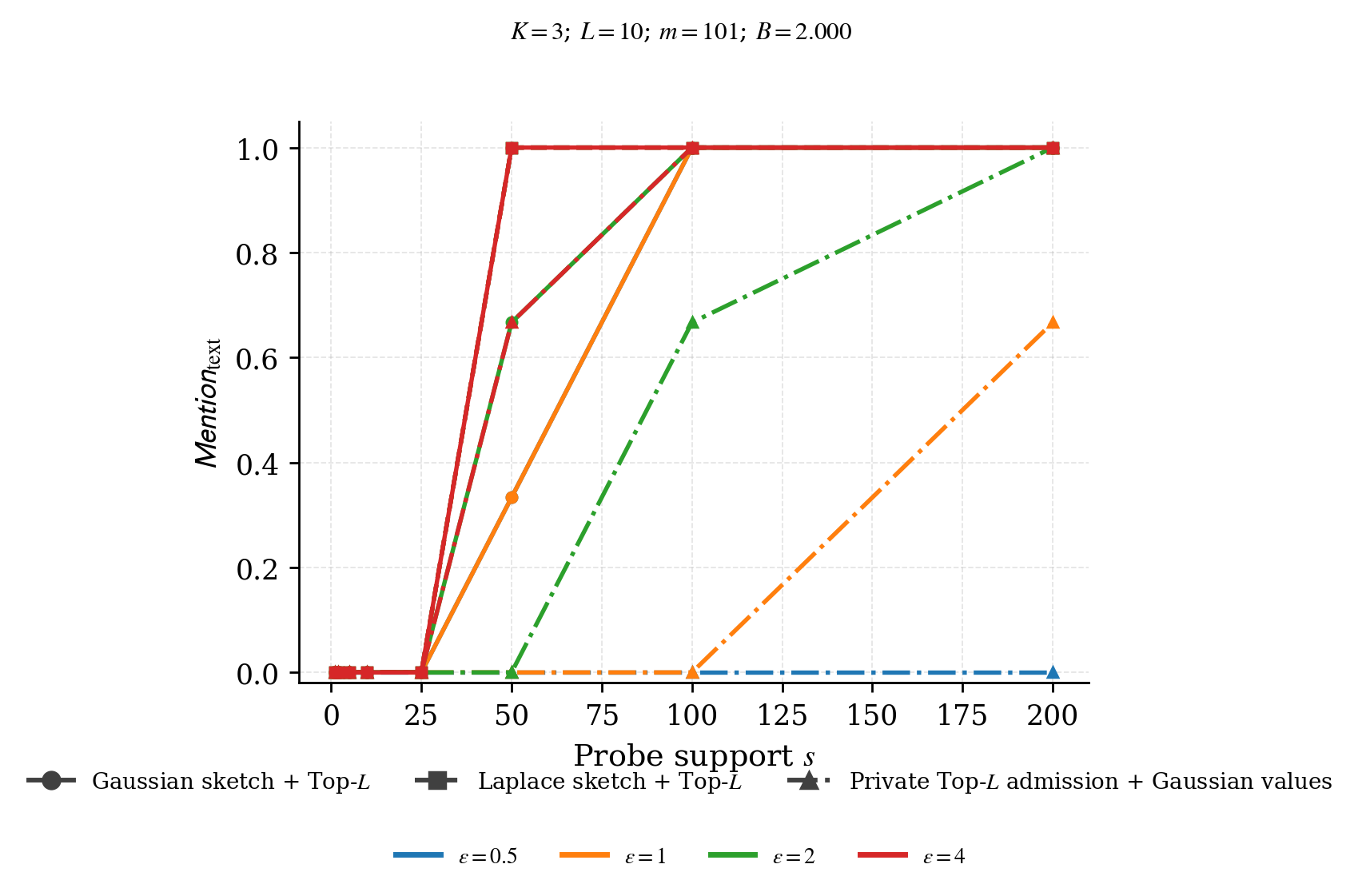}
        \vspace{2pt}
        \textbf{(b)} Summary mention.
    \end{minipage}
    \caption{
    Yelp user-level controlled probe-atom experiment as the injected probe support $s$ varies, with $K=3$, $L=10$, $m=101$ public atoms including the probe atom, and user contribution bound $B=2$. Panel~(a) reports the plan-admission rate $\mathsf{Adm}_{\mathrm{plan}}$, which measures whether the differentially private release admits the public probe atom into the released semantic plan. Panel~(b) reports the summary mention rate $\mathsf{Mention}_{\mathrm{text}}$, which measures whether the public probe  string is detected in the decoded summary. Curves compare privacy budgets and release mechanisms under the configured Yelp user-level probe-sweep setting.
    }
    \label{fig:yelp-user-probe-admission-mention}
\end{figure}

\clearpage

\begin{table*}[t]
\centering
\caption{Release-conditioned OpenAI evaluation of selected generated summaries on Yelp Restaurants reviews. Each summary is evaluated only with respect to the object released to its decoder. Scores are assigned on a 1--5 scale and report verbalization quality. The judge privacy-safety column evaluates only the generated text and is not a formal differential-privacy guarantee. Grouped values are reported as mean $\pm$ sample standard deviation across
the summaries in each row.}
\label{tab:openai-baseline-comparison-yelp-restaurants}
\small
\resizebox{0.99\textwidth}{!}{%
\begin{tabular}{lllccccccccc}
\toprule
Method & Release object & Configuration & $\varepsilon$ & Protected units & Judged summaries & Coverage & Specificity & Insightfulness & Faithfulness & Text safety & Clarity \\
\midrule
DP-SPIN private plan & DP semantic plan & record; Gaussian sketch + top-$L$; m=100; K=3; L=10 & 1.0 & 5000 & 3 & 4.67 $\pm$ 0.58 & 4.00 $\pm$ 0.00 & 3.33 $\pm$ 0.58 & 5.00 $\pm$ 0.00 & 5.00 $\pm$ 0.00 & 4.33 $\pm$ 0.58 \\
DP-SPIN private plan & DP semantic plan & user; Gaussian sketch + top-$L$; m=100; K=3; L=10; B=1 & 1.0 & 2000 & 3 & 4.67 $\pm$ 0.58 & 4.00 $\pm$ 0.00 & 4.00 $\pm$ 0.00 & 5.00 $\pm$ 0.00 & 5.00 $\pm$ 0.00 & 5.00 $\pm$ 0.00 \\
DP keyword histogram & DP keyword histogram & fixed public keyword vocabulary; Laplace histogram; released top-10 keywords & 1.0 & 2000 & 3 & 5.00 $\pm$ 0.00 & 4.33 $\pm$ 0.58 & 4.00 $\pm$ 0.00 & 5.00 $\pm$ 0.00 & 5.00 $\pm$ 0.00 & 5.00 $\pm$ 0.00 \\
DP category histogram & DP category histogram & fixed rating\_bin universe; Laplace histogram; released top-3 categories & 1.0 & 2000 & 3 & 5.00 $\pm$ 0.00 & 4.00 $\pm$ 0.00 & 3.33 $\pm$ 0.58 & 5.00 $\pm$ 0.00 & 5.00 $\pm$ 0.00 & 4.67 $\pm$ 0.58 \\
URANIA-style public keywords & DP public keywords & DP clustering; DP public-keyword histograms; clusters=10; keywords=10 & 1.0 & 2000 & 3 & 4.33 $\pm$ 0.58 & 3.67 $\pm$ 0.58 & 4.00 $\pm$ 0.00 & 3.67 $\pm$ 0.58 & 5.00 $\pm$ 0.00 & 4.67 $\pm$ 0.58 \\
\bottomrule
\end{tabular}%
}
\end{table*}

\begin{table*}[t]
\centering
\caption{Release-conditioned OpenAI evaluation of selected \texttt{DP-SPIN} summaries on Yelp Restaurants reviews. The decoder receives only the released differentially private semantic plan and public decoding instructions. Scores are assigned on the same 1--5 rubric used for the baseline-comparison table.}
\label{tab:openai-dpspin-yelp-restaurants}
\small
\resizebox{0.99\textwidth}{!}{%
\begin{tabular}{lllccccccccc}
\toprule
Unit & Mechanism & Configuration & $\varepsilon$ & Protected units & Judged summaries & Coverage & Specificity & Insightfulness & Faithfulness & Text safety & Clarity \\
\midrule
record & Gaussian sketch + top-$L$ & record; Gaussian sketch + top-$L$; m=100; K=3; L=10 & 1.0 & 5000 & 3 & 4.67 $\pm$ 0.58 & 4.00 $\pm$ 0.00 & 3.33 $\pm$ 0.58 & 5.00 $\pm$ 0.00 & 5.00 $\pm$ 0.00 & 4.33 $\pm$ 0.58 \\
record & Gaussian sketch + top-$L$ & record; Gaussian sketch + top-$L$; m=100; K=3; L=10 & 4.0 & 5000 & 3 & 5.00 $\pm$ 0.00 & 4.00 $\pm$ 0.00 & 3.67 $\pm$ 0.58 & 5.00 $\pm$ 0.00 & 5.00 $\pm$ 0.00 & 4.33 $\pm$ 0.58 \\
record & Private top-$L$ admission + Gaussian values & record; Private top-$L$ admission + Gaussian values; m=100; K=3; L=10 & 1.0 & 5000 & 3 & 4.33 $\pm$ 0.58 & 4.00 $\pm$ 0.00 & 3.33 $\pm$ 0.58 & 5.00 $\pm$ 0.00 & 5.00 $\pm$ 0.00 & 4.33 $\pm$ 0.58 \\
record & Private top-$L$ admission + Gaussian values & record; Private top-$L$ admission + Gaussian values; m=100; K=3; L=10 & 4.0 & 5000 & 3 & 5.00 $\pm$ 0.00 & 4.00 $\pm$ 0.00 & 3.33 $\pm$ 0.58 & 5.00 $\pm$ 0.00 & 5.00 $\pm$ 0.00 & 4.33 $\pm$ 0.58 \\
user & Gaussian sketch + top-$L$ & user; Gaussian sketch + top-$L$; m=100; K=3; L=10; B=1 & 1.0 & 2000 & 3 & 4.67 $\pm$ 0.58 & 4.00 $\pm$ 0.00 & 4.00 $\pm$ 0.00 & 5.00 $\pm$ 0.00 & 5.00 $\pm$ 0.00 & 5.00 $\pm$ 0.00 \\
user & Gaussian sketch + top-$L$ & user; Gaussian sketch + top-$L$; m=100; K=3; L=10; B=1 & 4.0 & 2000 & 3 & 4.33 $\pm$ 0.58 & 4.00 $\pm$ 0.00 & 3.67 $\pm$ 0.58 & 5.00 $\pm$ 0.00 & 5.00 $\pm$ 0.00 & 4.67 $\pm$ 0.58 \\
user & Private top-$L$ admission + Gaussian values & user; Private top-$L$ admission + Gaussian values; m=100; K=3; L=10; B=1 & 1.0 & 2000 & 3 & 4.33 $\pm$ 0.58 & 4.00 $\pm$ 0.00 & 3.67 $\pm$ 0.58 & 5.00 $\pm$ 0.00 & 5.00 $\pm$ 0.00 & 4.67 $\pm$ 0.58 \\
user & Private top-$L$ admission + Gaussian values & user; Private top-$L$ admission + Gaussian values; m=100; K=3; L=10; B=1 & 4.0 & 2000 & 3 & 4.33 $\pm$ 0.58 & 4.00 $\pm$ 0.00 & 4.00 $\pm$ 0.00 & 5.00 $\pm$ 0.00 & 5.00 $\pm$ 0.00 & 4.33 $\pm$ 0.58 \\
\bottomrule
\end{tabular}%
}
\end{table*}

\begin{table*}[t]
\centering
\caption{Similarity between OpenAI summaries generated from \texttt{DP-SPIN} private plans and summaries generated from the corresponding non-private top-$L$ plans on Yelp Restaurants reviews. Metrics compare generated texts and do not evaluate the underlying semantic sketches. Values are reported as mean $\pm$ sample standard deviation across runs.}
\label{tab:openai-private-non-private-yelp-restaurants-reviews}
\small
\resizebox{0.99\textwidth}{!}{%
\begin{tabular}{lllccccccc}
\toprule
Unit & Mechanism & Configuration & $\varepsilon$ & Protected units & Runs & $J_{\mathrm{tok}}$ & $J_2$ & $J_{\mathrm{kp}}$ & $\mathsf{Cos}_{\mathrm{tfidf}}$ \\
\midrule
record & Gaussian sketch + top-$L$ & record; Gaussian sketch + top-$L$; m=100; K=3; L=10 & 1.0 & 5000 & 3 & 0.49 $\pm$ 0.07 & 0.22 $\pm$ 0.09 & 0.23 $\pm$ 0.09 & 0.53 $\pm$ 0.09 \\
record & Gaussian sketch + top-$L$ & record; Gaussian sketch + top-$L$; m=100; K=3; L=10 & 4.0 & 5000 & 3 & 0.50 $\pm$ 0.07 & 0.21 $\pm$ 0.06 & 0.22 $\pm$ 0.06 & 0.54 $\pm$ 0.09 \\
record & Private top-$L$ admission + Gaussian values & record; Private top-$L$ admission + Gaussian values; m=100; K=3; L=10 & 1.0 & 5000 & 3 & 0.36 $\pm$ 0.07 & 0.12 $\pm$ 0.02 & 0.15 $\pm$ 0.03 & 0.40 $\pm$ 0.03 \\
record & Private top-$L$ admission + Gaussian values & record; Private top-$L$ admission + Gaussian values; m=100; K=3; L=10 & 4.0 & 5000 & 3 & 0.41 $\pm$ 0.05 & 0.14 $\pm$ 0.02 & 0.17 $\pm$ 0.01 & 0.42 $\pm$ 0.05 \\
user & Gaussian sketch + top-$L$ & user; Gaussian sketch + top-$L$; m=100; K=3; L=10; B=1 & 1.0 & 2000 & 3 & 0.56 $\pm$ 0.16 & 0.35 $\pm$ 0.12 & 0.34 $\pm$ 0.14 & 0.61 $\pm$ 0.06 \\
user & Gaussian sketch + top-$L$ & user; Gaussian sketch + top-$L$; m=100; K=3; L=10; B=1 & 4.0 & 2000 & 3 & 0.54 $\pm$ 0.11 & 0.32 $\pm$ 0.10 & 0.33 $\pm$ 0.09 & 0.57 $\pm$ 0.06 \\
user & Private top-$L$ admission + Gaussian values & user; Private top-$L$ admission + Gaussian values; m=100; K=3; L=10; B=1 & 1.0 & 2000 & 3 & 0.34 $\pm$ 0.01 & 0.09 $\pm$ 0.02 & 0.12 $\pm$ 0.02 & 0.34 $\pm$ 0.01 \\
user & Private top-$L$ admission + Gaussian values & user; Private top-$L$ admission + Gaussian values; m=100; K=3; L=10; B=1 & 4.0 & 2000 & 3 & 0.48 $\pm$ 0.10 & 0.25 $\pm$ 0.10 & 0.26 $\pm$ 0.10 & 0.50 $\pm$ 0.07 \\
\bottomrule
\end{tabular}%
}
\end{table*}

\begin{table*}[t]
\centering
\caption{Examples of aggregate summaries based on OpenAI-generated outputs on Yelp Restaurants reviews. Each row reports one selected \texttt{DP-SPIN} or baseline release object. The DP-SPIN row is terminology-normalized to the current semantic-support vocabulary; the admitted themes and released numeric values are unchanged.}
\label{tab:openai-qualitative-yelp-restaurants-reviews}
\scriptsize
\setlength{\tabcolsep}{3pt}
\renewcommand{\arraystretch}{1.04}
\resizebox{0.99\textwidth}{!}{%
\begin{tabular}{p{0.27\textwidth}p{0.68\textwidth}}
\toprule
Method and release configuration & Generated summary \\
\midrule
DP-SPIN private plan (record; Gaussian sketch + top-$L$; m=100; K=3; L=10; $\varepsilon=1$; $n=5000$) & The aggregate summary identifies low semantic support across the admitted dining-related themes, including long wait times, friendly staff, customer service, delicious food, pizza places, service staff, food places, Mexican food, lunch specials, and bar restaurants. - Long wait times and customer service have low released support. - Friendly staff has low released support. - Several food-related categories, including Mexican food and pizza places, have low released support. \\ \hline
\addlinespace[1pt]
DP keyword histogram (fixed public keyword vocabulary; vocab=200; Laplace histogram; released top-10 keywords; $\varepsilon=1$; $n=2000$) & The aggregate summary reveals a diverse range of keywords related to food, places, and services, with a notable emphasis on restaurants and specific food items. The most frequently mentioned category is food, followed by place and service, indicating a strong interest in dining experiences. Within the food category, pizza and sushi are highlighted, along with mentions of menus and descriptors like delicious, suggesting a focus on appealing meal options. - Food-related keywords dominate the summary. - Restaurants and specific food items like pizza and sushi are significant. - Descriptive terms such as delicious and meal types like lunch are also present. \\ \hline
\addlinespace[1pt]
DP category histogram (fixed rating\_bin universe; categories=3; Laplace histogram; released top-3 categories; $\varepsilon=1$; $n=2000$) & The aggregate summary indicates the distribution of ratings within the public category field, showing a predominance of high ratings, followed by low and medium ratings. - High ratings are the most common. - Low ratings are less frequent. - Medium ratings are the least common among the three categories. \\ \hline
\addlinespace[1pt]
URANIA-style public keywords (DP clustering; DP public-keyword histograms; clusters=10; keywords=10; $\varepsilon=1$; $n=2000$) & The preserved clusters reveal insights into various dining experiences, emphasizing the quality of food, service, and customer interactions across different types of restaurants. Notably, there is a focus on reasonable pricing and friendly staff in establishments serving Chinese cuisine, as well as a broader exploration of fast food options, including Vietnamese and burger offerings. The summaries also highlight the importance of food quality and taste, particularly in relation to sandwiches and breakfast items, suggesting a diverse culinary landscape. - Chinese restaurants are noted for their reasonable prices and friendly service. - Fast food options encompass a variety of cuisines, indicating a dynamic restaurant scene. - Quality and taste are emphasized across different food types, including breakfast and sandwiches. \\
\addlinespace[1pt]
\bottomrule
\end{tabular}%
}
\end{table*}

\end{document}